\documentclass[a4paper,12pt,times,print,index,authoryear]{Classes/PhDThesisPSnPDF}

\SetDraftVersion{v1.1}

\input{Preamble/preamble}

\title{Scalable Bayesian Inference\\in the Era of Deep Learning}

\subtitle{From Gaussian Processes to Deep Neural Networks}

\author{Javier Antorán Cabiscol}

\dept{Department of Engineering}

\university{University of Cambridge}
\crest{\includegraphics[width=0.2\textwidth]{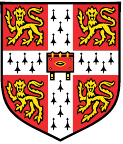}}

\degreetitle{Doctor of Philosophy}

\college{Darwin College}

\subject{PhD} \keywords{{Bayesian Inference} {Bayesian Neural Networks} {Gaussian Processes} {PhD Thesis} {Engineering} {University of
Cambridge}}

\ifdefineAbstract
\fi

\ifdefineChapter
\fi

\begin{document}

\frontmatter

\maketitle


\begin{dedication} 

I would like to dedicate this thesis to the memory of my grandmother, Carmen Mir Gavara, who kept reminding me that I had to submit my thesis.

\end{dedication}

\begin{declaration}

I hereby declare that except where specific reference is made to the work of 
others, the contents of this dissertation are original and have not been 
submitted in whole or in part for consideration for any other degree or 
qualification in this, or any other university. This dissertation is my own 
work and contains nothing which is the outcome of work done in collaboration 
with others, except as specified in the text and Acknowledgements. This 
dissertation contains fewer than 65,000 words including appendices, 
bibliography, footnotes, tables and equations and has fewer than 150 figures.

\end{declaration}

\begin{acknowledgements}      

First and foremost, I would like to thank my supervisor Miguel Hernández-Lobato. Miguel gave me the opportunity to pursue a PhD at a time when I was unsure what next step to take professionally. 
Throughout the PhD, Miguel has given me complete freedom to pursue my interests and to collaborate freely with other researchers. I have learnt a lot from Miguel's capacity to boil down complex topics to their simplest form and from his optimistic outlook on research. I would also like to thank him for his patience throughout all of the times I ignored his advice and went on to try things that didn't work just to find out that his previous suggestion indeed was the best way forward.

I have been very lucky to have worked with excellent collaborators during my PhD. 
There is a relatively widespread bias in academia by which most of the credit for research publications is assigned to the first author. This creates perverse incentives by which secondary authors are discouraged from making significant contributions to collaborative projects. 
In my experience, the easiest way to perform great research is to have multiple talented collaborators fully committed to a project.
Indeed, most of my work, and definitely my best work, has been co-first authored with James Allingham, Riccardo Barbano, Shreyas Padhi, and Andy Lin. 
Apart from helping me escape the academic credit assignment trap,
James has been a great friend. He is likely the person with whom I have achieved best working synergy within my professional career. I will always regret not having worked more closely together during the later stages of our PhDs. 
Riccardo Barbano is another good friend with whom I was privileged to work closely. Riccardo introduced me to the world of computed tomography, resulting in the final content chapter of this thesis. 
I started working with Shreyas and Andy in the last two years of my PhD, which allowed me to play a more senior role in our collaboration. I learned a lot from this arrangement, and watching them both grow into excellent researchers has felt very rewarding. 
I would be remiss to not also give special mention to Dave Janz, who has accompanied me throughout my research into linearised Laplace and Gaussian processes. I have learned a lot of maths from Dave and I am very grateful to him for his patience when teaching me new concepts.

I would also like to thank a number of additional collaborators: Alex Terenin, who introduced me to matrix-free linear algebra and from whom I learned a lot about academic writing, Laurence Midgeley, whose enthusiasm and endless stream of clever ideas are inspiring, Erik Daxberger,  Umang Bhatt, Johannes Leuschner, Austin Tripp, Vincent Stimper, Emile Mathieu, Tomas Geffner, Adam Foster, Wenbo Gong, Chao Ma, Chelsea Murray, Zeljko Kereta, Tameem Adel, Adrian Weller, Bernhard Schölkopf, Bangti Jin, Csaba Szepesvári, and Eric Nalisnick.
Apart from being a great collaborator, I am grateful to Eric for, together with Max Welling, hosting me during my visit to AMLab at the the university of Amsterdam.
I am also grateful to Artem Artemev and Mark Van der Wilk for our very insightful conversations.

I would like to thank Marine Schimel for patiently putting up with me during the highs and lows of the PhD and also to Manuel Escolá, Cristina Uruén, Laura Aznar, Beatriz Alegre, Juan Galvez, Pedro Cabeza, David del Río, Marta Parra, 
Gergely Flamich, Stratis Markou, and Miguel García Ortegon, for being good friends. Additionally, I am grateful to Adriá Garriga Alonso,
Andrew Foong,
Kris Jensen,
Sebastian Ober,
Matt Ashman,
Tor Erlend Fjelde,
Adrían Goldwaser,
Juyeon Hao,
Bruno Mlodozeniec,
Kenza Tazi,
Jonny So,
Valerii Likhosherstov,
Aliaksandra Shysheya,
Vincent Dutordoir,
Runa Eschenhagen,
Emile Mathieu,
Will Tebutt, and
Isaac Reid
for making the CBL a nice environment in which to have spent the past four years.

I am grateful to Yann Dubois for being the only person I know who shares my obsessive passion for machine learning, and to Antonio Miguel for very generously dedicating endless hours to teaching me machine learning during my years as an undergraduate student. 

Finally, I would like to thank my parents who gave me every opportunity.

My PhD research has been supported by Microsoft Research, through its PhD Scholarship Programme, and by the EPSRC. My work was also supported by a number of Tier-2 capital grants that allowed me access to the University of Cambridge Research Computing Services. I apologise to the Cambridge HPC staff for taking down the queuing server by submitting too many jobs one time.

\end{acknowledgements}

\begin{abstract}

\vspace{-0.5cm}
Large neural networks trained on large datasets have become the dominant paradigm in machine learning. 
These systems rely on maximum likelihood point estimates of their parameters, precluding them from expressing model uncertainty. This may result in overconfident predictions and it prevents the use of deep learning models for sequential decision making.

This thesis develops scalable methods to equip neural networks with model uncertainty.
To achieve this, we do not try to fight progress in deep learning but instead borrow ideas from this field to make probabilistic methods more scalable.
In particular, we leverage the linearised Laplace approximation to equip pre-trained neural networks with the uncertainty estimates provided by their tangent linear models. This turns the problem of Bayesian inference in neural networks into one of Bayesian inference in conjugate Gaussian-linear models.
Alas, the cost of this remains cubic in either the number of network parameters or in the number of observations times output dimensions.  By assumption, neither are tractable.

We address this intractability  by using stochastic gradient descent (SGD)---the workhorse algorithm of deep learning---to perform posterior sampling in linear models and their convex duals: Gaussian processes. 
With this, we turn back to linearised neural networks, finding the linearised Laplace approximation to present a number of incompatibilities with modern deep learning practices---namely, stochastic optimisation, early stopping and normalisation layers---when used for hyperparameter learning. We resolve these and construct a sample-based EM algorithm for scalable hyperparameter learning with linearised neural networks.

We apply the above methods to perform linearised neural network inference with {ResNet-50} ($25$M parameters) trained on Imagenet ($1.2$M observations and $1000$ output dimensions). To the best of our knowledge, this is the first time Bayesian inference has been performed in this real-world-scaled setting without assuming some degree of independence across network weights. 
Additionally, we apply our methods to estimate uncertainty for 3d tomographic reconstructions obtained with the deep image prior network, also a first. 
We conclude by using the linearised deep image prior to adaptively choose sequences of scanning angles that produce higher quality tomographic reconstructions while applying less radiation dosage.

\end{abstract}

\tableofcontents

\printnomenclature

\mainmatter

\chapter{Introduction}  %


Programs learnt from data are rapidly displacing programs based on human-designed rules as the dominant paradigm for computer-based automation.
We have seen this in the fields of
computer vision \citep{dosovitskiy2021an},
inverse problems \citep{arridge2019solving},
natural language processing \citep{wang2024improving},
information retrieval \citep{zhu2024large},
text and image generation \citep{saharia2022photorealistic,jiang2024mixtral},
system control \citep{Hu2022rldriving},
scientific discovery \citep{Graczykowski2022alice, Akiyama_2021},
and even computer programming \citep{chen2021code},
among others.
Practically all of these advances were enabled by large-scale deep learning \citep{Henighan2020scaling}.
Indeed, it is plausible that given enough data, a flexible enough neural network, and sufficient compute to train the artificial intelligence (AI), data-driven decision making methods will dominate all traditional computer programs.
\nomenclature[z-ai]{$AI$}{Artificial Intelligence}

The rules for optimally learning from data were codified in the framework of Bayesian probability well before the deep learning revolution of the past decade \citep{Jeffreys1939Probability,cox1946probability,Stigler1986Laplace,Jaynes_Justice_1986}. 
Under this framework, we represent our knowledge, or lack thereof, as probability distributions. When we observe new data, the information gained is used to update these prior distributions into less entropic posterior distributions \citep{Gull1988maxent,Skilling1989}. In turn, these act as priors for future inferences. 
Although probabilistic methods were extensively leveraged to build primordial neural network systems \citep{hinton1993keeping,salakhutdinov2009Boltzmann}, modern neural network methods rely on expressing our beliefs as point estimates instead of probability distributions. 
The lack of explicitly modelled uncertainty makes modern deep learning systems vulnerable to 
acting spuriously when they encounter situations which were not provided  sufficient coverage in the training data \citep{Goddard2023hallucnations,Weiser_Schweber_2023}.
Additionally, probabilistic methods remain state of the art for decision-making tasks that require uncertainty-based exploration, like automated chemical design \citep{Gomez-Bombarelli2018}.

From a Bayesian perspective, neural networks can be seen as an uncompromising model choice that puts very little restrictions on the function class to be learnt. 
The effects of individual weights are non-interpretable, precluding the design of informative Bayesian priors for neural network parameters.
However, it is likely this is the very feature that allows us to use neural networks to solve tasks in ways that can not easily be summarised by a human-readable list of rules. For instance, how to eloquently sustain a conversation or drive a car.
With this idea in place, an intuitive way to explain the seeming incompatibility between Bayesian inference an neural networks is to think of the former as scoring a set of prior hypotheses by how well each one it agrees with the data. 
The problem with modern neural networks is that there are just too many hypotheses to score. The scoring becomes prohibitively expensive, especially, when combined with large datasets which are likely to be fit well by a relatively small region of the neural network parameter space.
In other words, while maximum likelihood learning scales well to the modern big-network and big-data setting, Bayesian inference does not.

This thesis aims to bridge the gap between Bayesian methods and contemporary deep learning.
This endeavour was pioneered by \cite{Mackay1992Thesis} who extended Bayesian inference and hyperparameter selection in linear models (which is also attributable to \cite{Gull1989line}) to the neural network setting via the Laplace approximation, naming his class of methods the \emph{evidence framework}.
 In the last 30 years, the methods of machine learning have changed quite a bit; the scale of the problems tackled and models deployed has grown by multiple orders of magnitude, precluding the out-of-the-box application MacKay's methods, and giving me something to write my thesis about. 
In fact, similarly to \cite{mackay1992practical}, this thesis begins by making contributions to the field of linear models and Gaussian processes, uses the Laplace approximation to adapt these methods for approximate inference in neural networks, and finally applies the developed Bayesian neural networks to efficient data acquisition. 
Thus, this thesis is perhaps best described as a modern take on the evidence framework which makes it scalable to modern problem sizes and amenable to modern deep learning architectures.

To achieve our goals, we are not going to fight progress in deep learning by trying to re-build it from the ground up to natively use Bayesian inference, for instance by imposing fancy handcrafted priors on weights whose effect we dont understand. I believe this is a lost cause.
Instead, we are going to build upon the tremendous progress that has been made in deep learning, and borrow ideas from this field to make Bayesian methods more scalable.
For instance, in \cref{chap:SGD_GPs}, we will use stochastic gradient descent---the de-facto method for training neural networks---to make Bayesian inference in linear models and Gaussian processes more scalable.  Additionally, when dealing with neural networks, we will focus on the \emph{post-hoc inference} setting, in which we leverage \corr{approximate} Bayesian methods, to obtain uncertainty estimates for pre-trained neural networks. This will ensure the thesis' contributions remain compatible with the quickly evolving field of deep learning.

\section{Thesis outline and contributions}

This thesis is written with my past self, before embarking on the PhD, as a target audience.
Although some measure theoretic and functional analytic concepts are (infrequently) mentioned throughout the thesis, knowledge of these fields is not required to understand the thesis' contributions.
Additionally, I have tried to combine mathematical derivations with a number of less-technical remarks to help the reader build intuition about the material.

The rest of this thesis is organised as follows.

\begin{itemize}
    \item \cref{chap:linear_models} introduces Bayesian inference in conjugate Gaussian-linear models and Gaussian processes. Particular focus is placed on the duality between the two model classes because we will make heavy use of it throughout the thesis. We also introduce the model evidence and discuss hyperparameter learning in linear models. Readers who posses an adept grasp of Bayesian linear models may skip this chapter. However, \cref{sec:pathwise_view} on pathwise conditioning may still be of interest, since it is slightly more niche.
    
    \item \cref{chap:approx_inference} introduces approximate inference for large scale or non-conjugate linear models, and for neural networks. We discuss variational inference, conjugate gradient-based approximate inference, and the Laplace approximation. We discuss the limitations of each of these approximations emphasising the trade-off between crudeness of approximation and scalability that they all present. Special focus is placed on the linearised variant of the Laplace approximation, as its adaption to modern deep learning will be a key theme of this thesis.
    
    \item \cref{chap:SGD_GPs} leverages stochastic gradient descent to scale Bayesian inference in conjugate linear models and Gaussian processes to large scale problems. In particular, we develop a number of quadratic objectives whose minimisers represent samples from the posterior distribution of a Gaussian process. We analyse their properties and propose a series of recommendations on how best to apply stochastic gradient-based solvers to this setting. 
    We dub our approach stochastic dual descent. 
    It presents a linear computational cost in the number of observations per gradient step, and we find that a constant number of steps suffice to obtain good performance across a diverse range of problem settings. This starkly contrasts with the cubic in the number of observations cost of exact inference. We also extend stochastic gradient descent to inducing point posteriors, where a sub-linear cost per step can be achieved.
    We analyse the spectral bias of solutions found via stochastic gradient descent, showing that full convergence is not necessary to achieve strong performance. 
    Experimentally, we show that stochastic dual descent outperforms conjugate gradient-based inference and variational inference on standard regression benchmarks and on a large-scale Bayesian optimisation benchmark.
    When combined with stochastic dual descent, Gaussian processes are able match the performance of graph neural networks on a large scale molecular binding affinity prediction task. This chapter is based on  \cite{antoran2023samplingbased}, \cite{lin2023sampling} and \cite{lin2024stochastic}.
    
    \item \cref{chap:adapting_laplace} identifies a number of incompatibilities between  the classical linearised Laplace model evidence objective for model selection and modern deep learning methodologies, in particular, stochastic optimisation, early stopping, and the use of normalisation layers. These result in severe deterioration of the model evidence estimate. We provide recommendations on how to adapt linearised Laplace in light of these issues. Namely, every neural network weight setting has an associated tangent linear model, and we recommend using the evidence of this linear model to select the hyperparameters to be used for linearised Laplace uncertainty estimation. Additionally, we must select priors over the linearised network weights which counteract a number of scale invariances introduced by normalisation layers.  We empirically validate our recommendations  on MLPs, classic CNNs, residual networks with and without normalisation layers, generative autoencoders~and~transformers. This chapter is based on \cite{antoran2022adapting}  and \cite{antoran2021fixing}.
    
    \item \cref{chap:sampled_Laplace} combines the contributions of the previous two chapters to put forth a scalable sample-based EM algorithm for hyperparameter learning in linearised neural networks. The E-step is based on stochastic-gradient descent posterior sampling and the M-step leverages a sample-based estimate of the effective dimension-based hyperparameter update introduced by \cite{Mackay1992Thesis}. We also discuss a number of implementation details that allow us to work with linearised neural networks without ever instantiating these models' Jacobians explicitly. This suite of techniques allows us to scale linearised neural network inference to modern architectures and datasets, such as ResNet-50 on Imagenet. Distinctly from more crude, e.g. factorised, approximations to Bayesian inference in Neural networks, our approach improves upon the performance of the pre-trained network we build upon. It provides state of the art results in terms of joint predictions across multiple inputs, a task of special interest for uncertainty-guided exploration. 
    This chapter is based on \cite{antoran2023samplingbased}.
    
    \item \cref{chap:DIP}  applies the methods developed in this thesis to uncertainty estimation and experimental design for computed tomography (CT) image and volume reconstruction. In particular, we use the deep image prior architecture for reconstruction and linearise the network for uncertainty estimation. 
    We develop a novel total-variation based prior for the linearised deep image prior. Our scalable sample-based EM iteration allows our method to scale to high-resolution 3d volumetric reconstructions from real-measured micro CT data. To the best of our knowledge, our work is the first to perform uncertainty estimation for 3d neural reconstructions. We then go on to leverage the linearised deep image prior as a data-dependent prior for adaptive CT scanning angle selection. This allows us to design strategies that reduce by up to 30\% the number of scans needed to match the performance of an equidistant angle baseline on a synthetic task.
    This chapter is based on \cite{barbano2022probabilistic}, \cite{barbano2022design}, \cite{Antoran2022Tomography} and \cite{antoran2023samplingbased}.
    
    \item \cref{chap:conclusion} concludes the thesis with an outlook of this thesis' contributions in the context of the broader field of machine learning and a discussion of avenues for future work. 
\end{itemize}

\section{Full list of publications}

I now provide a full list of papers I have written during my time as a PhD student.  Titles are bolded for papers whose content is included in this thesis. I also give a brief description of my contribution to each of these works. An asterisk superscript $^*$ denotes co-first authorship. 

\begin{enumerate}
\item J.A. Lin$^*$, S. Padhy$^*$, \textbf{J. Antorán}$^*$, A. Tripp, A. Terenin, C. Szepesvári, J. M. Hernández-Lobato, D. Janz. "\textbf{Stochastic Gradient Descent for Gaussian Processes Done Right.}" In \textit{International Conference on Learning Representations (ICLR)}. 2024 \nocite{lin2024stochastic}

My contribution to to this project consisted of developing the idea, helping my co-authors debug their code, writing the paper, and helping orchestrate other authors' contributions. 

\item L.I. Midgley$^*$, V. Stimper$^*$, \textbf{J. Antorán}$^*$, E. Mathieu$^*$, B. Schölkopf, Hernández-Lobato. "SE (3) equivariant augmented coupling flows." In \textit{Advances in Neural Information Processing Systems (NeurIPS)}. 2023. Awarded spotlight presentation.  \nocite{midgley2023se}

\item J.A. Lin, \textbf{J. Antorán}, J.M. Hernández-Lobato. "Online Laplace Model Selection Revisited." In \textit{Symposium on Advances in Approximate Bayesian Inference (AABI)}. 2023. Awarded oral presentation. \nocite{lin2023online}

\item J.A. Lin$^*$, \textbf{J. Antorán}$^*$, S. Padhy$^*$, D. Janz, J.M. Hernández-Lobato, A. Terenin. "\textbf{Sampling from Gaussian Process Posteriors using Stochastic Gradient Descent}." In \textit{Advances in Neural Information Processing Systems (NeurIPS)}. 2023. Awarded oral presentation. \nocite{lin2023sampling}

My contribution to to this project consisted of developing the idea, writing some of the code, writing the paper, and helping orchestrate other authors' contributions. 

\item R. Barbano, \textbf{J. Antorán}, J. Leuschner, J.M. Hernández-Lobato, B. Jin, Z. Kereta. "Image Reconstruction via Deep Image Prior Subspaces." In \textit{Transactions on Machine Learning Research (TMLR)}. 2023 \nocite{barbano2023image}

\item J.U. Allingham, \textbf{J. Antorán}, S. Padhy, E. Nalisnick, J.M. Hernández-Lobato. "Learning Generative Models with Invariance to Symmetries." In \textit{NeurIPS 2022 Workshop on Symmetry and Geometry in Neural Representations}. 2022 \cite{allingham2022learning}

\item \textbf{J. Antorán}$^*$, S. Padhy$^*$, R. Barbano, E. Nalisnick, D. Janz,  J. M. Hernández-Lobato. "\textbf{Sampling-based inference for large linear models, with application to linearised Laplace.}" In \textit{International Conference on Learning Representations (ICLR)}. 2023 \nocite{antoran2023samplingbased}

My contribution to to this project consisted of developing the idea, writing the codebase that was used for experiments, running some experiments, writing the paper, and helping orchestrate other authors' contributions. 

\item \textbf{J. Antorán}$^*$, R. Barbano$^*$, J. Leuschner, J.M. Hernández-Lobato, B. Jin. "\textbf{Uncertainty Estimation for Computed Tomography with a Linearised Deep Image Prior.}" In \textit{Transactions on Machine Learning Research (TMLR)}. 2023\nocite{Antoran2022Tomography}

My contribution to to this project consisted of developing the idea, helping my co-authors debug their code, and writing the paper. 

\item R. Barbano$^*$, J. Leuschner$^*$, \textbf{J. Antorán}$^*$, B. Jin, J.M. Hernández-Lobato. "\textbf{Bayesian experimental design for computed tomography with the linearised deep image prior.}" In \textit{Adaptive Experimental Design and Active Learning workshop at ICML}. 2022 \nocite{barbano2022design}

My contribution to to this project consisted of developing the idea, helping my co-authors debug their code, and writing the paper. 

\item \textbf{J. Antorán}, D. Janz, J.U. Allingham, E. Daxberger, R.R. Barbano, E. Nalisnick, J. M. Hernández-Lobato. "\textbf{Adapting the linearised Laplace model evidence for modern deep learning.}" In \textit{International Conference on Machine Learning}. 2022 \nocite{antoran2022adapting}

My contribution to to this project consisted of developing the idea, writing the codebase that was used for the experiments, writing the paper, and orchestrating other co-authors' contributions.

\item C. Murray, J.U. Allingham, \textbf{J. Antorán}, J.M. Hernández-Lobato. "Addressing bias in active learning with depth uncertainty networks... or not." In \textit{Proceedings of Machine Learning Research (PMLR)  163:59-63}. 2022 \nocite{murray2022Addressing}

\item T. Geffner$^*$, \textbf{J. Antorán}$^*$, A. Foster$^*$, W. Gong, C. Ma, E. Kiciman, A. Sharma, A. Lamb, M. Kukla, N. Pawlowski, M. Allamanis, C. Zhang. "Deep end-to-end causal inference." \textit{arXiv preprint arXiv:2202.02195}. 2022 \nocite{geffner2022deep}

\item \textbf{J. Antorán}, J.U. Allingham, D. Janz, E. Daxberger, E. Nalisnick, J. M. Hernández-Lobato. "\textbf{Linearised Laplace inference in networks with normalisation layers and the neural g-prior.}" In \textit{Symposium on Advances in Approximate Bayesian Inference (AABI)}. 2022. Awarded oral presentation. \nocite{antoran2021fixing}

My contribution to to this project consisted of developing the idea, writing the codebase that was used for the experiments, writing the paper, and orchestrating other co-authors' contributions. 

\item R. Barbano$^*$, \textbf{J. Antorán}$^*$, J.M. Hernández-Lobato, B. Jin. "\textbf{A probabilistic deep image prior over image space.}" In \textit{Symposium on Advances in Approximate Bayesian Inference (AABI)}. 2022 \nocite{barbano2022probabilistic}

My contribution to to this project consisted of developing the idea, helping my co-authors debug their code, and writing the paper. 

\item C. Murray, J.U. Allingham, \textbf{J. Antorán}, J.M. Hernández-Lobato. "Depth Uncertainty Networks for Active Learning." In \textit{Bayesian Deep Learning Workshop at the 35th Conference on Neural Information Processing System}. 2021 \nocite{murray2022depth}

\item U. Bhatt, \textbf{J. Antorán}, Y. Zhang, Q.V. Liao, P. Sattigeri, R. Fogliato, G. G. Melancon, R. Krishnan, J. Stanley, O. Tickoo, L. Nachman, R. Chunara, A. Weller, A. Xiang. "Uncertainty as a Form of Transparency: Measuring, Communicating, and Using Uncertainty." In \textit{AAAI/ACM Conference on AI, Ethics, and Society}. 2021 \nocite{Bhatt2021transparency}

\item E. Daxberger, E. Nalisnick, J.U. Allingham, \textbf{J. Antorán}, J.M. Hernández-Lobato. "Bayesian Deep Learning via Subnetwork Inference." In \textit{International Conference on Machine Learning (ICML)}, 2021. \nocite{Daxberger21subnetwork}

\item \textbf{J. Antorán}, U. Bhatt, T. Adel, A. Weller, J. M. Hernández-Lobato. "Getting a CLUE: A Method for Explaining Uncertainty Estimates." In \textit{International Conference on Learning Representations (ICLR).} 2021. Awarded oral presentation.  \nocite{antoran2021clue}

\item E. Daxberger, E. Nalisnick, J. Allingham, \textbf{J. Antorán}, J.M. Hernández-Lobato. "Expressive yet tractable Bayesian deep learning via subnetwork inference." \textit{Symposium on Advances in Approximate Bayesian Inference (AABI)} 2020. Awarded oral presentation. \nocite{daxberger2020expressive}

\item \textbf{J. Antorán}$^*$, J.U. Allingham$^*$, J.M. Hernández-Lobato. "Depth uncertainty in neural networks." In \textit{Advances in Neural Information Processing Systems (NeurIPS)}. 2020 \nocite{Antoran20depth}

\item \textbf{J. Antorán}$^*$, J. U. Allingham$^*$, J. M. Hernández-Lobato. "Variational depth search in ResNets." In \textit{Workshop on Neural Architecture Search at International Conference on Learning Representations}. 2020 \nocite{antorán2020variational}

\end{enumerate}

\nocite{antoran2019understanding}

\nocite{antoran2019disentangling}

\chapter[Linear models and Gaussian processes]{Bayesian reasoning with Gaussian linear models and Gaussian processes}\label{chap:linear_models}

\ifpdf
    \graphicspath{{chapters/Chapter2/Figs/Raster/}{chapters/Chapter2/Figs/PDF/}{chapters/Chapter2/Figs/}}
\else
    \graphicspath{{chapters/Chapter2/Figs/Vector/}{chapters/Chapter2/Figs/}}
\fi

We start with \textit{linear regression}, where outputs are given by linear functions of some basis function expansion of the input variables, as these models play a central role in this thesis.  When a Gaussian prior is placed over the  parameters and the targets are assumed to have been corrupted by additive Gaussian noise, we obtain the Gaussian linear model. This setting is of special interesting because conjugacy between likelihood and prior leads to the equations of Bayesian inference admitting closed form solutions.   
This simplicity does not come at the cost of flexibility; the use of basis function expansion allows linear regressors to learn arbitrarily complex functions. 
This thesis will leverage this fact to tackle the analytical intractability of Bayesian inference in neural network models; in \cref{chap:sampled_Laplace} we will approximate the predictions of the neural network with those of a Gaussian linear model with an appropriate choice of basis function expansion.
The key limitation of Gaussian linear regression is its computational cost, which scales cubically with the number of observations or number of model parameters. This thesis addresses this limitation in \cref{chap:SGD_GPs}.

I would be remiss to not mention some other excellent references for Gaussian linear models, such as the seminal texts of \cite{Gull1989line} and \cite{MacKay1992interpolation}, and the books of
\cite{bishop2006pattern} (Chapter 3) and \cite{williams2006gaussian} (Chapter 2).
However, this chapter provides a presentation of the material that emphasises the duality between parameter-space and function-space, and the pathwise formulation of inference, which will hopefully make the contributions of the rest of the thesis easily accessible to the reader.
In particular, we start by providing 3 complementary views of Bayesian inference in Gaussian-linear models: \cref{sec:weight_space_view} introduces the parametric weight-space view of linear regression, \cref{sec:function_space_view} introduces Gaussian processes (GP) \nomenclature[z-gp]{$GP$}{Gaussian Processes}, the dual, non-parametric view of linear regression, and \cref{sec:pathwise_view} presents the pathwise view of inference in Gaussian processes which deals directly with random functions. The latter will be key to designing computationally efficient inference algorithms in \cref{chap:SGD_GPs}.
We then go on to discuss the importance of the choice of hyperparameters for linear models and how to select them via marginal likelihood maximisation in \cref{sec:model_selection}. The chapter concludes with a discussion of the limitations of linear models in \cref{sec:linear_model_limitations}.

\section[The weight space view]{The weight space view: Gaussian linear regression}\label{sec:weight_space_view}

We begin by introducing a multi-output basis-function linear model. 
We observe a set of $n$ inputs $x_1, \dotsc, x_n \in \c{X}$ and corresponding outputs $y_1, \dotsc, y_n \in \c{Y} \subseteq \R^c$, where $c$ is the number of output dimensions. 
We introduce a basis function expansion $\phi : \c{X} \to \R^c \times \c{H} $, that maps inputs into some Reproducing Kernel Hilbert Space (RKHS) $\c{H}$. \nomenclature[z-rkhs]{$RKHS$}{Reproducing Kernel Hilbert Space}
We will not provide a review of RKHS, but instead we refer the reader to \cite{paulsen16} for this.
Since we will always work with separable RKHS, we can treat their elements as vectors, that is $\c{H} \subseteq \R^d$, without loss of generality---although sometimes these vectors will be infinite dimensional---and thus we henceforth treat $\phi(x_i)$ like a $c \times d$ dimensional matrix. \nomenclature[z-pd]{$PD$}{Positive Definite}  \nomenclature[z-psd]{$PSD$}{Positive Semidefinite}

We assume that the targets are generated as a linear combination of our featurised inputs, weighted by a parameter vector $w \in \c{H}  \subset \R^d$, and corrupted by additive Gaussian noise with 0 mean and observation-dependent positive definite (PD) and symmetric precision matrix $B_i \in \R^{c \times c}$ for each $i \leq n$.
That is, each target is given by
\begin{gather}\label{eq:initial_bayesian_linear_model}
    y_i = \phi(x_i) w + \eps_i \spaced{with} w \sim \N(0, A^{-1}) \spaced{and} \eps_i \sim \N(0, B_i^{-1}).
\end{gather}
The parameter vector $w$ is an unobserved variable which we assume to have been drawn from a Gaussian prior distribution with precision given by the positive definite matrix $A$.

Henceforth, we will use the following stacked notation: we write $Y \in \R^{n c}$ for the concatenation of $y_1, \dotsc, y_n$.
We stack the expanded observations into the design matrix $\Phi = [\phi(x_1)^T, \phi(x_2)^T, \dotsc, \phi(x_n)^T]^T \in \R^{nc \times d}$. We concatenate the additive noise vectors into $\c{E} = [\eps_1^T, \eps_2^T, \dotsc, \eps_n^T]^T$.  Its distribution is a zero centred Gaussian with $B \in \R^{nc \times nc}$, the block diagonal matrix with blocks $B_1, \dotsc, B_n$, as its precision.
With this, our model is
\begin{gather}\label{eq:weight_space_gen_model}
    Y = \Phi w + \c{E} \spaced{with} w \sim \N(0, A^{-1} I) \spaced{and} \c{E} \sim \N(0, B^{-1}).
\end{gather}
Without loss of generality we assume $A = a I$ with $a \in \R_+$; any additional structure in $A$ can be absorbed into the basis functions $\phi$. Additionally, unless specified otherwise, we assume isotropic observation noise $B = b I$  with $b \in \R_+$. With this, each output dimension can be seen as an additional independent observation, and we are free to assume $c=1$ without loss of generality. 
Finally, for a vector $v$ and a  Positive SemiDefinite (PSD) matrix $G$ of compatible dimensions, $\|v\|_G^2=v^T G v$.

\begin{remark} \textbf{On our construction of the multioutput linear model} \\ 
Our model generates multiple outputs from a weight vector by multiplying with the matrix-valued features $\phi(x_i)$.
This choice differs from the presentation of \cite{bishop2006pattern}, where a weight matrix multiplies vector-valued features. Our choice is deliberate, as it will simplify notation when dealing linear approximations to multioutput neural network functions in later chapters. 
\end{remark}

\paragraph{Notation for probability distributions} 
We use capital letters to refer to probability measures, e.g., $\Pi = \N(0, A^{-1})$ and lowercase letters to refer to their density functions, i.e. $\pi$. These are defined in the standard way via the Radon–Nikodym derivative $d \Pi = \pi(w) d \nu$ with $\nu$ denoting the Lebesgue measure.
We will not concern ourselves any further with measure theoretic issues. Throughout the thesis we assume any necessary conditions hold for simplicity.
We refer to the parameters of probability distributions which we do not treat probabilistically as hyperparameters.
To make these explicit in our notation for a density function, we separate the hyperparameters from the point in the sample space at which the density is evaluated by a semicolon. For instance, we may write the density of our Gaussian prior at $w$ as $\pi(w;\, 0, A)$ where the mean $0$ and precision $A$ are the hyperparameters. However, whenever there is no ambiguity, we omit the hyperparameters to keep our notation uncluttered.
We write conditional density functions, e.g. likelihood functions, by separating variables being conditioned on with a vertical bar $|$. For instance, the likelihood of the linear regression weights is written as $p(Y|w) = \frac{\partial P_{Y|w}}{\partial \nu}$, where $P_{Y|w} = \N(\Phi w ,\, B^{-1})$. When there is no ambiguity, we do not explicitly include our set of inputs $X$ in our notation to further reduce clutter. 
Finally, we will also refer to the density of a distribution by prepending the point at which it being evaluated to the distribution's arguments, but separated by a semicolon. That is, $\N(w;\, 0, A^{-1}) = \pi(w)$.

\subsection{Understanding our choice of model}\label{subsec:basis_functions}

\begin{figure}[thb]
\centering
\includegraphics[width=\linewidth]{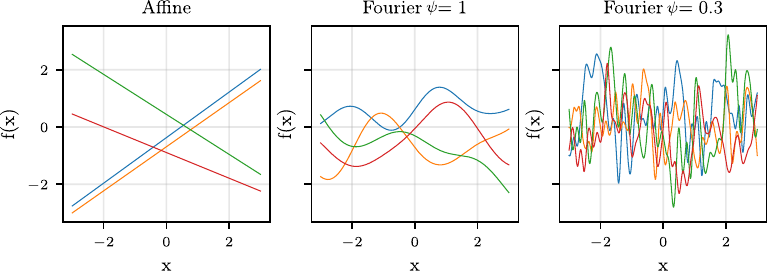}
\caption{Each plot displays four prior function samples, drawn using \cref{eq:prior_function_sampling_weight_space}. The left side plot uses an affine basis expansion \cref{eq:affine_expansion}, the middle one a 500 element random Fourier expansion with a Gaussian spectral measure and a lengthscale of $\psi=1$ \cref{eq:random_fourier_basis}, and the right side plot uses a similar Fourier expansion but with a lengthscale of $\psi=0.3$. 
}
\label{fig:basis_comparison}
\end{figure}

The choice of basis is perhaps the most important modelling decision when working with linear models; our flexibility in the choice of basis makes linear models very powerful. Indeed, every function can be expressed as a linear combination of a set of basis functions; to see this just choose an element of the basis to contain the target function. However, we seek more than just a representation from which our targets can be linearly decoded\footnote{Many trivial choices of basis, like ones with very short lengthscales, allow any target to be linearly decoded but are not practically useful.}. Our basis should reflect our prior knowledge (and uncertainty) over the target function.   

To illustrate the power of the basis function expansion, we provide some examples of common basis function choices: the affine basis and the random Fourier basis. We restrict ourselves to $\c{X} = \R$ and a single output dimension $c=1$ for the purpose of visualisation. 
The affine basis corresponds to regression with a single linear weight and a bias. That is
\begin{gather}\label{eq:affine_expansion}
    \phi(x) = [1, x].
\end{gather}
This model expresses the belief that our target function is a straight line, or plane. 
Furthermore, within the set of all possible lines, our 0-centred Gaussian prior over the parameters $w$ expresses a belief that lines corresponding to weight and bias choices of small magnitude are more likely a priori.
The random Fourier basis \citep{Rahimi2007random,Sutherland2015random} represents the input as a set of cosines with random frequency and phase
\begin{gather}\label{eq:random_fourier_basis}
    \phi_{s, r}(x) = \sqrt{\frac{2}{d}}[\cos(s_1^T x + r_1), \cos(s_2^T x + r_2), \dotsc, \cos(s_d^T x + r_d)] \\
    \text{with} \quad s_i \sim \N(0,  \psi^{-2}) \quad \spaced{and} \quad r_i \sim \U(0, 2\pi), \notag 
\end{gather}
where the subscript in $\phi_{s, r}$ makes explicit the features dependence on the source of randomness $s, r$. The lengthscale parameter $\psi$ controls the smoothness of the functions we can express through the choice of frequency variance. Small values lead to our prior placing most of weight on smooth functions, while large values generate a mix of functions of different smoothness.

We use $f : \c{X} \to \R$ to denote the random prior function implied by our model. We evaluate realisations of this random function by
multiplying weight vectors drawn from the prior over weights with the basis expanded inputs as
\begin{gather}\label{eq:prior_function_sampling_weight_space}
     f(\cdot) = \phi(\cdot) w \spaced{with} w \sim \N(0, A^{-1}),
\end{gather}
and display them in \cref{fig:basis_comparison}. We denote by $X$ the array of inputs $(\phi(x_i))^n_{i=1}$, and with $f(X) \in \R^n$ the vector given by our prior random function evaluated at these inputs. Pushing the prior distribution over weights through the product with the feature expansion, we obtain the prior distribution over function values evaluated at the inputs 
\begin{gather}\label{eq:prior_over_functions_evaled_weightspace}
    f(X) \sim \N(0, \Phi A^{-1} \Phi^T).
\end{gather}
We visualise the covariance matrices for our affine and random Fourier basis in \cref{fig:cov_mtx_comparison}.

\begin{figure}[thb]
\centering
\includegraphics[width=\linewidth]{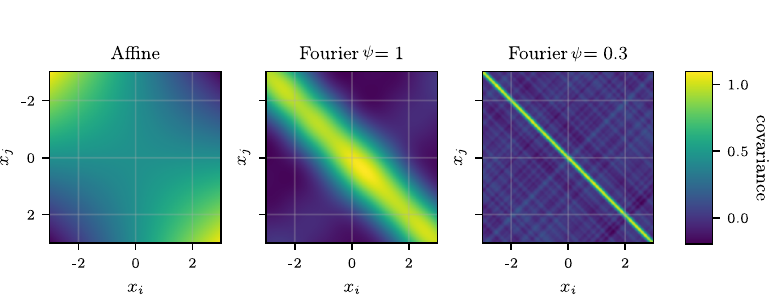}
\caption{Covariance matrices of the prior distribution over functions evaluated at 501 equally spaced points in the range $[-3, 3]$ The left side plot uses an affine basis expansion \cref{eq:affine_expansion}, the middle one a 500 element random Fourier expansion with a Gaussian spectral measure and a lengthscale of $\psi=1$ \cref{eq:random_fourier_basis}, and the right side plot uses a similar Fourier expansion but with a lengthscale of $\psi=0.3$. 
}
\label{fig:cov_mtx_comparison}
\end{figure}

The choice of basis affects our model's uncertainty a priori and thus how much data will be needed to pin down accurate values for the parameters.
If we choose a more flexible function class, then we will need more data to constrain the parameters and vice versa.
The Fourier model with a large value for $\psi$ is more flexible than the affine model since it can express non-linear functions in the inputs. This additional flexibility is reflected in the covariance matrix structures shown in \cref{fig:cov_mtx_comparison}. The linear model assumes strong correlations throughout the input space. Only a few observations will be enough to constrain its parameters everywhere.
On the other hand, the Fourier model's band diagonal covariance structure tells us the model assumes that targets are only correlated when their inputs are nearby. How close the inputs need to be is given by the width of the diagonal band. Since each observation will only constrain the random functions locally, many more observations are necessary to reduce the Fourier model's uncertainty. 
Since there are more ways for a function to change quickly than slowly, a smaller value of $\psi$ leads to an even more flexible random Fourier model with a thinner band-diagonal covariance structure. This model will require even more data to learn.

Suitable feature expansions exist for many types of data, such as images \citep{Wilk2017conv}, natural text \citep{Collins2001kernels}, and even graphs \citep{Tripp2023tanimoto}. Throughout this chapter we will use the Fourier basis as a recurring example. As we will see in \cref{subsec:kernel_trick}, the random Fourier linear model is intimately tied to stationary Gaussian processes.

\subsection{Posterior inference: from loss functions to distributions}\label{subsec:linear_posterior_inference}

Having discussed the choice of model, we now turn to learning from data. Intuitively, learning can be thought of as combining what we knew a priori with the information that the newly observed data tells us. We can achieve this by scoring candidate parameter vectors by their density under our prior $\Pi$ and how closely the corresponding functions pass to the observed targets (the mapping between weight vectors and functions is given in \cref{eq:prior_function_sampling_weight_space}). The latter requirement is quantified by the probability density of our observations given the weights, which is known as the likelihood when taken as a function of the weights.  
Our assumption on the Gaussianity of the observation noise implies the conditional density over the targets is $p(Y|w) = \N(Y; \Phi w ,\, B^{-1})$. We assume iid inputs, making this density factorise across observations as $\prod_{i=1}^n \N(y_i; \phi(x_i) w ,\, B_i^{-1})$.

Since we require our functions to simultaneously be constrained by our prior \emph{and} likelihood, we construct an objective function by multiplying the two, obtaining the joint density $p(Y | w) \pi(w) = \prod_{i=1}^n p(y_i | w) \pi(w)$. 
Taking a logarithm\footnote{The monotonicity of the logarithm ensures the optima of the function do not change} for numerical stability, we find that the likelihood corresponds to the least squares regression loss and the prior, to the sum of squares regulariser, both up to an additive constant. That is, we obtain the loss $\c{L}: \c{R}^d \to R_+$ given by
\begin{gather}\label{eq:linear_regression_loss}
    \log p(Y | w) + \log \pi(w) + C = \underbrace{\frac{1}{2} \sum_{i=1}^n \|y_i - \phi(x_i) w\|^2_{B_i}}_{\text{least squares loss}} + \underbrace{\frac{1}{2} \|w\|^2_{A}}_{\text{regulariser}} \coloneqq \c{L}(w),
\end{gather}
where $C$ is the additive constant independent of $w$ and $Y$.
Both terms in the expression are quadratic, with the curvature of the fit term being 
$M = \Phi^T B \Phi^T$ and the regulariser's curvature being given by $A$. The curvature of the full loss is thus $\nabla^2_w \c{L} = M + A \coloneqq H$. This allows for a closed form solution for the maximum a posteriori (MAP) \nomenclature[z-map]{$MAP$}{Maximum A Posteriori} estimate of the parameters $w_\star = H^{-1} \Phi^T B Y$. Thus the MAP function is $f_\star(\cdot) = \phi(\cdot) w_\star$. We refer to \cite{bishop2006pattern} for more detailed derivations.

Only finding the optima of the loss does not tell us how confident we should be in the corresponding parameter setting. For instance, if there are many parameter settings obtaining similar loss values but mapping to very different functions, i.e. the determinant of $H$ is small, we might become less confident in the MAP estimate. To fully capture the uncertainty in our parameter estimate we resort to Bayesian inference. 
We obtain the posterior density over parameters through Bayes rule
\begin{gather}\label{eq:bayes_rule_c2}
    \pi(w|Y) = \frac{p(Y | w) \pi(w)}{\int_w p(Y | w)\, d \pi(w)}.
\end{gather}
From \cref{eq:linear_regression_loss} it is clear that the posterior relates to the linear regression loss as $\pi(w|Y) \propto \exp(- \c{L}(w))$\footnote{It is worth noting that we can use this strategy to construct probability distributions from other positive-valued functions.}. Since the loss is quadratic, the posterior is also Gaussian with mean $w_\star$ and covariance $H^{-1} \coloneqq \Sigma$. We illustrate this for our affine model in \cref{fig:linear_posterior_functions}.
The ratio of the joint density to the posterior $p(Y) = \int_w p(Y | w)\, d \pi(w)$ 
 is known as the ``evidence'', a constant independent of $w$, which we will discuss in detail in \cref{sec:model_selection}. 
 
 \begin{figure}[thb]
\centering
\includegraphics[width=0.85\linewidth]{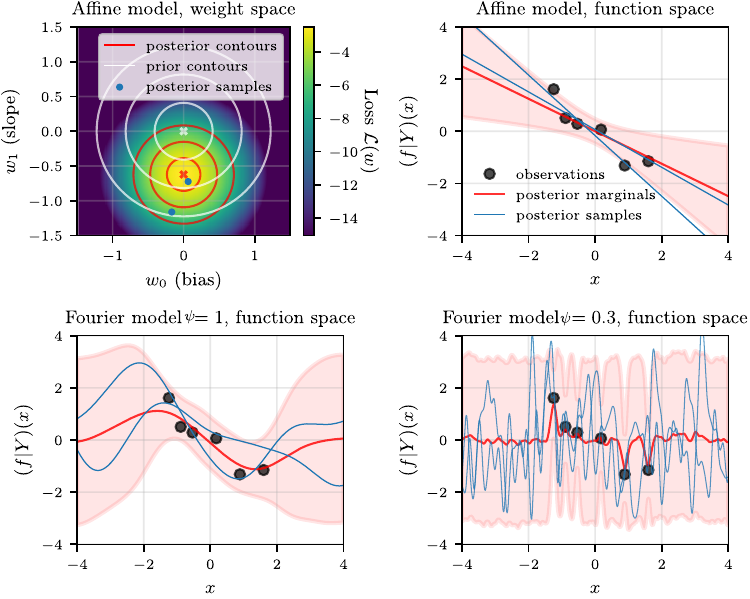}
\caption{ The top left plot shows the $d=2$ dimensional posterior landscape of our affine model fit on a $n=6$ observation dataset with $B=2I$ and $A=6I$. The 1, 2 and 3 standard deviation prior and posterior contours are overlayed on top. We draw 2 samples from the weight space posterior, which we plot as function samples in the top right plot. The top right plot also displays the mean and 2 standard deviation contours of the posterior random function $f | Y$. The bottom left and bottom right plots display the same objects as the top right, but for the 500 element random Fourier basis with a Gaussian spectral measure. We set $A=0.4I$ for the Fourier models. The lengthscale on the left is $\psi=1$ and the right side plot uses $\psi=0.3$.
}
\label{fig:linear_posterior_functions}
\end{figure}

We draw from the posterior distribution over functions by multiplying posterior weight samples with our basis expansion 
\begin{gather}
 (f|Y)(\cdot) = \phi(\cdot) w \spaced{with} w \sim \N(w_\star, \Sigma) \notag\\
  \Sigma = H^{-1} = (A + M)^{-1} \spaced{and} w_\star = \Sigma \Phi^T B Y,\label{eq:posterior_linear_function_draws}
\end{gather}
and illustrate this for the different priors introduced in \cref{subsec:basis_functions}, and a small dataset, in
\cref{fig:linear_posterior_functions}. Computationally, the cost of evaluating this posterior is dominated by computing the inverse of the Hessian $H$ which presents cubic cost in the number of observations times output dimensions $\c{O}\left((nc)^3\right)$.

At an array of test inputs $X' = (x'_{i})_{i=1}^{n'}$ with corresponding featurisation $\Phi' \in \R^{n'c \times d}$, we evaluate the posterior distribution over function values by marginalising out the parameters in \cref{eq:posterior_linear_function_draws}. Since we are dealing with a linear combination of Gaussian variables, the distribution over function evaluations will be jointly Gaussian
\begin{gather}\label{eq:posterior_predicivte_evaluated}
    (f|Y)(X') \sim \N(\Phi' w_\star,\, \Phi' \Sigma \Phi'^T).
\end{gather}
We illustrate the marginals of this distribution in \cref{fig:linear_posterior_functions}. The affine model presents the smallest posterior errorbars, as it is the least flexible. We can pin down the value of its parameters with the least amount of data. The $\psi=1$ Fourier model presents a smoother posterior mean and larger errorbars, a consequence of the model's increased flexibility. 
Additionally, the band diagonal structure of the Fourier model's covariance (recall \cref{fig:cov_mtx_comparison}) results in the posterior returning to the prior covariance far enough away from our observations. For the model with lengthscale $\psi=0.3$, this happens so fast that the posterior ends up matching the 0-mean prior almost everywhere, except very close to the data. 
Visual inspection reveals this model choice is too flexible for our toy dataset. We would not expect this solution to generalise to additional observations. This lack of generalisation is also reflected in the large errorbars of the posterior.

So far, we have looked at the posterior distribution over functions. However, if we want to make predictions about observations, we need to take into account that these are generated as noisy function realisations. The output space distribution that accounts for both the uncertainty in our parameters and  observations is the posterior predictive. For a new input $x_{n+1}$, the posterior predictive density over the corresponding target $y_{n+1}$ is given by
\begin{gather}
    p(y_{n+1} | Y) = \int_w p(y_{n+1} | w) \, d \pi(w|Y).
\end{gather}
In the linear-Gaussian setting, assuming a homoscedastic observation noise of precision $b$, this density corresponds to the distribution $\N(\phi(x') w_\star, \,\phi(x') \Sigma \phi(x')^T + b^{-1} I$).

\section[The function space view]{The function space view: Gaussian processes}\label{sec:function_space_view}

 A stochastic process is a potentially infinite set of random random variables. 
We say that a random function $f : \c{X} \to \R^c$ is a \emph{Gaussian process} if, for every finite set of points $X \in \c{X}^n$, $f(X)$ is jointly Gaussian. Both of the expressions we derived in the previous section for the prior \cref{eq:prior_over_functions_evaled_weightspace} and posterior \cref{eq:posterior_predicivte_evaluated}
 distributions over function evaluations are multivariate Gaussians, satisfying this definition.
Viewing the Gaussian linear model as a Gaussian process (GP) \nomenclature[z-gp]{$GP$}{Gaussian Processes} will allow us to perform Bayesian inference without ever having to work with the parameters $w$ directly. The use of stochastic processes as priors is known as Bayesian nonparametrics \citep{ghosal2017nonparametrics}.

\subsection{Duality}\label{subsec:duality}

Instead of the usual measure theoretic definition of stochastic processes \citep{matthews2017scalable}, we will derive the function-space view as the convex dual formulation of the Gaussian linear model \citep{khan2014decoupled}.
We will make heavy use of this duality throughout the rest of the thesis. 

\begin{derivation} \textbf{Convex duality of Gaussian linear regression loss}\\
    We begin  by formulating the linear regression loss $\c{L}$ introduced in \cref{eq:linear_regression_loss} as the constrained optimisation problem
    \begin{gather}
    \min_{w \in \R^d} \frac{1}{2} \sum_{i=1}^n \|y_i - u_i\|^2_{B_i} + \frac{1}{2} \|w\|^2_{A} \notag \\
        \text{s.t.} \quad u_i = \phi(x_i) w \quad \forall i.
    \end{gather}
    We introduce the Lagrangian $L : \R^d \times \R^{nc} \times \R^{nc} \to \R_+$ with Lagrange multiplier $\alpha \in \R^{nc}$ as
    \begin{gather}
        L(w, u, \alpha) = \frac{1}{2} \sum_{i=1}^n \|y_i - u_i\|_{B_i}^2 + \frac{1}{2} \|w\|^2_{A} + \sum_{i=1}^n \innerprod{\alpha_i}{u_i - \phi(x_i) w}.
    \end{gather}
    This problem is quadratic in both $w$ and $u$ and Slater's condition holds (see 5.2 in \cite{Boyd2014convex}). Thus there is strong duality
    \begin{gather}
        \min_{w \in \R^d}\c{L}(w) =  \max_{\alpha \in \R^{nc}}  \inf_{w \in \R^d} \inf_{u \in \R^{nc}} L(w, u, \alpha) = \max_{\alpha \in \R^{nc}} L(w', u', \alpha) 
    \end{gather}
    where the optimal primal variables can be shown, by solving the respective quadratic problems, to be given by $w' = A^{-1}\Phi\alpha, \spaced{and} u'_i = y_i - B_i^{-1}\alpha_i$. Plugging these in to the Lagrangian yields the dual loss 
    \begin{gather}
        L(w', u', \alpha) = - \frac{1}{2} \|\alpha\|^2_{(B^{-1} + \Phi A^{-1} \Phi^T)} + \alpha^T Y
    \end{gather}
    which is also quadratic but with curvature $B^{-1} + \Phi A^{-1} \Phi^T$. It is optimised by taking
    \begin{gather}\label{eq:c2_gp_mean_representer}
        \alpha_\star \coloneqq (B^{-1} + \Phi A^{-1} \Phi^T)^{-1} Y
    \end{gather}
\end{derivation}
Thus we can reparametrise the maximum a posteriori function estimate in terms of the optimal Lagrange multipliers $\alpha_\star$ as 
\begin{gather}\label{eq:dual_map}
    f_\star(\cdot) = \phi(\cdot) w_\star = \phi(\cdot) A^{-1}\Phi^T \alpha_\star.
\end{gather}
There are $nc$ Lagrange multipliers in the vector $\alpha_\star$, one per observation and output dimension. Obtaining them requires solving a $nc$ dimensional system at cost $\c{O}{((nc)^3)}$. This is in contrast to the  $\c{O}{(d^3)}$ cost of the primal solution $w_\star$. Thus, the dual formulation will be preferable when $nc < d$.

An analogous derivation to the one above, given in \citep{khan2014decoupled}, can be used to find the dual formulation of the full Gaussian posterior, including the covariance. However, a faster route is to use the Woodbury matrix to re-write the expression for the posterior covariance into a form that depends on $(B^{-1} + \Phi A^{-1} \Phi^T)^{-1}$, the curvature of the dual problem, as
\begin{gather}
    \Sigma = (A + \Phi^T B \Phi)^{-1} = A^{-1} - A^{-1} \Phi^T (B^{-1} + \Phi A^{-1} \Phi^T)^{-1} \Phi A^{-1}.
\end{gather}
With this, the posterior distribution over functions evaluated at a set of test points $X' = (x'_{i})_{i=1}^{n'}$ with featurisation $\Phi' \in \R^{n'c \times d}$ can be written as  
\begin{align}\label{eq:posterior_predicivte_evaluated2}
    (f|Y)(X') \sim \N(&\Phi' A^{-1} \Phi^T (B^{-1} + \Phi A^{-1} \Phi^T)^{-1} Y, \notag \\
    &\Phi' A^{-1} \Phi'^T - \Phi' A^{-1} \Phi^T (B^{-1} + \Phi A^{-1} \Phi^T)^{-1} \Phi A^{-1} \Phi'^T).
\end{align}
Again, evaluating this expression presents cost $\c{O}\left((nc)^3\right)$ as opposed to $\c{O}(d^3)$ for the primal form \cref{eq:posterior_predicivte_evaluated}.

\subsection{From features to kernels}\label{subsec:kernel_trick}

When working with the dual form of the Gaussian linear model, we no longer encounter the featurised design matrix $\Phi \in \R^{nc \times d}$ explicitly; it only shows up as part of the $nc \times nc$ matrix $\Phi A^{-1} \Phi^T \coloneqq K$, which we will refer to as the \emph{kernel matrix}. 

Taking $c=1$ for simplicity of notation but without loss of generality, 
any feature map of the form $\phi(\cdot) : \c{X} \to \c{H}$ defines a symmetric and positive definite kernel $k : \c{X} \times \c{X} \to \R$ as $k(x_i, x_j) = \innerprod{\phi(x_i)}{\phi(x_j)}$ for $x_i, x_j \in \c{X}$. The converse is also true; any symmetric and positive definite kernel $k$ can be written as an inner product in some RKHS $\c{H}$ \citep{Aronszajn1950reproducing}. 
We now note that the positive definite matrix $A$ can be absorbed into the featurised design matrices. It can simply be seen as a rotation and shear of the features. Thus, there exists a kernel that generates our kernel matrix such that $[K]_{i, j} = k(x_i,x_j) \,\forall\,\, {i,j=1,..,N}$. 
This fact will allow us to avoid working with features entirely in favour of their inner products. In turn, this will allow us to use potentially infinite dimensional feature expansions, where it may be impossible to explicitly compute the features. The substitution of input inner products $\innerprod{x_i}{x_j}$ with kernel function evaluations $k(x_i, x_j)$ to obtain a non-linear (in the inputs) version of existing algorithms is known as the \emph{kernel trick} \citep{Scholkopf2001learning}.  Additionally, we will henceforth denote the matrix built by evaluating our kernel at all pairs in two arrays of inputs $X$ and $X'$ as $K_{X X'}$. That is $[K_{X X'}]_{i,j} = k(x_i, x'_j) : i=1,2,\dotsc, n, \, j=1,2,\dotsc,n'$.  We refer to \cite{Hofmann2006rkhsreview} for a tutorial on RKHS.

To illustrate the kernel trick, we consider the random Fourier basis given in \cref{eq:random_fourier_basis} and let the number of features $d$ go to infinity. We recover the squared exponential or Radial Basis Function (RBF) \nomenclature[z-rbf]{$RBF$}{Radial Basis Function} kernel
\begin{align}\label{eq:random_feature_convergence}
    k(x_i, x_j) &= \innerprod{\phi_{s, r}(x_i)}{\phi_{s,r}(x_j)} \notag \\
    &= \frac{2}{d} \sum_{l=1}^d  \cos(s_l^T x_i + r_l) \cos(s_l^T x_j + r_l) \underset{d = \infty}\to \exp\left(\frac{-\|x_i - x_j\|^2}{\psi^2}\right)\notag \\ 
    \text{with} \quad&  s_l \sim \N(0, \psi^{-2}) \quad \spaced{and} \quad r_l \sim \U(0, 2\pi).
\end{align}
Thus, when we use the RBF kernel we are leveraging an infinite dimensional feature expansion without ever having to compute Fourier features explicitly. We will discuss random feature approximations to kernels in more detail in \cref{subsec:random_features}.

We refer to the partial evaluation of the kernel $k(x, \cdot) = \phi(x) A^{-1}\phi(\cdot)^T  : \c{H} \to \R$ as the \emph{evaluation functional}, which is an element of RKHS in its own right\footnote{To see this note that elements of the RKHS can be written as $\sum_{i=1}^d \alpha_i k(x_i, \cdot)$ and then choose all but 1 $\alpha_i$ to be 0.}. Its name comes from the fact that for a kernel $k$, there is a unique $k(x, \cdot) \in \c{H}$ which evaluates a function $\varphi \in \c{H}$ at the input $x\in \c{X}$ through the inner product
\begin{gather}
    \varphi(x) = \innerprod{\varphi}{k(x, \cdot)} = \sum_{i=1}^d \alpha_i k(x_i, x).
\end{gather}
This is the reproducing property, which gives name to the RKHS. 
A consequence of this property is that $\innerprod{k(x, \cdot)}{k(x', \cdot)} = k(x_i, x_j)$.

\begin{figure}[thb]
\centering
\includegraphics[width=\linewidth]{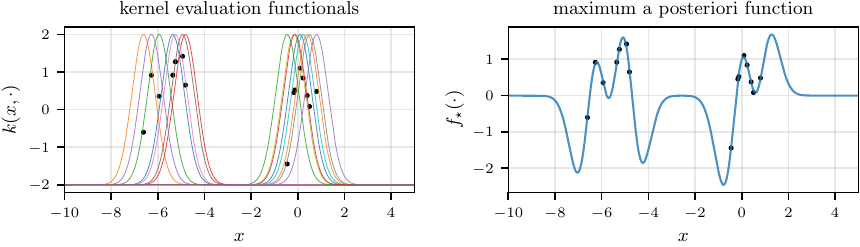}
\caption{Left: RBF kernel ($\psi=0.5$) evaluation functionals for each observation (black dots) in a toy 1d dataset. Right: the posterior mean function is a linear combination of evaluation functionals.  
}
\label{fig:evaluation_functionals_and_map}
\end{figure}

We can now identify the dual expression for the posterior mean function $f_\star(\cdot) = \phi(\cdot) A^{-1} \Phi^T\alpha_\star$, given in \cref{eq:dual_map}, as a linear combination of evaluation functionals
\begin{gather}\label{eq:representer_optimal_function}
    f_\star(\cdot) = \sum_{i=1}^n \alpha_{\star, i} k(\cdot, x_i) = K_{(\cdot)X} \alpha_{\star},
\end{gather}
where for the last equality we write $K_{(\cdot)X}$ for the stacked evaluation functionals at the observed datapoints $k(\cdot, x_1), \dotsc, k(\cdot, x_n)$, allowing us express functions in $\c{H}$ as matrix vector products.
We can think of the evaluation functionals as a basis function expansion of the inputs $x_i, \, i<n$.
The entries of the linear coefficient vector $\alpha_\star$ are known as the \emph{representer weights}\footnote{This name is due to the representer theorem of \cite{Scholkopf2001learning}.}. \Cref{fig:evaluation_functionals_and_map} depicts a set of evaluation functionals for the RBF kernel \cref{eq:random_feature_convergence} and how the posterior mean function is constructed as a linear combination of these functions. The local nature of the kernel leads to the evaluation functionals going to 0 far enough away from the observations and this behaviour translates to the MAP function $f_\star$.

\subsection{Bayesian reasoning about functions: Gaussian processes}\label{subsec:bayesian_reasoning_about_functions}

We now leverage duality and the kernel trick to re-state the Bayesian model from \cref{sec:weight_space_view} directly as a Gaussian process
\begin{gather}\label{eq:gp_bayesian_model}
    Y = f(X) + \c{E} \spaced{with} f  \sim \GP(\mu, k) \spaced{and} \c{E} \sim \N(0, B^{-1}).
\end{gather}
The \emph{mean function} $\mu(\cdot) = \E(f(\cdot))$ and a \emph{covariance kernel} $k(\cdot,\cdot') = \text{cov}(f(\cdot),f(\cdot'))$ uniquely identify the Gaussian process prior. Without loss of generality, we will assume $\mu(\cdot)=0$ throughout the rest of this chapter. 

The posterior distribution over functions is another Gaussian process $f|Y \sim \GP(f_{\star}, k_{\star})$ with
\begin{gather}
    f_{\star}(\cdot) = {K}_{(\cdot)X}(K + B^{-1})^{-1} Y \notag\\
    k_{\star}(\cdot,\cdot') = k(\cdot,\cdot') - {K}_{(\cdot)X}(K + B^{-1})^{-1}{K}_{X(\cdot')}
    .\label{eqn:posterior_GP_moments}
\end{gather}
Evaluating both of these expressions present a cost cubic in the number of observations and output dimensions $\c{O}\left((nc)^3\right)$.

\subsection{Sampling from Gaussian processes \& random features}\label{subsec:random_features}

We saw in \cref{eq:prior_function_sampling_weight_space} how to sample from the prior distribution over functions by first sampling the weights from the prior $w \sim \N(0, A^{-1})$ and taking an inner product with the feature expansion $\innerprod{\phi(\cdot)}{w}$. This operation presents a linear cost in the number of features $d$, resulting computationally intractable when dealing with an infinite dimensional feature space, such as the one associated with the RBF kernel \cref{eq:random_feature_convergence}.

\paragraph{Matrix square root sampling}
Instead, from \cref{eq:prior_over_functions_evaled_weightspace},
we know that the distribution over prior function samples evaluated at a pre-fixed set of points $X' \in \c{X}^{n'}$ is $\N(0, K_{X' X'})$. Thus, we can evaluate a prior sample at $X'$ by transforming an $n'$ dimensional vector of standard Gaussian noise with a matrix square root of the covariance. For instance, we may use the Cholesky decomposition $\mathrm{L} \mathrm{L}^T = K_{X' X'}$ to compute
\begin{gather}
    f(X') = \mathrm{L} u \spaced{with} u \sim \N(0, I_{cn'}).
\end{gather}
Be that as it may, this approach requires knowing the points at which we want to evaluate our prior functions a priori and presents a cost cubic in the number of points we want to evaluate at $\c{O}\left((n'c)^3\right)$. Furthermore, if $X'$ contains repeated points or pairs of points for which the kernel evaluates to very small values, $K_{X' X'}$ may be singular or close to singular, resulting in numerical instability when computing its square root.

\begin{figure}[t]
\centering
\includegraphics[width=0.75\linewidth]{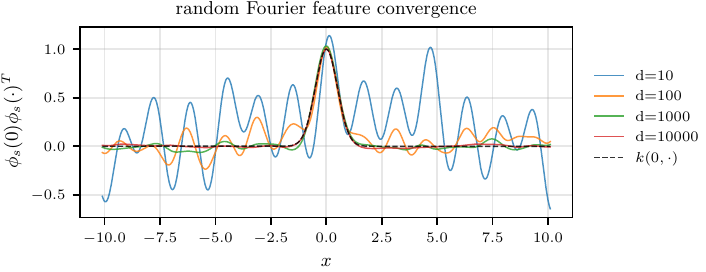}
\caption{Convergence of random Fourier feature basis (given in \cref{eq:random_fourier_basis}) to the RBF kernel's evaluation functional $k(0, \cdot)$ using the estimator in \cref{eq:unbiased_rff_kernel} as the number of random features $d$ increases.  
}
\label{fig:rff_convergence_to_kernel}
\end{figure}

\paragraph{Random feature prior sampling}
Fortunately, we may approximate prior function samples to high accuracy using random features \citep{Rahimi2007random,wilson20,terenin2022thesis}. 
In particular, we may use some feature expansion $\phi_s : \c{X} \to R^{c \times d}$ parametrised by a random variable $s$ with law $\Omega$ to construct an unbiased estimator of a kernel function as
\begin{gather}\label{eq:unbiased_rff_kernel}
    k(x,x') = \E_{s \sim \Omega} \, \phi_s(x)\phi_s(x')^T.
\end{gather} 
We can use these random features to construct a Monte Carlo estimator of a prior function sample $f \sim \text{GP}(\mu, k)$ as
\begin{gather}
    \label{eqn:rff_approx-prior}
    f(\cdot) \approx \tl{f}(\cdot) = \phi_{s}(\cdot) w
    \spaced{with}
    w \sim\N(0, I_d) \spaced{and} s \sim \Omega
\end{gather}
at $\c{O}(d)$ cost, where $d$ is the dimensionality of the feature expansion, often referred to as the number of random features. This parameter controls the error in the approximation, which goes to 0 as $d$ goes to infinity. 
We have approximately reversed the kernel trick, recovering a finite dimensional linear model. We may now evaluate our prior function sample at any $x \in \c{X}$ by simply evaluating the random features at $x$ and taking an inner product with the random weights.
Following \cite{wilson20}, both the next section and \cref{chap:SGD_GPs} will efficiently draw approximate posterior function samples by replacing instances of $f$ with $\tl{f}$.

Random Fourier features \cref{eq:random_fourier_basis} can be used to approximate any stationary kernel---that is, those that can be written as $k(x,x') = k'(x - x')$ for $k' : \c{X} \to \R\:$---by taking the distribution from which the cosine frequencies are sampled $\Omega$ to be the normalised spectral measure of the kernel $k$. As we saw in \cref{eq:random_feature_convergence}, the RBF kernel is recovered when $\Omega$ is chosen to be Gaussian. We illustrate the convergence of this estimator in \cref{fig:rff_convergence_to_kernel}.
More sophisticated Fourier feature sampling strategies have been developed to reduce the variance of the above estimators \citep{Felix2016Orthogonal,Reid2023Simplex}.
Some non-stationary kernels also admit random features. For instance, there exist random features that describe graphs \cite{Reid2023Graph}, ones that describe sets of binary attributes \citep{Tripp2023tanimoto}, and ones that approximate the attention mechanism \citep{Peng2021randomfeatureattention}.

\begin{figure}[thb]
\centering
\includegraphics[width=0.65\linewidth]{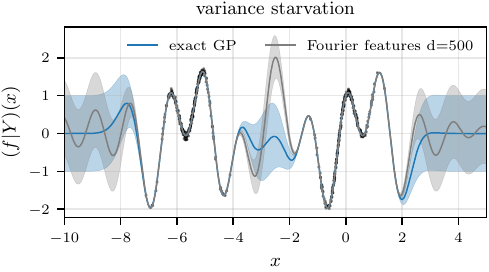}
\caption{Illustration of variance starvation when using random Fourier features to approximate a posterior GP. The shaded region represents a one standard deviation credible interval. Our $n=5000$ datapoints are placed close to each other and are largely redundant. This consumes the degrees of freedom of our $d=500$ random features, leading to arbitrary extrapolation and reduced uncertainty away from the training data. This doesn't happen with our exact GP, which effectively uses an infinite number of basis functions.
}
\label{fig:variance_starvation}
\end{figure}

\paragraph{Variance Starvation}
Random Fourier features can also be used for approximate posterior inference at cost $\c{O}(d^3)$. For this, we simply approximate the infinite feature expansion with $d$ random features and then proceed with linear model inference as in \cref{eq:posterior_linear_function_draws}. However, this is not advisable, as the degrees of freedom needed to represent posterior functions grow with the number of observations and $d\gg n$ features are often needed to obtain a good approximation. This issue is known as variance starvation. It is intimately related to Gibbs ringing and is discussed in detail in 2.4.2 of \cite{terenin2022thesis}. We illustrate variance starvation in \cref{fig:variance_starvation}.

\section[The Pathwise view]{Pathwise view: working with the posterior random function}\label{sec:pathwise_view}

We have so far characterized inference in the Gaussian linear model in terms of the posterior distribution over its weights and the posterior Gaussian process. These require dealing with the posterior covariance matrices over weights and observations, respectively. Even storing these in memory can be computationally intractable when the number of parameters or observations is large, which is the setting of interest of this thesis.
This section introduces \emph{pathwise conditioning} \corr{\citep{wilson20,Wilson2021Pathwise}}, a formulation of inference that deals only with posterior weight or function samples, as opposed to complete posterior distributions. In turn, this will allow us to deal with vectors of dimension $d$ or $nc$, as opposed to covariance matrices, which are quadratic in those quantities. We will build upon the pathwise view of inference to design scalable approximate inference algorithms in \cref{chap:SGD_GPs}.

For the Bayesian model in \cref{eq:gp_bayesian_model}, one can write the posterior random function directly as
\begin{gather}
    (f| Y)(\cdot) = f(\cdot) + {K}_{(\cdot)X} (K + B^{-1})^{-1}(Y - f(X) - \c{E}) \notag \quad \text{with} \\
    \c{E} \sim\N(0,B^{-1})
    \spaced{and}
    f\sim\GP(\mu, k)\label{eqn:pathwise_samples}
    .
\end{gather}
It is straight forward to check that the moments of $(f| Y)$ match those of the posterior GP given in \cref{eqn:posterior_GP_moments}.
Thus, evaluating this expression for a particular prior function sample $f$ and noise sample $\c{E}$ yields a posterior function sample. Although we retain the cubic cost of a linear solve against $(K + B^{-1})$, this only needs to be done once. Then we are free to evaluate the posterior sample at any set of test points $X'$ at only linear cost in $nc$. Additionally, we avoid the need to store the covariance matrix explicitly.

To gain a better understanding of the pathwise form of the posterior, we can rewrite it as a sum of three terms
\begin{gather}
  \label{eqn:pathwise-zero-mean}
  (f| Y)(\cdot) = \ubr{f_{\star}(\cdot) \vphantom{{K}_{(\cdot)X}} }_{\text{posterior mean}} + \ubr{f(\cdot) \vphantom{{K}_{(\cdot)X}} }_{\text{prior sample}}  - \ubr{{K}_{(\cdot)X} (K + B^{-1})^{-1}(f(X) + \c{E})
  }_{\text{uncertainty reduction term}} \\
  \text{with} \quad
  \c{E}\sim\N(0, B^{-1}) \spaced{and}
  f\sim\GP(0, k) \notag,
\end{gather}
which are illustrated in \Cref{fig:pathwise_illustration}.
The first component is the posterior mean function $f_{\star}(\cdot) = {K}_{(\cdot)X}\alpha_\star$, which we analysed in \cref{subsec:kernel_trick}. Its job is ensuring our posterior function samples pass near the datapoints. To it, we add a prior function sample, whose value will vary across input space in a data-independent way. The uncertainty reduction term cancels the effect of the prior function sample near the datapoints.\footnote{We say that two points $x_i$ and $x_j$ are ``near'' when $k(x_i, x_j)$ is small.} It ensures the posterior function sample takes values close to the posterior mean, and thus close to our observed targets, near the training data. 
Just like the posterior, the uncertainty reduction term takes the form of a linear combination of evaluation functionals ${K}_{(\cdot)X}\alpha_u$, with $\alpha_u = (K + B^{-1})^{-1}(f(X) + \c{E})$. 
Consequently, far away from the observed data, the posterior function samples revert to the prior function samples, inflating the uncertainty in the posterior to match the prior uncertainty. 

\begin{figure}[thb]
\centering
\includegraphics[width=\linewidth]{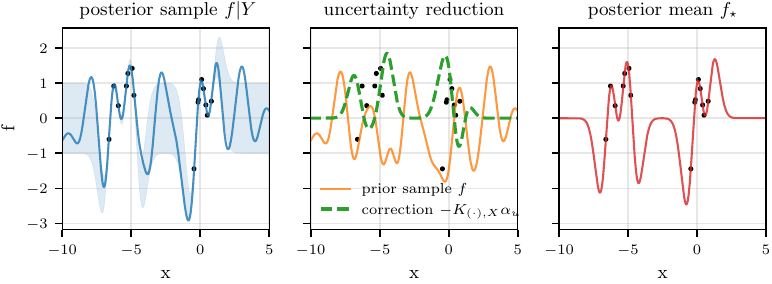}
\caption{Illustration of the pathwise construction of the posterior function sample, shown on the left together with a single standard deviation posterior credible region contour. The middle plot shows a prior function sample together with its corresponding uncertainty reduction term, which cancels the prior sample near the training data. The right side plot shows the GP posterior mean function, which is added to the prior sample and uncertainty reduction term to build the posterior sample.
}
\label{fig:pathwise_illustration}
\end{figure}

The pathwise formulation first appeared in the field of geostatistics, where it was referred to as ``Matheron's rule'' \citep{Journel1978mining}. It has been used to perform inferences in astrophysics models \citep{Hoffman1991Constrained,Hoffman2009Fields} and Gaussian Markov random fields \citep{Papandreou2010perturbations}. More recently, it was re-discovered and popularised among the Gaussian process community by \cite{wilson20}, to whom the form \cref{eqn:pathwise_samples} is due. 

\begin{remark} \textbf{Are GP function samples in the RKHS?} \\
    The pathwise formulation of GP posterior samples \cref{eqn:pathwise-zero-mean} allows us to answer this question; we must simply check the norm of each term in the RKHS. For the posterior mean, we have
    \begin{gather*}
        \|f_\star\|^2_{\c{H}} = \innerprod{K_{(\cdot) X} \alpha_\star}{K_{(\cdot) X} \alpha_\star} = \alpha_\star^T K \alpha_\star
    \end{gather*}
    which will be a finite number as long as the number of observations times output dimensions $nc$ is finite. We can use the same argument for the uncertainty reduction term $K_{(\cdot) X} \alpha_U$. However, this is not necessarily true for the prior sample. For infinite dimensional feature expansions, it can not be written as a linear combination of a finite number of basis functions. Its RKHS norm may be infinite $\|f\|_{\c{H}} = \infty$. Thus, neither the GP prior or posterior functions live in the RKHS associated with the GP's covariance kernel $k$. However, the difference between the GP prior and posterior functions always lives in the RKHS $f - (f|Y) \in \c{H}$.
\end{remark}

\subsection{Efficiently sampling from GP posteriors with random features}\label{subsec:efficient_prior_sampling}

The practical utility of the pathwise formulation  \cref{eqn:pathwise_samples} rests on our ability to efficiently evaluate a prior function sample.
In the infinite-dimensional feature case, this can present a number of challenges, discussed in \cref{subsec:random_features}. However, following \cite{wilson20}, we can efficiently approximate the pathwise form of posterior functions using a random feature approximation of the prior
\begin{gather}
    \label{eqn:pathwise_samples_fourier}
    (f| Y)(\cdot) \approx \tl{f}(\cdot) + {K}_{(\cdot)X} (K + B^{-1})^{-1}(Y - \tl{f}(X) - \c{E}) \notag \spaced{with} \c{E} \sim\N(0,B^{-1}) \\
    \spaced{and}
    \tl{f}(\cdot) =  \phi_{s}(\cdot) w \quad 
    w \sim\N(0, I_d) \quad s \sim \Omega
    .
\end{gather}
Importantly, random features are only used to approximate the prior function sample. Conditioning on the data is done via the exact linear solve, at cubic cost in the number of observations and output dimensions, avoiding variance starvation. 
Pathwise sampling combined with random features provides a very powerful toolkit for decision making under uncertainty which we will use throughout this thesis.

\subsection{Duality between pathwise conditioning and sample-then-optimise}

We now present the primal form of the pathwise formulation of posterior samples for finite dimensional feature spaces and show that it is equivalent to the ``sample-then-optimise'' posterior sampling strategy for Bayesian linear models \citep{Matthews2017SamplethenoptimizePS}\footnote{Although, I believe this observation to first have been made in \cite{antoran2023samplingbased}, which forms the basis of \cref{chap:sampled_Laplace}, it is presented here as it constitutes a useful building block for the rest of the thesis.}. For this, we start from the pathwise expression of the posterior over weights
\begin{gather}\label{eq:pathwise_weights}
    w | Y = w_0 +   A^{-1} \Phi^T (\mathrm{B}^{-1} +  K )^{-1}  (\c{E} - \Phi w_0) \\
    \text{with} \quad
  \c{E}\sim\N(0, B^{-1}) \spaced{and}
  w_0 \sim \N(0, A^{-1}) \notag,
\end{gather}
which matches \cref{eqn:pathwise_samples} but we have removed the product with the feature expansion that maps weight samples to function samples, that is $f(\cdot) = \phi(\cdot) w$. Despite returning a $d$ dimensional weight sample, \cref{eq:pathwise_weights} retains a linear solve against $ (\mathrm{B}^{-1} +  K )$ at cost $\c{O}\left( (nc)^3 \right)$.

\begin{derivation} \textbf{Duality of pathwise conditioning and sample-then-optimise} \\
    We recall that $H = A + \Phi^T \mathrm{B}\Phi$ and then apply the following series of matrix identities to \cref{eq:pathwise_weights} 
\begin{align}
    w | Y &= w_0 +   A^{-1} \Phi^T (\mathrm{B}^{-1} +  \Phi A^{-1} \Phi^T )^{-1}  (\c{E} + Y - \Phi w_0) \\
    &=  w_0 +   A^{-1} \Phi^T \mathrm{B} (I +  \Phi A^{-1} \Phi^T \mathrm{B})^{-1}  (\c{E} + Y - \Phi w_0) \\ 
    &= w_0 +   A^{-1} (I +  \Phi^T \mathrm{B} \Phi A^{-1})^{-1} \Phi^T \mathrm{B} (\c{E} + Y - \Phi w_0)\\
    &= w_0 + H^{-1}\Phi^T \mathrm{B} (\c{E} - \Phi w_0) \\
   &= H^{-1}((H - \Phi^T \mathrm{B}\Phi)w_0 + \Phi^T \mathrm{B} (\c{E} + Y)) \\
   &= H^{-1}(\Phi^T \mathrm{B} (\c{E} + Y)  + A w_0) \label{eq:sample_then_optimise_solution} \\
   &\text{with} \quad
  \c{E}\sim\N(0, B^{-1}) \spaced{and}
  w_0 \sim \N(0, A^{-1}) \notag.
\end{align}
\end{derivation}
Equation \Cref{eq:sample_then_optimise_solution} recovers an expression containing a linear solve against $H$, with cost $\c{O}\left( d^3\right)$. By visual inspection, we can identify that \cref{eq:sample_then_optimise_solution} matches the form of the maximum a posteriori weight setting for weight-space model \cref{eq:posterior_linear_function_draws}, but where our targets are perturbed by adding $\c{E}$ and our prior mean is $w_0$. Thus, \cref{eq:sample_then_optimise_solution} represents the solution to a quadratic problem analogous to the linear regression loss 
\begin{gather}\label{eq:sample_then_optimise}
    w | Y = \min_{w \in \R^d} \frac{1}{2} \|Y + \c{E} - \Phi w\|^2_{\mathrm{B}} + \frac{1}{2} \|w - w_0\|^2_{A} \\
    \text{with} \quad  \c{E} \sim\N(0, B^{-1}) \spaced{and}
    w_0 \sim \N(0, A^{-1}) \notag.
\end{gather}
\begin{sloppypar}
This expression is known in the literature as the ``sample-then-optimise'' objective.  \cite{Matthews2017SamplethenoptimizePS} use it in the noiseless case ($B^{-1}=0$) to study the connection between Bayesian inference and optimisation in overparametrised linear models. \cite{Osband2018randomised} and \cite{Pearce2020ensembling} apply variants of the objective to draw 
approximate posterior samples from the posterior distribution over the weights of a neural network, although this approximation may be very crude. 
\end{sloppypar}

Leveraging the pathwise or sample-then-optimise formulations for posterior sampling at scale (large $d$ and large $nc$) still requires solving large linear systems, which remains an open problem. Standard methods, such as conjugate gradients and matrix sketching are discussed in \cref{chap:approx_inference}. In \cref{chap:SGD_GPs} we instead propose to use stochastic gradient descent for pathwise inference.

\subsection[Decision making: Bayesian optimisation]{Demonstration: pathwise conditioning for Bayesian optimisation}\label{subsec:decision_making}
We conclude by reviewing how we may use pathwise inference for efficient decision making. In particular, consider the problem of finding the input which maximises some unknown function $g : \c{X} \to \R$ in the least number of function evaluations. To this end, we place a GP prior over the function and choose new points at which to evaluate $g$ as
\begin{gather}\label{eq:decision_theory}
  x_{\text{new}} = \argmax_{x' \in \c{X}} \int \c{U}(x', f) \, d P_{f|Y},
\end{gather}
where $P_{f|Y}$ is the measure of the posterior GP and $\c{U} : \c{X} \times \R^{\c{X}} \to R$ is a utility function \citep{Hansson2011decision}. The latter is chosen to trade-off exploration and exploitation. 

\begin{figure}[thb]
\centering
\includegraphics[width=0.9\linewidth]{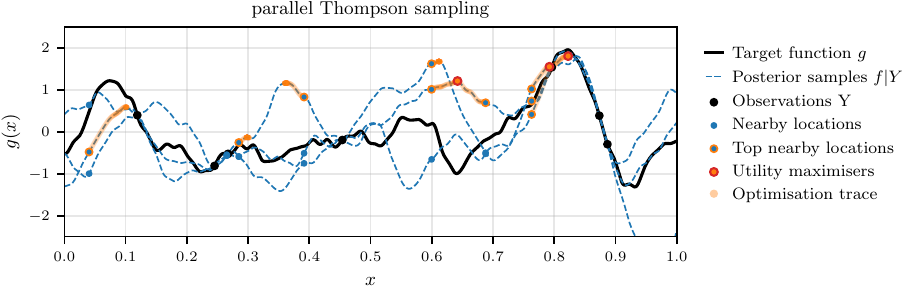}
\caption{ Illustration of parallel Thompson sampling procedure with multistart gradient-based optimisation of posterior function samples. Our GP is initialised with 7 observations from $g$ corresponding to inputs chosen uniformly at random from $[0, 1]$.
}
\label{fig:thompson_c2}
\end{figure}

For our example, we will use Thompson sampling \citep{thompson1933sampling}, where $\c{U}(x, f) = \mathbbm{1}(f(x) = \max_{x' \in \c{X}} f(x'))$ where $\mathbbm{1}$ is the indicator function. At each step, we approximate the integral in \cref{eq:decision_theory} with a single Monte Carlo (MC) sample. \nomenclature[z-mc]{$MC$}{Monte Carlo} That is, we draw a single posterior function sample and choose the input that maximises it $x_{\text{new}} = \argmax_{x' \in \c{X}} (f|Y)(\cdot)$. We then evaluate $g(x_{\text{new}})$ and add it to the dataset we use to perform posterior inference in our GP.   

\corr{
In \cref{fig:thompson_c2}, we demonstrate a single step of parallel Thompson sampling \citep{hernandezlobato2017Parallel} in a 1d toy problem with $\c{X} = [0,1]$ and where $g \sim \GP(0, k)$ with $k$ being the Matérn $\nicefrac{3}{2}$ kernel. The parallel variant differs from the above explained algorithm in that at each step we draw multiple posterior functions which we maximise to add multiple observations to our dataset at each step. We use 3 posterior functions, depicted as dashed blue lines. We add homoscedastic Gaussian noise of precision $\sqrt{1000}$ to target function evaluations. We maximise each posterior function by first evaluating it at 7 inputs chosen uniformly at random from $[0,1]$. These are labelled ``nearby locations'' in the legend. We keep the inputs corresponding to the top 3 posterior function evaluations, labelled ``top nearby locations'', and use the Adam optimiser to improve them until corresponding local optima of the posterior functions are found. These are labelled ``utility maximisers''. We evaluate the target function at the utility maximisers and add the corresponding input-observation pairs to our dataset.

The pathwise formulation of posterior functions is critical to make this algorithm computationally efficient.  It allows us to solve a single linear system to obtain each posterior function per Thompson step. 
After this, we may evaluate each posterior function $(f|Y)$ an arbitrary number of times for its maximisation at only linear cost in the number of observations. This contrasts with the cubic cost per evaluation that we would have had to incur had we used the more-traditional definition of a posterior GP in terms of its first and second moments \cref{eqn:posterior_GP_moments}. The remaining bottleneck is solving the linear system to update the posterior functions when the number of observations $n$ becomes large. This challenge will be dealt with in \cref{chap:SGD_GPs}.  

}

\section[Model selection]{Model selection: the marginal likelihood, or evidence, and empirical Bayes}\label{sec:model_selection}

So far we have seen how to perform Bayesian inference over model parameters, how to transform the posterior distribution over parameters into predictions and how these predictions can be used to make decisions under uncertainty.
All of these techniques rest on our prior modelling choices. In \cref{eq:weight_space_gen_model}, we assumed that our targets are generated as a noisy linear combination of basis functions $\phi$. Furthermore, we assume that the weights of this linear combination were sampled from a zero-centred Gaussian with precision $A$ and that the additive observation noise is also Gaussian with precision $B$. We refer to these quantities (i.e. $\phi, A, B$), over which we do not perform Bayesian inference, as hyperparameters and denote them by $\theta \in \Theta$. We henceforth refer to the choice of hyperparameters and the choice of model interchangeably\footnote{Any model can be written in terms of a broader model class $\Theta$ which is indexed by a set of hyperparameters $\theta \in \Theta$}. The quality of the inferences that we do make rests on the appropriateness of our hyperparameter choices 
\citep{Masegosa2020Misspecification}. Although the Bayesian framework forces us to make our modelling choices explicit, it does not tell us which choices to make.

Intuitively, we should choose our model such that it incorporates all our knowledge about the generative process of the data. As we saw in \cref{subsec:basis_functions}, the more restrictive the model class, 
the less degrees of freedom will be left to be pinned down by the data, and the more confident we can be in our inferences. 
However, if our strong prior assumptions are wrong, we risk our inferences being biased and our predictions not reflecting real world outcomes.

In this section, we will depart slightly from the Bayesian framework to introduce model selection tools that efficiently navigates the bias-variance trade off. 
To this end, consider the integral of the likelihood against the prior, which featured as the denominator in Bayes rule \cref{eq:bayes_rule_c2}. For the weight space linear model, this is 
\begin{gather}\label{eq:model_evidence_with_p_calculus}
    \log p(Y; \,\phi, B, A) = \log \E_{w \sim \Pi} [p(Y | w; \, \phi, B)] = \log \int_w p(Y | w; \, \phi, B) \pi(w; A) \, d w
\end{gather}
which is known as the log marginal likelihood, or the model evidence. We use the semicolon $;$ to separate model parameters from hyperparameters on which the likelihood and the prior depend but over which we do not place a prior or perform inference. We have written out these hyperparameters in \cref{eq:model_evidence_with_p_calculus} for clarity, but we henceforth group them into the tuple $\theta = (\phi, A, B)$ for brevity.
The evidence measures the degree of overlap between the prior and the likelihood, thus rewarding a choice of prior that concentrates its mass on parameter settings that fit the training data well. Too broad a prior will spread its probability mass across many models, only some of which will fit the data, decreasing the evidence.
In this way, the model evidence differs from the training loss; the latter can always be improved by using a more flexible model.
See chapter 28 of \cite{MacKay2011Information} for additional discussion and illustrative examples.

\begin{remark} \textbf{Automatic Occam's razor}\\
    In the literature, the model evidence is said to automatically incorporate ``Occam's razor'' since it implicitly favours ``simpler'' priors \citep{Jeffreys1939Probability,Gull1988maxent,Mackay1992Thesis,Rasmussen2000Occams}. In this context, the notion of complexity refers to the degree of diversity of the hypotheses supported by the prior. For instance, we would say that the class of affine models is simpler than the class of third order polynomials, since the latter contains the former and many more functions. Intuitively, there are many more complex functions than simple ones.

    It is important to note that ``simple'' does not mean more linear, continuous, or having a lower Lipschitz constant. 
    For instance, a prior over third degree polynomials, where all of the coefficients of order greater than 0 are set to a fixed quantity---only the bias is left to be inferred---would be considered simpler and, thus preferred by the evidence, to a prior over affine models, assuming both families of functions fit the data equally well.
\end{remark}

A complementary point of view of \cref{eq:model_evidence_with_p_calculus} is that it is the log-density of the training data when our model is set to the prior. 
If our prior is able to predict our training observations, then our posterior will not differ much from our prior, yielding credence to it also being able to predict yet-unseen datapoints. This intuition is formalised in the framework of PAC-Bayes bounds \citep{Germain2016PAC,Masegosa2020Misspecification}. Also intimately related to the model evidence are the framework of minimum description length \citep{grunwald2004MDL} and other model selection criteria such as Akaike information criterion \citep{Akaike1970Statistical} and Bayesian information criterion \citep{neath2012bic}.

\subsection{Comparing two models}

The marginal likelihood of of some model $M_1$ differs from the regular likelihood in that the model parameters have been marginalised out. In this sense, it can be seen as a quantity at the second level of inference. The first level is inference over parameters, the second is over model class. We could apply this idea again to construct a third level likelihood to score members of a family of meta-model classes and so on. Thus, if we want to decide which model is best among a pair of models  ${M_1, M_2}$ we can compute the ratio of their posterior probabilities at the second level of inference as
\begin{gather}\label{eq:likelihood_ratios}
    \frac{p(M_1 | Y)}{p(M_2 | Y)} = \frac{p(Y | M_1) p(M_1)}{p(Y | M_2) p(M_2)}.
\end{gather}
Often the priors are chosen to be uniform over models $p(M_1) =p(M_2)$ and the posterior probability ratio matches the likelihood ratio $ \frac{p(Y | M_1)}{p(Y | M_2)} $. Likelihood ratios provide a Bayesian alternative to hypothesis tests. See chapter 37 of \cite{MacKay2011Information} for a detailed discussion.

\begin{remark} \textbf{On the dangers of model comparison with the evidence} \\
\corr{
A criticism of marginal likelihood-based model comparison is its sensitivity to the choice of prior \citep{Kass1995BayesFactors}. This is especially concerning when placing (seemingly) uninformative priors over our models' parameters. Intuitively, as our prior becomes fully uninformative (e.g. an improper uniform distribution over the parameters), its marginal likelihood goes to 0. 
In the almost-fully-uninformative regime, small changes in the prior hyperparameters, which have very little effect on posterior inferences, can have large effects on the model evidence. 
In response, a number of ``sensitivity analysis'' methods have been introduced to characterise the sensitivity of the evidence to the prior hyperparameters \cite{Sinharay2002sensitivity}.}

One\corr{ could argue that if we are fully uncertain about the parameters of a model, and the model's predictions depend strongly on those parameters, we should be happy to throw the model in the trash. Yet, this is roughly the case with neural networks, and here we are. 
A different, perhaps more Bayesian view is that we should not use the evidence to perform model selection at all. Instead of discarding one model with less evidence than another, we should expand out model class and consider both models in our Bayesian model average. If we were willing to consider both models for comparison in the first place, then we must have assigned some credibility to both models a priori, and our inferences should reflect this. This is roughly the view expressed by Adrew Gelman in a blog post addressing \cite{MacKay2011Information} chapter on Bayesian model comparison.}\footnote{\href{https://statmodeling.stat.columbia.edu/2011/12/04/david-mackay-and-occams-razor/}{statmodeling.stat.columbia.edu/2011/12/04/david-mackay-and-occams-razor}}

\end{remark}

\subsection{Hyperparameter optimisation}

We now extend the notion of model comparison to a continuous model space. 
For our linear model, the fully Bayesian approach would introduce a prior over $\theta = (\phi, A, B)$ and then perform inference. Unfortunately, this is rarely done. Performing inference at the higher levels of a Bayesian hierarchical model is often too computationally expensive to be practical outside of toy settings.  
As an alternative, when the number of hyperparameters is small relative to the number over observations, the posterior distribution over hyperparameters may be well approximated by a point mass at its mode $p(\theta | Y) \approx \delta(\theta - \theta_{\star})$ with 
\begin{gather}
\theta_\star = \argmax_{\theta} p(Y; \theta).
\end{gather}
In this setting, the likelihood over hyperparameters dominates the prior, and thus the latter is ignored.
Here, the model evidence can provide us with a learning objective to select our hyperparameters. The possibility of performing gradient-based optimisation of $\log p(Y; \theta)$ makes this approach an attractive alternative to traditional cross validation. We must be cautious when using this technique when the dataset is small or our hyperparameter space is large however, as the point-mass-posterior assumption can break, leaving us susceptible to overfitting.

\subsection{The evidence of the linear model}

For the Gaussian linear model, the model evidence can be computed in closed form 
\begin{gather}
    \log p(Y; \,\phi, B, A) = \log \int_w \sqrt{\frac{\det{B}}{(2 \pi)^{n}}} \exp\left(-\frac{1}{2} \|Y - \Phi w\|^2_{B} \right) \, d \N(0, A^{-1}) \notag \\
    =  - \frac{n}{2}\log(2 \pi) -\frac{1}{2}\logdet{(\Phi A^{-1} \Phi^T + B^{-1})^{-1}} -\frac{1}{2} \|Y\|^2_{(\Phi A^{-1} \Phi^T + B^{-1})^{-1}}\label{eq:GP_evidence}
\end{gather}
which amounts to the log density of the targets under a multivariate Gaussian with mean 0 and covariance $\Phi A^{-1} \Phi^T + B^{-1}$. The equivalent kernelised form, which can be used to optimise kernel hyperparameters, like the lengthscale, is obtained by again substituting $K = \Phi A^{-1} \Phi^T$. The cost of evaluating \cref{eq:GP_evidence} is cubic in $nc$ because of both the linear solve against $K+B^{-1}$, and because of the appearance of the same matrix's log-determinant. The former appeared in the expression for the posterior distribution (e.g.\cref{eqn:posterior_GP_moments}) but the latter presents a new challenge, which we will also tackle in the later chapters of this thesis.  

\begin{figure}[tbh]
\centering
\includegraphics[width=\linewidth]{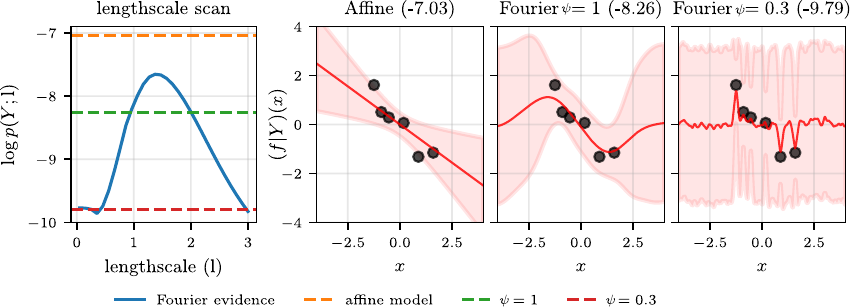}
\caption{The leftmost plot displays the evidence of a $d=500$ random Fourier basis function linear model as a function of the lengthscale $\psi$ parameter for the toy 1d dataset of \cref{fig:linear_posterior_functions}. The evidence of the Affine model, and the $\psi=1$ and $\psi=0.3$ Fourier models are indicated as dashed horizontal lines. The posterior mean function, along with 2 standard deviation errorbars are displayed for each of these models in the three plots on the right. In these, each model's evidence is provided in parenthesis in the plot title. Other hyperparameters match those of \cref{fig:linear_posterior_functions}.
}
\label{fig:evidence_comparison}
\end{figure}

Using the Woodbury matrix identity and the matrix determinant lemma we recover the primal form of \cref{eq:GP_evidence}, with cost cubic in the number of parameters $d$
\begin{gather}
    - \frac{n}{2}\log(2 \pi) + \frac{1}{2}\logdet{B} - \frac{1}{2} \underbrace{\|Y - \Phi w_\star \|^2_{B}}_{\text{data fit}} - \frac{1}{2} \underbrace{\|w_\star \|^2_{A}}_{\text{parameter norm}} - \frac{1}{2}\underbrace{\logdet{\frac{H + A}{A}}}_{\text{posterior contraction}}.
\end{gather}
This expression more intuitively captures the quality of fit vs simplicity trade-off discussed at the beginning of the section. There is a data fit term that rewards the posterior mean for passing near the targets. There is a prior fit term that ensures the norm of the posterior mean weights are small in the metric given by the prior precision. Finally, the determinant ratio term measures the contraction of the posterior covariance's volume relative to the prior covariance. Up to a constant factor, this quantity matches the information that was gained by our model by seeing the data, in nats. This captures the intuition that the marginal likelihood rewards models that are able to explain the targets well a priori, and thus do not learn much from conditioning on the targets.

\Cref{fig:evidence_comparison} compares the model evidence for the Affine and random Fourier models introduced in \cref{subsec:basis_functions} and the dataset from \cref{fig:linear_posterior_functions}. The targets are roughly arranged in a straight line, making the affine model a good fit. Although different lengthscale Fourier models can also fit this data, their additional flexibility penalises them; there is no lengthscale setting for which the Fourier model's evidence surpasses the Affine model's. The leftmost plot shows the evidence as a function of the lengthscale. Too small lengthscale values lead to too flexible models that overfit. This is the case for the model in the rightmost plot. Too large lengthscale values would under fit. The optima is somewhere in the middle.

\begin{remark}
\textbf{All linear models are wrong, but the evidence can tell us which are useful} \\
``All models are wrong, but some are useful'' -- George Box
\\
We almost never expect the data we are modelling to have been generated via a noisy linear combination of basis functions. On the other hand, we usually judge models on whether their predictions about quantities we care about match empirical outcomes to a desirable tolerance. It is well known that the Bayesian posterior does not provide optimal predictions under model misspecification \citep{draper2010calibration,Masegosa2020Misspecification}. It may seem surprising then, that the linear model's evidence can be shown, using the PAC-Bayes framework, to provide guarantees about generalisation performance \citep{Germain2016PAC}, informing us about whether our models are ``useful''.

It is worth noting that PAC-Bayes guarantees no longer hold if we use the model evidence for hyperparameter selection. If we select among a discrete set off hyperparameters, we could obtain relaxed guarantees via a union bound, but this would not work for a continuous hyperparameter space.
\end{remark}

\subsection{Effective dimension}\label{subsec:effective_dimension}

We conclude with a discussion of the effective dimension, a quantity \corr{intimately related to the evidence of the Gaussian linear model \citep{Mackay1992Thesis,Wipf2007Determination,Maddox2020Rethinking}. We will make heavy use this quantity to derive efficient algorithms for hyperparameter learning in \cref{chap:sampled_Laplace}.} Let $\lambda_1, \lambda_2, \dotsc, \lambda_d$ denote the sorted (in descending order) eigenvalues of the weight space loss curvature $M = \Phi^T B \Phi$. For a regulariser of the form $A = a I$, the effective dimension $\gamma$ is given by
\begin{gather}
    \gamma = \sum_{i=1}^d \frac{\lambda_i}{\lambda_i + a}.
\end{gather}
When $\lambda_i \gg a$ the term inside of the sum will roughly be of value 1. When $\lambda_i \ll a$ it will be roughly 0. Thus, the effective dimension $\gamma \in [0, \min(nc, d)]$ counts the number of directions in parameter space which are determined by the data. 

Following \cite{Mackay1992Thesis}, we can construct a more general definition for the effective dimension that doesn't require an isotropic regulariser, by taking it to be the trace of the matrix that maps the maximum likelihood parameter vector $H^{-1}Y$ into the maximum a posteriori weights $(\Phi^T B \Phi + A)^{-1}Y$. That is,
\begin{gather}\label{eq:forms_effective_dim}
   \gamma \coloneqq \tr{\left(\Phi^T B \Phi H^{-1}\right)} = d - \tr{\left(A H^{-1}\right)} =  \tr{\left(K (K + B^{-1})^{-1}\right)},
\end{gather}
where we have provided an additional two forms of the quantity, each providing for a complementary interpretation. Using the cyclical property of the trace we can see that the leftmost form is equivalent to $\tr{\left(B \Phi H^{-1} \Phi^T \right)}$. That is, the sum of the ratios of the marginal posterior predictive variance to noise variance at the observed inputs. Since each observation reduces the marginal uncertainty in the posterior over functions at that point to at least the corresponding diagonal entry of $B$, each diagonal entry of $B \Phi H^{-1} \Phi^T$ must be smaller or equal to 1. The degree to which the predictive variance is smaller than the observation noise depends on how well the datapoints explain each other. If they explain each other a lot, i.e. many inputs map to nearby points in the RKHS, the effective dimension decreases. The middle form of the effective dimension in \cref{eq:forms_effective_dim} provides us with the same intuition, but through the ratio of the prior and posterior covariance over the weights. The rightmost form is the trace of the matrix that maps the representer weights obtained by fitting the data without regularisation $K^{-1}Y$ onto the representer weights corresponding to the posterior mean function $(K+B^{-1})^{-1}Y$.

\begin{derivation} \textbf{Relating the forms of the effective dimension}\\
We first relate the first and second equalities in \cref{eq:forms_effective_dim}.
\begin{align*}
    \tr\left(H^{-1}\Phi^T B \Phi\right) &= \tr\left((I+A^{-1}\Phi^T B \Phi)^{-1}A^{-1}\Phi^T B \Phi\right)  \\ &= \tr\left( I - (I+A^{-1}\Phi^T B \Phi)^{-1} \right) = d - \tr\left(AH^{-1} \right).
\end{align*}
We now connect the first and third equalities
\begin{align*}
    \tr\left(H^{-1}\Phi^T B \Phi\right) &= \tr\left( B (K - K(K + B^{-1})^{-1}K) \right) \\ &= \tr\left( BK (I - (KB + I)^{-1}KB) \right) = \tr\left( BK (KB + I)^{-1} \right) \\ &= \tr\left( K (K + B^{-1})^{-1} \right).
\end{align*}
\end{derivation}

\section{Limitations of conjugate Gaussian-linear Bayesian reasoning}\label{sec:linear_model_limitations}

We have seen how the linear model with a Gaussian prior over its weights, or Gaussian process, acts as a conjugate prior for the likelihood induced by Gaussian observation noise, providing us with a closed form expression for the Bayesian posterior \cref{eq:posterior_linear_function_draws}, and model evidence \cref{eq:GP_evidence}, both Gaussian forms. 
Alas, conjugacy is quickly lost when constructing more sophisticated Bayesian models that more accurately describe real-world systems of interest. It is lost if we define a non-Gaussian prior over the weights, for instance heavy tailed priors used to model outlier events \citep{west2018Outlier} or priors designed to favour sparse posteriors, like the  horseshoe \citep{carvalho09Horseshoe}. It is lost if we use
non-Gaussian likelihoods, like the categorical used in classification \citep{Bishop2003classification}, or the Poisson used to count neural spikes \citep{Heeger2000Poisson} and X-ray quanta \citep{Elbakri2003poissonCT} in computed tomography. Conjugacy is also lost if our model presents a non-linear relationship between its parameters and outputs, for instance due to the use of a linking function that constrains the output range.

Of special interest for this thesis is the use of the neural network function class. These models can be thought of as basis function linear models in which the basis function parameters are treated as model parameters, instead of hyperparameters, and thus inferred from the data. Neural networks are used to model processes where we have little intuition of what the data-generating process might look like, and thus we can not manually choose a set of basis functions. To make up for this lack of prior knowledge, very large and flexible models are paired with vast datasets.

This leads to our second major setback. The closed form expressions of linear model inference involve cubic operations: linear system solves and log-determinant computations, both of which present cubic time complexity. We may choose to either pay this cost in terms of the number of observations times output dimensions $\c{O}\left( (nc)^3 \right)$ or model parameters  $\c{O}\left( d^3 \right)$ (when the feature space is finite-dimensional). This provides little consolation in the modern setting where it is common to work with large datasets. For instance, the Imagenet dataset \citep{imagenet}, which is a benchmark three orders of magnitude smaller than the datasets used to train the largest models in deployment \citep{dosovitskiy2021an}, has $nc\approx10^9$. The ResNet-50 neural network \citep{he2016deep}, another common benchmark model that is around 10 times smaller than the state of the art models, presents a parameter space with $d\approx 25 \cdot 10^6$. One may think that linear models could scaled up to problems of modern interest  via efficient numerical linear algebra routines implemented on GPU accelerators. However, at these scales, even storing covariance matrices, whose number of entries are quadratic, becomes intractable due to the  $\c{O}\left( (nc)^2 \right)$ or $\c{O}\left( d^2 \right)$ memory cost. For instance storing a covariance matrix for a parameter space the size of ResNet-50's would require around $2500$ Terabytes.

The following chapter reviews approximations to Bayesian inference which may be tractably computed when faced with non-conjugacy or large covariance matrices. Unfortunately, we will see how these approximations tend to break down when faced with the neural network model class and real-world sized datasets. The rest of the thesis aims to fill this gap by introducing methods for very large scale Bayesian reasoning with linear models and neural networks.

\chapter[Approximate inference]{Approximate inference methods\\for linear models and neural networks}\label{chap:approx_inference}

\ifpdf
    \graphicspath{{Chapter3/Figs/Raster/}{Chapter3/Figs/PDF/}{Chapter3/Figs/}}
\else
    \graphicspath{{Chapter3/Figs/Vector/}{Chapter3/Figs/}}
\fi

The need for approximate inference arises in the linear-Gaussian model when the problem setting becomes too large, making closed form expressions too computationally expensive to evaluate. It also arises when working with non-conjugate Bayesian models. This thesis deals with both settings, 1) Bayesian inference in Gaussian linear models with millions of parameters and observations, and 2) Bayesian inference in neural networks. On our way to tackling these problems, this chapter reviews approximate inference methods for linear models and Gaussian processes, and how these can be extended to neural networks. \Cref{sec:VI} covers Variational Inference (VI) \nomenclature[z-vi]{$VI$}{Variational Inference} in both its parameter-space and inducing point flavours. \Cref{sec:CG} covers the use of Conjugate Gradient (CG) \nomenclature[z-cg]{$CG$}{Conjugate Gradient} methods. Finally, \cref{sec:Laplace} introduces the Laplace approximation as well as its linearised variant for neural networks.
Through different paths, all of these methods provide both an approximation to the posterior as well as the model evidence. 
We do not delve into Markov Chain Monte Carlo (MCMC) \nomenclature[z-mcmc]{$MCMC$}{Markov Chain Monte Carlo} techniques, but instead refer
refer to \cite{Andrieu2003MCMC} for a general overview and to \cite{neal1992bayesian} for a discussion of their application to neural networks. 

\section[Variational Inference]{Approximating the posterior distribution:\\ Variational inference}\label{sec:VI}

We commence from Bayes rule \cref{eq:bayes_rule_c2}. Making the set of hyperparameters $\theta \in \Theta$ explicit in the notation, we take logs on both sides of the equality, and re-arrange it as
\begin{gather}
  \log p(Y; \, \theta) =   \log p(Y | w; \, \theta) + \log \pi(w; \, \theta) - \log \pi(w | Y; \, \theta),
\end{gather}
to evaluate the evidence. This expression holds for any value of $w$, allowing us to take expectations on both sides of the equality with respect to any distribution over $w$. We thus introduce the variational distribution $Q$, with density $q(w)$ such that $d Q = q(w) d\nu$, and use it to derive the lower bound
\begin{align}
  \log p(Y; \, \theta) &=   \E_{w \sim Q}\left[ \log p(Y | w; \, \theta) + \log \pi(w; \, \theta) - \log \pi(w | Y; \, \theta)\right]\ \\
  &\geq \E_{w \sim Q}\left[ \log p(Y | w; \, \theta) + \log \pi(w; \, \theta) \right] + \mathbb{H}\left(Q\right) \coloneqq \c{M}(Q, \theta). \label{eq:ELBO}
\end{align}
We refer to $\c{M}(Q, \theta)$ as the Evidence Lower BOund (ELBO) \nomenclature[z-elbo]{$ELBO$}{Evidence Lower Bound} and $\mathbb{H}$ is the differential entropy.
The inequality is true because the cross entropy can be decomposed into a sum of an entropy and the KL divergence between the distributions being compared, and the latter term is greater or equal to 0. That is, adopting the density-based notation for the KL divergence $\KL{q(w)}{\pi(w|Y)} = \int \log \frac{Q}{\Pi_{w|Y}} \, dQ $, we have  $\E_{w \sim Q}\left[- \log \pi(w | Y; \, \theta)\right] = \mathbb{H}\left(Q\right) + \KL{q(w)}{\pi(w|Y)} \geq \mathbb{H}\left(Q\right)$. Thus, when $\KL{q(w)}{\pi(w|Y)}=0$ and thus the variational posterior matches the Bayesian posterior $q(w) = \pi(w|Y)$, \cref{eq:ELBO} becomes an equality, and the ELBO matches the evidence $\log p(Y; \, \theta) = \c{M}(Q, \theta)$.

The ELBO allows us to transform the problem of Bayesian inference into one of variational optimisation. By maximising $\c{M}$ with respect to our variational distribution $Q \in \c{Q}$, we approximate the Bayesian posterior distribution in the sense of minimising $\KL{q(w)}{\pi(w|Y)}$ \citep{hinton1993keeping}. 
We may do this even if our search space, the variational family $\c{Q}$, does not contain the true posterior $\Pi_{w|Y} \notin \c{Q}$. This allows us to tractably approximate the Bayesian posterior even when this distribution is analytically or computationally intractable \citep{Attias1999variational}.
Evaluating the ELBO does not require conjugacy, only being able to evaluate the log-likelihood function and the prior log-density. We demonstrate this for a 1d toy classification example, where the likelihood is Bernoulli, in \cref{fig:classification}.  The expectation in \cref{eq:ELBO} is often unbiasedly estimated via Monte Carlo. Thus, the requirements on the variational distribution are that we can sample from it and that we can compute its entropy. Relaxing the latter constraint is an active area of research \citep{Titsias2019unbiased,Uppal2023implicit}.

\begin{figure}[t]
\centering
\includegraphics[width=\linewidth]{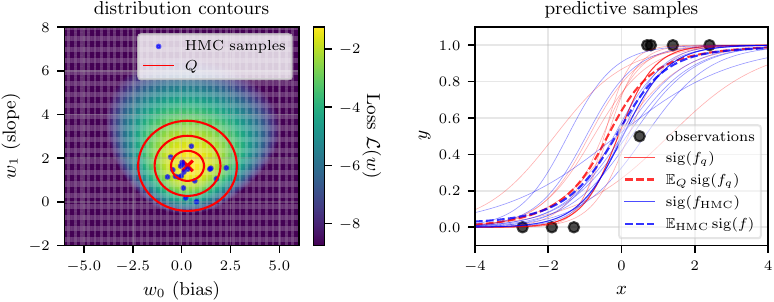}
\caption{Classification example with our affine linear model, where we place a Gaussian prior over the weights, we use a sigmoid linking function and a Bernoulli likelihood. The left plot shows how the loss landscape, which up to a constant matches the log posterior density, presents a non quadratic form; the top of the distribution is wider than the bottom. We approximate this posterior with a Gaussian variational distribution $Q$ and with Hamiltonian Monte Carlo (HMC). In this setting, only the latter method provides an unbiased approximation. \nomenclature[z-hmc]{$HMC$}{Hamiltonian Monte Carlo} Despite this, the plot on the right shows how both approximations lead to similar predictions. However, the variational approximation places more mass on low slope functions, resulting in slight underestimation of the steepness of the sigmoid. 
}
\label{fig:classification}
\end{figure}

\begin{remark} \textbf{Protection against overfitting} \\
It is often said that variational parameters are protected against overfitting. This is because optimising the ELBO with respect to these parameters always brings the variational distribution closer to the true posterior. Thus, choosing a more flexible variational family that leads to a tighter ELBO should always lead to a better posterior approximation. Unfortunately, the same is not true about the model hyperparameters, whose optimisation with the ELBO can lead to overfitting (see, for instance, \cite{ober2021promises}).
\end{remark}

\begin{figure}[t]
\centering
\includegraphics[width=\linewidth]{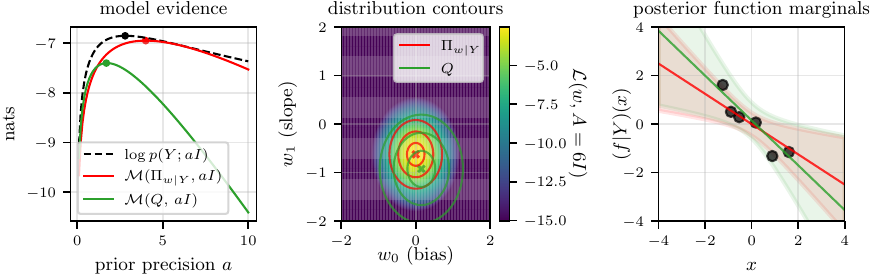}
\caption{Variational inference in the Gaussian affine linear model, fit to the toy dataset dataset in \cref{fig:linear_posterior_functions}. The leftmost plot shows the model evidence as a function of the isotropic prior covariance $A=aI$. We also display an ELBO where the variational posterior is set to the true posterior when $a=6$, denoted $\Pi_{w|Y}$ in the plot. The bound is tight at $a=6$, as predicted by \cref{eq:ELBO}. However, since the posterior over the weights does not change as we scan $a$, the optima of the ELBO, marked with a red dot, differs from the optima of the evidence. Hyperparameter selection with this objective would be biased. We also display, in green, the ELBO corresponding to a different variational posterior $Q$. Since $Q$, doesn't match the true posterior for any value of $a$, the bound is never tight. It is also a biased estimate of the evidence. The middle plot shows the loss function when $a$ is set to 6 as well as the 1, 2 and 3 standard deviation contours for the log-density of $\Pi_{w|Y}$ and $Q$. Finally, the rightmost plot shows the mean and 2 standard deviation errorbars of the posterior distribution over functions corresponding to each of the 2 variational posteriors under consideration.
}
\label{fig:linear_ELBO}
\end{figure}

The ELBO can also act as a hyperparameter selection objective, acting as a substitute for the model evidence when the later is not tractable. However, if the variational posterior differs from the true posterior, the hyperparameter learning objective will be biased (see, for instance, \cite{turnersahani2011bias}). We illustrate this bias in \cref{fig:linear_ELBO}.
The variational EM algorithm \citep{Dempster1977EM,Nealhinton1998EM,bishop2006pattern} implements this idea by iterating variational posterior optimisation and hyperparameter optimisation steps: 1) setting $Q = \argmax_{Q \in \c{Q}}\c{M}(Q, \theta)$ in the E step and 2)  $\theta = \argmax_{\theta \in \Theta}\c{M}(Q, \theta)$ in the M step. If  $\Pi_{w|Y} \in \c{Q}$, the E step will attain the exact posterior and the EM algorithm is guaranteed to not decrease the model evidence. Alternatively, one may optimise $\c{M}$ with respect to $\{Q, \theta\}$ jointly using gradient-based optimisation.

\begin{remark} \textbf{The dangers of model comparison with the ELBO}\\
    The ELBO is not a reliable tool for model comparison. If a model obtains a larger ELBO than another, it is not guaranteed to have a larger evidence. The model with the smaller ELBO could have a larger evidence and the difference in ELBO values could be due to there being more slack in the second model's bound. 
\end{remark}

Beyond approximate inference in predictive models, the ELBO also plays an important role in information theory and data compression; we refer to \cite{hinton1993keeping,greg2019compression} and chapter 33 of \cite{MacKay2011Information} for in-depth discussion.

\subsection{VI in the parameter space of the linear model}\label{sec:weight_space_VI}

We now provide the explicit form of the ELBO for the weight-space Gaussian linear model introduced in \cref{eq:weight_space_gen_model} paired with a multivariate Gaussian
variational family $Q = \N(w_q, \Sigma_q)$ with variational parameters $w_q \in \R^d$ and $\Sigma_q \in \R^{d \times d}$. In this case, the true posterior is contained within the variational family. The ELBO is 
\begin{align}
   \c{M}(w_q, \Sigma_q, A, B, \phi) = 
   \frac{1}{2}\E_{w \sim \N(w_q, \Sigma_q)}\bigl[  & - n\log(2 \pi) -\logdet{B^{-1}} - \|Y - \Phi w \|^2_{B}   \notag \\ 
  &  -\logdet{A^{-1}} - \|w \|^2_{A}  +  \logdet \Sigma_q + d\, \bigr],
\end{align}
where we have substituted $Q$ for its variational parameters, which uniquely define the distribution, in the ELBO's arguments.
Evaluating the expectation we obtain
\begin{align}\label{eq:expanded_linear_ELBO}
    \c{M}(w_q, \Sigma_q, A, B, \phi) =  \frac{1}{2}\bigr( &- n\log(2 \pi)  -\logdet{B^{-1}} - \logdet{A^{-1}} - \|w_q\|_A^2 - \tr(\Sigma_q A) \notag \\
    &   - \|Y - \Phi w_q \|^2_{B} - \tr(\Phi \Sigma_q \Phi^T B) + \logdet \Sigma_q + d \bigl).
\end{align}
This expression will \corr{be} of particular interest in \cref{chap:adapting_laplace} and \cref{chap:sampled_Laplace}, where we will use the Laplace approximation to the posterior, a multivariate Gaussian, as the variational distribution for large scale models.
The variational posterior distribution over functions is computed analogously to \cref{eq:posterior_linear_function_draws} by substituting the Bayesian weight posterior with its approximation 
\begin{gather}
   f_q(\cdot) = \phi(\cdot) w \spaced{with} w \sim \N(w_q, \Sigma_q).
\end{gather}

\subsection{VI in function space: inducing points}\label{sec:inducing_point_VI}

We now look at the dual form of variational inference for linear models where the approximate distribution is specified directly over function outputs. To this end, we introduce an array of $m$ inducing points $Z = (z_1, z_2, \dotsc, z_m)$ with $z_i \in X$. The variational inducing point framework of \cite{titsias09,titsias2009report} substitutes our observed targets $Y$ with the inducing targets $U \in \R^{cm}$, each of which is associated with an inducing point.
We start by constructing a Gaussian process conditioned on the set of inducing locations and targets 
\begin{gather}\label{eq:conditional_variational_GP}
     (f^{(Z)}|U) \sim \GP( \mu^{(Z)}_{f|U}, k^{(Z)}_{f | U}),
\end{gather}
where the superscript notation $^{(Z)}$ makes explicit that the input locations correspond to $Z$ and not $X$. The mean and covariance functions are given by
\begin{gather}
    \mu^{(Z)}_{f|U}(\cdot) = {K}_{(\cdot) Z} {K}_{Z Z}^{-1}U
    \quad
    k^{(Z)}_{f | U}(\cdot, \cdot') = {K}_{(\cdot, \cdot')} - {K}_{(\cdot) Z} {K}_{Z Z}^{-1} {K}_{Z (\cdot')},
\end{gather}
where $[K_{ZZ}]_{ij} = k(z_i, z_j), \, i,j \leq m$ and we again use ${K}_{Z (\cdot)}$ for the stacked evaluation functionals $k(z_i, \cdot), \, i \leq m$.
These expressions match \cref{eqn:posterior_GP_moments}, with the observed inputs $X$ and targets $Y$ replaced by the inducing inputs $Z$ and inducing targets $U$. 

We now place a multivariate Gaussian variational distribution over the inducing targets $Q = \N(u^{(Z)}_q, {K}^{(Z)}_q)$, with $u^{(Z)}_q \in \R^{cm}$ and ${K}^{(Z)}_q \in \R^{cm \times cm}$. 
Following \cite{titsias09}, we choose the mean and covariance of this distribution that minimises the KL divergence between the variational Gaussian process $\E_{U\sim Q}[f^{(Z)}|U]$ and the posterior Gaussian process $f|Y$ \citep{matthews2016sparse}\footnote{In practise, this KL divergence between stochastic processes can be minimised by minimising the KL divergences between the multivariate Gaussians given by evaluating the variational GP and posterior GP at the set of observed and inducing inputs $\{X, Z\}$ jointly. }.
These are
    \begin{align}
        u^{(Z)}_q & = {K}_{ZZ} ({K}_{ZZ} + {K}_{ZX} B {K}_{XZ})^{-1} {K}_{ZX} B  Y
        \\
        {K}^{(Z)}_q     & = {K}_{ZZ} ({K}_{ZZ} + {K}_{ZX} B {K}_{XZ})^{-1} {K}_{ZZ},
    \end{align}
where $[K_{XZ}]_{ij} = k(x_i, z_j),\, i<n\, j<m$. Using this, we marginalise out the inducing targets in \cref{eq:conditional_variational_GP}, arriving at the optimal variational Gaussian process
\begin{gather}\label{eq:variational_post_GP}
     (f^{(Z)}|Y) \sim \GP( \mu^{(Z)}_{f|Y}, k^{(Z)}_{f | Y}),
\end{gather}
with mean and covariance functions
    \begin{align}
        \label{eq:titsias_predictive_mean_fn}
        \mu_{f| Y}^{(Z)}(\cdot)      & = {K}_{(\cdot) Z} ({K}_{ZZ} + {K}_{ZX} B {K}_{XZ})^{-1} {K}_{ZX} B  Y
        \\
        k_{f| Y}^{(Z)}(\cdot,\cdot') & = {K}_{(\cdot, \cdot')} + {K}_{(\cdot) Z} (({K}_{ZZ} + {K}_{ZX} B {K}_{XZ})^{-1} - {K}_{Z Z}^{-1} ){K}_{Z (\cdot')}\label{eq:titsias_predictive_cov_kernel}
        .
    \end{align}
These expressions contain linear solves against ${K}_{ZZ}$ instead of $K$. The number of inducing points is typically chosen to be smaller than the number of observations $m<n$ and thus the cost is lowered from $\c{O}\left( (nc)^3 \right)$ to $\c{O}\left( (mc)^3 \right)$. 

\subsubsection{Connecting inducing points to the Nyström approximation}
The expressions \cref{eq:titsias_predictive_mean_fn} and \cref{eq:titsias_predictive_cov_kernel} match those that we obtain if we substitute our Gaussian process prior with $\GP(0, K_{(\cdot), Z} K_{ZZ}^{-1}K_{Z, (\cdot')})$, and proceed with exact GP inference, as in \cref{subsec:bayesian_reasoning_about_functions}.
With this, every instance of $K$ is replaced with $K_{XZ} K_{ZZ}^{-1} K_{ZX}$, revealing that the variational Gaussian process amounts to a Nyström approximation of the kernel matrix \citep{wild2021connections}.

\begin{derivation} \textbf{Nyström pathwise representation of the optimal variational GP} \\
To show the connection between the Nyström approximation and variational inducing point GPs, we leverage the pathwise formulation of the GP random function \cref{eqn:pathwise_samples} but replace every instance of with $K_{XZ} K_{ZZ}^{-1} K_{ZX}$, yielding
  \begin{gather}
    \begin{gathered}
      (f^{(Z)}| Y)(\cdot) = f(\cdot) +  {K}_{(\cdot)Z} {K}_{ZZ}^{-1} {K}_{ZX } ({K}_{X Z} {K}_{ZZ}^{-1} {K}_{ZX } + B^{-1})^{-1}(Y - f^{(Z)}(X ) - \eps)
      \\
      \eps\sim\N(0,B^{-1})
      \qquad
      f\sim\GP(0, k)
      \qquad
      f^{(Z)}(\cdot) = {K}_{(\cdot)Z}{K}_{ZZ}^{-1}f(Z).
    \end{gathered}
  \end{gather}
  We now check the correctness of this expression by calculating the moments of this Gaussian process' marginal distributions and show them to match those of the KL-optimal variational Gaussian process given in \cref{eq:titsias_predictive_mean_fn} and \cref{eq:titsias_predictive_cov_kernel}.
  Write
  \begin{align}
    \label{eq:app_mean_equivalency}
    \E[(f^{(Z)}| Y)(\cdot)] & = {K}_{(\cdot)Z} {K}_{ZZ}^{-1} {K}_{ZX } ({K}_{X Z} {K}_{ZZ}^{-1} {K}_{ZX } + B^{-1})^{-1}Y
    \\
                                    & = {K}_{(\cdot)Z} {K}_{ZZ}^{-1} {K}_{ZX } B ({K}_{X Z} {K}_{ZZ}^{-1} {K}_{ZX }B + {I})^{-1}Y
    \\
                                    & = {K}_{(\cdot)Z} ({K}_{X Z} B {K}_{X Z} + {K}_{ZZ}  )^{-1} {K}_{ZX } BY
    \\
                                    & = \mu^{(Z)}_{f|Y}(\cdot)
  \end{align}
  and
  \begin{align}
     & \!\!\!\!\cov((f^{(Z)}| Y)(\cdot) - \mu^{(Z)}_{f|Y}(\cdot))
    \\
     & = \E((f^{(Z)}| Y)(\cdot) - \mu^{(Z)}_{f|Y}(\cdot), (f^{(Z)}| Y)(\cdot') - \mu^{(Z)}_{f|Y}(\cdot'))
    \\
     & = {K}_{(\cdot, \cdot')} - {K}_{(\cdot)Z} {K}_{ZZ}^{-1} {K}_{ZX } ({K}_{X Z} {K}_{ZZ}^{-1} {K}_{ZX } + B^{-1})^{-1}{K}_{X  Z} {K}_{ZZ}^{-1} {K}_{Z(\cdot')}
    \\
     & = {K}_{(\cdot, \cdot')} + {K}_{(\cdot)Z} {K}_{ZZ}^{-1} \left( - {I} + {I} -  {K}_{ZX } B ({K}_{X Z} {K}_{ZZ}^{-1} {K}_{ZX } B + {I})^{-1}{K}_{X  Z} {K}_{ZZ}^{-1} \right){K}_{Z(\cdot')}
    \\
     & = {K}_{(\cdot, \cdot')} + {K}_{(\cdot)Z} {K}_{ZZ}^{-1} \left( - {I} + ({K}_{ZX } B{K}_{X  Z} {K}_{ZZ}^{-1} + {I})^{-1} \right) {K}_{Z(\cdot')}
    \\
     & = {K}_{(\cdot, \cdot')} + {K}_{(\cdot)Z} \left(- {K}_{ZZ}^{-1} + ({K}_{ZX } B{K}_{X  Z}  + {K}_{ZZ})^{-1}  \right) {K}_{Z(\cdot')}
    \\
     & = k_{f|Y}^{(Z)}(\cdot,\cdot')
  \end{align}
  which recovers \cref{eq:titsias_predictive_mean_fn} and \cref{eq:titsias_predictive_cov_kernel}, as claimed.
\end{derivation}

\begin{figure}[t]
\centering
\includegraphics[width=0.85\linewidth]{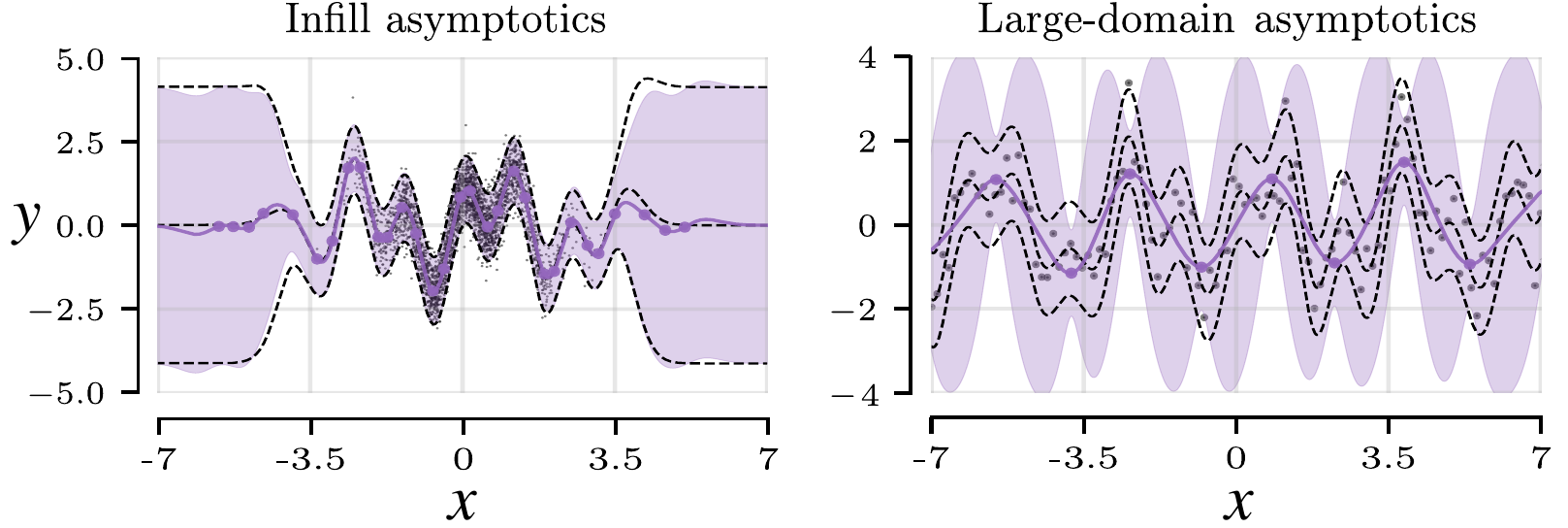}
\caption{Illustration of variational inducing point GP inference with a squared exponential kernel on $10\text{k}$ datapoints from $\sin(2x)+\cos(5x)$ with observation noise distribution $\N(0, 0.5)$.  The inducing point locations are marked with purple dots. All variational parameters are fit with SVGP \cref{eq:hensman_bound}. The true GP posterior is marked with a think black dashed line. Contours denote 2 standard deviation credible intervals for the predictive posterior. 
\emph{Infill asymptotics} considers  $x\sim\N(0,1)$. A large number of points near zero result in a very ill-conditioned kernel matrix. VI can summarise the data with only 20 inducing points. 
\emph{Large domain asymptotics} considers data on a regular grid with fixed spacing. Note that most of the data is not visible in the plot.
This problem is better conditioned. However, 1024 inducing points are not enough to summarise the data, leading to poor performance.
}
\label{fig:VI_GP_c3}
\end{figure}

This relationship allows us to gain intuition about the properties of inducing point approximations. These will work well when the conditioning number of $K$ is large. Intuitively, if multiple observed inputs are similar they can be modelled with a single inducing point and analogously if multiple rows of $K$ nearly linearly dependent, their action can be captured by a single row of $K_{ZZ}$.  On the other hand, a dataset where different inputs map to distant points in the RKHS will be poorly approximated by $m<n$ inducing points. We illustrate these properties in \cref{fig:VI_GP_c3}.

\subsubsection{Hyperparameter learning with inducing points}

\cite{titsias09} uses the optimal variational GP, given in \cref{eq:variational_post_GP}, to construct the ELBO
\begin{align}\label{eq:titsias_ELBO}
    \c{M}(Z, \theta) = &- \frac{n}{2}\log(2 \pi) -\frac{1}{2}\logdet{\left(B^{-1} + K_{XZ} K_{ZZ}^{-1} K_{ZX}\right)} \notag \\
    &- \frac{1}{2} \|Y  \|^2_{(B^{-1} + K_{XZ} K_{ZZ}^{-1} K_{ZX})^{-1}}  - \frac{1}{2} \tr{\left(B (K - K_{XZ} K_{ZZ}^{-1} K_{ZX})\right)},
\end{align}
where the inducing target values are marginalised, leaving the inducing point locations $Z$ as the only variational parameters to be optimised together with the hyperparameters $\theta$.

Two limitations of the bound in \cref{eq:titsias_ELBO} are that it's data-fit term can not be decomposed into a sum of each observation's contributions, precluding minibatch estimators, and that it is only valid for Gaussian likelihoods.
\cite{Hensman2013big} addresses both of these issues by introducing the ELBO
\begin{align}\label{eq:hensman_bound}
    \c{M}( u^{(Z)}_q, K^{(Z)}_q, Z, \theta) = &\E_{\substack{ (f^{(Z)}|U) \sim \GP( \mu^{(Z)}_{f|U}, k^{(Z)}_{f | U}) \\ U \sim \N(u^{(Z)}_q, K^{(Z)}_q) }} \log p\left(Y| (f^{(Z)}|U)\right)  \notag \\
     &+ \KL{\N(u^{(Z)}_q, K^{(Z)}_q)}{\N(0, K_{ZZ})},
\end{align}
where $p\left(Y| (f^{(Z)}|U) \right)$ is the conditional density of the targets given the variational GP. \Cref{eq:hensman_bound} can be shown to be a lower bound on \cref{eq:titsias_ELBO}. Here, the mean and covariance of the variational distribution over $U$ are left as variational parameters to be optimised. However, learning a quadratic number of parameters for the covariance can lead to numerical instability. GPs fit with \cref{eq:hensman_bound} are often referred to as Stochastic Variational Gaussian Processes (SVGP). \nomenclature[z-svgp]{$SVGP$}{Stochastic Variational Gaussian Process}

\subsection{Expectation propagation and non-KL divergences}

So far, we have discussed algorithms that choose the variational posterior such that its KL divergence to the Bayesian posterior is minimised. However, there is a rich literature that studies the minimisation of other divergences. We only review these works briefly, as they play no role in the later chapters of this thesis. 

The power expectation propagation (EP) \nomenclature[z-ep]{$EP$}{expectation propagation} algorithm  \citep{minka2004power,minka2007EPschemes} targets the alpha-divergence between a variational posterior, built as a series of site approximations, one for each observation, and the true posterior.
Power EP is a generalisation of regular EP \citep{Minka2001EP,Opper2005EP} with the latter targeting reverse KL divergences at each site. In turn, EP can be understood as a generalisation of the belief propagation algorithm \citep{Pearl1982belief,Pearl1988belief}.
\cite{hernandez2015probabilistic} extended the EP framework to neural networks, developing an algorithm coined ``probabilistic backpropagation''.
Furthering this line of work, \cite{Lobato2016black} applied alpha divergences to black box variational inference problems, doing away with the EP framework.
\cite{li2018thesis} extends variational inference to target the family of Rényi divergences \citep{renyi1961measures}, which also generalise the KL-divergence.
EP can also be shown to target a dual of the variational lower bound \citep{li2018thesis}. This idea has been used to construct hybrid algorithms, which may present better properties for hyperparameter optimisation \citep{Adam2021Dual,Lisolin2023EPGPs}.

\subsection{Variational inference for neural networks and its limitations}

Neural networks present very high dimensional and strongly multimodal posterior distributions. This has made it difficult to develop variational inference methods for neural networks that effectively navigate the trade-off between scalability and accuracy of approximation.

The most common choice of variational distribution is a Gaussian that factorises across dimensions\footnote{Factorised approximations are also referred to as \emph{mean field} approximations.}.
This choice allows for simple implementation and is relatively computationally inexpensive. As a result, it has persisted from the first works on VI for neural networks
\citep{hinton1993keeping,Saul1998Meanfield} to more modern approaches \citep{graves2011practical,blundell2015weight}. 
However, it can be shown that modelling dependencies between posterior weights is necessary to obtain calibrated uncertainty estimates \citep{Foong20Approximate}. 

There have been efforts to leverage more flexible variational distributions. \cite{louizos2017multiplicative} use normalising flows as variational approximations.
\cite{dusenberry2020efficient} target multiple posterior modes with rank-1 Gaussian approximations. 
\cite{ober2021global} construct an inducing-point based variational distribution with autoregressive structure across layers. 
On the other hand, \cite{gal2016dropout} and \cite{antoran2020depth} obtain scalability to very large neural networks by using very 
crude variational distributions that consist of randomly zeroing subsets of network weights, and network layers, respectively.
Another family of approaches re-cast popular optimisation algorithms, like Adam, as variational inference \citep{khan2018fast,Osawa2019practical,Khan2021rule}.
Unfortunately, despite these efforts, variational methods often reach solutions that under-perform traditional maximum likelihood learning of NN parameters in terms of predictive accuracy \citep{ashukha2020pitfalls,Wenzel2020good} or underestimate predictive uncertainty \citep{foong2019between,Foong20Approximate}.

\section[Conjugate Gradients]{Approximating the posterior computation: \\ Conjugate Gradients}
\label{sec:CG}

As we saw in \cref{chap:linear_models}, the main impediment to posterior inference in the Gaussian linear model is having to solve large systems of linear equations
(see \cref{eq:posterior_linear_function_draws}
\cref{eqn:posterior_GP_moments}
\cref{eqn:pathwise_samples}). These present time complexity $\c{O}\left((nc)^3\right)$ and memory complexity $\c{O}\left((nc)^2\right)$ in kernelised form and the same complexity, but in the number of parameters $d$, when dealing with the weight-space form. 
The most widely used algorithm to solve linear systems, both in the context of GPs \citep{Gibbs_MacKay97a,WangPGT2019exactgp,Artemev2021conjugate}, and also more generally \citep{numerical07,Boyd2014convex}, is Conjugate Gradients (CG).

CG is an iterative algorithm. Given the system ${(K + B^{-1})}^{-1} Y$, 
CG performs a single matrix-vector product ${(K + B^{-1})} Y$, with cost $\c{O}\left((nc)^2\right)$, at each iteration. The algorithm recovers the exact solution after at most $nc$ steps, asymptotically recovering the cubic cost. However, the algorithm  often converges much faster, delivering very accurate approximations of the linear system solution after only a few iterations.  
The speed of convergence depends on system conditioning, which we discuss in detail in \cref{sec:CG_limiations}.

\subsection{Hyperparameter learning with CG}\label{sec:CG_hyperparam_learning}

Optimising linear model hyperparameters with the marginal likelihood requires both solving linear systems against the loss Hessian matrix and also computing its log-determinant. Although, we can not compute the log-determinant with CG, we can can compute its gradient as $\partial_{\theta}\logdet{(K + B^{-1})} = \tr{\left((K + B^{-1})^{-1} \partial_{\theta} (K + B^{-1})\right)}$. We now apply \cite{Hutchinson90trace}'s trick to substitute the trace with an expectation, obtaining
\begin{align}\label{eq:GP_evidence_gradient}
    \partial_{\theta}\log p(Y; \,\theta) =  &\frac{1}{2} \E_{z \sim \N(0, I_{nc})} z^T (K + B^{-1})^{-1} \partial_{\theta} (K + B^{-1}) z \notag \\
    &+ \frac{1}{2} Y^T(K + B^{-1})^{-1}  \partial_{\theta} \left(K + B^{-1} \right) (K + B^{-1})^{-1} Y.
\end{align}
The above expression is approximated by constructing a MC estimator of the expectation. Evaluating each MC sample requires a linear solve against $K + B^{-1}$.
It is straight forward to apply the same trick to the primal form of the model evidence. 
This approach, which was first used by \cite{Gibbs_MacKay97a}, has become the most popular approximation for hyperparameter learning with large-scale GPs \citep{gardner18GPYtorch}. CG for linear model hyperparameter learning can been paired with preconditioning and low precision computation \citep{maddison2016concrete} or variational lower bounds  \citep{Artemev2021conjugate} to reduce time-to-convergence.

\begin{figure}[t]
\centering
\includegraphics[width=0.85\linewidth]{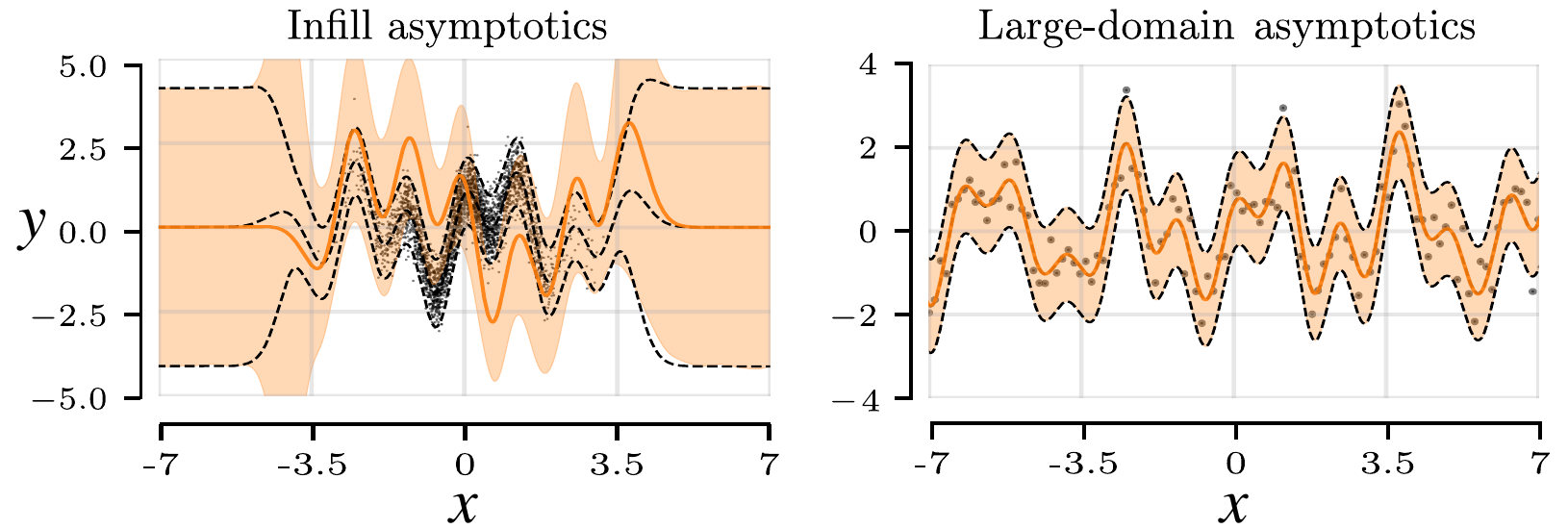}
\caption{Illustration of variational inducing point GP inference with a squared exponential kernel on $10\text{k}$ datapoints from $\sin(2x)+\cos(5x)$ with observation noise distribution $\N(0, 0.5)$.  The inducing point locations are marked with purple dots. All variational parameters are fit with SVGP \cref{eq:hensman_bound}. The true GP posterior is marked with a think black dashed line. Contours denote 2 standard deviation credible intervals for the predictive posterior. 
\emph{Infill asymptotics} considers  $x\sim\N(0,1)$. A large number of points near zero result in a very ill-conditioned kernel matrix, preventing CG from converging (we draw 2000 posterior by CG for 10 minutes on an RTX 2070 GPU.
).
\emph{Large domain asymptotics} considers data on a regular grid with fixed spacing. Note that most of the data is not visible in the plot.
This problem is better conditioned, allowing CG to recover the exact solution.
}
\label{fig:CG_approx_GP}
\end{figure}

\subsection{Limitations of Conjugate Gradient inference}\label{sec:CG_limiations}

The chief limitation of CG is that its convergence speed decreases as the matrix we are solving against becomes more ill-conditioned. 
Given the system ${(K + B^{-1})}^{-1} Y$, the number of matrix-vector products needed to guarantee convergence of CG to within a tolerance of $\eps$ is
\begin{gather}\label{eq:cg_convergence}
\c{O}\del{\sqrt{\cond(K + B^{-1})}\log \frac{\cond(K + B^{-1})\|Y\|}{\eps}}
               \\
\text{with} \quad  \cond(K + B^{-1})  = \frac{\lambda_{\max}(K + B^{-1})}{\lambda_{\min}(K + B^{-1})}, \notag
\end{gather}
where $\lambda_{\max}(K + B^{-1})$ and $\lambda_{\min}(K + B^{-1})$ are the maximum and minimum eigenvalues of $K + B^{-1}$. See \citep{terenin23} for further discussion on \cref{eq:cg_convergence}.
Although CG performs well in many GP use cases, for instance \citep{gardner18GPYtorch,WangPGT2019exactgp}, 
the condition number $\cond({K} {+} B^{-1})$ need not be bounded, and conjugate gradients may fail to converge quickly \citep{terenin23}. We illustrate this in \cref{fig:CG_approx_GP}.
Nonetheless, by exploiting the quadratic structure of the objective, substantially better worst-case convergence rates can be shown for CG than alternatives, like gradient descent \citep{BlanchardK2010optimalkernel,ZouWBGK2021benign}. This makes the results of \cref{chap:SGD_GPs}, where we show that SGD can be used to approximate GP posteriors notably faster than alternative methods, surprising.

\section[The linearised Laplace approximation]{Approximating the function class: \\ the linearised Laplace approximation} \label{sec:Laplace}

The Laplace approximation is a classical technique in Bayesian statistics for constructing Gaussian surrogates for analytically intractable posterior distributions. 
The Laplace approximation was first applied to neural networks by \cite{mackay1992practical}. We will also focus on the neural network setting here, as it is of primary interest for the rest of this thesis, forming the basis of \cref{chap:adapting_laplace}, \cref{chap:sampled_Laplace} and \cref{chap:DIP}. In doing this, we will also see how the Laplace approximation can be applied to non-conjugate linear models which can be seen as a particular case of neural networks.

Let the the function $g : \R^d \times \c{X} \to \R^c$ be a neural network and $v \in \c{V} \subseteq \R^d$ refer to its parameters, flattened into a single vector. We train it to solve a $c$ output prediction problem, by minimising a loss of the form
\begin{gather}\label{eq:NN_loss}
    \c{L}_{g}(v) = \sum_{i=1}^n \ell(y_i, g(v, x_i)) + \c{R}(v),
\end{gather}
where the subscript in $\c{L}_{g}$ makes explicit that our model is the NN $g$, $\ell$ is a data fit term (a negative log-likelihood) which we assume to include any linking functions, and $\c{R}$ is a regulariser. We do not assume either to be quadratic. This procedure returns the weights $v_\star \in \argmin_{v \in \R^{d}} \c{L}_g(v)$ \footnote{We use $\in$ since NN loss functions are almost always multimodal and thus there exist a set of multiple minimisers.}. 

\paragraph{Notation for gradients and Hessians}
We use m $\partial^m_v[g(v, x)](v ')$ to denote the mth order mixed partial derivatives of $g$ with respect to $v$ evaluated at $(v', x)$. We use $\partial^m_x f(x')$ to refer to $\partial^m_x[f(x)](x')$ for single argument functions, where no ambiguity exists.

With that, the Laplace method constructs a locally quadratic approximation to $\c{L}_g$ around the mode 
\begin{gather}\label{eq:Laplace_method_no_lin}
    \c{L}_g(v) = \c{L}_g(v_\star) + \frac{1}{2} \|v - v_\star\|^2_{\partial^2_v  \c{L}_g(v_\star)}  + \c{O}(v^3),
\end{gather}
where the first order term cancels since  $\partial_v \c{L}_g(v_\star) = 0$ and $\partial^2_v  \c{L}_g(v_\star) \in \R^{d \times d}$ is the Hessian of the loss at $v_\star$. We use this quadratic approximation to the loss to define the negative log density of an approximate posterior, which by inspection corresponds to the Gaussian 
\begin{gather}\label{eq:traditional_laplace_posterior}
    \N \left(v_\star, \, (\partial^2_v  \c{L}_g(v_\star))^{-1}\right).
\end{gather}

\begin{remark} \textbf{The asymptotic exactness of the Laplace approximation} \\
The Bernstein–von Mises theorem tells us that for any likelihood function $\ell$ and under relatively weak conditions, the posterior distribution converges to a Gaussian centred at the maximum likelihood parameter setting as the number of observations goes to infinity, i.e. $n \to \infty$ \citep{Bernstein_1946,Walker1969bigdatalikelihoodconvergence}.
This result yields credence to the Laplace approximation in big-data settings. Indeed, \cite{mackay1992practical} reports increased approximation accuracy for larger number of observations.
\end{remark}

Despite the Laplace approximation being the first method developed for Bayesian reasoning with NNs \citep{mackay1992practical}, modern adaptions of the method, some of which are introduced in \cref{chap:adapting_laplace} and \cref{chap:sampled_Laplace}, represent the state-of-the-art in the field of Bayesian deep learning \citep{daxberger2021laplace,antoran2023samplingbased}. \corr{The method has also seen success when applied to non-conjugate linear models, where the likelihood is non-Gaussian \citep{Rue2009Inla}.}
We go on to discuss the use of the Laplace approximation, in its linearised variant, for predictive variance estimation and for model evidence approximation in neural networks.  

\subsection{Linearising our network at prediction time}

Despite the closed form of the Laplace posterior over NN parameters, integrating out the parameters to evaluate the posterior distribution over functions $g(v, \cdot), \, v \sim \N(v_\star, \, (\partial^2_v  \c{L}_g(v_\star))^{-1})$ remains analytically intractable.
\cite{mackay1992practical} resolves this by introducing an additional approximation: a local linearisation of the neural network function around $v_\star$. We also do this, introducing the affine model $h : \R^d \times \c{X} \to \R^c$, which performs the map
\begin{gather}\label{eq:linearised_model}
    h(w, x) \coloneqq g(v_\star, x) + J(x) (w - v_\star)
\end{gather}
where $J(x_i) \coloneqq \partial_v [f(v, x_i)](v_\star) \in \R^{c \times d}$ is the Jacobian of the NN function evaluated at $x$ with respect to its weights, and we denote the approximate model's parameters as $w\in \R^d$ to highlight their linear relationship with the output. With this, the marginals of the posterior distribution over functions $h(w, \cdot), \, w \sim \N(v_\star, \, (\partial^2_v  \c{L}_g(v_\star))^{-1})$ become closed form and Gaussian 
\begin{gather}\label{eq:linearised_predictive_1}
    \N(g(v_\star, x'), \, J(x') (\partial^2_v \c{L}_g(v_\star))^{-1} J(x')^T).
\end{gather}
Here, we have used that the expectation of an affine transform of a Gaussian random variable is the affine transformation of the mean, and since the mean is $v_\star$, the first order term in \cref{eq:linearised_model} cancels, leaving only $g(v_\star, x')$. 

\cite{mackay1992practical} makes one final approximation.
He substitutes the Hessian of the data-fit loss $ \partial^2_{v} [\ell(y, g(v, x))](v_\star)$ for the \nomenclature[z-ggn]{$GGN$}{Generalised Gauss Newton Matrix}  Generalised Gauss Newton matrix (GGN) $J(x)^T \partial^2_{\hat{y}} \ell(y, \hat{y}) J(x)^T$ evaluated at the MAP predictions $\hat{y_i} = g(v_\star, x_i)$.
With this, the precision of the Laplace posterior becomes 
\begin{gather}\label{eq:GGN}
    \sum_{i=1}^{n} \underbrace{J(x_i)^T \partial^2_{\hat{y}_i} \ell(y_i, \hat{y_i}) J(x_i)^T}_{\text{GGN}} + \underbrace{\partial^2_{v} \c{R}(v_\star)}_{\text{Hessian of regulariser}}
\end{gather}
where $\partial^2_{v} \c{R}(v_\star) \in \R^{d\times d}$ is the Hessian of the regulariser, and $\partial^2_{y_i} \ell(y_i, \hat{y_i}) \in \R^{c\times}c$ is the GGN corresponding to the contribution of each observation to the likelihood.

\begin{remark} \textbf{Comparing the Hessian and the GGN}\\
Using the chain rule of the product, we can decompose the Hessian into the GGN, which captures the curvature of the likelihood function but linearises the NN function, and a second term consisting of the gradient of the log-likelihood multiplied with the Hessian of the NN function
\begin{gather}
  \partial^2_{v} [\ell(y, g(v, x))](v_\star)  =  \underbrace{J(x)^T \partial^2_{\hat{y}} \ell(y, \hat{y}) J(x)^T}_{\text{GGN}} + \partial_{\hat{y}} [\ell(y, \hat{y})](g(v_\star, x)) \partial^2_v [g(v, x)] (v_\star).
\end{gather}
From this, we can see that the GGN will be a good approximation to the Hessian when the gradient of the data-fit loss $\partial_{\hat{y}} [\ell(y, \hat{y})](g(v_\star, x))$ is small. It will be exact when we are at an optima of the fit term. For instance, when the NN parametrises the mean of a Gaussian likelihood and the NN output perfectly interpolates the training targets.

Unlike the exact Hessian, the GGN is guaranteed to be PSD. This makes it often preferred in the second order optimisation literature  \citep{Becker1989ImprovingGGN,Schraudolph2002fastcurvature,martens2014new}, since negative curvature results in linear system solutions lying at infinity, causing optimisers to diverge. 
Furthermore, the GGN is cheaper to compute than the full Hessian and better lends itself to efficient block-wise approximations. Examples of the latter are the iLQR algorithm \citep{BEMPORAD202ILQR} and the Kronecker factored approximation \citep{Martens15kron}.
\end{remark}

\begin{figure}[t]
\centering
\includegraphics[width=\linewidth]{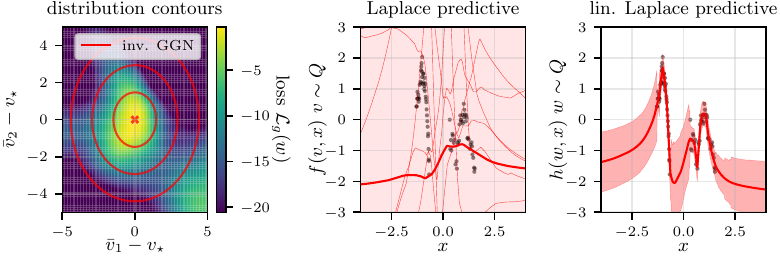}
\caption{Left: 2d projection of a neural network loss landscape around a mode $v_\star$. We also display the 1, 2 and 3 standard deviation contours of the linearised Laplace (i.e. using the GGN approximation to the Hessian) posterior computed at the mode. Middle: we push Laplace posterior through the NN function and display mean and 2 standard deviation credible regions of the posterior predictive distribution. These do not fit the data. We also display the functions corresponding to 4 posterior samples. Right: the linearised Laplace predictive distribution fits the data well and provides sensible errorbars (2 standard deviation credible regions of the posterior predictive distribution). 
}
\label{fig:laplace_posterior_is_bad}
\end{figure}

\cite{Lawrence_thesis} found that the Laplace approximation, without the linearisation step, resulted in very poor quality predictive distributions that did not even assign high density to the train targets. \cite{ritter2018scalable} make a similar observation, but ameliorate the issue by introducing additional hyperparameters that decrease the variance of the posterior over the weights. We reproduce this result in \cref{fig:laplace_posterior_is_bad}. We also show how the true NN posterior can present strongly non-Gaussian features near a mode, leading the Laplace approximation to place some of its mass in very low density regions of the true posterior. It is this that causes poor predictions. However, local linearisation resolves the issue. 
This incongruence was resolved recently, roughly 30 years after the publication of \cite{Mackay1992Thesis}, by the modern formulation of the linearised Laplace approximation \citep{Khan2019approximate,Immer2021Improving,antoran2022adapting}, which we describe in the next section.

\subsection{A modern view of linearised Laplace} \label{subsec:modern_view_laplace}

We now present a modern re-interpretation of the linearised Laplace methodology described in the previous section. The key observation, made by \cite{Khan2019approximate}, is that the GGN-Laplace posterior matches the true posterior of the tangent linear model $h$. Using a similar reasoning, \cite{Immer2021Improving} argue that the GGN-Laplace posterior should be paired with the tangent linear model at prediction time. The authors show that this results in more accurate posterior predictive distributions, which we reproduce in \cref{fig:laplace_posterior_is_bad}.
Building on this, \cite{antoran2022adapting} and \cite{antoran2023samplingbased} present a linearisation-first derivation of linearised Laplace, which we go on to present here.

The linearised Laplace method consists of two consecutive approximations, the latter of which is necessary only if $\ell$ or $\c{R}$ are non-quadratic. That is, if the likelihood or prior are non-Gaussian.
\begin{enumerate}[topsep=0pt]
    \item We take a first-order Taylor expansion of $g$ around $v_\star$, yielding the surrogate model given in \cref{eq:linearised_model}.
    This model's weights linearly combine the rows of the Jacobian matrix, which can be seen as a feature expansion of the input. To make this connection explicit, we henceforth adopt the notation $\phi(x) = J(x)$. 
    The linear model's loss is 
    \begin{gather}\label{eq:linear_model_loss}
        \c{L}_h(w) = \sum_{i=1}^n \ell(y_i, h(w, x_i)) + \c{R}(w).
    \end{gather}
    If this expression is quadratic, we may proceed with conjugate linear-Gaussian inference as described in \cref{chap:linear_models}. The Laplace approximation is not needed. 
  \item If the linear model's loss is non quadratic, we locally approximate it with the Laplace method. This yields a Gaussian posterior of the form
  \begin{equation}\label{eq:lin_laplace_posterior}
    \N(v_\star, \, (\partial^2_v \c{L}_h (v_\star))^{-1}).
\end{equation}
Since the NN and tangent linear model share gradients, that is $\partial_v \c{L}_g(v) = \partial_h \c{L}_h(w)$, if $v_\star$ is a local optima of $\c{L}_g$ it will also be one of $\c{L}_h$.
Direct calculation shows that $\partial^2_v \c{L}_h (v_\star) = A + \Phi^T B \Phi = H$, for $\nabla^2_w\c{R}(v_\star)=A$ and $B$ a block diagonal matrix with blocks $B_i = \nabla^2_{\hat y_i} \ell(y_i, \hat y_i)$ evaluated at $\hat y_i = h(v_\star, x_i) = g(v_\star, x_i)$. We have once again used notation matching the one used for conjugate Gaussian-linear models in \cref{chap:linear_models} to highlight the Laplace approximation's Gaussianisation of the likelihood and prior.
\end{enumerate}
Linearised Laplace has returned us a conjugate Gaussian multi-output linear model with the GGN as its posterior precision.  

\begin{remark} \textbf{Linearisation as a modelling choice} \\
We could go a step further and view the linearisation step as a modelling choice. That is, we would be adopting a Gaussian linear model with basis functions matching the NN's Jacobian around $v_\star$. This would not be a data independent basis, however. The NN has been fit to our dataset in order to find $v_\star$. We overlay the Jacobian basis functions on top of a toy 1d dataset on which the corresponding NN was trained in \cref{fig:linear_NN_prior}. The basis functions present sharp changes in the input regions where there is data, and are more smooth elsewhere. The same is true for the equivalent kernel. Given this data dependence, it is somewhat surprising that linearised Laplace does not result in uncertainty underestimation given that we are using our data twice: one to train our NN and another for inference in the tangent linear model. The most common case of double use of data resulting in overfitting is hyperparameter learning with the model evidence. This overfitting happens, for instance, in the deep kernel learning model \citep{ober2019benchmarking,ober2021promises}, which uses the- linear model evidence to fit basis functions parametrised by neural networks. My intuition is that linearised Laplace escapes overfitting because the NN's weights are not trained with the linearised model's evidence, but with some bespoke NN loss function. As a result, the Jacobian basis functions do not perfectly pass through the training targets.
\end{remark}

\begin{figure}[h]
\centering
\includegraphics[width=\linewidth]{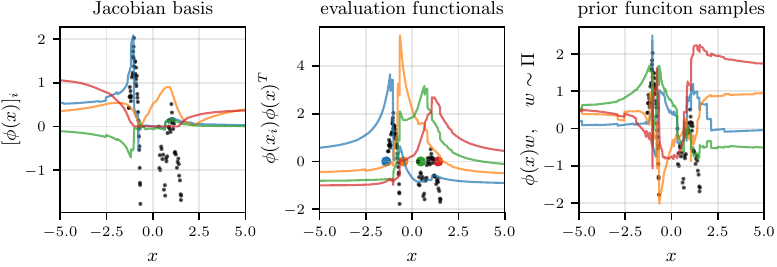}
\caption{Illustration of the prior implied by the linearised NN. The leftmost plot shows 4 dimensions of the Jacobian basis (i.e. the Jacobian with respect to 4 of the NN weights) function of a 2 layer residual MLP trained  on the 1d toy dataset introduced by \cite{antoran2020depth}. This dataset is displayed as black dots. The middle plot shows the kernel implied by the Jacobian basis. It is non-stationary. The rightmost plot shows 4 samples drawn from the linearised NN prior, with the NN loss mode's prediction $g(v_\star, \cdot)$ removed.
}
\label{fig:linear_NN_prior}
\end{figure}

The linearised posterior distribution over NN outputs at a new input $x' \in \c{X}$ is thus
\begin{gather}\label{eq:linearised_predictive_2}
    \N(g(v_\star, x'), \, \phi(x') H^{-1} \phi(x')^T),
\end{gather}
matching the expression used by \citep{MacKay1992interpolation} and given in \cref{eq:linearised_predictive_1}.
In other words, linearised Laplace simply augments our pre-trained NN's predictions with with Gaussian errorbars. Keeping the NN outputs as the mean presents a large advantage over alternative approaches to Bayesian inference in deep learning which often trade off goodness of fit with quality of uncertainty estimates \citep{snoek2019can, daxberger2021bayesian, daxberger2021laplace}. 
Additionally, linearised Laplace tends to provide sensibly shaped errorbars, contrasting with other approximations which fail simple tests like ``in-between'' uncertainty \citep{foong2019inbetween} or ``far-away'' uncertainty \citep{Kristiadi2020being}. 

\begin{remark} \textbf{Connections to the neural tangent kernel and infinitely wide NNs}\\
    The Neural Tangent Kernel (NTK) \citep{Jacot2018NTK,Lee2019wide}\nomenclature[z-ntk]{$NTK$}{Neural Tangent Kernel} is intimately related to linearised Laplace. The NTK matches the linearised model given in \cref{eq:linearised_model}, but with the Taylor expansion point being the point where the NN weights are initialised, instead of an optima of the loss. As the NN width increases, and under some relatively weak conditions which we will not discuss here, the mode of the NN loss goes to the initialisation point. In this setting, the linearised Laplace posterior matches the posterior of a GP with the NTK as its covariance kernel. This distribution is different, however, from the true posterior of the infinitely wide NN model \citep{matthews2018gaussian}. The latter also corresponds to a GP, but its kernel is not the NTK. It is the outer product of the Jacobians of the NN's last layer weights. Thus, the NTK is a sum of the infinitely wide NN kernel and also some other kernels with features matching the NN's non-last layer Jacobians.
\end{remark}

\subsection{Learning hyperparameters with the Laplace evidence}\label{subsec:original_laplace_hyperparam_learning}

An important limitation of linearised Laplace is its predictive variance's sensitivity to the curvature of the likelihood $B$ and regulariser $A$. The values of these matrices derived from the loss used to train the NN often result in miss-calibrated uncertainty, and large performance gains can be obtained by tuning them. To this end, the Laplace approximation provides us with and an estimate of the model evidence, which may be used to learn $A$ and $B$ as well as other hyperparameters. 

Again denoting our set of hyperparameters as $\theta$, we re-arrange Bayes rule \cref{eq:bayes_rule_c2} to expose the model evidence on the left hand side, substitute the posterior density for our Gaussian approximation \cref{eq:lin_laplace_posterior}, and evaluate the functions at $v_\star$, obtaining
\begin{align}\label{eq:laplace_evidence_non_gauss}
   \log  p(Y; \, \theta) &\approx \log p(Y|v_\star) + \log \pi(v_\star) - \N(v_\star; v_\star, H^{-1}) \notag \\
    &= \log p(Y|v_\star) + \log \pi(v_\star) - \frac{1}{2}\logdet H + \frac{d}{2} \log(2 \pi) \coloneqq \c{G}_{v_\star}(\theta),
\end{align}
where we make the dependence of the approximation on the linearisation point $v_\star$ explicit by introducing it as a subscript to $\c{G}$.
If we do not have access to the explicit joint density of parameters and observations $p(Y|v_\star)\pi(v_\star)$, we may use the loss function $\c{L}(v_\star)$ instead. However, we have to introduce the normalisation factors for the respective Gaussian approximations of the likelihood and prior
\begin{gather}\label{eq:laplace_evidence_normalised}
   \c{G}_{v_\star}(\theta) = -\c{L}_f(v_\star)  - \frac{1}{2}\logdet H + \frac{1}{2}\logdet A + \frac{1}{2}\logdet B - \frac{n}{2} \log(2 \pi).
\end{gather}

We may tune the hyperparameters $\theta$ for a NN by choosing them to maximise $\c{G}_{v_\star}(\theta)$.  This may improve our errorbar calibration\footnote{It most likely will not and fixing this is the object of \cref{chap:adapting_laplace}. But one could plausibly conclude that it might from reading the relevant literature.}, but will not change our NN's outputs however, as its parameters are held fixed at $v_\star$. 
\cite{Mackay1992Thesis} proposes to re-train the NN from scratch using the new hyperparameters. Steps of NN training and hyperparameter optimisation are iterated until a joint stationary point of the parameters and hyperparameters is found
\begin{gather}
    v_\star \in \argmin_{v \in \R^d} \c{L}_f(v, \theta_\star) \spaced{and} \theta_\star \in \argmin_{\theta \in \Theta} \c{G}_{v^\star}(\theta)
\end{gather}
where we have made the loss' dependence on the hyperparameters explicit by adding them as an argument.

\subsection{Online Laplace methods}\label{subsec:online_Laplace}

The size of the neural networks and datasets has grown dramatically since 1992. As a result, nowadays, re-training our NN multiple times after hyperparameter updates introduces a prohibitive computational cost. 
This motivates Online Laplace (OL) \nomenclature[z-ol]{$OL$}{Online Laplace} approaches which, at timestep $t$ with parameters $v_t$ and hyperparameters $\theta_t$, perform a step of NN parameter optimisation to minimise $\c{L}_f(v_t; \theta_t)$, obtaining $v_{t+1}$, followed by a hyperparameter update to maximise $\c{G}_{v_{t+1}}( \theta_t)$ \citep{Foresee97BayesNewton,Friston07variational,Immer21Selection}\footnote{\cite{Friston07variational} refers to the described online Laplace procedure as \emph{Variational Laplace}.}.
Critically, the Laplace approximation of the evidence is constructed with both the NN loss and GGN evaluated at the current NN parameter setting $v_{t+1}$. 
Since optimisation has not converged, $v_{t+1} \notin \argmin_{v\in \R^d} \c{L}_f(v; \theta)$. Thus $\c{G}_{v_{t+1}}(\theta_t)$, which discards the first-order Taylor expansion term, is unlikely to provide a local approximation to the true model evidence. Despite this, online Laplace methods have seen success recently, for instance for learning data augmentation hyperparameters \citep{immer2022invariance} and model invariance hyperparameters \citep{vanderouderaa2023learning}. 

In \cite{lin2023online}, a piece of work not covered in this thesis, we construct a Taylor expansion-based Gaussian approximation to the evidence that does not discard the first order term, and thus may be more suitable for online use. We then show that this approximation corresponds to the exact evidence of the tangent linear model. Interestingly, when we drop the first order term, we recover a variational lower bound on the evidence of the linear model, providing some justification for the online approaches of \cite{Foresee97BayesNewton,Friston07variational,Immer21Selection}.

\subsection{Limitations of the linearised Laplace approximation}

Linearised Laplace presents a number of critical limitations and addressing these is the object of much of the rest of this thesis.

Linearised Laplace shares the limitations of linear model inference discussed in \cref{chap:linear_models}: cubic compute cost and quadratic memory cost, in either the number of NN parameters $d$ or the number of outputs times observations $nc$. In modern deep learning problems, both of these quantities tend to be in the tens of millions, or larger. Fortunately, linearising the NN allows us to leverage the approximations for linear models discussed in \cref{sec:VI}, \cref{sec:CG}, and the ones we will introduce next, in \cref{chap:SGD_GPs}. Additionally, there are a number of approximations that exploit the structure of the linearised NN, such as last layer methods \citep{Kristiadi2020being,Eschenhagen2021mixtures}, Kronecker factorised approximations \citep{ritter2018scalable,immer2023stochasticNTK}, and subnetwork methods \citep{daxberger2021bayesian}.

Additionally, linearised Laplace presents some  unique limitations.
Firstly, its basis functions, the NN Jacobians, are computationally expensive to deal with. For a given in put $x \in \c{X}$, computing the Jacobian expansion $J(x)$ requires either $c$ passes of backward mode Automatic Differentiation (AD)\nomenclature[z-ad]{$AD$}{Automatic Differentiation} or $d$ passes of forward mode AD. Clearly the former is preferable for most neural networks, but it still presents an issue when dealing with high-dimensional output spaces, for instance, in many-way image classification, image-restoration, or language modelling. For a large enough model, even storing the $c \times d$ dimensional Jacobian features may be too expensive.
In a textbook implementation of linearised Laplace, Jacobians appear at two different points: 1) when constructing the GGN matrix at inference time, and 2) when making predictions for a new observation $x'$.  
The former problem can be partially ameliorated by leveraging the equivalency between the GGN and the Fisher information matrix for exponential family likelihoods---almost all of the ones used in machine learning. In particular, the Fisher admits unbiased stochastic estimation by sub-sampling output dimensions \citep{martens2014new,Kunstner2019limitations}. Unfortunately, the predictive covariance does not admit this sort of approximation. \cref{chap:sampled_Laplace} presents an implementation of linearised Laplace that completely avoids instantiating Jacobian matrices.

Finally, astute readers may have noticed that the assumption that we find a local optima of the NN training loss is unrealistic. NN optimisation landscapes present a large number of symmetries and invariances, and stochastic optimisation is almost always used to minimise them. Furthermore, early stopping is also almost always used. These techniques prevent us from finding a local minima of the loss. \corr{This is intentional, it can be thought of as regularisation, because reaching a very low loss value would almost surely mean we are overfitting.} One may thus wonder how not having access to a minimum of the loss affects the linearised Laplace approximation? This is addressed in \cref{chap:adapting_laplace}.

\chapter[Stochastic Gradient Descent for Gaussian Processes]{Sampling from Gaussian Process posteriors using Stochastic Gradient descent}\label{chap:SGD_GPs}  %

\ifpdf
  \graphicspath{{Chapters/Chapter4/Figs/Raster/}{Chapters/Chapter4/Figs/PDF/}{Chapters/Chapter4/Figs/}}
\else
  \graphicspath{{Chapters/Chapter4/Figs/Vector/}{Chapters/Chapter4/Figs/}}
\fi

\textit{``When solving a given problem, try to avoid solving a more general problem as an intermediate step.''} --- Vladimir Vapnik\\

Gaussian processes (GPs) provide a comprehensive framework for learning unknown functions in an uncertainty-aware manner.
This often makes GPs the model of choice for sequential decision-making, achieving state-of-the-art performance in tasks such as optimising molecules in computational chemistry settings \citep{Gomez-Bombarelli18} and automated hyperparameter tuning \citep{Snoek2012,Hernandez-Lobato14}.

As we have seen in previous chapters, the main limitation of Gaussian processes is that their computational cost is cubic in the training dataset size.
Significant research efforts have been directed at addressing this limitation, resulting in two key classes of scalable inference methods: (i) \emph{inducing point} methods \citep{titsias09,Hensman2013big}, which approximate the GP posterior, and (ii) \emph{conjugate gradient} methods \citep{Gibbs_MacKay97a,gardner18GPYtorch,Artemev2021conjugate}, which approximate the computation needed to obtain the GP posterior.
Note that in structured settings, such as geospatial learning in low dimensions, specialised techniques are available \citep{Wilson15,Wilkinson2019thesis}.
Throughout this chapter, we focus on the generic setting, where scalability limitations are as of yet unresolved.

In recent years, stochastic gradient descent (SGD) \nomenclature[z-sgd]{$SGD$}{Stochastic Gradient Descent} has emerged as the leading technique for training deep learning models at scale \citep{Ruder16SGD,Zhang2023sgd}. It has also been applied to kernel methods
\citep{dai2014scalable}, and even connected to variational Bayesian inference \citep{Mandt17Descent}.
While the principles behind the effectiveness of SGD are not yet fully understood, empirically, SGD often leads to good predictive performance---even when it does not fully converge.
The latter is the default regime in deep learning, and has motivated researchers to study \emph{implicit biases} and related properties of SGD \citep{Belkin2019,ZouWBGK2021benign}.

In the context of GPs, SGD is commonly used to learn kernel hyperparameters---by optimising the marginal likelihood \citep{gardner18GPYtorch,chen20,chen22} or closely related variational objectives \citep{titsias09,Hensman2013big}.
In this chapter, we explore applying SGD to the complementary problem of approximating GP posterior samples given fixed kernel hyperparameters.
In one of his seminal books on statistical learning theory, Vladimir \cite{vapnik95} famously said: \emph{"When solving a given problem, try to avoid solving a more general problem as an intermediate step."}
Motivated by this viewpoint, as well as the aforementioned property of good performance often not requiring full convergence when using SGD, we ask: \emph{Do the linear systems arising in GP computations necessarily need to be solved to a small error tolerance? If not, can SGD help accelerate these computations?}

We answer the latter question affirmatively, with specific contributions as follows.
(i) In \cref{sec:stochastic_sampling_estimators}, we develop a scheme for drawing GP posterior samples by applying SGD to a quadratic problem.
In particular, we re-cast the pathwise conditioning technique of \citep{wilson20} as an optimisation problem,
and, in \cref{sec:inducing_SGD}, extend the method to inducing point GPs.
In \cref{subsec:variance_reduction_SGD}, we develop a novel low-variance SGD sampling estimator applicable to both linear models, where the kernel is finite dimensional, and GPs.
\nomenclature[z-sdd]{$SDD$}{Stochastic Dual Descent}
For the kernelised setting, in \cref{subsec:SDD}, we introduce Stochastic Dual Descent (SDD), an optimisation scheme that targets a better conditioned dual objective in place of the more-common kernel ridge regression objective. 
(ii) In \cref{sec:implicit_bias}, we characterise the implicit bias in SGD-approximated GP posteriors showing that despite optimisation not fully converging, these match the true posterior in regions both near and far away from the data.
(iii) Finally, in \cref{sec:SGDGP_experiments}, we present the following experimental evidence:
\begin{enumerate}
    \item On standard UCI regression benchmarks with up to 2 million observations, stochastic dual descent either matches or improves upon the performance of conjugate gradients, while strictly outperforming other baselines.
    \item On large-scale parallel Bayesian optimisation, stochastic gradient descent is shown to be superior to preconditioned conjugate gradients and inducing point variational inference, both in terms of the number of iterations and in terms of wall-clock time. In turn, stochastic dual descent is shown to be superior to vanilla stochastic gradient descent.
    \item On a molecular binding affinity prediction task, where Gaussian processes have not previously been shown to be competitive with deep learning approaches, the performance of stochastic dual descent matches that of graph neural networks.
\end{enumerate}

The methods and insights developed in this chapter, in particular the low-variance SGD estimator of weight-space posterior samples, will play a key role in scaling the linearised Laplace method to large scale neural networks and datasets in \cref{chap:sampled_Laplace}.

\section{Pathwise conditioning as an optimisation problem}\label{sec:pathwise_as_optim}

Both a Gaussian process' posterior mean and posterior samples can be expressed as solutions to quadratic optimisation problems.
For the primal, weight-space form, the expressions for the mean and samples were provided in \cref{chap:linear_models}, in  \cref{eq:linear_model_loss} and \cref{eq:sample_then_optimise}, respectively. Here we study the more general kernelised form. To simplify notation, we assume the output dimension is $c=1$ throughout this chapter. As a result, our noise precision matrix $B$ is diagonal. Additionally, we assume our kernel $k$ is stationary, or at least admits random features.

The GP posterior mean minimises the ridge regression loss over functions in the RKHS:
\begin{gather}
    f_\star (\cdot) =  \argmin_{f\in\c{H}} \sum_{i=1}^n [B]_{ii} (y_i - \innerprod{k(x_i, \cdot)}{f}^2  + \|f\|_{\c{H}}^2.
\end{gather}
Using the representer theorem \citep{Scholkopf2001representer}, we transform this objective into a quadratic problem over the representer weights $\alpha \in \R^n$
\begin{align}
  \label{eqn:mean-optim}
  f_\star (\cdot) & = {K}_{(\cdot)X }\alpha_\star = \sum_{i=1}^n \alpha_{*\, i} k(x_i,\cdot)
                       &
  \alpha_\star              & = \argmin_{\alpha\in\R^n} \sum_{i=1}^n  [B]_{ii} (y_i - {K}_{x_iX }\alpha)^2 + \norm{\alpha}_{K}^2
  .
\end{align}
Its optima is $\alpha_\star = (K + B^{-1})^{-1}Y$, matching \cref{eq:c2_gp_mean_representer}.
Recall that we refer to $k(x_i,\cdot)$ as the \emph{evaluation functionals}, and we henceforth refer to $\norm{\alpha}_{K}^2 = \alpha^TK\alpha$ as the \emph{regulariser}.
To construct respective optimisation problem for obtaining posterior samples, we part from the decomposed pathwise expression given in \cref{eqn:pathwise-zero-mean}, which we repeat here for the reader's convenience
\begin{gather}
  \label{eqn:pathwise-zero-mean_c3}
  (f| Y)(\cdot) = \ubr{f_{\star}(\cdot) \vphantom{{K}_{(\cdot)X}} }_{\text{posterior mean}} + \ubr{f(\cdot) \vphantom{{K}_{(\cdot)X}} }_{\text{prior sample}}  - \ubr{{K}_{(\cdot)X} (K + B^{-1})^{-1}(f(X) + \c{E})
  }_{\text{uncertainty reduction term}} \\
  \text{with} \quad
  \c{E}\sim\N(0, B^{-1}) \spaced{and}
  f\sim\GP(0, k) \notag.
\end{gather}
The posterior mean can be obtained by solving \cref{eqn:mean-optim}.
We approximate the prior function sample $f$ using a sum of random Fourier features $\tl{f}$, as described in \cref{subsec:efficient_prior_sampling}.
Each posterior sample's uncertainty reduction term is parametrised by a set of representer weights. These are given by a linear solve against a noisy prior sample evaluated at the observed inputs $(K + B^{-1})^{-1}(\tl{f}(X ) + \eps)$.
Thus, by analogy to \cref{eqn:mean-optim}, we can construct an optimisation objective targeting a sample's representer weights as
\begin{gather}\label{eqn:samples-optim}
  \argmin_{\alpha\in\R^n} \sum_{i=1}^n [B]_{ii} (\tl{f}(x_i) + \eps_i - {K}_{x_iX }\alpha)^2 + \norm{\alpha}_{K}^2  \\
   \text{with} \quad \c{E} \sim\N(0,B^{-1}) 
    \spaced{and}
    \tl{f}(\cdot) =  \phi_{s}(\cdot) w \quad 
    w \sim\N(0, I_d) \quad s \sim \Omega, \notag
\end{gather}
where $\eps_i$ are the individual entries of $\c{E} = [\eps_1, \eps_2, \dotsc, \eps_n]^T$.
We denote as $\phi_s$ a $d$ dimensional random feature expansion. Unless specified otherwise, we assume a stationary kernel and use a cosine expansion with random frequencies drawn from our kernel's spectral density $\Omega$.

\begin{remark} \textbf{Dividing and conquering}\\
    We explicitly separate the posterior mean and 0-mean samples into two separate optimisation problems: \cref{eqn:mean-optim} and \cref{eqn:samples-optim}. However, this need not be the case. We could shift the solution of the sampling objective by exactly the mean function by regressing onto $(\tl{f}(X ) + \c{E} + Y)$ instead of $(\tl{f}(X ) + \c{E})$. Unfortunately, when doing this,  we find the target vector $Y$ to often dominate the objective, resulting in a worse quality estimates of the GP posterior variance. 
\end{remark}

\section{Stochastic estimators of the sampling objective}\label{sec:stochastic_sampling_estimators}

We now develop and analyse techniques for drawing samples from GP posteriors using stochastic gradient descent.
We provide three different stochastic estimators. First a simple, general purpose one in \cref{subsec:SGD_RFF_estimator}. This objective will prove useful when dealing with inducing point GPs, where the innovations discussed next are not applicable.
Then, in \cref{subsec:variance_reduction_SGD}, one with reduced variance when drawing 0-mean posterior samples. We will also provide the weight-space counterpart of this estimator. We will go on to investigate the conditioning of the quadratic objectives targeted by these estimators. This will lead us to develop our third method \emph{Stochastic Dual Descent} in \cref{subsec:SDD}, which brings favourable conditioning to the kernelised setting.
Finally, \cref{subsec:the_right_optimiser} compares different approaches to stochastic optimisation and provides guidelines on best practices.
As a preview of this section's contributions, we showcase SGD's performance, and compare it to CG and inducing point VI, on a pair of toy problems designed to capture complementary computational difficulties, in \cref{fig:scalable-learning-comparison}.

\begin{figure}
\includegraphics{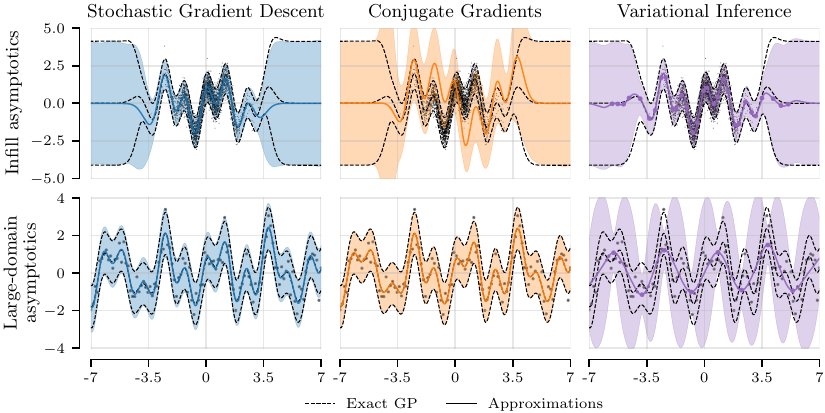}
\caption{Comparison of SGD, CG \citep{WangPGT2019exactgp} and SVGP \citep{Hensman2013big} for GP inference with a squared exponential kernel on $10\text{k}$ datapoints from $\sin(2x)+\cos(5x)$ with observation noise distribution $\text{N}(0, 0.5)$.  We draw 2000 function samples with all methods by running them for 10 minutes on an RTX 2070 GPU.
\emph{Infill asymptotics} considers  $x_i\~[N](0,1)$. A large number of points near zero result in a very ill-conditioned kernel matrix, preventing CG from converging. SGD converges in all of input space except at the edges of the data. SVGP can summarise the data with only 20 inducing points. Note that CG converges to the exact solution if one uses more compute, but produces significant errors if stopped too early, as occurs under the given compute budget.
\emph{Large domain asymptotics} considers data on a regular grid with fixed spacing.
This problem is better conditioned, allowing SGD and CG to recover the exact solution. However, 1024 inducing points are not enough for SVGP to summarise the data.
}
\label{fig:scalable-learning-comparison}
\end{figure}

\subsection{A first approach: mini batching and unbiased random features}\label{subsec:SGD_RFF_estimator}

The optimisation problem \cref{eqn:mean-optim},
requires $\c{O}(n^2)$ operations to compute both the square error and regulariser terms exactly.
The square error loss term is amenable to minibatching, which gives an unbiased estimate in $\c{O}(n)$ operations.
Assuming that $k$ admits random features, we can stochastically estimate the regulariser by expressing the kernel matrix as the expectation of an outer product of feature expansions (see \cref{subsec:random_features}). That is, 
$\norm{\alpha}_{K}^2 = \E_{s \sim \Omega} \alpha^T \Phi_s\Phi^T_s \alpha$ where $\Phi_s \in \R^{ n \times d}$ is the stacked $d$-dimensional random feature expansion of the $n$ inputs. 
Combining both estimators gives our SGD~objective
\begin{gather}
  \label{eqn:stochastic_mean-objective}
  \frac{n}{r} \sum_{i=1}^r [B]_{ii} (y_i - {K}_{x_iX }\alpha)^2 + \alpha^T \Phi_s\Phi^T_s \alpha
\end{gather}
where $r$ is the minibatch size.
This regulariser estimate is unbiased even when drawing a single Fourier feature per step $d=1$. The number of features controls the variance.
Evaluating \cref{eqn:stochastic_mean-objective} presents $\c{O}(n)$ complexity, in contrast with the $\c{O}(n^2)$ complexity of one CG step.
It is straight forward to apply the same estimators to the 0-mean sampling objective in \cref{eqn:samples-optim} obtaining
\begin{gather}
  \label{eqn:stochastic_sample-objective}
  \frac{n}{r} \sum_{i=1}^r[B]_{ii} (\tl{f}(x_i) + \eps_i - {K}_{x_iX }\alpha)^2 +    \alpha^T \Phi_s\Phi^T_s \alpha,
\end{gather}
with a per-step cost of $\c{O}(ns)$, for $s$ the number of posterior samples drawn.
We discuss sublinear inducing point techniques further on, in \cref{sec:inducing_SGD}.

\subsection{A lower variance estimator for SGD-based sampling}\label{subsec:variance_reduction_SGD}

Empirically, the minibatch estimator in \cref{eqn:stochastic_sample-objective} results in high gradient variance. This is because our targets contain unstructured noise $\eps_i$, which is difficult to predict. We propose an alternative sampling objective function which shares the same gradient in expectation, but whose stochastic estimates may present lower variance. We provide both kernelised and weight-space forms for the new objective. We then analyse the variance of the new weight-space objective.

\subsubsection{Kernelised form}

We modify the sampling objective \cref{eqn:samples-optim} by moving the noise into the regulariser term
\begin{gather}\label{eqn:samples-optim-reduced}
   \argmin_{\alpha\in\R^n} \sum_{i=1}^n [B]_{ii} (\tl{f}(x_i) - {K}_{x_iX }\alpha)^2 + \norm{\alpha - \c{E}'}_{K}^2  \\
   \text{with} \quad \c{E}' \sim\N(0,B) 
    \spaced{and}
    \tl{f}(\cdot) =  \phi_{s}(\cdot) w \quad 
    w \sim\N(0, I_d) \quad s \sim \Omega, \notag
\end{gather}
which inverts the covariance of the distribution the noise is sampled from. We highlight this change with the prime notation $\c{E}'$. This modification \emph{preserves the optimal representer weights} since objective \cref{eqn:samples-optim-reduced} equals \cref{eqn:samples-optim} up to a constant.  

\begin{derivation}\textbf{Equivalency of kernelised sampling objectives} \\
To show the equality of both objectives up to a constant, we show both have the same gradient. 

  Let ${L}{L}^T = B^{-1}$ be the Cholesky factorisation of the noise covariance, let $f(X) \sim\N(0, K)$, and let $\epsilon \sim\N(0, {I}_n)$.
  Our objectives are
  \begin{gather}
    \norm[0]{f(X ) + {L} \epsilon - K\alpha}_{B}^2 + \norm{\alpha}_{K}^2
  \end{gather}
  and
  \begin{gather}
    \norm[0]{f(X )  - K\alpha}_{B}^2 + \norm{\alpha - {L}^{-T}\epsilon}_{K}^2
    .
  \end{gather}
  Taking derivatives with respect to $\alpha$, we have
  \begin{gather}
    \partial{\alpha} \left(\norm[0]{f(X ) + {L} \epsilon - K\alpha}_{B}^2 + \norm{\alpha}_{K}^2 \right)
    \\
    \quad= -2K B \left( f(X ) + {L} \epsilon - K\alpha \right) + 2 K\alpha
    \\
    \quad= -2K ( B f(X ) - B K\alpha + {L}^{-T} \epsilon - \alpha),
  \end{gather}
  and
  \begin{gather}
    \partial{\alpha} \left( \norm[0]{f(X ) - K\alpha}_{B}^2 + \norm{\alpha - {L}^{-T}\epsilon}_{K}^2 \right)
    \\
    \quad= -2K B \left( f(X ) - K\alpha \right) + 2 K(\alpha - {L}^{-T} \epsilon )
    \\
    \quad= -2K ( B f(X ) - B K\alpha + {L}^{-T} \epsilon - \alpha),
  \end{gather}
  respectively.
  These expressions match, giving the claim. 
  Furthermore, since both objectives are strictly convex, they both have the same unique minimum.
\end{derivation}

\subsubsection{Weight-space form}

We now apply the same trick for the weight-space form of the sample-then-optimise objective \cref{eq:sample_then_optimise}. This will allow us to scale the linearised Laplace method to real-world sized deep learning problems in \cref{chap:sampled_Laplace}. We begin by stating the zero-mean sample-then-optimise objective:
\begin{gather}\label{eq:old_loss}
         L(w) = \frac{1}{2} \|\c{E} - \Phi w\|^2_{{B}} + \frac{1}{2} \|w - w_0\|^2_{A} \\
    \text{with} \quad  \c{E} \sim\N(0, B^{-1}) \spaced{and}
    w_0 \sim \N(0, A^{-1}) \notag.
\end{gather}
Inspecting the expression, we consider how it may be stochastically estimated:
\begin{itemize}[topsep=0pt]
  \item The first term is data dependent. It corresponds to the scaled squared error in fitting $\mathcal{E}$ as a linear combination of $\Phi$. Its gradient requires stochastic approximation for large datasets.
  \item The second term, a regulariser centred at $w_0$, does not depend on the data. Its gradient can thus be computed exactly at every optimisation step. This differs from the kernelised setting, where the regulariser contained the kernel matrix and required stochastic estimation.
\end{itemize}
Again, we encounter random noise in the targets, and thus, the variance of a mini-batch estimate of the gradient of $\|\Phi z - \mathcal{E}\|_{{B}}^2$ may be large. %
Instead, for $\mathcal{E}$ and $w_0$ defined as above, we propose the following alternative loss, again equal to \cref{eq:old_loss} up to an additive constant independent of the variable being optimised:
\begin{gather}\label{eq:new-loss}
  L'(w) = \frac{1}{2}\|\Phi w\|_{{B}}^2 + \frac{1}{2}\|w - w'_0\|_A^2 \spaced{with} w'_0 = w_0 + A^{-1}\Phi^T {B} \c{E} \\
  \text{where} \quad  \c{E} \sim\N(0, B^{-1}) \spaced{and}
    w_0 \sim \N(0, A^{-1}) \notag.
\end{gather}
The mini-batch gradients of $L'$ and $L$ are equal in expectation and both objective's optima is the same.
However, in $L'$, the randomness from the noise samples $\mathcal{E}$ and the prior sample $w_0$ both feature within the regularisation term---the gradient of which can be computed exactly---rather than in the data-dependent term.

\begin{derivation} \textbf{Equivalency of weight space sampling objectives} \\
Again, the losses $L$ and $L'$ are strictly convex, thus to confirm they have the same unique minimum, it suffices to consider the respective first order optimality conditions. We introduce $\zeta \in \R^d$ and $\zeta' \in \R^d$ such that $\partial_w L(\zeta) = 0$ and $\partial_w L'(\zeta') = 0$. We have,
\begin{equation}
  \partial_w L (\zeta) = \Phi^T{B} (\Phi \zeta - \mathcal{E}) + A(\zeta - w_0),
\end{equation}
and 
\begin{align}
  \partial_{w} L' (\zeta') &= \Phi^T{B}\Phi \zeta' + A(\zeta' - A^{-1}\Phi^T{B}\mathcal{E} - w_0) \\&= \Phi^T{B} (\Phi \zeta' - \mathcal{E}) + A(\zeta' - w_0)
\end{align}
Thus $\zeta = \zeta'$ almost surely. Moreover, $L'(w)=L(w)+C$ for all $w$, for $C$ a constant independent of $w$.
\end{derivation}

For completeness, we provide an alternative path to checking the validity of our sampling objectives. We study the distribution of their optima.
\begin{derivation} \textbf{Distribution of optima of weight-space sampling losses}\\
To determine the distribution of $\zeta = \argmin_{w \in \R^d} L(w)$, we note that it is a linear transformation of zero-mean Gaussian random variables, and thus itself a zero-mean Gaussian random variable. Rearranging the first order optimality condition, we find that
\begin{equation}
  \zeta = H^{-1}(\Phi^T{B}\mathcal{E} + A\theta^0).
\end{equation}
Thus 
\begin{align}
  \E[\zeta\zeta^T] &= H^{-1} \E[(\Phi^T{B}\mathcal{E} + A\theta^0)(\Phi^T{B}\mathcal{E} + A\theta^0)^T] H^{-1} \\
  &= H^{-1} \left(\Phi^T {B} \E[\mathcal{E}\mathcal{E}^T] {B} \Phi + A\E[\theta^0\theta^0] A + 2\Phi^T{B}\E[\mathcal{E}(\theta^0)^T]A\right)H^{-1} \\
  &= H^{-1} (\Phi^T {B} \Phi + A) H^{-1} = H^{-1} H H^{-1} = H^{-1},
\end{align}
and so $\zeta \sim \N(0, H^{-1})$.
\end{derivation}

\subsubsection{Analysis of minibatch gradient variance}

Consider the variance of the single-datapoint stochastic gradient estimators for both weight-space objectives' data dependent terms. At $\smash{z \in \R^{d}}$, for datapoint indices sampled as $j \sim \U(\{1, \dotsc, n\})$, these are
\begin{equation}
  \hat g = n\phi(x_j)^T (\phi(x_j) z - \eps_j) \spaced{and} \hat g' = n\phi(x_j)^T \phi(x_j) z
\end{equation}
for $L$ and $L'$, respectively. Direct calculation, shows that
\begin{equation}
  \textstyle{\frac{1}{n}\left[\var \hat g - \var \hat g'\right]
    = \var(\Phi^T {B} \mathcal{E}) - 2\cov(\Phi^T{B} \Phi z, \Phi^T {B} \mathcal{E})
    \eqqcolon \Delta.}
\end{equation}
Note that both $\var \hat g$ and $\var \hat g'$ are $d \times d$ matrices. We impose an order on these by considering their traces: we prefer the new gradient estimator $\hat g'$ if the sum of its per-dimension variances is lower than that of $\hat g$; that is if $\tr \Delta > 0$. We analyse two key settings:
\begin{itemize}
  \item At initialisation, taking $w = w_0$ (or any other initialisation independent of $\c{E}$),
        \begin{equation}
          \tr\Delta = \tr\{\Phi^T{B}\E[\mathcal{E} \mathcal{E}^T]{B}\Phi\} - \tr\{\Phi^T{B}\Phi \E[w_0  \mathcal{E}^T]{B} \Phi \} = \tr M  > 0.
        \end{equation}
        Recall that $M = \Phi^T B \Phi$. We used that $\E[\mathcal{E}\mathcal{E}^T] = {B}^{-1}$ and since $\mathcal{E}$ is zero mean and independent of $w_0$, we have $\E[w_0\mathcal{E}^T] = \E w_0 \E \mathcal{E}^T = 0$.
        Thus, the new objective $L'$ is always preferred at initialisation.
  \item At convergence, that is, at $\zeta = \argmin_{w\in \R^d} L(w)$, assuming a prior precision of the form $A=aI$, a more involved calculation, contained in Appendix C.3 of \cite{antoran2023samplingbased}, shows that $L'$ is preferred if
        \begin{equation}\label{eq:effective-dim-condition}
          2 a \gamma > \tr M,
        \end{equation}
        where $\gamma$ is the effective dimension \cref{eq:forms_effective_dim}.
        This is satisfied if the regulariser $a$ is large relative to the eigenvalues of $M$, (see Appendix C.4 of \cite{antoran2023samplingbased}), that is, when the effective dimension is low and the parameters are not strongly determined by the data relative to the prior. In practise, we find this to be the case for most heavily overparametrised models, like linearised neural networks, which are central to the following chapters of this thesis.
\end{itemize}
When $L'$ is preferred both at initialisation and at convergence, we expect it to have lower variance for most minibatches throughout training. Even if the proposed objective $L'$ is not preferred at convergence, it may still be preferred for most of the optimisation, before the noise is fit well enough.

\begin{remark} \textbf{Sticking the landing... or not.} \\ 
    The regular sample-then-optimise objective \cref{eq:old_loss} uses random noise as targets in its fit term. There may be an optima of this fit term where we can perfectly interpolate the noise targets. At this point, not only is the gradient of the fit term 0, but so is any minibatch estimate we construct. Of course, the regulariser prevents the optima of the full objective from matching the optima of the fit term.
    
    On the other hand, the fit term of our proposed objective \cref{eq:new-loss} has no targets (or 0 targets). Clearly, the regulariser prevents the optima of the proposed objective being the zero vector. Thus, at the optima, the data-fit term's gradient variance will not go to zero. 

    With this, we can build intuition for the result in \cref{eq:effective-dim-condition}. The weaker the regulariser, the closer the full objective optima is to the optima of the fit term, where the regular sample-then-optimise objective \cref{eq:old_loss} presents lower variance.
\end{remark}

\begin{figure}[h]
\includegraphics[width=0.9\textwidth]{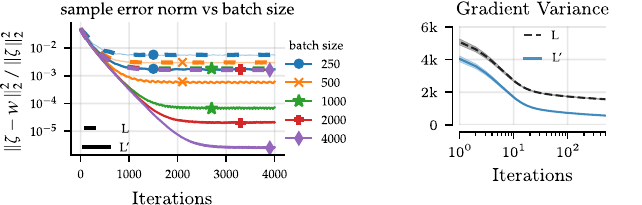}
\caption{Left: optimisation traces for the relative $L_2$ error in the weight-space posterior sample using  our proposed  sample-then-optimise  objective $L'$ \cref{eq:new-loss} and the existing one $L$ \cref{eq:old_loss}. The model is a linearised NN \cref{subsec:linearised_prior} and the task is MNIST. The plotted lines are averaged across 16 samples and 5 seeds.  The low variance objective allows a $\approx16\times$ reduction in batch size without reduction in weight-space posterior sample accuracy. Right: gradient variance throughout optimisation for a single-sample minibatch estimator ($r=1$) of the kernelised sampling objectives. We use an RBF kernel  on the \textsc{elevators} dataset ($n\approx16\text{k}$). Again $L'$ refers to the low variance estimator \eqref{eqn:samples-optim-reduced}. In both plots we run SGD with Nesterov momentum $\rho=0.9$ and geometric averaging. 
}
\label{fig:grad_var_comparison}
\end{figure}

\subsubsection{Demonstration of low-variance sampling objective}

\Cref{fig:grad_var_comparison} illustrates the benefits of both our weight space and kernelised low variance sampling objectives.  For the weight-space version, we use the Jacobian feature expansion corresponding to a LeNet style CNN with $d=29226$ weights. Its linearisation point is found by pre-training the model on MNIST.

\subsection{Stochastic Dual Descent}\label{subsec:SDD}

We now analyse the curvature of the quadratic objectives used for GP posterior sampling in \cref{subsec:SGD_RFF_estimator} and \cref{subsec:variance_reduction_SGD}. This leads us propose a better conditioned objective. In the context of this new objective, we question our previous choice of stochastic approximation. We compare mini-batching and random-feature approximations, and building upon the insights gained, propose a random-coordinate estimator with more desirable properties than either. 
The resulting algorithm: Stochastic dual descent (SDD) can be seen as an adaptation of the stochastic dual coordinate ascent algorithm of \cite{shalev2013stochastic} to the large-scale deep-learning-type gradient descent framework. We also incorporate insights on stochastic approximation from the theoretical work of \cite{dieuleveut2017harder} and \cite{varre2021last}.

Assuming an isotropic noise precision, $B=bI$, the kernelised posterior sampling objectives provided in \cref{subsec:SGD_RFF_estimator} and \cref{subsec:variance_reduction_SGD} are of the form
\begin{equation*}
L_p(\alpha) \coloneqq \frac{1}{2}\norm{z - K \alpha}^2 + \frac{b^{-1}}{2} \norm{\alpha}^2_K 
\end{equation*}
over $\alpha \in \R^n$ and for some choice of target vector $z \in \R^n$.
In the kernel literature \citep{smola1998learning,shalev2013stochastic}, the kernel ridge regression objective $L_p$ is known as the \emph{primal} objective, and thus the subscript $_p$. We adopt this naming in the context of this subsection\footnote{In the Bayesian linear model and Gaussian process literature, the weight space view is refereed as the \emph{primal form}, while the kernelised view is refereed to as the \emph{dual form}. In the kernel literature, the opposite is true; methods that deal with objects living in the RKHS are referred to as dual.}. 
The \emph{primal} gradient and Hessian are
\begin{equation}
    \label{eq:primal_gradient}
\partial_\alpha L_p(\alpha) = K (b^{-1} \alpha - z + K\alpha)
\spaced{and}
\partial_\alpha^2 L_p(\alpha) = K(K+b^{-1} I),
\end{equation}
respectively.
Recall that the speed at which our optimiser approaches $\alpha_\star = (K+b^{-1} I)^{-1}z \in \R^n$ is determined by the condition number of the Hessian: the larger the condition number, the slower the convergence speed.
The intuitive reason for this correspondence is that, to guarantee convergence, the step-size needs to scale inversely with the largest eigenvalue of the Hessian, while progress in the direction of an eigenvector underlying an eigenvalue is governed by the step-size multiplied with the corresponding eigenvalue. 
Letting $\lambda_i: \R^{n\times n} \to \R$ return the $i$th largest eigenvalue of a matrix, for the primal objective, the tight bounds on the relevant eigenvalues are
\begin{equation*}
0 \leq \lambda_n(K(K+b^{-1} I )) \leq \lambda_1(K(K+b^{-1} I)) \leq \kappa n(\kappa n+b^{-1}),
\end{equation*}
where $\kappa = \sup_{x\in \c{X}} k(x,x)$ is finite by assumption. These bounds only allow for a step-size $\beta$ on the order of $(\kappa n(\kappa n+ b^{-1}))^{-1}$ and, since they do not bound the minimum eigenvalue away from zero, we do not have a~priori guarantees for the performance of gradient descent.

\subsubsection{A dual objective}

Consider, instead, minimising the \emph{dual objective}
\begin{equation}\label{eq:sdd_objective}
    L_d(\alpha) = \frac12 \norm{\alpha}_{K+b^{-1} I}^2 - \alpha^T z.
\end{equation}
The dual $L_d$ has the same unique minimiser as $L_p$, namely $\alpha_\star$. We go on to show the duality of $(L_p, b^{-1} L_d)$; the factor of $b^{-1}$ is immaterial.

\begin{derivation} \textbf{Strong duality of SDD objective $L_d$ \cref{eq:sdd_objective}} \\
    We claim that
    \begin{equation*}
        \min_{\alpha \in \R^n} L_p(\alpha) = - b^{-1} \min_{\alpha \in \R^n} L_d(\alpha),
    \end{equation*}
    with $\alpha_\star$ minimising both $L_p$ and $L_d$.

    That $\alpha_\star$ minimises both $L_p$ and $L_d$ can be established from the first order optimality conditions. Now, for the duality, observe that we can write $\min_{\alpha \in \R^n} L_p(\alpha)$ equivalently as the constrained optimisation problem
    \begin{equation*}
        \min_{u \in \R^n} \min_{\alpha \in \R^n} \frac{1}{2} \|u\|^2 + \frac{b^{-1}}{2}\|\alpha\|^2_K \spaced{subject to} u=K\alpha-z\,.
    \end{equation*}
    Observe that this is quadratic in both $u$ and $\alpha$. Introducing Lagrange multipliers $\eta \in \R^n$, in the form $b^{-1} \eta$, where we recall that $b^{-1} > 0$, the solution of the above is equal to that of
    \begin{equation*}
         \min_{u \in \R^n} \min_{\alpha \in \R^n} \sup_{\eta \in \R^n} \frac{1}{2} \|u\|^2 + \frac{b^{-1}}{2}\|\alpha\|^2_K + b^{-1} \eta^T (z-K\alpha - u)\,.
    \end{equation*}
    This is a finite-dimensional quadratic problem, and thus we have strong duality \cite[see, e.g., Examples 5.2.4 in][]{Boyd2014convex}. We can therefore exchange the order of the minimum operators and the supremum, yielding the again equivalent problem
    \begin{equation*}
        \sup_{\eta \in \R^n} \left\{\min_{u \in \R^n} \frac{1}{2} \|u\|^2 - b^{-1} \eta^T  u \right\} + \left\{\min_{\alpha \in \R^n} \frac{b^{-1}}{2}\|\alpha\|^2_K - b^{-1}\eta^T  K\alpha\right\}\, + b^{-1} \eta^T  z.
    \end{equation*}
    Noting that the two inner minimisation problems are quadratic, we solve these analytically using the first order optimality conditions, that is $\alpha = \eta$ and $u = b^{-1}  \eta$, to obtain that the above is equivalent to
    \begin{equation*}
        \sup_{\eta \in \R^n} - b^{-1}\left(\frac{1}{2}\|\eta\|^2_{K+b^{-1} I} - \eta^T  z\right) = -b^{-1} \min_{\eta \in \R^n} L_d(\eta)\,.
    \end{equation*}
    The result follows by chaining the above equalities.
\end{derivation}

The \emph{dual gradient} and Hessian are given by
\begin{equation}
    \label{eq:dual_gradient}
    \partial_\alpha L_d(\alpha) = b^{-1}\alpha - z + K\alpha \spaced{and} \partial_\alpha^2L_d(\alpha) = K + b^{-1} I.
\end{equation}
Observe that when running gradient descent on the dual objective $L_d$, we can use a step-size of order $(\kappa n + b^{-1})^{-1}$. That is, $\kappa n$ higher than before.
Moreover, since the condition number of the dual satisfies $\mathrm{cond}(K+b^{-1} I) \leq1+ \kappa n b $, we have faster convergence, and can provide an a~priori bound on the number of iterations required for any fixed error level for any length $n$ sequence of observations.

\begin{remark} \textbf{The conditioning of weight-space sampling objectives} \\
The weight space sampling objectives \cref{eq:old_loss} and \cref{eq:new-loss} both present the Hessian $M + A = H$. Here, the Hessian of the data fit term appears only once, instead of twice (like in the Hessian of $L_p$ \cref{eq:primal_gradient}). For $A = aI$, $a \leq \lambda_d(H)$, and thus the conditioning number can be bounded from above. \corr{Thus, the improvements from the dual kernelised objective are already baked-into non-kernelised weight-space objectives.} This makes sense, since in \cref{subsec:duality} we saw that the weight space problem is also a dual to the kernelised regression problem (but a different dual than $L_d$).
\end{remark}

\begin{remark} \textbf{An RKHS view of the dual gradient} \\
In \cref{sec:pathwise_as_optim}, we derived the primal objective $L_p$ \cref{eq:primal_gradient} by applying the representer theorem to the regularised regression problem formulated in the RKHS
\begin{gather*}
    \frac{1}{2} \sum_{i=1}^n  (y_i - \innerprod{k(x_i, \cdot)}{f}^2  + \frac{1}{2}b^{-1}\|f\|_{\c{H}}^2.
\end{gather*}
We did this because we can not fit infinite dimensional objects into our computers. However, an alternative could have been to take gradients directly in the RKHS
\begin{gather*}
    f' =    K_{(\cdot)X} (Y - f(X))  + b^{-1} f,
\end{gather*}
apply them to get our updated function
\begin{gather*}
    f_{\text{new}} = f - \beta f' =  f - \beta ( K_{(\cdot)X} (Y - f(X)) +  b^{-1} f),
\end{gather*}
and then express this in terms of representer weights
\begin{gather*}
    f_{\text{new}} = f - \beta f' = K_{(\cdot)X} \left( \alpha - \beta  ( Y - f(X) +  b^{-1} \alpha) \right ),
\end{gather*}
and thus we have 
\begin{gather*}
    \alpha_{\text{new}} = \alpha - \beta  ( Y - f(X) +  b^{-1} \alpha) = \alpha - \beta \partial_{\alpha} L_d(\alpha).
\end{gather*}
In other words, we can derive the dual objective by performing gradient descent in the RKHS and projecting back onto the representer weights once the gradient update has been performed.
Thus, the dual objective is ``dual'' in the sense that it operates directly in the RKHS. 
\end{remark}

\begin{figure}[t]
    \centering
    \includegraphics[width=5.5in]{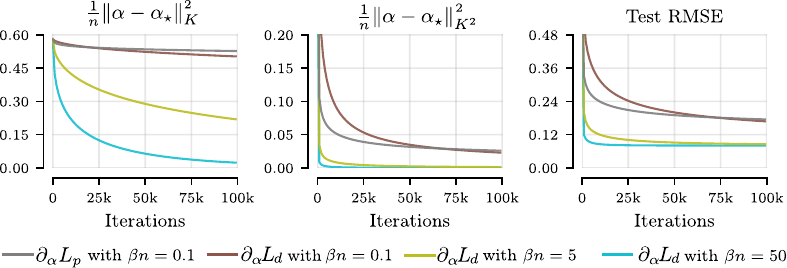}
    \caption{Comparison of full-batch primal and dual gradient descent on \textsc{pol} with varying step-sizes. Primal gradient descent becomes unstable and diverges for $\beta n$ greater than $0.1$. Dual gradient descent is stable with larger step-sizes, allowing for markedly faster convergence than the primal.
    For $\beta n =0.1$, the dual method makes more progress in the $K$-norm, whereas the primal in the $K^2$-norm.
    }
    \label{fig:toy-gradient-primal-vs-dual}
\end{figure}

\subsubsection{Demonstration: dual gradients}

To illustrate the discussion so far, we compare the progress of dual and the primal gradient descent when computing the GP posterior mean representer weights on the UCI \textsc{pol} regression task, with results shown in \cref{fig:toy-gradient-primal-vs-dual}. 
There, for the step-sizes we tried, gradient descent with the primal objective was only stable up to  $\beta n=0.1$, and diverged for larger step-sizes. In contrast, gradient descent with the dual objective is stable with a step-size as much as $500\times$ higher. It converges faster and to a better solution. We show this on three evaluation metrics: 1) distance to $\alpha^\star$ measured in $\|\cdot\|^2_K$, the $K$-norm (squared), 2) in $\|\cdot\|^2_{K^2}$, the $K^2$-norm (squared), and 3) test set root mean square error (RMSE). To understand the difference between the two norms, note that the $K$-norm error bounds the error of approximating $f_\star = K_{(\cdot) X} \alpha_\star$ with $f = K_{(\cdot) X} \alpha$ uniformly. Indeed, as shown below, we have the~bound
\begin{equation}
    \label{eq:pointwise-error}
    \| f -f_\star\|_\infty \leq \sqrt{\kappa} \| \alpha - \alpha_\star\|^2_K\,,
\end{equation}
where $\kappa = \sup_{x\in \c{X}} k(x,x)$. Uniform norm guarantees of this type are crucial for sequential decision making tasks, such as Bayesian optimisation, where test input locations may be arbitrary. The $K^2$-norm metric, on the other hand, reflects training error. Examining the gradients, it is immediate that the primal gradient optimises for the $K^2$-norm, while the dual for the $K$-norm. And indeed, we see in \cref{fig:toy-gradient-primal-vs-dual} that when both methods use $\beta n=0.1$, up to $70$k iterations, the dual method is better on the $K$-norm metric and the primal on $K^2$. Later, the dual gradient method performs better on all metrics. This, too, is to be expected, as the minimum eigenvalue of the Hessian of the dual loss is higher than that of the primal loss.

\begin{derivation} \textbf{Uniform bound on function approximation error} \\ 
We first introduce some machinery that allows us to formally reason about elements in a (potentially) infinite dimensional RKHS $\c{H}$.
For observations $X$, let $K_{X(\cdot)} \colon \c{H} \to \R^n$ be the linear operator mapping $h \mapsto f(X)$, where $f(X) = (f(x_1), \dotsc, f(x_n))$. We will write $K_{(\cdot)X}$ for the adjoint of $K_{X(\cdot)}$, and observe that $K$ is the matrix of the operator $K_{X(\cdot)}K_{(\cdot)X}$ with respect to the standard basis. 

With that, first, observe that,
\begin{align*}
\|f - f_\star\|_\infty 
&= \sup_{x \in \c{X}} |f(x) - f_\star(x)| \tag{defn. of sup norm}\\
&= \sup_{x \in \c{X}} |\langle k(x, \cdot), f - f_\star \rangle| \tag{reproducing property} \\
&\leq \sup_{x \in \c{X}} \|k(x, \cdot)\|_\c{H} \|f - f_\star\|_\c{H} \tag{CBS}\\
&\leq \sqrt{\kappa} \|f - f_\star\|_\c{H}\,, \tag{defn. of $\kappa$}
\end{align*}
Now, observe that $f = K_{(\cdot)X}\alpha$ and $f_\star = K_{(\cdot)X}\alpha_\star$, and so we have the equalities
\begin{align*}
    \|f - f_\star\|^2_\c{H}
    &= \langle K_{(\cdot)X}(\alpha - \alpha_\star), K_{(\cdot)X} (\alpha - \alpha_\star) \rangle \tag{defn. norm}\\
    &= \langle \alpha - \alpha_\star, K_{X(\cdot)}K_{(\cdot)X} (\alpha - \alpha_\star) \rangle \tag{defn. adjoint}\\
    &= \| \alpha - \alpha_\star\|^2_K\,. \tag{defn. $K$}
\end{align*} 
Combining the above two equalities yields the claim.
\end{derivation}

\subsubsection{Randomised Gradients: Random Features versus Random Coordinates}\label{subsec:stochastic_approximation}

We now study the construction of a stochastic objective to estimate the dual gradient \cref{eq:dual_gradient} in linear time. The minibatching plus random feature estimator presented in \cref{subsec:SGD_RFF_estimator} is not suitable for the dual objective because the kernel matrix does not appear in the regulariser of the dual objective. However, it does appear in the data fit term. Thus, we compare random feature and minibatch estimators.

We begin with random features. Recall that $K = E_{s \sim \Omega}  \Phi_s\Phi^T_s$ where $\Phi_s \in \R^{n \times d}$ is a $d$ dimensional random feature expansion of $X$.
It follows that
\begin{equation*}
    \widetilde \partial_\alpha L_d(\alpha) = b^{-1}\alpha - z + \Phi_s\Phi^T_s \alpha
\end{equation*}
gives an unbiased estimate of $\partial_\alpha L_d(\alpha)$.

An alternative is to use minibatching
\begin{equation} \label{eq:random_coordinate}
    \widehat \partial_\alpha L_d(\alpha) =  n e_i e_i^T \partial_\alpha L_d(\alpha) = n e_i (b^{-1} \alpha_i - z_i + [K]_i \alpha) 
\spaced{with}
i \sim \U{\{1, \dotsc, n\}},
\end{equation}
where $e_i$ are the elements of the canonical basis, i.e. $e_1 = [1, 0, 0, \dotsc]^T$.
Thanks to $\E[ n e_i e_i^T ] = I$, this is also an unbiased estimate of $\partial_\alpha L_d(\alpha)$.
Note that the cost of calculating either $\widetilde \partial_\alpha L_d(\alpha)$ or $\widehat \partial_\alpha L_d(\alpha)$ is linear in $n$, achieving our goal of reduced computation time. 
Also, note that while $\widetilde \partial_\alpha L_d(\alpha)$, which we call the random feature estimate, is generally a dense vector, $\widehat \partial_\alpha L_d(\alpha)$,  is sparse. Since all but one coordinates of $\widehat \partial_\alpha L_d(\alpha)$ are zero, we refer to this as the \emph{random coordinate estimate}\footnote{In practise, we implement this estimator by sampling multiple coordinates at each step, not just one.}.

The nature of the noise introduced by these estimators, and thus their qualities, are quite different. In particular, one can show that
\begin{align*}
\norm{ \widehat \partial_\alpha L_d(\alpha) - \partial_\alpha L_d(\alpha) } \le \norm{ (n e_i e_i^T -I) (K+b^{-1} I) } \norm{\alpha - \alpha_\star}\,.
\end{align*}
As such, the noise introduced by $\widehat \partial_\alpha L_d(\alpha)$ is proportional to the distance between the current iterate $\alpha$, and the optima $\alpha^\star$. The noise goes to 0 when the optima is reached. This estimator does stick the landing!
For $\widetilde \partial_\alpha L_d(\alpha)$, letting $\widetilde K = \Phi_s\Phi^T_s$, we have
\begin{align*}
\norm{ \widetilde \partial_\alpha L_d(\alpha) - \partial_\alpha L_d(\alpha) } = \norm{(\widetilde K-K) \alpha}\,.
\end{align*}
As such, the error in $\widetilde \partial_\alpha L_d(\alpha)$ is \emph{not reduced} as $\alpha$ approaches $\alpha_\star$.
In the optimisation literature, $\widetilde \partial_\alpha $ would be classed as an \emph{additive noise} gradient oracle, whereas $\widehat \partial_\alpha$ as a \emph{multiplicative noise} oracle \citep{dieuleveut2017harder}. 
Intuitively, multiplicative noise oracles automatically reduce the amount of noise injected as the iterates get closer to their target. While harder to analyse, multiplicative noise oracles often yield better performance: see, for example, \cite{varre2021last}.

\begin{remark} \textbf{Incompatibility of multiplicative noise with the weight-space problem} \\
Unfortunately, the random coordinate estimator \cref{eq:random_coordinate} is not applicable to the weight space sample-then-optimise objective \cref{eq:new-loss}. The multiplicative noise estimator relies on the data-fit and regulariser terms being minibatched jointly. This is possible in the kernelised setting, since they are both $n$ dimensional. However, in the weight-space setting, the data-fit objective is composed of $n$ terms, while the regulariser is composed of $d$ terms, one per model parameter. One could subsample the parameters of the linear model in the data-fit term, but this would require us having to compute the full dataset's feature expansion at each step. This would be computationally intractable for the Jacobian basis functions we deal with in the later chapters of this thesis.
\corr{Thus, for the weight-space formulation, we must fall-back on additive-noise and the reduced variance estimator corresponding to the objective in \cref{eq:new-loss} is the best we can do.}
\end{remark}

\begin{remark} \textbf{Can we apply the variance reduction strategy of \cref{subsec:variance_reduction_SGD} to the random coordinate estimator of the dual objective?} \\ 
The variance reduction strategy presented earlier in this chapter amounts to moving the random noise in the sample-then-optimise targets from the data fit term to the regulariser. Here, we do this for the dual gradient \cref{eq:dual_gradient}. Letting the noisy targets be $z = f(X) + \c{E}$, we move $\c{E}$ to the regulariser while inverting its covariance by premultiplying by the noise precision $b$
\begin{gather*}
    \partial_\alpha L_d = (K\alpha - f(X) - \c{E}) + b^{-1} \c{E} = \underbrace{K\alpha - f(X)}_{\textrm{new fit gradient}} - \underbrace{b^{-1} (\alpha - b \c{E})}_{\textrm{new reg. gradient}}.
\end{gather*}
Now, by inspection, it is clear that the random coordinate estimator, which subsamples the entries of the fit term and regulariser jointly, produces the same result when applied to both of the forms of the dual gradient written above. It is clear that we have nothing to gain by applying the variance reduction strategy.
\end{remark}

\begin{figure}[t]
    \begin{algorithm2e}[H]\LinesNotNumbered
    	\KwData{Kernel matrix $K$ with rows $K_1,\dots,K_n\in \R^n$, targets $z \in \R^n$,\\ 
            likelihood precision $b > 0$, number of steps $T \in \N^+$, 
            batch size $r \in \{1, \dotsc, n\}$, \\
            step size $\beta > 0$, 
            momentum parameter $\rho \in [0,1)$, 
            averaging parameter $\chi \in (0,1]$ }
    Set $v_0 = 0$; $\alpha_0 = 0$; $\overline \alpha_0 = 0$ \tcp*{all in $\R^n$}
    	\While{$t \in \{1, \dotsc, T \}$}{
    	 Sample $\mathcal{I}_t = (i^t_1, \dotsc, i^t_r) \sim \U{\{1,\dotsc, n\}}$ independently \tcp*{rand. coord.}
        $g_t = \frac{n}{r} \sum_{i \in \mathcal{I}_t} (\smash{( K_i + b^{-1} e_i)^T (\alpha_{t-1} + \rho v_{t-1})} - b_i)e_i$ \tcp*{gradient estimate}
        $v_t = \rho v_{t-1} - \beta g_t$ \tcp*{velocity update}
        $\alpha_t = \alpha_{t-1} + v_t$ \tcp*{parameter update}
        $\overline\alpha_{t} = \chi\alpha_t + (1-\chi)\overline\alpha_{t-1}$ \tcp*{geometric averaging}
    	}
    \KwResult{$\overline\alpha_T$ }
    \caption{\emph{Stochastic dual descent} for approximating $\alpha^\star = (K+ b^{-1} I)^{-1}z$}
    \label{alg:compute-mean}
    \end{algorithm2e}
   
\end{figure}

We term the combination of the dual gradient \cref{eq:dual_gradient} with random coordinate estimation \cref{eq:random_coordinate} as \emph{Stochastic Dual Descent} (SDD). The corresponding algorithm is provided in \cref{alg:compute-mean}. We discuss optimisation strategies in \cref{subsec:the_right_optimiser}. We henceforth distinguish this algorithm from the one that uses the primal loss \cref{eqn:samples-optim} and minibatching of only the fit term, as opposed to random coordinate estimation, by referring to the latter as SGD.

\begin{figure}[t]
    \centering
    \includegraphics[width=5.5in]{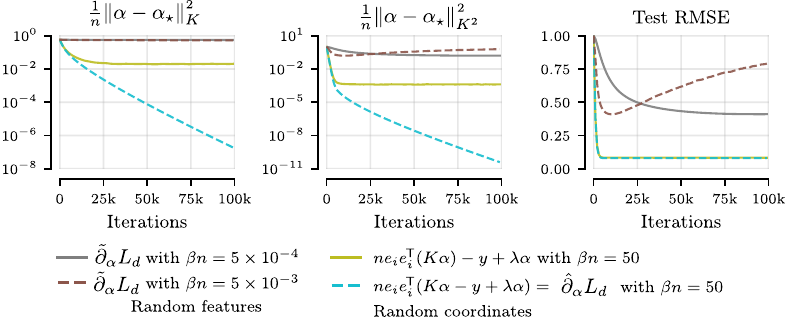}
    \caption{A comparison of dual (stochastic) gradient descent on the \textsc{pol} data set with either random Fourier features or random coordinates, using batch size $r=512$, momentum $\rho=0.9$ and averaging parameter $\chi=0.001$ (see \cref{subsec:the_right_optimiser} for explanation of latter two). Random features converge with $\beta n=5\times 10^{-4}$ but perform poorly, and diverge with a higher step-size. Random coordinates are stable with $\beta n = 50$ and show much stronger performance on all metrics. We include a version of random coordinates where only the $K\alpha$ term is subsampled: this breaks the multiplicative noise property, and results in an estimate which is worse on both the $K$-norm and the $K^2$-norm metric.
}
    \label{fig:toy-features-vs-coordinates}
\end{figure}

\subsubsection{Demonstration: additive vs multiplicative noise}

In \cref{fig:toy-features-vs-coordinates}, we compare variants of stochastic dual descent with either random (Fourier) features or random coordinates.
We see that random features, which produce high-variance additive noise, can only be used with very small step-sizes and have poor asymptotic performance. 
We test two versions of random coordinates: $\widehat \partial_\alpha L_d(\alpha)$, where, as presented, we subsample the whole gradient, and an alternative, $n e_i e_i^T(K\alpha) - y - b^{-1}\alpha$, where only the $K\alpha$ term is subsampled. While both are stable with much higher step-sizes than random features, the latter has worse asymptotic performance. This is a kind of \emph{Rao-Blackwellisation trap:} introducing the known value of $-y+b^{-1}\alpha$ in place of its estimate $ne_ie_i^T (-y + b^{-1}\alpha)$ destroys the multiplicative property of the noise, making things worse, not better.

\begin{figure}[t]
    \centering
    \includegraphics[width=5.5in]{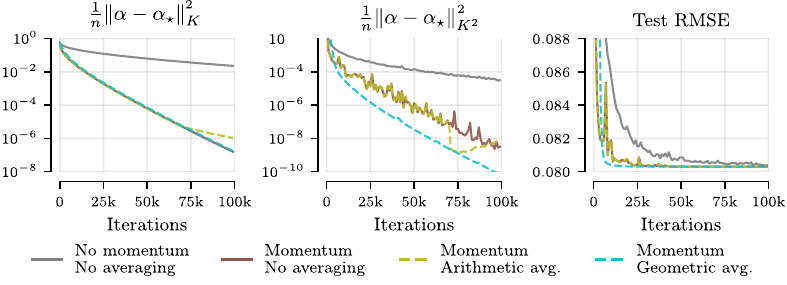}
    \caption{Comparison of optimisation strategies for random coordinate estimator of the dual objective on the \textsc{pol} data set, using momentum $\rho = 0.9$, averaging parameter $\chi = 0.001$, batch size $r=128$, and step-size $\beta n = 50$. Nesterov's momentum significantly improves convergence speed across all metrics. The dashed olive line, marked \emph{arithmetic averaging}, shows the regular iterate up until 70k steps, at which point averaging commences and the averaged iterate is shown. Arithmetic iterate averaging slows down convergence in $K$-norm once enabled. Geometric iterate averaging, on the other hand, outperforms arithmetic averaging and unaveraged iterates throughout optimisation.}\label{fig:toy-iterate-averaging}
\end{figure}

\subsection{Getting the optimiser \emph{right}}\label{subsec:the_right_optimiser}

Momentum, or acceleration, is a range of modifications to the usual gradient descent updates that aim to improve the rate of convergence, in particular with respect to its dependence on the curvature of the optimisation problem \citep{polyak1964some}.  
We use Nesterov's momentum \citep{nesterov1983method}, as adapted for deep learning by \cite{sutskever2013importance}, because it is readily available in standard deep learning libraries.
\Cref{alg:compute-mean} gives the precise updates. We use a momentum of $\rho=0.9$ throughout.
Comparing the plots in \Cref{fig:toy-iterate-averaging}, we see that momentum is vital on this problem, independently of iterate averaging.

When gradient updates are stochastic with additive noise and step-size is constant, the resulting iterates will bounce around the optimum rather than converging. 
In this setting, one can recover a convergent algorithm by arithmetically averaging the tail iterates, a procedure called \emph{Polyak--Ruppert averaging} \citep{polyak1990new,ruppert1988efficient,polyak1992acceleration}. 
While Polyak--Ruppert averaging is necessary with constant step-size and additive noise, it is not under multiplicative noise \citep{varre2021last}, and indeed can slow convergence. We recommend using \emph{geometric averaging} instead, where we let $\overline\alpha_0 = \alpha_0$ and, at each step, compute
\begin{equation*}
    \overline\alpha_t = \chi\alpha_t + (1-\chi) \overline\alpha_{t-1} \spaced{for an averaging parameter} \chi \in (0,1],
\end{equation*}
and return $\overline\alpha_T$. Geometric averaging is an anytime approach. It does not rely on fixed averaging-window size, and thus can be used in combination with early stopping, and the value of $\chi$ can be tuned adaptively. Here and throughout, we set $\chi = 100/T$, for $T$ the total number of steps we perform.
\Cref{fig:toy-iterate-averaging} shows that geometric averaging outperforms both arithmetic averaging, and simply returning the last iterate $\alpha_T$ without any averaging.

 \begin{figure}[t]
    \centering
    \includegraphics{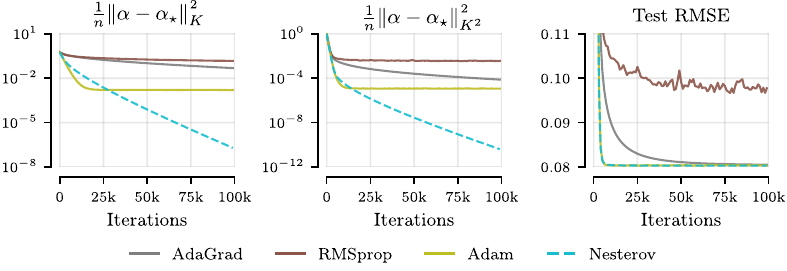}
    \caption{Comparison of stochastic dual descent on \textsc{pol} with batch size $r = 512$ and averaging parameter $\chi = 0.001$, using different optimisers. While Adam and Nesterov perform similarly on Test RMSE, the latter has much closer convergence in both $K$-norm and $K^2$-norm}
    \label{fig:opt_comparison}
\end{figure}

\subsubsection{Demonstration: Comparing Polyak momentum and geometric averaging with popular optimisers}

In \cref{fig:opt_comparison}, we report the performance of different optimisers on the dual problem. While algorithms such as AdaGrad, RMSprop, and Adam are designed to tackle problems with a non-constant curvature, the problem of sampling from a GP posterior is quadratic. Here, Nesterov-type momentum is theoretically rate-optimal. For all optimisers, we tune the step-size in a range of $[0.01, 100]$ and report the best performance; Adam with $0.05$, AdaGrad with $10$, RMSProp with $0.05$, and Nesterov's momentum with $50$.  As predicted by the theory, Nesterov does best.

\section{SGD for inference with inducing points}\label{sec:inducing_SGD}

So far, our sampling objectives have presented linear cost in the dataset size.
In the large-scale setting, algorithms with costs independent of the dataset size are often preferable.
For GPs, this can be achieved through \emph{inducing point posteriors} \citep{titsias09,Hensman2013big}, reviewed in \cref{sec:inducing_point_VI}, to which we now extend SGD sampling.

Let $Z = (z_1, z_2, \dotsc, z_m) \in X^m$ be a set of $m \in \N$ inducing points.

By inspecting \cref{eq:titsias_predictive_mean_fn}, we see that the optimal inducing point mean $\mu_{f|Y}^{(Z)}$ can be written
\begin{align}
  \label{eqn:inducing-optim}
  \mu_{f|Y}^{(Z)}(\cdot)  = {K}_{(\cdot)Z}\alpha_\star = \sum_{j=1}^m \alpha_{*\, i} k(z_j,\cdot) \quad
  \alpha_\star = \argmin_{\alpha\in\R^m} \sum_{i=1}^n [B]_{ii} (y_i - {K}_{x_i Z}\alpha)^2 + \norm{\alpha}_{{K}_{ZZ}}^2,
\end{align}
and we can parameterise the uncertainty reduction term in the same way but with representer weights given by
\begin{gather}
  \label{eqn:inducing-samples-ht-optim}
  \argmin_{\alpha\in\R^m} \sum_{i=1}^n [B]_{ii} (f(x_i) + \eps_i- {K}_{x_i Z}\alpha)^2 + \norm{\alpha}_{{K}_{ZZ}}^2\\
\text{with} \quad f^{(Z)}(x_i) \sim\N(0, {K}_{x_i Z}{K}_{ZZ}^{-1}{K}_{Z x_i }) \spaced{and} \c{E}\sim\N(0, B^{-1}).\notag
\end{gather}

\begin{derivation}
     \textbf{Derivation of inducing point sampling objectives \cref{eqn:inducing-optim} and \cref{eqn:inducing-samples-ht-optim}.}\\

Both expressions are derived in the same way, with only the targets we regress against changing between the objective for the variational posterior mean and samples. We part from 
 the pathwise form of the 0-mean Kullback--Leibler-optimal inducing point GP
\begin{align*}
    (f^{(Z)}| Y)(\cdot) &=\\
    &= f(\cdot) + u^{(Z)}_{f|Y}(\cdot) - {K}_{(\cdot)Z} {K}_{ZZ}^{-1} {K}_{ZX } ({K}_{X Z} {K}_{ZZ}^{-1} {K}_{ZX } + B^{-1})^{-1}(f^{(Z)}(X ) + \c{E})
    \\
    &\quad \c{E}\sim\N(0,B^{-1})
    \qquad
    f\sim\GP(0, k)
    \qquad
    f^{(Z)}(\cdot) = {K}_{(\cdot)Z}{K}_{ZZ}^{-1}f(Z).
\end{align*}
and apply the Woodbury identity to obtain
  \begin{align}
     & {K}_{(\cdot)Z} {K}_{ZZ}^{-1} {K}_{ZX } ({K}_{X Z} {K}_{ZZ}^{-1} {K}_{ZX } + B^{-1})^{-1}(f^{(Z)}(X ) + \eps)
    \\
     & \quad={K}_{(\cdot)Z} ({K}_{X Z} B {K}_{X Z} + {K}_{ZZ}  )^{-1} {K}_{ZX } B(f^{(Z)}(X ) + \eps)
    \\
     & \quad= {K}_{(\cdot)Z} \alpha_\star
    .
  \end{align}
  Now, we recognize $({K}_{X Z} B {K}_{X Z} + {K}_{ZZ}  )^{-1} {K}_{ZX } B(f^{(Z)}(X ) + \eps) = \alpha_\star$ as the expression for the
optimiser of a ridge-regularised linear regression problem—see \cref{eq:linear_regression_loss}—with parameters $\alpha$, features $K_{XZ}$ , Gaussian noise of covariance $B^{-1}$, and regulariser curvature $K_{ZZ}$.
The targets are given by the random variable $(f^{(Z)}(X ) + \c{E})$.
\end{derivation}

\corr{Exact implementation of \cref{eqn:inducing-samples-ht-optim} is precluded by the need to draw prior samples from a Gaussian with covariance ${K}_{X Z}{K}_{ZZ}^{-1}{K}_{ZX }$. This would require inverting ${K}_{ZZ}$, which presents cubic cost in $m$ and may be poorly conditioned.} 
However, we identify this matrix as a Nyström \corr{(i.e. low rank for $m < n$)} approximation to $K$.
\corr{Thus, we can approximate \cref{eqn:inducing-samples-ht-optim} by replacing $f^{(Z)}\sim\GP(0, {K}_{(\cdot) , Z}{K}_{Z, Z}^{-1}{K}_{Z, (\cdot)})$ with $f\sim\GP(0, k(\cdot, \cdot))$, which can be, in turn, accurately approximated with random features \cref{eqn:rff_approx-prior}.
} The error in approximating $f^{(Z)}$ with $f$ is small when the number of inducing points $m$ is large and the inducing points are close enough to the data. That is, whenever the inducing point GP is a good approximation to the posterior GP.

\begin{figure}[t]
\includegraphics{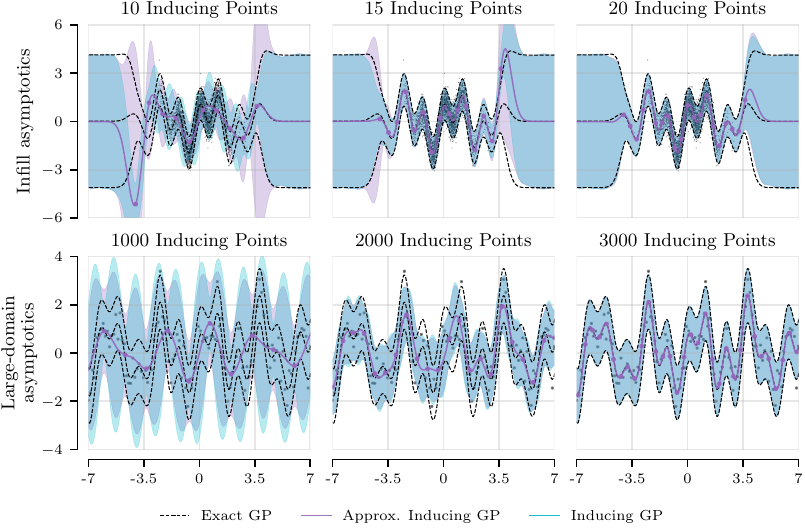}
\caption{Comparison of exact and approximate inducing point posteriors for a GP with squared exponential kernel and $10\text{k}$ data points generated using the true regression function $\sin(2x) + \cos(5x)$ under two different data-generation schemes: \emph{infill asymptotics}, which considers $x_i \sim\N(0,1)$, and \emph{large-domain asymptotics}, which considers $x_i$ on an evenly spaced grid with fixed spacing. We see that the approximation needed to apply inducing points is only inaccurate in situations where the inducing point posterior itself has significant error, which generally manifests itself as error bars that are larger than those of the exact posterior.}
\label{fig:inducing_point_approx}
\end{figure}

\begin{remark}\textbf{on the Error in the Nyström Approximation ${K}_{X , Z}{K}_{Z, Z}^{-1}{K}_{Z, X } \approx {K}$} \\
  \Cref{fig:inducing_point_approx} compares the KL-optimal inducing point posterior GP with that obtained when taking the prior function samples which we fit with the representer weighed evaluation functionals to be $f(X)$ with $f \sim \GP(0, k)$ instead of
  $f^{(Z)}(X ) = {K}_{X Z}{K}_{ZZ}^{-1}f(Z)$. This amounts to approximating the Nyström-type matrix ${K}_{X , Z}{K}_{Z, Z}^{-1}{K}_{Z, X }$ with its exact counterpart ${K}_{X , X }$.
  Both of these matrices become very similar if there is an inducing point placed sufficiently close to every data point.
  In practice, this tends to occur when an inducing point is placed within roughly a half-length-scale of every observation.
  This is effectively what is needed for inducing point methods to provide a good approximation of the exact GP.
  This is reflected in \Cref{fig:inducing_point_approx}, where we see that our approximate inducing point posterior differs from the exact inducing point posterior only in situations where the latter fails to be a good approximation to the exact GP in the first place.
  This manifests as the approximate method providing larger error bars.
  When the number of inducing points increases, both methods become indistinguishable from each other and the exact GP.
  Fortunately, the linear cost of SGD in the number of inducing points allows us to use a very large number of these in practice.
\end{remark}

We now turn to stochastic estimation of the inducing point sampling objectives \cref{eqn:inducing-optim} and \cref{eqn:inducing-samples-ht-optim}. Sadly, none of the tricks developed in this chapter are applicable.
Since the curvature of the data-fit term $K_{ZX} B K_{XZ}$ differs from that of the regulariser $K_{zz}$, we can not apply stochastic dual descent. Additionally, the data fit term is a sum of $n$ terms, while the regulariser is a sum of $m$ terms; we can not apply the random coordinate estimator either.
Finally, the low-variance estimator of \cref{subsec:variance_reduction_SGD}, would require sampling the noise in the regularisation term from $\N(0, K_{zz}^{-1} K_{ZX} B K_{XZ}K_{zz}^{-1})$, which is also intractable for large numbers of inducing points. 
With this, we apply the simple minibatching plus random feature stochastic estimator given in \cref{eqn:stochastic_mean-objective} to the inducing point sampling objective.

The inducing point objectives differ from those presented in previous sections in that there are $\c{O}(m)$ and not $\c{O}(n)$ learnable parameters, and we may choose the value of $m$ and locations $Z$ freely.
The cost of inducing point representer weight updates is thus $\c{O}(s m)$, where $s$ is the number of samples.

\subsubsection{Demonstration: inducing point SGD}
\begin{figure}[t]
\includegraphics[width=1\textwidth]{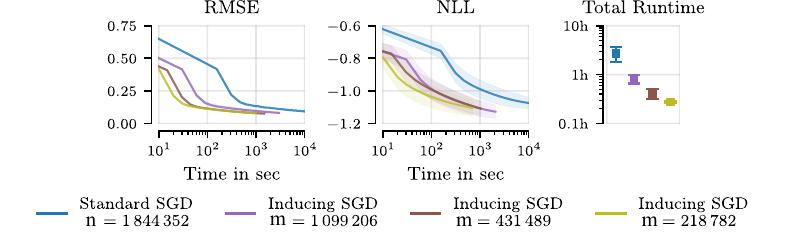}
\caption{Test RMSE and negative log-likelihood (NLL) obtained by SGD and its inducing point variants, for decreasing numbers of inducing points, given in the rightmost plot, as a function of time on an A100 GPU, on the \textsc{houseelectric} dataset ($n\approx2$M). }
\label{fig:grad_var_inducing_sgd}
\end{figure}

We demonstrate the inducing point variant of our method, on \textsc{houseelectric}, our largest dataset ($n{=}2$M).
We select varying numbers of inducing points from the pool of train points. In particular, we use a $K$-nearest-neighbour algorithm to find and eliminate the points nearest to other points (in terms of euclidean distance).
\Cref{fig:grad_var_inducing_sgd} shows the time required for $100\text{k}$ SGD steps scales roughly linearly with inducing points. It takes 68m for full SGD and 50m, 25m, and 17m for $m{=}1099\text{m}, \,728\text{k}$, and $ 218\text{k}$, respectively.
Performance in terms of RMSE and NLL degrades less than $10\%$ even when using~$218\text{k}$~points.

\begin{table}[ht]
\caption{\ Time to convergence (on an A100 GPU) and predictive performance for all approximate inference methods under consideration in this chapter, including inducing point SGD, on the \textsc{houseelectric} dataset. Experimental details are provided below, in \cref{sec:SGDGP_experiments}.}
\small
\setlength{\tabcolsep}{3pt}
\begin{tabular}{c c c c c c c}
\toprule
Model & \multicolumn{3}{c}{Inducing SGD} & Standard SGD & CG & SVGP \\
$m$ & 218\,782 & 431\,489 & 1\,099\,206 & 1\,844\,352 & 1\,844\,352 & 1\,024 \\
\midrule
RMSE & \textbf{0.08\,$\pm$\,0.00} & \textbf{0.08\,$\pm$\,0.00} & \textbf{0.08\,$\pm$\,0.01} & 0.09\,$\pm$\,0.00 & 0.87\,$\pm$\,0.14 & 0.10\,$\pm$\,0.02 \\
Hours & 0.28\,$\pm$\,0.01 & 0.41\,$\pm$\,0.09 & 0.83\,$\pm$\,0.17 & 2.69\,$\pm$\,0.91 & 2.62\,$\pm$\,0.01 & \textbf{0.04\,$\pm$\,0.00} \\
NLL & \textbf{-1.10\,$\pm$\,0.05} & \textbf{-1.11\,$\pm$\,0.04} & \textbf{-1.13\,$\pm$\,0.04} & \textbf{-1.09\,$\pm$\,0.04} & 2.07\,$\pm$\,0.58 & -0.94\,$\pm$\,0.13 \\
\bottomrule \\
\end{tabular}
\label{tab:inducing_regression}
\end{table}

\Cref{tab:inducing_regression} provides quantitative results for inducing point SGD on the \textsc{houseelectric} dataset.
SGD's time to convergence is shown to scale roughly linearly in the number of inducing points.
However, for this dataset, keeping only $10\%$ of observations as inducing points and thus obtaining $10\times$ faster convergence leaves performance unaffected.
This suggests the dataset can be summarised well by a small number of points.
Indeed, SVGP obtains almost as strong performance as SGD in terms of RMSE with only $1024$ inducing points.
SVGP's NLL is weaker however, which is consistent with known issues of uncertainty overestimation when using a too small amount of inducing points.
On the other hand, the large and potentially redundant nature of this dataset makes the corresponding optimisation problem ill-conditioned, hurting CG's performance.

\section{Analysing the implicit bias of stochastic gradient descent}\label{sec:implicit_bias}

We have detailed an SGD-based scheme for obtaining approximate samples from a posterior Gaussian process.
Despite SGD's significantly lower cost per-iteration than CG, its convergence to the true optima, shown in \Cref{fig:exact_vs_low_noise}, is much slower in both Euclidean representer weight space, and the reproducing kernel Hilbert space (RKHS) induced by the kernel.
Despite this, the predictions obtained by SGD are very close to those of the exact GP, and effectively achieve the same test RMSE.
Moreover, \Cref{fig:error-and-eigenfunctions} shows the SGD posterior on a 1D toy task exhibits error bars of the correct width close to the data, and which revert smoothly to the prior far away from the data.
Empirically, differences between the SGD and exact posteriors concentrate at the borders of data-dense regions.

\begin{figure}[t]
\includegraphics{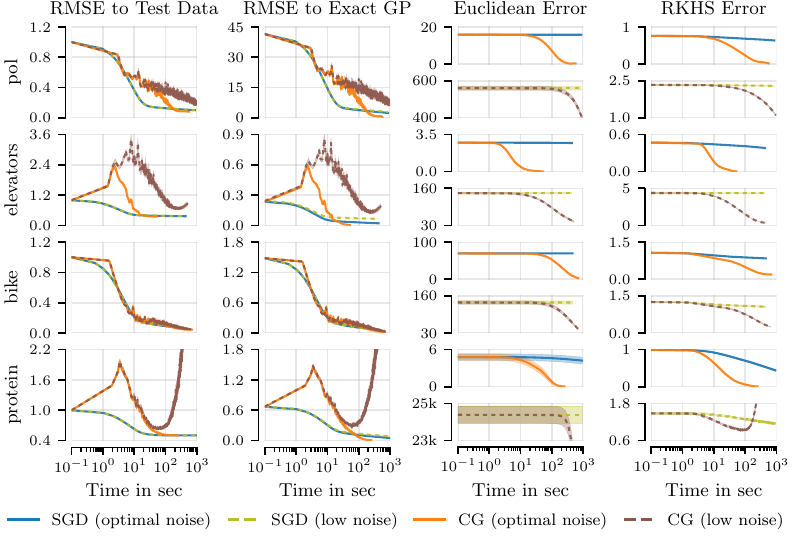}
\caption{
Convergence of the GP posterior mean with SGD and CG as a function of time (on an A100 GPU) on the 
\textsc{pol} ($N \approx 15\text{k}$),
\textsc{elevators} ($N \approx 16\text{k}$),
\textsc{bike} ($N \approx 17\text{k}$) and
\textsc{protein} ($N \approx 46\text{k}$) datasets, while setting the noise scale to (i) maximise exact GP marginal likelihood and (ii) to $10^{-3}$, labelled \emph{low noise}. We plot, in left-to-right order, test RMSE, RMSE to the exact GP mean at the test inputs, which is related to the $K^2$ norm $\|\alpha - \alpha_\star\|_{K^2}$, representer weight euclidean error $\norm{\alpha - \alpha^{*}}$, and RKHS error $\norm[0]{\mu_{f|Y} - \mu_{\text{SGD}}}_{\c{H}} = \|\alpha - \alpha_\star\|_{K}$, i.e. $K$ norm. In the latter two plots, the low-noise setting is shown on the bottom.
}
\label{fig:exact_vs_low_noise}
\end{figure}

We now argue the behavior seen in \Cref{fig:error-and-eigenfunctions} is a general feature of SGD: one can expect it to obtain good performance even in situations where it does not converge to the exact solution.
Consider posterior function samples in pathwise form, namely $(f| Y)(\cdot) = f(\cdot) + {K}_{(\cdot)X }\alpha$, where $f\sim\GP(0, k)$ is a prior function sample and $\alpha$ are the learnable representer weights.
We characterise the behavior of SGD-computed approximate posteriors by splitting the input space $\c{X}$ into 3 regions, which we call the \emph{far-away}, \emph{interpolation}, and \emph{extrapolation} regions.
This is done as follows.

\paragraph{\textnormal{\emph{(I) The Far-away Region.}}}
This corresponds to points sufficiently distant from the observed data.
Here, for kernels that decay over space, the evaluation functionals $k(x_i, \cdot)$ go to zero.
Thus, both the true posterior and any approximations formulated pathwise revert to the prior. More precisely, let $X = \R^d$, let $k$ satisfy $\,\lim_{c \to\infty} k(x', c\cdot x) = 0$ for all $x'$ and $x$ in $\c{X}$, and let $(f| Y)(\cdot)$ be given by $(f| Y)(\cdot) = f(\cdot) + {K}_{(\cdot)X }\alpha$, with $\alpha \in \R^n$. Then, for any fixed $\alpha$, any choice of $x \in \c{X}$, it follows immediately that $ \lim_{c\to\infty} (f| Y)(c\cdot x) = f(c\cdot x)$. Therefore, SGD cannot incur error in regions which are sufficiently far away from the data.
This effect is depicted in \Cref{fig:error-and-eigenfunctions}.

\begin{figure}[t]
\includegraphics{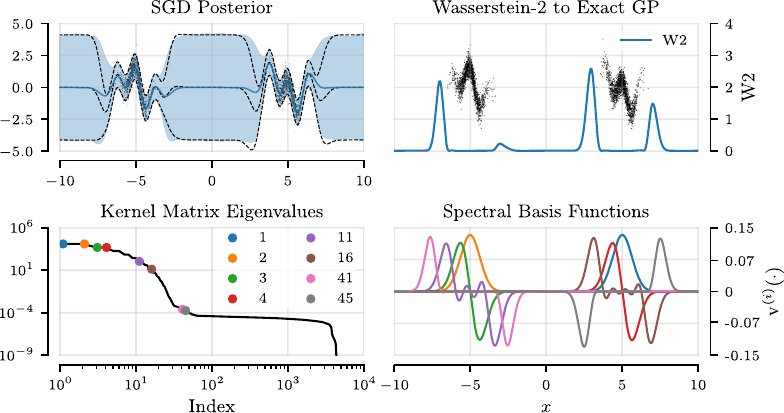}
\caption{SGD error and spectral basis functions. Top-left: SGD (blue) and exact GP (black, dashed) fit to a $n{=}10\text{k}$, toy regression dataset. Top-right: 2-Wasserstein distance (W2) between both processes' marginals. The W2 values are low near the data (interpolation region) and far away from the training data. The error concentrates at the edges of the data (extrapolation region). Bottom: The low-index spectral basis functions lie on the interpolation region, where the W2 error is low, while functions of index $10$ and larger lie on the extrapolation region where the error is large.}
\label{fig:error-and-eigenfunctions}
\end{figure}

\paragraph{\textnormal{\emph{(II) The Interpolation Region.}}}
This includes points close to the training data.
We characterise this region through subspaces of the RKHS, where we show SGD incurs small error.

Let $K = {V}\Lambda{V}^T$ be the eigendecomposition of the kernel matrix.
We index the eigenvalues $\Lambda = \text{diag}(\lambda_1,\dotsc, \lambda_n)$ in descending order.
Define the \emph{spectral basis functions} as eigenvector-weighed linear combinations of evaluation functionals
\begin{gather}
  v^{(i)}(\cdot) = \sum_{j=1}^n \frac{[V]_{ji}}{\sqrt{\lambda_i}} k(x_j, \cdot)
  .
\end{gather}
These functions are orthonormal with respect to the RKHS inner product.
To characterise them further, consider the following characterisation of eigenvalues and eigenvectors in the RKHS $\c{H}$ 
\begin{gather}
  v^{(i)}(\cdot) = \argmax_{v \in \c{H}} \cbr{\sum_{i=1}^n v(x_i)^2 : \norm{v}_{\c{H}} = 1, \innerprod[0]{v}{v^{(j)}} = 0, \forall j < i}.
\end{gather}
This tells us that the top spectral basis function, $v^{(1)}(\cdot)$, is a function of fixed RKHS norm---that is, of fixed degree of smoothness, as defined by the kernel $k$---which takes maximal values at the observations $x_1,..,x_n$.
Thus, $v^{(1)}$ will be large near clusters of observations.
The same will be true for the subsequent spectral basis functions, which also take maximal values at the observations, but are constrained to be RKHS-orthogonal to previous spectral basis functions.
\Cref{fig:error-and-eigenfunctions} confirms that the top spectral basis functions are indeed centred on the observed data.

Empirically, SGD matches the true posterior in the region of the top spectral basis functions, i.e. in the data dense regions.
We now formalise this observation by showing that SGD converges quickly in the directions spanned by spectral basis functions with large eigenvalues. For this, we consider the primal objective \cref{eqn:mean-optim} with minibatching for the data fit but no random feature estimation of the regulariser; this provides us with a sub-Gaussian additive noise estimator of the gradient. To simplify the analysis, we assume the use of arithmetic iterate averaging, as opposed to geometric averaging, and no momentum.
Let $\proj_{v^{(i)}}(\cdot)$ be the orthogonal projection onto the subspace spanned by $v^{(i)}$.

\begin{restatable}{proposition}{PropSGD}
  Let $\delta>0$.
  Let $B^{-1} = b^{-1}{I}$ for $b^{-1} > 0$.
  Let $\mu_{\text{SGD}}$ be the predictive mean function obtained by arithmetically-averaged SGD after $t$ steps, starting from an initial set of representer weights equal to zero, and using a sufficiently small learning rate of $0 < \beta <\frac{b^{-1}}{\lambda_1(\lambda_1 + b^{-1})}$.
  Assume the stochastic estimate of the gradient is $G$-sub-Gaussian.
  Then, with probability $1-\delta$, we have for $i=1,..,N$ that
  \begin{gather}
    \norm[1]{\proj_{v^{(i)}} \mu_{f|Y} - \proj_{v^{(i)}} \mu_{\text{SGD}}}_{\c{H}} \leq \frac{1}{\sqrt{\lambda_i t}}\del{\frac{b \norm{Y}_2}{\eta} + G\sqrt{ \frac{2}{t} \log\frac{N}{\delta}}}
    .
  \end{gather}
\end{restatable}
This is an extension of a standard result on the convergence of SGD \citep{LeCun1992Automatic} to the span of the spectral basis functions.
For the proof, as well as an additional pointwise convergence bound, and a variant that handles projections onto general subspaces spanned by basis functions, we refer to Appendix E of \cite{lin2023sampling}.
In general, we expect $G$ to be at most $\c{O}(\lambda_1^2 \norm{Y}_\infty)$ with high probability. An analogous result is straightforward to obtain for the dual gradient \cref{eq:dual_gradient}. It allows us to raise our learning rate to $0 < \beta <\frac{1}{\lambda_1 + b^{-1}}$.

The result extends immediately from the posterior mean to posterior samples.
As consequence, \emph{SGD converges to the posterior GP quickly in the data-dense region}, namely where the spectral basis functions corresponding to large eigenvalues are located.
Since convergence speed on the span of each basis function is independent of the magnitude of the other basis functions' eigenvalues, SGD can perform well even when the kernel matrix is ill-conditioned.
This is shown in \Cref{fig:exact_vs_low_noise}.

\paragraph{\textnormal{\emph{(III) The Extrapolation Region.}}}
This can be found by elimination from the input space of the far-away and  interpolation regions, in both of which SGD incurs low error.
Consider the spectral basis functions $v^{(i)}(\cdot)$ with small eigenvalues.
By orthogonality of $v^{(1)},..,v^{(N)}$, such functions cannot be large near the observations while retaining a prescribed norm.
Their mass is therefore placed away from the observations.
SGD converges slowly in this region, resulting in a large error in its solution in both a Euclidean and RKHS sense, as seen in \Cref{fig:exact_vs_low_noise}.
Fortunately, due to the lack of data in the extrapolation region, the excess test error incurred due to SGD nonconvergence may be low, resulting in benign nonconvergence \citep{ZouWBGK2021benign}.
\Cref{fig:error-and-eigenfunctions} shows the Wasserstein distance to the exact GP predictions is high in this region, as SGD tends to return small representer weights, thereby reverting to the prior.

\subsubsection{Demonstration: resilience to ill conditioning}

\begin{table}
  \scriptsize
  \setlength{\tabcolsep}{2.5pt}
  \renewcommand{\arraystretch}{1.1}
  \begin{tabular}{l c c c c c c c c c c}
    \toprule
    \multicolumn{2}{c}{Dataset} & \textsc{pol}                & \textsc{elevators}         & \textsc{bike}               & \textsc{protein}           & \textsc{keggdir}            & \textsc{3droad}             & \textsc{song}              & \textsc{buzz}              & \textsc{houseelec}          \\
    \multicolumn{2}{c}{$N$}     & 15000                       & 16599                      & 17379                       & 45730                      & 48827                       & 434874                      & 515345                     & 583250                     & 2049280                     \\
    \midrule
    \multirow{3}{*}{\rotatebox[origin=c]{90}{RMSE}}
 & SGD
 & 0.13\,$\pm$\,0.00 & 0.38\,$\pm$\,0.00 & 0.11\,$\pm$\,0.00 & \textbf{0.51\,$\pm$\,0.00} & 0.12\,$\pm$\,0.00 & \textbf{0.11\,$\pm$\,0.00} & \textbf{0.80\,$\pm$\,0.00} & 0.42\,$\pm$\,0.01 & \textbf{0.09\,$\pm$\,0.00} \\
 & CG
 & \textbf{0.08\,$\pm$\,0.00} & \textbf{0.35\,$\pm$\,0.00} & \textbf{0.04\,$\pm$\,0.00} & \textbf{0.50\,$\pm$\,0.00} & \textbf{0.08\,$\pm$\,0.00} & 0.15\,$\pm$\,0.01 & 0.85\,$\pm$\,0.03 & 1.41\,$\pm$\,0.08 & 0.87\,$\pm$\,0.14 \\
 & SVGP
 & 0.10\,$\pm$\,0.00 & 0.37\,$\pm$\,0.00 & 0.08\,$\pm$\,0.00 & 0.62\,$\pm$\,0.00 & 0.10\,$\pm$\,0.00 & 0.64\,$\pm$\,0.01 & 0.82\,$\pm$\,0.00 & \textbf{0.34\,$\pm$\,0.00} & 0.10\,$\pm$\,0.02 \\
\midrule
    \multirow{3}{*}{\rotatebox[origin=c]{90}{RMSE\textsuperscript{\ensuremath{\dagger}}}}
                              & SGD
                              & \textbf{0.13\,$\pm$\,0.00}  & \textbf{0.38\,$\pm$\,0.00} & 0.11\,$\pm$\,0.00           & \textbf{0.51\,$\pm$\,0.00} & \textbf{0.12\,$\pm$\,0.00}  & \textbf{0.11\,$\pm$\,0.00}  & \textbf{0.80\,$\pm$\,0.00} & \textbf{0.42\,$\pm$\,0.01} & \textbf{0.09\,$\pm$\,0.00}  \\
                              & CG
                              & 0.16\,$\pm$\,0.01           & 0.68\,$\pm$\,0.09          & \textbf{0.05\,$\pm$\,0.01}  & 3.03\,$\pm$\,0.23          & 9.79\,$\pm$\,1.06           & 0.34\,$\pm$\,0.02           & 0.83\,$\pm$\,0.02          & 5.66\,$\pm$\,1.14          & 0.93\,$\pm$\,0.19           \\
                              & SVGP
 & --- & --- & --- & --- & --- & --- & --- & --- & --- \\
    \bottomrule                                                                                                                                                                                                                                                                                         \\
  \end{tabular}
  \caption{Mean and std. err. of the test RMSE and low-noise test RMSE ($\dagger$) obtained by the GP predictive mean computed with SGD, CG and SVGP. The latter method is omitted for the low noise setting, where it fails to run. Metrics are reported for the datasets normalised to zero mean and unit variance. The full experimental setup is described below in \cref{sec:SGDGP_experiments}.}
  \label{tab:low_noise_regression}
\end{table}

We explore how this section's result affects the algorithms under consideration by setting a small isotropic noise variance of $b^{-1} = 10^{-6}$ and running them on our set of UCI regression datasets.
\Cref{tab:low_noise_regression} shows the performance of CG severely degrades on all datasets. SVGP diverges for all datasets.
SGD's results remain essentially-unchanged. This is because the noise only changes the smallest kernel matrix eigenvalues substantially and these do not affect convergence in the direction of the top spectral basis functions.
This mirrors results presented in \Cref{fig:exact_vs_low_noise}.

\section{Experiments and benchmarks}\label{sec:SGDGP_experiments}

We now turn to empirical evaluation of SGD GPs and SDD GPs.
We compare these with the two most popular scalable Gaussian process techniques: preconditioned conjugate gradient (CG) optimisation \citep{gardner18GPYtorch,WangPGT2019exactgp} and sparse stochastic variational inference (SVGP) \citep{titsias09,Hensman2013big}.
We employ the \textsc{jax.scipy} CG implementation and follow \cite{WangPGT2019exactgp} in using a pivoted Cholesky preconditioner of size 100. Our preconditioner implementation resembles the implementation of the \textsc{TensorFlow Probability} library.
For a small subset of datasets, we find the preconditioner to lead to slower convergence, and we report the results for conjugate gradients without preconditioning instead.
We employ the GPJax \citep{Pinder2022} SVGP implementation and initialise inducing point locations with the $K$-means algorithm.
In all SGD and SDD experiments, we use a Nesterov momentum value of $\rho=0.9$ and geometric averaging with $\chi=100/T$ for $T$ the total number of steps. The latter is chosen on a task-dependent basis.
For SGD, at each step, we draw $100$ random features to unbiasedly estimate the regulariser term.
When drawing posterior samples with all methods we use pathwise conditioning with $2000$ random Fourier features to draw each prior function.

\subsection{UCI benchmark datasets}

We first compare SGD-based predictions with baselines in terms of predictive performance, scaling of computational cost with problem size, and robustness to the ill-conditioning of linear systems.
Following \cite{WangPGT2019exactgp}, we consider 9 datasets from the UCI repository \citep{UCI_repo} ranging in size from $n=15\text{k}$ to $n \approx 2$M datapoints and input dimensionality from $d'=3$ to $d' =90$.
We report mean and standard deviation over five $90\%$-train $10\%$-test splits for the small and medium datasets, and three splits for the largest dataset.

\paragraph{GP hyperparameters}
\label{app:hyperparameters}
We use a zero prior mean function and the Matérn-$\nicefrac{3}{2}$ kernel, and share hyperparameters across all methods, including baselines.
For each dataset, we choose a homoscedastic Gaussian noise variance, a single kernel variance, and a separate length scale per input dimension.
For datasets with less than $50\text{k}$ observations, we tune these hyperparameters to maximise the exact GP marginal likelihood \cref{eq:GP_evidence}.
The cubic cost of this procedure makes it intractable at a larger scale: instead, for datasets with more than $50\text{k}$ observations, we obtain hyperparameters using the following procedure:
\begin{enumerate}
  \item  From the training data, select a \emph{centroid} data point uniformly at random.
  \item  Select the subset of $10\text{k}$ data points with the smallest Euclidean distance to the centroid.
  \item  Find hyperparameters by maximizing the exact GP marginal likelihood using this subset of data.
  \item  Repeat the preceding steps for $10$ different centroids, and average the resulting hyperparameters.
\end{enumerate}
This approach avoids aliasing bias \cite{barbano2022design} due to data subsampling and is tractable for large datasets.

\paragraph{Inference method hyperparameters}
We run SGD for $100\text{k}$ steps, with a fixed batch size of $512$ for both the mean function and samples.
For all regression experiments, we use a learning rate of $0.5$ to estimate the mean function representer weights, and a learning rate of $0.1$ to draw samples.
For SDD, we use step-sizes $100\times$ larger than SGD, except for \textsc{elevators}, \textsc{keggdirected} and \textsc{buzz}, where this causes divergence and we use $10\times$ larger step-sizes instead. We run CG to a tolerance of $0.01$, except for the 4 largest data sets, where we stop CG after 100 iterations---this still provides CG with a larger compute budget than first-order methods. 
For SVGP, we use $3,000$ inducing points for the smaller five data sets and $9,000$ for the larger four, so as to match the runtime of the other methods. For all methods, we estimate predictive variances for log-likelihood computations from 64 function samples drawn using pathwise conditioning.

\begin{table}[t]
\caption{Root mean square error (RMSE), compute time (on an A100 GPU), and negative log-likelihood (NLL), for 9 UCI regression tasks for all methods considered. We report mean values and standard error across five $90\%$-train $10\%$-test splits for all data sets, except the largest, where three splits are used. Targets are normalised to zero mean and unit variance. This work denoted by SDD.}
\label{tab:UCI_regression}
\vspace{8pt}
\centering
\scriptsize
\setlength{\tabcolsep}{2.5pt}
\renewcommand{\arraystretch}{1.1}
\begin{tabular}{l l c c c c c c c c c}
\toprule
\multicolumn{2}{c}{Data} & \textsc{pol} & \textsc{elevators} & \textsc{bike} & \textsc{protein} & \textsc{keggdir} & \textsc{3droad} & \textsc{song} & \textsc{buzz} & \textsc{houseelec} \\
\multicolumn{2}{c}{Size} & 15k & 17k & 17k & 46k & 49k & 435k & 515k & 583k & 2M \\
\midrule
\multirow{4}{*}{\rotatebox[origin=c]{90}{RMSE}}
 & SDD
 & \textbf{0.08\,$\pm$\,0.00} & \textbf{0.35\,$\pm$\,0.00} & \textbf{0.04\,$\pm$\,0.00} & \textbf{0.50\,$\pm$\,0.01} & \textbf{0.08\,$\pm$\,0.00} & \textbf{0.04\,$\pm$\,0.00} & \textbf{0.75\,$\pm$\,0.00} & \textbf{0.28\,$\pm$\,0.00} & \textbf{0.04\,$\pm$\,0.00} \\
 & SGD
 & 0.13\,$\pm$\,0.00 & 0.38\,$\pm$\,0.00 & 0.11\,$\pm$\,0.00 & 0.51\,$\pm$\,0.00 & 0.12\,$\pm$\,0.00 & 0.11\,$\pm$\,0.00 & 0.80\,$\pm$\,0.00 & 0.42\,$\pm$\,0.01 & 0.09\,$\pm$\,0.00 \\
 & CG
 & \textbf{0.08\,$\pm$\,0.00} & \textbf{0.35\,$\pm$\,0.00} & \textbf{0.04\,$\pm$\,0.00} & \textbf{0.50\,$\pm$\,0.00} & \textbf{0.08\,$\pm$\,0.00} & 0.18\,$\pm$\,0.02 & 0.87\,$\pm$\,0.05 & 1.88\,$\pm$\,0.19 & 0.87\,$\pm$\,0.14 \\
 & SVGP
 & 0.10\,$\pm$\,0.00 & 0.37\,$\pm$\,0.00 & 0.08\,$\pm$\,0.00 & 0.57\,$\pm$\,0.00 & 0.10\,$\pm$\,0.00 & 0.47\,$\pm$\,0.01 & 0.80\,$\pm$\,0.00 & 0.32\,$\pm$\,0.00 & 0.12\,$\pm$\,0.00 \\
\midrule
\multirow{4}{*}{\rotatebox[origin=c]{90}{Time (min)}}
 & SDD
 & 1.88\,$\pm$\,0.01 & 1.13\,$\pm$\,0.02 & 1.15\,$\pm$\,0.02 & 1.36\,$\pm$\,0.01 & 1.70\,$\pm$\,0.00 & \textbf{3.32\,$\pm$\,0.01} & \textbf{185\,$\pm$\,0.56} & \textbf{207\,$\pm$\,0.10} & \textbf{47.8\,$\pm$\,0.02} \\
 & SGD
 & 2.80\,$\pm$\,0.01 & 2.07\,$\pm$\,0.03 & 2.12\,$\pm$\,0.04 & 2.87\,$\pm$\,0.01 & 3.30\,$\pm$\,0.12 & 6.68\,$\pm$\,0.02 & 190\,$\pm$\,0.61 & 212\,$\pm$\,0.15 & 69.5\,$\pm$\,0.06 \\
 & CG
 & \textbf{0.17\,$\pm$\,0.00} & \textbf{0.04\,$\pm$\,0.00} & \textbf{0.11\,$\pm$\,0.01} & \textbf{0.16\,$\pm$\,0.01} & \textbf{0.17\,$\pm$\,0.00} & 13.4\,$\pm$\,0.01 & 192\,$\pm$\,0.77 & 244\,$\pm$\,0.04 & 157\,$\pm$\,0.01 \\
 & SVGP
 & 11.5\,$\pm$\,0.01 & 11.3\,$\pm$\,0.06 & 11.1\,$\pm$\,0.02 & 11.1\,$\pm$\,0.02 & 11.5\,$\pm$\,0.04 & 152\,$\pm$\,0.15 & 213\,$\pm$\,0.13 & 209\,$\pm$\,0.37 & 154\,$\pm$\,0.12 \\
\midrule
\multirow{4}{*}{\rotatebox[origin=c]{90}{NLL}}
 & SDD
 & \textbf{-1.18\,$\pm$\,0.01} & \textbf{0.38\,$\pm$\,0.01} & -2.49\,$\pm$\,0.09 & \textbf{0.63\,$\pm$\,0.02} & \textbf{-0.92\,$\pm$\,0.11} & \textbf{-1.70\,$\pm$\,0.01} & \textbf{1.13\,$\pm$\,0.01} & \textbf{0.17\,$\pm$\,0.06} & \textbf{-1.46\,$\pm$\,0.10} \\
 & SGD
 & -0.70\,$\pm$\,0.02 & 0.47\,$\pm$\,0.00 & -0.48\,$\pm$\,0.08 & 0.64\,$\pm$\,0.01 & -0.62\,$\pm$\,0.07 & -0.60\,$\pm$\,0.00 & 1.21\,$\pm$\,0.00 & 0.83\,$\pm$\,0.07 & -1.09\,$\pm$\,0.04 \\
 & CG
 & \textbf{-1.17\,$\pm$\,0.01} & \textbf{0.38\,$\pm$\,0.00} & \textbf{-2.62\,$\pm$\,0.06} & \textbf{0.62\,$\pm$\,0.01} & \textbf{-0.92\,$\pm$\,0.10} & 16.3\,$\pm$\,0.45 & 1.36\,$\pm$\,0.07 & 2.38\,$\pm$\,0.08 & 2.07\,$\pm$\,0.58 \\
 & SVGP
 & -0.67\,$\pm$\,0.01 & 0.43\,$\pm$\,0.00 & -1.21\,$\pm$\,0.01 & 0.85\,$\pm$\,0.01 & -0.54\,$\pm$\,0.02 & 0.60\,$\pm$\,0.00 & 1.21\,$\pm$\,0.00 & 0.22\,$\pm$\,0.03 & -0.61\,$\pm$\,0.01 \\
\bottomrule
\end{tabular}

\end{table}

\paragraph{Results}
The results, reported in \Cref{tab:UCI_regression}, show that SDD matches or outperforms all baselines on all UCI data sets in terms of root mean square error of the mean prediction across test data. SDD strictly outperforms SGD on all data sets and metrics, matches CG on the five smaller data sets, where the latter reaches tolerance, and outperforms CG on the four larger data sets. The same holds for the negative log-likelihood metric (NLL), except on \textsc{bike}, where CG marginally outperforms SDD. 
Since SDD requires only one matrix-vector multiplication per step, as opposed to two for SGD, it provides about $30\%$ wall-clock time speed-up relative to SGD.
Although we run SDD for 100k iterations to match the SGD baseline, SDD often converges earlier than that.

\begin{remark} \textbf{Why SGD and SDD outperform CG on large problems}\\
SGD and SDD present two key advantages which make them perform well on very large scale tasks. The first is their relative insensitivity to problem conditioning (see \cref{sec:implicit_bias}). We only expect SGD to converge in the direction of the top eigenvectors of the curvature matrix. Poor conditioning will make it converge slowly in the bottom eigendirections, but optimisation noise would prevent convergence in those directions anyway. On the other hand, CG's runtime is heavily determined by conditioning (see \cref{sec:CG_limiations}).
The second is SGD and SDD's compatibility with early stopping.
From \Cref{fig:exact_vs_low_noise}, we see that SGD makes the vast majority of its progress in prediction space in its first few iterations, improving roughly monotonically with the number of steps.
Thus, early stopping after $100\text{k}$ iterations incurs only moderate errors.
In contrast, CG monotonically decreases euclidean error and error measured in the RKHS norm but its initial steps actually increase test error (which is more related to the $K^2$ norm), resulting in very poor performance if stopped too early.
\end{remark}

\subsection{Large-scale Bayesian optimisation}

A fundamental goal of scalable Gaussian processes is to produce uncertainty estimates useful for sequential decision making.
Motivated by problems in large-scale recommender systems, where both the initial dataset and the total number of users queried are simultaneously large \citep{rubens2015active, elahi2016survey}, we benchmark SGD on a large-scale Bayesian optimisation task.
We draw a target function from a GP prior $f \sim\GP(0, k)$ and optimise it on $\c{X} = [0, 1]^{d'}$ using parallel Thompson sampling \citep{hernandezlobato2017Parallel}, which we described in \cref{subsec:decision_making}.
We use an acquisition batch size of $1000$ samples, and maximise them with a multi-start gradient descent-based approach described in \Cref{subsec:decision_making}\footnote{Please refer to Appendix A.3 of \cite{lin2023sampling} for more details on our function maximisation strategy.}. We run 30 acquisition steps, acquiring a total of $30k$ observations.
We set the search space dimensionality to $d' = 8$, the largest considered by \cite{wilson20}, and initialise all methods with the same dataset of $50\text{k}$ observations sampled uniformly at random from $\c{X}$.
To eliminate model misspecification confounding, we use a Matérn-$\nicefrac{3}{2}$ kernel and consider length scales of $(0.1, 0.2, 0.3, 0.4, 0.5)$ for both the target function and our models. For each length scale, we repeat the experiment for 10 seeds.

In large-scale Bayesian optimisation, training and posterior function optimisation costs can become significant, and predictions may be needed on demand.
For this reason, we include two variants of the experiment, one with a small compute budget, where SGD and SDD are run for 15k steps, SVGP is given 20k steps and CG is run for 10 steps, and one with a large budget, where all methods are run for 5 times as many steps. We present the results on this task, broken down by lengthscale value, in \cref{fig:bayesopt-task}. In both large and small compute settings, and across lenghtscales, SDD makes the most progress, in terms of maximum value found, while using the least compute. Unlike SVGP and CG, The performance of SDD and SGD degrades gracefully when compute budget is limited.
Here, SVGP performs well---on par with SGD---in the large length scale setting, where many observations can likely be summarised with 1024 inducing points.
CG suffers from slow convergence due to ill-conditioning here.
On the other hand, CG performs on par with SGD in the better-conditioned small length scale setting, while SVGP suffers.
In the large compute setting, all methods perform similarly per acquisition step for all length scales except the small one, where SVGP suffers.

\begin{figure}[htb]
    \centering
    \includegraphics[width=0.8\textwidth]{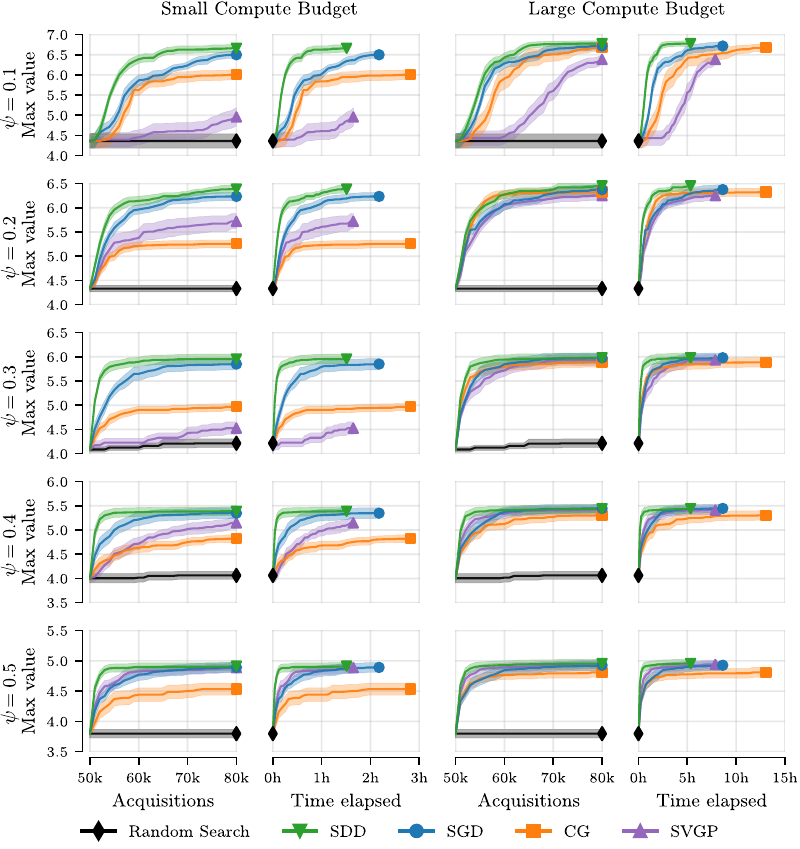}
    \caption{Maximum function values, with mean and standard error across 10 seeds, obtained by parallel Thompson sampling, for functions with different length-scales $\psi$, plotted as functions of acquisition steps and the compute time on an A100 GPU. All methods share an initial data set of 50k points, and take 30 Thompson steps, acquiring a batch of 1000 points in each. The algorithms perform differently across the length-scales: CG performs better in settings with smaller length-scales, which give better conditioning; SVGP tends to perform better in settings with larger length-scales and thus higher smoothness; SGD and SDD perform well in both settings. }
    \label{fig:bayesopt-task}
\end{figure}
\clearpage

\subsection{Molecule-protein binding affinity prediction}\label{expt:molecules}

The binding affinity between a molecule and certain proteins is a widely used preliminary filter in drug discovery \citep{pinza2019docking}, and machine learning is increasingly used to estimate this quantity \citep{yang2021efficient}. In this final experiment, we show that Gaussian processes with SDD are competitive with graph neural networks for binding affinity prediction.

\paragraph{Dataset setup}

We use the \textsc{dockstring} regression benchmark of \cite{Ortegon2022dockstring}, which contains five tasks, corresponding to five different proteins. 
The inputs are the graph structures of $250$k candidate molecules, and the targets are real-valued affinity scores from the docking simulator \emph{AutoDock Vina}~\citep{trott2010autodock}. 
We perform all of the preprocessing steps for this benchmark outlined by \cite{Ortegon2022dockstring}, 
including limiting the maximum docking score to $5$.
For each protein, we use a standard train-test splits of $210$k and $40$k molecules, respectively. These were produced by structure-based clustering to avoid similar molecules from occurring both in the train and test set.
We use Morgan fingerprints of dimension $1024$~\citep{rogers2010extended} to represent the molecules.%

\paragraph{Primer on fingerprints, Tanimoto Kernel and its random features}
Molecular fingerprints are a way to encode the structure of molecules
by indexing sets of subgraphs present in a molecule.
There are many types of fingerprints.
Morgan fingerprints represent the subgraphs up to a certain radius around each atom in a molecule \citep{rogers2010extended}.
The fingerprint can be interpreted as a sparse vector of counts, analogous to a `bag of words' representation of a document.
Accordingly, the Tanimoto coefficient $\mathfrak{T}(x,x')$, also called the Jaccard index, is a way to measure similarity between fingerprints, given by
\begin{equation*}
    \mathfrak{T}(x,x')=\frac{\sum_i\min(x_i,x'_i)}{\sum_i{\max(x_i,x'_i)}}.
\end{equation*}
This function is a valid kernel and has a known random feature expansion using random hashes \citep{Tripp2023tanimoto}. We use this kernel for our GPs.
The feature expansion builds upon prior work for fast retrieval of documents using random hashes that approximate the Tanimoto coefficient; that is, a distribution $P_\mathfrak{h}$ over hash functions $\mathfrak{h}$ such that
\begin{equation*}
    P_(\mathfrak{h}(\mathfrak{h}(x)=\mathfrak{h}(x')) = \mathfrak{T}(x,x')\ .
\end{equation*}
Per \cite{Tripp2023tanimoto}, we extend such hashes into random \emph{features} by using them to index a random tensor whose entries are independent Rademacher random variables. We use the random hash of \cite{ioffe2010improved}.

\paragraph{Gaussian process Setup}

As the Tanimoto kernel itself has no hyperparameters,
the only kernel hyperparameters are a constant scaling factor $a^{-1} > 0$ for the kernel,
the noise variance $b^{-1}$,
and a constant GP prior mean $\mu_0$ (the Gaussian process regresses on $y-\mu_0$ in place of $y$).
These were chosen by \cite{Tripp2023tanimoto} by maximising the evidence of an exact GP given a randomly chosen subset of the data and held constant during the optimisation of the inducing points.
The values are given in \cref{tab:molecule_regression_hyperparamters}. The same values are also used for SGD and SDD to ensure that the differences in accuracy are solely due to the quality of the GP posterior approximation.
The SGD method uses 100-dimensional random features for the regulariser.

\begin{table}[ht]
    \caption{Hyperparameters for all Gaussian process methods used in the molecule-protein binding affinity experiments of \cref{expt:molecules}.}\label{tab:molecule_regression_hyperparamters}
    \vspace{8pt}
    \centering
    \setlength{\tabcolsep}{2.5pt}
    \renewcommand{\arraystretch}{1.1}
    \begin{tabular}{l c c c c c}
    \toprule
    Data & \textsc{ESR2} & \textsc{F2} & \textsc{KIT} & \textsc{PARP1} & \textsc{PGR} \\
    \midrule
    $a^{-1}$   & 0.497 & 0.385 & 0.679 & 0.560 & 0.630 \\
    $b^{-1}$   & 0.373 & 0.049 & 0.112 & 0.024 & 0.332 \\
    $\mu_0$   & -6.79 & -6.33 & -6.39 & -6.95 & -7.08 \\
    \bottomrule \\
    \end{tabular}
\end{table}

\subsubsection{Results}

\begin{table}[t]
    \caption{Test set $R^2$ scores obtained for each target protein on the \textsc{dockstring} molecular binding affinity prediction task. Results with (\ensuremath{\cdot})\textsuperscript{\ensuremath{\dagger}} are from \cite{Ortegon2022dockstring}, those with (\ensuremath{\cdot})\textsuperscript{\ensuremath{\ddagger}} are from \cite{Tripp2023tanimoto}. SVGP uses 1000 inducing points. SDD denotes this work.}
    \label{tab:molecule_regression}
    \vspace{8pt}
    \centering
    \scriptsize
    \setlength{\tabcolsep}{2.5pt}
    \renewcommand{\arraystretch}{1.1}
    \hfill
\begin{tabular}{l c c c c c}
\toprule
Method & \textsc{ESR2} & \textsc{F2} & \textsc{KIT} & \textsc{PARP1} & \textsc{PGR} \\
\midrule
Attentive FP\textsuperscript{\ensuremath{\dagger}}      & \textbf{0.627} & \textbf{0.880} & \textbf{0.806} & \textbf{0.910} & \textbf{0.678} \\
MPNN\textsuperscript{\ensuremath{\dagger}}              & 0.506 & 0.798 & 0.755 & 0.815 & 0.324 \\
XGBoost\textsuperscript{\ensuremath{\dagger}}             & 0.497 & 0.688 & 0.674 & 0.723 & 0.345 \\
\bottomrule
\end{tabular}
\hfill
\begin{tabular}{l c c c c c}
\toprule
Method & \textsc{ESR2} & \textsc{F2} & \textsc{KIT} & \textsc{PARP1} & \textsc{PGR} \\
\midrule
SDD             & \textbf{0.627} & \textbf{0.880} & 0.790 & 0.907 & 0.626 \\
SGD               & 0.526 & 0.832 & 0.697 & 0.857 & 0.408 \\
SVGP\textsuperscript{\ensuremath{\ddagger}}     & 0.533 & 0.839 & 0.696 & 0.872 & 0.477 \\ 
\bottomrule
\end{tabular}
\hfill

\end{table}

In \cref{tab:molecule_regression}, following \cite{Ortegon2022dockstring}, we report $R^2$ values. Alongside results for SDD and SGD, we incldue results from \cite{Ortegon2022dockstring} for XGBoost, and for two graph neural networks, MPNN \citep{gilmer2017neural} and Attentive FP \citep{xiong2019pushing}, the latter of which is the state-of-the-art for this task. 
We also include the results for SVGP reported by \cite{Tripp2023tanimoto}. 
These results show that SDD matches the performance of Attentive FP on the ESR2 and FP2 proteins, and comes close on the others. 
To the best of our knowledge, this is the first time Gaussian processes have been shown to be competitive on a large-scale molecular prediction task.

\section{Discussion}

In this chapter, we explored using stochastic gradient algorithms to approximately compute Gaussian process posterior means and function samples at scale.
We derived optimisation objectives with linear and sublinear (via inducing points) cost for both. 
We studied variance reduction techniques for the posterior sampling objective and also its conditioning. The latter investigation led to the development of stochastic dual descent, a specialised first-order stochastic optimisation algorithm for  Gaussian processes. To design this algorithm, we combined a number of ideas from the optimisation literature with Gaussian-process-specific ablations, arriving at an algorithm which is simultaneously simple and matches or exceeds the performance of relevant baselines.
We showed that SGD can produce accurate predictions, even in cases when it is early stopped and does not converge to an optimum.
We developed a spectral characterisation of the effects of non-convergence, showing that it manifests itself mainly through error in an extrapolation region located away---but not too far away---from the observations.
We benchmarked SGD and SDD, showing they yield strong performance on standard regression benchmarks and on a large-scale Bayesian optimisation benchmark. SDD matches the performance of state-of-the-art graph neural networks on a molecular binding affinity prediction task.

Being able to perform posterior inference in large-scale linear models, through stochastic optimisation, is the first step towards tackling the ultimate goal of this thesis: performing Bayesian reasoning with large-scale neural networks. \cref{chap:adapting_laplace} will further pursue this goal by studying connections between linear models and neural networks through the linearised Laplace approximation. In particular, the next chapter will focus on hyperparameter optimisation using the model evidence. \corr{Building upon this, \cref{chap:sampled_Laplace} will use the methods introduced in this chapter to scale Bayesian inference to large-scale linearised NNs. For this, we will have to work with the non-kernelised setting, where SDD is not-applicable. We will rely on SGD instead.}

\chapter[A modernised Laplace approximation]{Adapting the linearised Laplace model evidence for modern deep learning}  %
\label{chap:adapting_laplace}

\ifpdf
    \graphicspath{{Chapter5/Figs/Raster/}{Chapter5/Figs/PDF/}{Chapter5/Figs/}}
\else
    \graphicspath{{Chapter5/Figs/Vector/}{Chapter5/Figs/}}
\fi

Model selection and uncertainty estimation are two important open problems in deep learning. The former aims to select network hyperparameters and architectures without costly cross-validation \citep{Mackay1992Thesis,Immer21Selection,immer2022invariance}. The latter provides a measure of fidelity of network predictions that can be used in downstream tasks such as experimental design \citep{barbano2022design}, sequential decision making \citep{Janz19successor}, and in safety-critical settings \citep{fridman2019arguing}.
\corr{This thesis does not attempt to compute exact Bayesian posterior credible regions or the exact model evidence for NNs. This is likely impossible when dealing with large-scale networks. Instead, we will sacrifice orthodoxy and pursue Bayesian-inspired methods that scale well and provide good results.
To this end, we focus on a classical approximate approach to these two problems: the linearised Laplace method \citep{Mackay1992Thesis},} which has recently been shown to be one of the best performing methods for \corr{approximate} inference in neural networks \citep{Khan19approximate,Kristiadi2020being,Immer2021Improving,daxberger2021bayesian,daxberger2021laplace}.

Linearised Laplace approximates the output of a neural network (NN) with a first order Taylor expansion (a linearisation) around optimal NN parameters. It then uses standard linear-model-type error bars to approximate the uncertainty in the output of the NN, while retaining the NN point-estimate as the predictive mean.
The latter feature means that, unlike other Bayesian deep learning procedures, the linearised Laplace uncertainty estimates do not come at the cost of the accuracy of the predictive mean  \citep{snoek2019can,ashukha2020pitfalls,antoran2020depth}. A downside of the method is that its uncertainty estimates are very sensitive to the choice of the prior precision hyperparameter \citep{daxberger2021bayesian}. Our work looks at the model evidence maximisation method for choosing this hyperparameters, as used in the seminal work of \citet{Mackay1992Thesis}.
In contrast with often used cross-validation, evidence maximisation reduces model selection to an (often convex) optimisation problem, and can scale to a large number of hyperparameters. 

The methods studied in this chapter differ from those of \cite{Mackay1992Thesis} in that we deal with the fully \emph{post-hoc setting}. In modern settings, retraining our NN every time we update the hyperparameters is prohibitively expensive. Thus, we work with a pre-trained NN and do not re-train it once the hyperparameters have been updated. This chapter also differs from the recent body of work of \cite{Immer21Selection,immer2022invariance,immer2023effective}, since the latter focuses on the \emph{online} setting, where the NN is trained and the hyperparameters are optimised concurrently.  We consider the post-hoc setting to be the one of most general interest, since it ensures compatibility with existing and future deep learning training techniques.

Our contributions, presented after a review of the necessary preliminaries in \cref{sec:c5_preliminaries}, are the identification of certain incompatibilities between the assumptions underlying the classical linearised Laplace model evidence and modern deep learning methodology, and a number of recommendations on how to adapt the method in light of these. In particular:
\begin{itemize}[topsep=0pt]
    \item A core assumption of linearised Laplace is that the point of linearisation is a minimum of the training loss. When the neural network is not trained to convergence (and this is almost never done), this does not hold and results in
    severe deterioration of the model evidence estimate. In \cref{sec:choice_of_posterior_mode}, we show that this can be corrected by instead considering the optima of the linearised model's loss, that is solving a quadratic optimisation problem.    
    \item In \cref{sec:normalised-networks}, we show that for networks with normalisation layers (such as batch norm \citep{ioffe2015batch}), the linearised Laplace predictive distribution can fail to be well-defined. However, this can be resolved by separately parametrising the prior corresponding to normalised and non-normalised network parameters. We also show that a standard feature-normalisation method, the g-prior \citep{zellner_1986,minka2000bayesian}, resolves this pathology.
\end{itemize}
We provide both theoretical and, in \cref{sec:adapting_experiments}, empirical justification for both points above. The resulting recommended procedure significantly outperforms a na\"ive linearised Laplace implementation on a series of standard tasks and a wide range of neural architectures:  MLPs, classic CNNs, residual networks with and without normalisation layers, generative autoencoders and transformers.

\section{Post-hoc linearised neural net hyperparameter selection}\label{sec:c5_preliminaries}

We consider the problem of selecting a Gaussian prior precision, hereon also referred to as the regulariser, with the objective of obtaining calibrated linearised Laplace uncertainty estimates. 
We go on to review the aspects of linearised Laplace that pertain to post-hoc selection of this hyperparameter. We refer the reader to \cref{sec:Laplace} for a detailed review of linearised Laplace.

\subsubsection{Setup and notation}

We consider the \emph{post-hoc} setting. We work with a neural network $g \colon  \c{V} \times \c{X} \mapsto \c{Y}$ with parameter space $ \c{V} \subseteq \R^d$, input space $\c{X}$ and output space $\c{Y} \subseteq \R^{c}$.
We assume access to a pre-trained set of weights $\tilde v \in  \c{V}$ which we keep fixed throughout the chapter.
Additionally, we assume these were obtained by minimising a regularised objective of the form
\begin{gather}
    \c{L}_{g, A}(v) =  L(g(v,\cdot)) + \| v \|^2_A,
\end{gather}
for $L \colon  \c{Y}^{ \c{V}\times\c{X}} \mapsto \R_+$ of the form  $L(g(v, \cdot)) = \sum_i^n \ell(y_i, g(v, x_i))$ where $\ell$ is a negative log-likelihood function.  We assume any linking functions are absorbed into $\ell$. $\| v\|^2_A$ corresponds to the log density of a Gaussian prior over $v$ for some initial value the of the positive-definite prior precision matrix $A \in \R^{d\times d}$. However, we henceforth treat $A$ as a model hyperparameter.  
Throughout this chapter we use $\cdot$ to denote by a
vector-matrix or matrix-matrix product where this may help
with clarity. 

\subsubsection{Linearisation and posterior approximation}
The parameter setting $\tilde v$ acts as the linearisation point around which we approximate $g$ with the affine function 
\begin{equation}\label{eq:h-def_c5}
    h(w, x) = g(\tilde v, x) + \partial_v [g(v, x)](\tilde v) \cdot ( w - \tilde v ),
\end{equation}
with parameters $w\in \R^d$. We then approximate the loss function for the linearised model, $ \c{L}_{h,  A}(w) = L(h(w, \cdot)) + \|w\|^2_ A$, with a second order Taylor expansion about $\tilde v$. Since $\partial_w h (\tilde v, \cdot) \;{=}\; \partial_v g(\tilde v, \cdot)$ and, by assumption, $\tilde v {\in} \argmin_{v} \c{L}_{g, A}$, we have that $\partial_w  \c{L}_{h,  A}(\tilde v) \:{=}\: \partial_v  \c{L}_{g,  A}(\tilde v) \:{=}\: 0$, and thus the first order term vanishes. This leaves us with the approximation
\vspace{-0.07cm}
\begin{equation}\label{eq:G_function}
    \c{L}_{h,  A}(\tilde v) + \frac{1}{2}\| w-\tilde v\|^{2}_{\partial^2_w  \c{L}_{h,  A}(\tilde v)}.
    \vspace{-0.05cm}
\end{equation}
We define an approximate posterior $Q$ by taking its Lebesgue density to be proportional to the exponential of minus this approximate loss. That is, a Gaussian with mean $\tilde v$ and covariance $(\partial^2_w  \c{L}_{h,  A}(\tilde v))^{-1}$. We henceforth adopt the notation of \cref{chap:linear_models} and \cref{sec:Laplace}, writing the Hessian of $ \c{L}_{h,  A}$ at $\tilde v$ as
\vspace{-0.05cm}
\begin{equation}
    \smash{\partial^2_w  \c{L}_{h,  A}(\tilde v) = M +  A} \,\,\,\,\text{with}\,\,\,\, M = \partial^2_w [L(h(w, \cdot))](\tilde v),
    \vspace{-0.05cm}
\end{equation}
and $J(\cdot) = \partial_v g(\tilde v, \cdot)$ for the Jacobian of $g$ at $\tilde v$. The approximate predictive posterior is given by the GP $h(w, \cdot)$, $w \sim Q$. Since $h$ is affine, this is again Gaussian. Its marginal at a test point $x' \in \c{X}$ is
\begin{gather}\label{eq:linearised_predictive_c5}
    \N(g(w_\star, x'), \, J(x') (M+A)^{-1} J(x')^T).
\end{gather}
\subsubsection{Model selection} 

The predictive posterior $h(w, \cdot)$, $w \sim Q$ corresponds to a GP with covariance kernel $J(\cdot)(M + A)^{-1}J(\cdot')^T$, revealing an explicit dependence on the regulariser $A$.  This parameter significantly affects the predictive posterior variance, but we have no simple method for choosing it \textit{a~priori}. We instead follow an empirical-Bayes procedure: we interpret $ A$ as the precision of a prior $\Pi = \N(0,  A)$ and choose $ A$ as that most likely to generate the observed data given the prior linearised model $h(w, \cdot);\, w \sim \Pi$. This yields the objective 
\begin{equation}\label{eq:model-evidence}
    \c{G}_{\tilde v}( A) = -\frac{1}{2}\left[ \|\tilde v\|^{2}_{ A}  + \log \det( A^{-1}M + I)\right] + C,
\end{equation}
where $C$ is independent of $A$. \corr{We have made explicit the objective's dependence on the linearisation point, which we assume fixed throughout optimisation of $A$, with the subscript $_{\tilde v}$.}
Equation \cref{eq:model-evidence} is called the model evidence.  
 Throughout, we will constrain $ A$ to the set of positive diagonal matrices, as in \citet{Mackay1992Thesis}. Maximising $ \c{G}_{\tilde v}$ is a concave optimisation problem. 

\subsubsection{Discussion: advantages and limitations of linearised Laplace in the modern setting}

The posterior predictive mean is fixed to match $g(\tilde v, \cdot)$, ignoring that a change in $A$ will almost surely change the modes of $\mathcal{L}_{g,A}$.
This choice keeps the NN's predictions unchanged. This is considered an advantage of linearised Laplace over competing Bayesian deep learning methods, which are often forced to compromise the accuracy of their predictive mean for better calibrated uncertainty.

We made a number of assumptions in our derivation. First, that the data-fit term $L$ is convex. This is satisfied by the standard losses used to train neural networks. We also assumed that the true posterior over the NN weights is sharply peaked around its optima such that it can be approximated well by a quadratic expansion and that $h$ is a good approximation to $g$ near the linearisation point. These assumptions we do not question further. We made one further important assumption, that the linearisation point $\tilde v$ is a local minimum of $ \c{L}_{g,  A}$ and thus it is also a  minima of the linearised loss $ \c{L}_{h,  A}$. This final assumption will be the focus of our work.

Since the linearised Laplace method with model-evidence maximisation was first introduced by \citet{Mackay1992Thesis}, deep learning training procedures and architectures have changed. 
Stochastic first order methods are used to minimise the loss function in place of the second order full-batch methods common in classical literature \citep{lecun1996effiicient,Amari2000info}. We  often do not use a low value of the loss $\c{L}_{g,  A}$ as a stopping criterion, but instead monitor some separate validation metric. Also, normalisation layers are ubiquitous. 

Since the derivations of this section assume that we linearise $g$ and expand $\c{L}_{h,  A}$ about a local minimum of $\c{L}_{g,  A}$ (and thus of $\c{L}_{h,  A}$), modern practises pose difficulties for the presented method. The rest of this chapter explores these issues, proposes a modern adaptation of the linearised Laplace method, and discusses some interesting special cases. 

\section{On the choice of posterior mode}\label{sec:choice_of_posterior_mode}

We consider a na\"ive implementation of the linearised Laplace method in the context of modern neural networks as that using the linearisation point $\tilde v$ in the expression for the model evidence $\c{G}_{\tilde v}$ \cref{eq:model-evidence}, even if this point is known to not be a local minimum of $\c{L}_{g,  A}$. 
We now propose an~alternative. %

We begin, as before, by linearising $g$ about  $\tilde v$, the point returned by (possibly stochastic or incomplete) optimisation of the neural network loss $\c{L}_{g,  A}$ and constructing the feature expansion $J\colon x \mapsto \partial_v g(\tilde v, \cdot)$. Under broad assumptions discussed in \cref{sec:dense-final-layer}, this choice means the posterior mean $g(\tilde v, \cdot)$ is contained within the linear span of the Jacobian features $J(\cdot)$\footnote{\corr{When the NN's output layer is linear, $g(\tilde v, \cdot)$ is a linear combination of the final layer activations and the Jacobian contains the last layer activations.}}. This yields credence to the interpretation of the linear model's error bars as uncertainty about the NN output.

We diverge from \cref{sec:c5_preliminaries} in how we approximate $\c{L}_{h,  A}$. We start by noting that since $\tilde v$ is not a local minimum of $\c{L}_{g,  A}$, it is not one of $\c{L}_{h,  A}$ either.

\begin{observation} \label{result:f-not-then-h-not} For network $g$ with linearisation $h$ about $\tilde v$ and a positive definite regulariser $ A$, if $\tilde v$ is not a stationary point of $\c{L}_{g,  A}$, it is not a local minimum of $\c{L}_{h, A}$.
\end{observation}

\begin{derivation} \textbf{Proof of \cref{result:f-not-then-h-not} } \\

\begin{proof}
Since $\tilde v$ is not a stationary point of $\c{L}_{g,  A}$, the gradient $\partial_v \c{L}_{g, A}(\tilde v)$ is not identically zero. But $\partial_v \c{L}_{g, A}(\tilde v)$ equal to
\begin{align*}
    \sum_i \partial_{\hat{y}_i} &[\ell(\hat{y}_i, y_i)](g(\tilde v, x_i))  \partial_v [g(v, x_i)](\tilde v) + \partial_v [\|v\|^2_{ A}](\tilde v)\\
    &= \sum_i \partial_{\hat{y}_i} [\ell(\hat{y}_i, y_i)](h(\tilde v, x_i))  \partial_w [h(w, x_i)](\tilde v) + \partial_w [\|w\|^2_{ A}](\tilde v),
\end{align*}
which is in turn equal to $\partial_w \c{L}_{h, A}(\tilde v)$. Since this is thus non-zero, $\tilde v$ cannot be a local minimum of $\c{L}_{h, A}$.
\end{proof}
\end{derivation}

Thus, $\tilde v$ is not a suitable point for a quadratic approximation to $\c{L}_{h, A}$ without a first order term. However, for any given $ A$, the loss for the linearised model $\c{L}_{h,  A}$ is a convex function of $w$ ($L$ is convex and $h$ is linear in $w$), and thus has a well defined minimiser; expanding the loss about this minimiser will yield a more faithful approximation to the evidence. 
Moreover, for each fixed $w$, $\c{G}_{w}( A)$ is concave in $ A$, yielding a maximiser. Iteratively minimising the convex $\c{L}_{h, A}(w)$ and maximising the concave $\c{G}_{w}( A)$ yields a simultaneous stationary point $(w_\star,  A_\star)$ satisfying  
\begin{equation*}
    \textstyle{w_\star \!\in \argmin_w \c{L}_{h, A_\star}(w) \!\!\spaced{and}\!\!  A_\star \!\in \argmax_ A \c{G}_{w_\star}( A).}
\end{equation*} 
Our adaption performs evidence maximisation with an affine model $h$ where the basis expansion $J$ is fixed.
Unlike \citet{Mackay1992Thesis}, we do not retrain the neural network. Instead, we re-fit the linear model. \corr{\cref{chap:sampled_Laplace} will introduce methods that efficiently implement this iterative optimisation scheme.}

\begin{figure}[t]
\vspace{-0.0cm}
    \centering
    \includegraphics[width=0.7\linewidth]{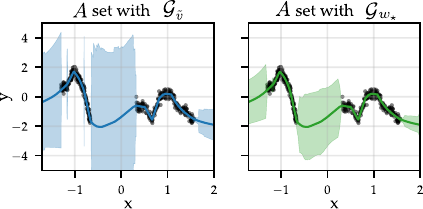}
    \caption{Linearised Laplace predictive mean and std-dev for a 2.6k parameter MLP trained on toy dataset from \citet{antoran2020depth}. Choosing $ A$ with $\c{G}_{\tilde v}$ yields error bars larger than the marginal std-dev of the targets. \Cref{rec:linear-weights} (using $\c{G}_{w_{\star}}$) solves~this. }
    \label{fig:lin_vs_nn}
\end{figure}

In practice, we make one further approximation: rather than evaluating the curvature $\partial^2_w L(h(w, \cdot))$ afresh at 
successive modes of $\c{L}_{h,  A}$ found during the iterative procedure for computing $(w_\star,  A_\star)$, we use the curvature at the linearisation point $M=\partial^2_w [L(h(w, \cdot))](\tilde v)$ throughout. This avoids the  expensive re-computation of the Hessian; experimentally we find that this does not affect the results\footnote{The Hessian depends on the linear model weights $w$ only trough the predictions made by the linear model. If our pre-trained NN is well-fit to the data, we do not expect the linearised NN's MAP predictions to differ much from the NN prediction at the linearisation point.} (see \cref{subsec:exp_assumptions}). The resulting model evidence expression matches that in \cref{eq:model-evidence}, with only the weights featuring in the norm~changed,
\begin{equation}\label{eq:model-evidence-new}
    \vspace{-0.05cm}
    \c{G}_{w_{\star}}( A) = -\frac{1}{2}\left[ \|w_{\star}\|^{2}_{ A}  + \log \det( A^{-1}M + I) \right] + C.
\end{equation}

\begin{recommendation}\label{rec:linear-weights}
While using the linearisation point $\tilde v$ in the construction of the feature expansion $J$ and the Hessian $M$ (as introduced in \cref{sec:c5_preliminaries}), find a joint optimum $(w_\star,  A_\star)$ for the feature-linear model and employ these to construct the corresponding model evidence $\c{G}_{w_\star}$ (equation \cref{eq:model-evidence-new}) and to compute the predictive variance (equation \cref{eq:linearised_predictive_c5}).
\end{recommendation}

\begin{figure}[htbp]
\vspace{-0.0cm}
    \centering
    \includegraphics[width=0.5\linewidth]{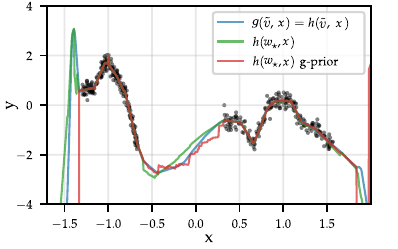}
    \caption{\corr{Comparison of the predictive means a a 2.6k parameter MLP trained on toy dataset from \citet{antoran2020depth} (blue) with the posterior mean of its tangent linear model with an isotropic Gaussian prior (green) and the posterior mean of the tangent linear model with the diagonal g-prior, introduced in \cref{subsec:g-prior} (red).}}
    \label{fig:toy_mean_functions}
\end{figure}

We thus recommend employing a posterior distribution for $h$ of the same form as given in \cref{eq:linearised_predictive_c5}, but with prior precision $ A_\star$. \corr{We do not recommend using the mean predictions of the tangent linear model $h(w_\star, \cdot)$ as the posterior mean function since this introduces additional computational load while empirically providing little to no benefit.
We verified this across a range of tasks, including image classification and tomographic image reconstruction. This is illustrated for a 1d toy problem in \cref{fig:toy_mean_functions}. Here, the linearised model's mean resembles the NN's mean but is less smooth. We attribute this non-smoothness to the inclusion of a linear dependence on ReLU features from network layers near the input. 
}

\subsubsection{Demonstration: 1d regression and a simple CNN}

\begin{figure}[t]
\vspace{-0.0cm}
    \centering
    \includegraphics[width=\linewidth]{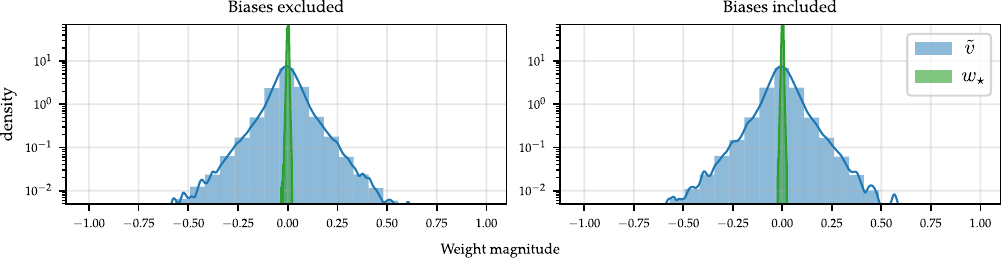}
    \caption{Histograms of the individual entries of $\tilde v$ and $w_{\star}$ for the models in the bias exclusion experiment of \cref{subsec:exp_assumptions}. We use a $d=46$k ResCNN described in \cref{sec:adapting_experiments} and train it on MNIST.}
    \label{fig:lin_vs_nn_hist}
\end{figure}

\cref{fig:lin_vs_nn} shows how choosing the prior precision with the evidence objective that contains the linearisation point $\tilde v$ results in uncertainty overestimation; the predictive distribution's marginal standard deviation is much larger than the marginal standard deviation of the targets. This is resolved by applying \cref{rec:linear-weights}.
We further explore the differences between the norm of the linear model MAP $w_\star$ and the linearisation point $\tilde v$ by, in \cref{fig:lin_vs_nn_hist}, plotting the histogram for both given a small ResNet-style CNN described in \cref{sec:adapting_experiments}. The linear model weights present a much narrower distribution around 0. This is commensurate with their use in the model evidence resulting in larger prior precisions and thus smaller errorbars. The plot also ablates whether considering model biases in the linearisation makes a difference to this recommendation, and it does not.

\section{Linearised Laplace with normalised networks} \label{sec:normalised-networks}

We now study linearised Laplace in the presence of \corr{scale-invariance introduced by} normalisation layers. For this, we put forth the following formalism:

\begin{definition}[Normalised networks] \label{def:normalised_networks}
We say that a set of networks $\mathfrak{G} \subset \c{Y}^{ \c{V}\times \c{X}}$ is normalised if $ \c{V}$ can be written as a \emph{direct sum} $ \c{V}' \oplus  \c{V}''$, with $ \c{V}''$ non-empty, such that for all networks $g \in \mathfrak{G}$ and parameters $v' + v'' \in  \c{V}' \oplus  \c{V}''$,
\begin{equation*}
   g(v' + v'', \,\cdot\,) = g(v' + \mathfrak{c}v'', \,\cdot\,)
\end{equation*}
for all $\mathfrak{c} \in \R_+$.
\end{definition}
Throughout, we write $v',v''$ to denote the respective projections onto $ \c{V}', \c{V}''$ of a parameter $v \in  \c{V}$; for ease of notation, we will assume these projections are aligned with a standard basis on $ \c{V}$. That is, we write our parameter vectors as the sum of $\mathfrak{c} v''$, which is only non-zero for normalised weights and $g$ is invariant to $\mathfrak{c}$, and $v'$, for which the opposite holds.  

\begin{remark} \textbf{Example to illustrate how \cref{def:normalised_networks}} \\ 

Consider an MLP $g:  \c{V} \times \mathcal{X} \to \mathcal{Y}$ with a single input dimension $\mathcal{X} = \mathbb{R}$, a single output dimension $\mathcal{Y} = \mathbb{R}$, single hidden layer and 2 hidden units $\c{V} = \R^4$. We apply layer norm after the input layer parameters. The model parameters are $v = [v_{1}, v_{2}, v_{3}, v_{4}] \in  \c{V}$ with $v_{1}, v_{2}$ belonging to the input layer and $v_{3}, v_{4}$ to the readout layer.  We assume there are no biases without loss of generality. 

Denoting the outputs of the first parameter layer $\mathfrak{a} = [v_{1}x, v_{2}x]$, layer norm applies the function 
\begin{gather*}
   \frac{ \mathfrak{a} - \mathbb{E}[\mathfrak{a}]}{\sqrt{\text{Var}(\mathfrak{a})}} \mathfrak{b} + \mathfrak{e} \spaced{with} \mathbb{E}[\mathfrak{a}] = 0.5 v_{1}x + 0.5 v_{2}x  \\ 
   \textrm{and} \quad \text{Var}(\mathfrak{a}) = 0.5 (v_{1}x - \mathbb{E}[\mathfrak{a}])^2 + 0.5 (v_{2}x - \mathbb{E}[a])^2
\end{gather*}
for $\mathfrak{b} \in \mathbb{R},\, \mathfrak{e} \in \mathbb{R}$. Now take $\mathfrak{c} \in \mathbb{R}_{+}$ to see that the output of the layernorm layer is invariant to scaling the input layer parameters by $\mathfrak{c}$
\begin{gather*}
   \frac{ \mathfrak{c} \mathfrak{a} - \mathbb{E}[\mathfrak{c} \mathfrak{a}]}{\sqrt{\text{Var}(\mathfrak{c} \mathfrak{a})}}  = 
   \frac{\mathfrak{c} (\mathfrak{a} - \mathbb{E}[\mathfrak{a}])}{\mathfrak{c}\sqrt{\text{Var}(\mathfrak{a})}} 
 =
   \frac{ \mathfrak{a} - \mathbb{E}[\mathfrak{a}]}{\sqrt{\text{Var}(\mathfrak{a})}}.
\end{gather*}
Thus, we have $g([v_{1}, v_{2}, v_{3}, v_{4}], \cdot) = g([\mathfrak{c}v_{1}, \mathfrak{c}v_{2}, v_{3}, v_{4}], \cdot)$.

Now let $ \c{V}$ be the result of the internal \underline{\emph{direct sum}} $ \c{V} =  \c{V}' \oplus  \c{V}''$, and $v', \, v''$ be projections of $v$ onto the subspaces $ \c{V}' \,\&\,  \c{V}''$, respectively, so that $v = v' + v''$. The operator $+$ is defined as the vector sum, as usual. In the simplest case where $ \c{V}'$ and $ \c{V}''$ are aligned with the standard basis, this corresponds to  vectors in $ \c{V}'$ having zero valued entries in the place of parameters to which normalisation is applied, $v' = [0, 0, v_{3}, v_{4}]$. Vectors in $ \c{V}''$ are non-zero for normalised parameters, $v'' = [v_{1}, v_{2}, 0, 0]$.
Finally, we write the property of interest
\begin{gather*}
    g(v' + \mathfrak{c} v'', \cdot) = g(v' +  v'', \cdot).
\end{gather*}
\end{remark}

Our formalism requires only that a single group of normalised parameters $ \c{V}''$ exists. However, by applying the definition repeatedly\corr{, introducing a separate scaling constant per layer,} we encompass networks with any number of normalisation layers, and all our results extend to this case. This formalism can be used to model the \corr{scale-invariant effect} of layer norm \citep{Ba2016layer}, group norm \citep{He2020group} or batch norm \citep{ioffe2015batch}, and even some so-called normalisation-free methods \citep{Brock21propagation,Brock2021free}. \corr{However, it is worth noting that each of these normalisation strategies introduce additional effects that are not of interest to this chapter and are deliberately not described by our formalism.}

\begin{figure}[tbh]
    \centering
    \includegraphics[width=\linewidth]{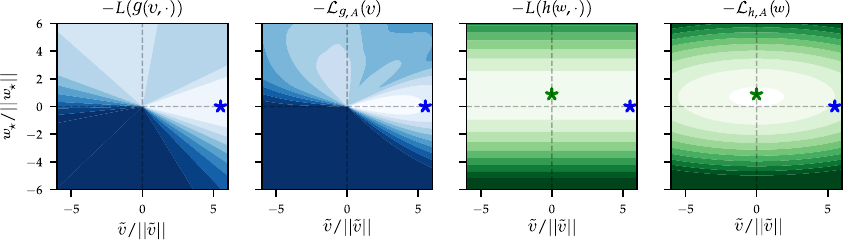}
    \caption{ Log likelihood $L(g(v, \cdot))$ (left) and log posterior $\c{L}_{g,A}$ (left middle) density for an MLP with layer norm, both plotted as functions of a 2d slice of the input layer weights.
    The horizontal axis corresponds to the direction of $\tilde v$ while the vertical to $w_{\star}$. The linearisation point found with SGD $\tilde v$ (\textcolor{blue}{$\bigstar$})  is not an optima of $\c{L}_{g,  A}$. We can always increase the value of  $\c{L}_{g,A}$ by moving towards the origin along the horizontal axis, without changing the likelihood.  $\tilde v$ is not an optima of the linear model's Log likelihood  $L(h(w, \cdot))$ (middle right) or log posterior $\c{L}_{g, A}$ (right) either. The linear model log posterior $\c{L}_{h,  A}$ is convex and optimised~by~$w_{\star}$~(\textcolor{ForestGreen}{$\bigstar$}).}
    \label{fig:contour}
\end{figure}

Our focus on normalised networks is motivated by the following observation:
\begin{proposition}\label{proposition:no-normalised-minimum}
For any normalised network $g$ and positive definite matrix $ A$, the loss $\c{L}_{g, A}$ has no local minima. 
\end{proposition}

To see this, note that the data term fit $L(g(v' + \mathfrak{c}v'', \cdot))$ is invariant to the choice of $\mathfrak{c}>0$, but we can always decrease the prior term $\|v'+ \mathfrak{c}v''\|^2_ A$ by decreasing $\mathfrak{c}$. Since $\mathfrak{c} \in \R_{+}$ has no minimal value,  $\c{L}_{g, A}(v'+ \mathfrak{c}v'') = L(g(v'+\mathfrak{c}v'', \cdot)) + \|v'+\mathfrak{c}v''\|^2_ A$ has no local minima. This is illustrated in \cref{fig:contour}.

As in \cref{sec:choice_of_posterior_mode}, minimisers of the linear loss $\c{L}_{h,  A}$ remain well-defined (the loss remains strictly convex). However, in this case, $\tilde v$ cannot minimise $\c{L}_{h, A}$: the linearisation point minimises $\c{L}_{h,  A}$ only if it minimises $\c{L}_{g,  A}$ (recall \cref{result:f-not-then-h-not}), and this is now impossible! To correct this, from hereon we follow \cref{rec:linear-weights}.

An even larger concern raised by \cref{proposition:no-normalised-minimum} is that the linearisation point is identified only up to the scaling $\mathfrak{c}$ of the normalised parameters $\tilde v''$. Since $\mathfrak{c}$ is arbitrary, and does not affect the predictions of the neural network (by definition), it ought not affect the predictive variance returned by the linearised Laplace method. \corr{However, due to scaling of the Jacobian features with $\mathfrak{c}$ which we go on to show in the following section, in general, it does.}
 See \cref{fig:k_scan} for a demonstration of this.

\subsection{The layerwise prior}\label{subsec:layerwise_prior}

As shown by the following proposition, it suffices to regularise linear model weights corresponding to the normalised parameters $w'' \in \c{V}''$ separately from $w'\in \c{V}'$, and choose both regularisation strengths with the model evidence \cref{eq:model-evidence-new}, to recover a unique predictive posterior independent of $\mathfrak{c}$.

\begin{figure}[t]
    \centering
    \includegraphics[width=0.7\linewidth]{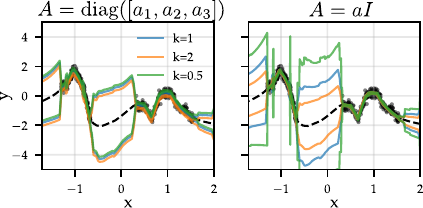}
    \caption{For a normalised MLP with an isotropic prior precision, modifying the scale of the normalised weights $\tilde v''$ in the linearisation point changes the error bars after hyper-parameter optimisation (right). Incorporating \cref{rec:factorised_prior} fixes the issue (left). }
    \label{fig:k_scan}
\end{figure}

\begin{proposition}\label{proposition:unique-posterior}
For normalised neural networks, using a regulariser of the form $\|w'\|^2_{ A'} + \|w''\|^2_{ A''}$ with $ A'$ and $ A''$ parametrised independently and chosen according to \cref{rec:linear-weights}, the predictive posterior $h(w,\cdot)$, $w \sim Q$ induced by a linearisation point $\tilde v'+\mathfrak{c} \tilde v''$ is independent of the choice of $\mathfrak{c}>0$.
\end{proposition}

Briefly, the result follows because the Jacobian entries corresponding to weights $\mathfrak{c} v''$ scale with $\mathfrak{c}^{-1}$.
This is illustrated in \cref{fig:contour} (leftmost plot), where
as we move further from the origin, weight settings of equal likelihood $L(g(v, \cdot))$ move further from each other. Given an un-scaled reference solution $(w_\star,  A_\star)$,
as we vary $\mathfrak{c}$ in $\mathfrak{c} \tilde v''$, the linear model weights and prior precisions 
that simultaneously optimise $\c{L}_{h,  A_{\star}}$ and~$\c{G}_{w_{\star}}$
 scale as $(w_\star', \mathfrak{c}w_\star'')$, and $( A_{\star}', \mathfrak{c}^{-2}  A''_\star)$ respectively. These scalings cancel each other in the predictive posterior, which remains invariant.
When $ A''$ can not change independently of $ A'$, this cancellation does not occur. An empirical demonstration is provided in \cref{fig:k_scan}. We now present the full proof.

\begin{derivation} \textbf{Proof of \cref{proposition:unique-posterior}} \\
    \paragraph{Notation} Consider a linearisation point $\tilde v' + \tilde v''$, with corresponding linearised function $h$, basis function $J$ and Hessian $M$. For $\mathfrak{c}>0$, $h_\mathfrak{c},J_\mathfrak{c},M_\mathfrak{c}$ denote these quantities corresponding to a linearisation point $\tilde v_{\mathfrak{c}} \coloneqq \tilde v' + \mathfrak{c} \tilde v''$. Moreover, we write
\begin{equation*}
    J = \begin{bmatrix} J' \\ J''
    \end{bmatrix} 
    \spaced{and} 
    M = \begin{bmatrix}
        M' & X^T \\
        X & M''
    \end{bmatrix}
\end{equation*}
for the sub-entries of $J$ and $M$ with dependencies on $v'$ and $v''$ respectively, with $X$ containing cross-terms. We refer to sub-entries of $J_\mathfrak{c}$ and $M_\mathfrak{c}$ in the same manner. 

With notation in place, we have the following scaling result:

\begin{lemma}\label{lemma:scaling}
Let $g$ be a normalised network and consider two alternative linearisation points $\tilde v = \tilde v' + \tilde v''$ and $\tilde v_\mathfrak{c} = \tilde v' + \mathfrak{c} \tilde v''$ for some $\mathfrak{c}>0$. Then,
\begin{equation*}
    \corr{\begin{bmatrix} J'_\mathfrak{c} \\ \mathfrak{c} J''_\mathfrak{c}
    \end{bmatrix} = J
    \spaced{and}
    \begin{bmatrix}
        M'_\mathfrak{c} & \mathfrak{c} X_\mathfrak{c}^T \\
        \mathfrak{c}X_\mathfrak{c} & \mathfrak{c}^2 M''_\mathfrak{c}
    \end{bmatrix} = M.}
\end{equation*}
Moreover, for all $w \in  \c{V}$, $h_\mathfrak{c}(w' + \mathfrak{c}w'', \cdot) = h(w' + w'', \cdot)$.
\end{lemma}

\begin{proof} 
First, we consider $J'_\mathfrak{c}$ and $J''_\mathfrak{c}$. For $J'_\mathfrak{c}$, take any $\tilde v' \in  \c{V}'$ and consider the directional derivative $D_{v'} g(\tilde v' + \mathfrak{c} \tilde v'')$. From the limit definition,
\begin{align*}
    D_{v'} g(\tilde v' + \mathfrak{c} \tilde v'', \cdot) 
    &= \lim_{\delta \downarrow 0} \frac{1}{\delta}\left[g ((\tilde v' + \delta v' + \mathfrak{c} \tilde v'' ), \cdot) - g ((\tilde v' + \mathfrak{c} \tilde v''), \cdot) \right] \\
    &= \lim_{\delta \downarrow 0} \frac{1}{\delta}\left[g ((\tilde v' + \delta v' + \tilde v'' ), \cdot) - g ((\tilde v' + \tilde v''), \cdot) \right] \\
    &= D_{v'} g(\tilde v' + \tilde v'', \cdot).
\end{align*}

From the Jacobian-product definition, we have $J'_\mathfrak{c} \cdot v' = J' \cdot v'$. Since $v' \in  \c{V}'$ was arbitrary and we are working on a finite-dimensional Euclidean space, this shows $J'_\mathfrak{c} = J'$. For $J''_\mathfrak{c}$, consider $D_{v''} g(\tilde v' + \mathfrak{c} \tilde v'')$ for $v'' \in  \c{V}''$ arbitrary. We have
\begin{align*}
    D_{v''} g(\tilde v' + \mathfrak{c} \tilde v'', \cdot) 
    &= \lim_{\delta \downarrow 0} \frac{1}{\delta}\left[g ((\tilde v' + \mathfrak{c} \tilde v'' + \delta v'' ), \cdot) - g ((\tilde v' + \mathfrak{c} \tilde v''), \cdot) \right] \\
    &= \lim_{\delta \downarrow 0} \frac{1}{\delta}\left[g (\tilde v' + \tilde v'' + \frac{\delta}{\mathfrak{c}} v''), \cdot) - g ((\tilde v' + \tilde v''), \cdot) \right] \\
    &= \frac{1}{\mathfrak{c}}\lim_{\delta' \downarrow 0} \frac{1}{\delta'}\left[g ((\tilde v' + \tilde v'' + \delta' v''), \cdot) - g ((\tilde v' + \tilde v''), \cdot) \right] \\
    &= \frac{1}{\mathfrak{c}}D_{v''} g(\tilde v' + \tilde v'', \cdot).
\end{align*}
Repeating the same argument as for $J'_\mathfrak{c}$, we obtain $J''_\mathfrak{c} = \frac{1}{\mathfrak{c}} J''$.

Now, we look at the scaling of $h_\mathfrak{c}$. By definition, using that $g$ is normalised and the previously derived scaling for $J_\mathfrak{c}$, 
\begin{align*}
    h_\mathfrak{c}(w' + \mathfrak{c}w'', \cdot) &= g(\tilde v' + \mathfrak{c} \tilde v'', \cdot) + J'_\mathfrak{c} (w' -\tilde v') + J''_\mathfrak{c} (\mathfrak{c}w'' - \mathfrak{c} \tilde v'')\\
    &= g(\tilde v' + \tilde v'', \cdot) + J'(w' -\tilde v') + J''(w'' -\tilde v'') \\
    &= h(w' + w'', \cdot),
\end{align*}
which is the claimed result.

For $M_\mathfrak{c}$, we examine it entry-wise. We have,
\begin{align*}
    [M_\mathfrak{c}]_{mn} 
    &= \partial_{w_m}\partial_{w_n} [L(h_\mathfrak{c}(w, \cdot))](\tilde v_\mathfrak{c}) \\
    &= \sum_i \partial_{w_m} [h_\mathfrak{c}(w, x_i)](\tilde v_\mathfrak{c}) \cdot \partial^2_{\hat{y_i}} [\ell(\hat{y}_i, y_i)](h_\mathfrak{c}(\tilde v_\mathfrak{c}, x_i)) \cdot \partial_{w_n} [h_\mathfrak{c}(w, x_i)](\tilde v_\mathfrak{c})
    \\ &\quad\quad\quad+ \sum_j \partial_{\hat{y_i}} [\ell(\hat{y}_i, y_i)](h_\mathfrak{c}(\tilde v_\mathfrak{c}, x_i)) \cdot \partial_{w_m} \partial_{w_n} [h_\mathfrak{c}(w, x_j)](\tilde v_\mathfrak{c}).
\end{align*}
Now since $h_\mathfrak{c}$ is affine, it has no curvature and thus $\partial_{w_m} \partial_{w_n} h_\mathfrak{c}(w, x)$ is identically zero for all $w\in\c{V}$ and $x \in \c{X}$. With that, the second term in the sum vanishes. For the first sum, consider the \emph{middle term}, the curvature of the negative log-likelihood function, and use $h_\mathfrak{c}(w' + \mathfrak{c}w'', \cdot) = h(w' + w'', \cdot)$ to see that it is invariant to $\mathfrak{c}$. Finally, note that $\partial_{w_m} h_\mathfrak{c}$ and $\partial_{w_n} h_\mathfrak{c}$ are entries of $J_\mathfrak{c}$ and inherit scaling from therein. Specifically, if both $w_m$ and $w_n$ belong to $ \c{V}''$, we obtain $\mathfrak{c}^2$ scaling; if just one belongs to $ \c{V}''$, we get $\mathfrak{c}$ scaling, and otherwise we obtain constant scaling. This completes the result for $M_\mathfrak{c}$.
\end{proof}

We now turn to how the optimal weights and regularisation parameters scale with the parameter $\mathfrak{c}$.

\begin{lemma}\label{lemma:stationary-points}
For $\mathfrak{c} > 0$, let $h_\mathfrak{c}$ be a linearisation of a normalised network $g$ about $\tilde v'+\mathfrak{c} \tilde v''$. Then $(w_\mathfrak{c},  A_\mathfrak{c})$ are an optima of the resulting objectives $( \c{L}_{h_\mathfrak{c}, A_\mathfrak{c}},  \c{G}_{w_\mathfrak{c}})$ respectively if and only if they are of the form 
\begin{equation*}
    (w_\mathfrak{c},  A_\mathfrak{c}) = (w'_\star + \mathfrak{c}w''_\star,  \: A'_\star + \mathfrak{c}^{-2} A''_\star)
\end{equation*}
where $(w_\star,  A_\star)$ are optima of $( \c{L}_{h, A_\star}, \c{G}_{w_\star})$ with $h$ a linearisation of $g$ about $\tilde v' + \tilde v''$.
\end{lemma}

\begin{proof} To prove the result, we will show that $ \c{L}_{h_\mathfrak{c},  A_\mathfrak{c}}(w'+\mathfrak{c}w'') =  \c{L}_{h, A_\star}(w' + w'')$ for all $w' + w'' \in  \c{V}$ and $ \c{G}_{w_\mathfrak{c}}( A' + \mathfrak{c}^{-2} A'') =  \c{G}_{w_\star}( A' +  A'')$ for all strictly diagonal positive matrices $ A', A''$ of compatible sizes. Then, the result follows by noting that for $\mathfrak{c}>0$ fixed, the mappings $w' + w'' \mapsto w' + \mathfrak{c}w''$ and $ A' +  A'' \mapsto  A' + \mathfrak{c}^{-2} A''$ are bijections.

Consider the objective $ \c{L}_{h_\mathfrak{c}, A_\mathfrak{c}}$. By definition, $ \c{L}_{h_\mathfrak{c}, A_\mathfrak{c}}(w' + \mathfrak{c}w'')$ is given by
\begin{align*}
    L(h_\mathfrak{c}(w' + &\mathfrak{c}w'', \cdot)) + \|w'\|^2_{ A'_\mathfrak{c}} + \|\mathfrak{c}w''\|^2_{ A''_\mathfrak{c}} \\
    &= L(h(w' + w'', \cdot)) + \|w'\|^2_{ A'_\star} + \|w''\|^2_{ A''_\star},
\end{align*}
where the equality follows by \cref{lemma:scaling} and the definition of $ A_\mathfrak{c}$. The bottom expression is equal to $ \c{L}_{h, A_\star}(w' + w'')$ proving the equality for the loss term.

Consider the objective $ \c{G}_{w_\mathfrak{c}}$. For our claim, we need to show that
\begin{align*}
    \|w_\mathfrak{c}\|^2_{ A' + \mathfrak{c}^{-2} A''} &+ \log \frac{\det(M_\mathfrak{c} +  A' + \mathfrak{c}^{-2} A'')}{\det( A' + \mathfrak{c}^{-2} A'')} \\ 
    &= \|w_\star\|^2_{ A' +  A''} + \log \frac{\det(M +  A' +  A'')}{\det( A' +  A'')}.
\end{align*}
The equality $\|w_\mathfrak{c}\|_{ A' + \mathfrak{c}^{-2} A''} = \|w_\star\|_{ A' +  A''}$ holds trivially. We now show equality of the determinants. Let $d',d''$ denote the dimensions of $ \c{V}'$ and $ \c{V}''$ respectively. By the Schur determinant lemma, the numerator $\det(M_\mathfrak{c} +  A' + \mathfrak{c}^{-2} A'')$ is equal to
\begin{align*}
    \det(M''_\mathfrak{c} + \frac{ A''_{d':}}{\mathfrak{c}^2})\,\det(M'_\mathfrak{c} +  A'_{:d'} - X(M''_\mathfrak{c} + \frac{ A''_{d':}}{\mathfrak{c}^2})^{-1}X^T),
\end{align*}
where $ A'_{:d'} = [ A'_{ij} \colon i,j \leq d']$ and $ A''_{d':}$ is defined similarly. Using \cref{lemma:scaling}, $\det(M''_\mathfrak{c} + \frac{ A''_{d':}}{\mathfrak{c}^2}) = (\frac{1}{\mathfrak{c}^2})^{d''} \det(M'' +  A''_{d':})$. Expanding the Schur complement term and using \cref{lemma:scaling} shows that it is independent of $\mathfrak{c}$. In turn, the denominator is given by
\begin{align*}
    \det( A' + \mathfrak{c}^{-2} A'') &= (\frac{1}{\mathfrak{c}^2})^{d''}\det( A'_{:d'})\,\det( A''_{d':}) \\
    &= (\frac{1}{\mathfrak{c}^2})^{d''} \det( A' +  A''),
\end{align*}
The $(\frac{1}{\mathfrak{c}^2})^{d''}$ terms in the numerator and denominator cancel, yielding the claim.
\end{proof}

\begin{proof}[Proof of \cref{proposition:unique-posterior}] 
Using \cref{lemma:scaling} and \cref{lemma:stationary-points} and the notation defined therein,
\begin{equation*}
    \|J\|^2_{(M +  A_\star)^{-1}} = \|J_\mathfrak{c}\|^2_{(M_\mathfrak{c} +  A_\mathfrak{c})^{-1}}.
\end{equation*}
Thus the errorbars induced by linearising about $\tilde v'+\tilde v''$ and $\tilde v'+\mathfrak{c} \tilde v''$ are equal for all $\mathfrak{c} > 0$.
\end{proof}

We note that \cref{proposition:unique-posterior} holds even when $ A_{\star}$ is found by evaluating the Hessian at the optima of the linear model loss instead of $w_{\star}$, instead of linearisation point $\tilde v$---the latter is our suggestion in \cref{sec:choice_of_posterior_mode}. This is because $w''_{\star}$ scales with $\tilde v''$ 
(\cref{lemma:stationary-points}). 

\end{derivation}

By induction, \cref{proposition:unique-posterior} applies to networks with multiple normalisation layers. Note that the proof of the results required for \cref{proposition:unique-posterior} depends crucially on being able to scale $ A''_\mathfrak{c}$ with $\mathfrak{c}$ while keeping $ A'_\mathfrak{c}$ fixed. This motivates our recommendation:
\begin{recommendation}\label{rec:factorised_prior}
When using the linearised Laplace method with a normalised network, use an independent regulariser for each normalised parameter group present.
\end{recommendation}
An example of a suitable regulariser for a network with normalised parameter groups $v^{(1)}, v^{(2)}, \dotsc, v^{(L)}$ and non-normalised parameters $v'$ would be
\begin{equation*}
     a'\|w'\|^2 +  a_1 \| w^{(1)}\|^2 +  a_2 \| w^{(2)}\|^2 + \dotsc +  a_L \| w^{(L)}\|^2
\end{equation*}
for independent parameters $ a',  a_1,  a_2, \dotsc,  a_L > 0$ and $w^{(1)}, w^{(2)}, \dotsc, w^{(L)}$ referring to the linear model weights corresponding to the NN weights in each normalised parameter group. Usually, this involves setting independent priors for each layer of the network. 

\subsection{The diagonal g-prior}\label{subsec:g-prior}

We now present a different class of diagonal prior which exploits the scaling of the likelihood curvature with the linearisation point (\cref{lemma:scaling}) to resolve the issue of scale indeterminacy in the predictive posterior.

\begin{proposition}\label{proposition:g-posterior}
For normalised neural networks, using a regulariser of the form $\|w\|^2_A$ with 
\begin{gather*}
    A = a \diag{M}
\end{gather*}
for $a \in \R_+$ and \corr{$M = \partial^2_w [L(h(w, \cdot))](\tilde v)$}, the predictive posterior $h(w,\cdot)$, $w \sim Q$ induced by a linearisation point $\tilde v'+\mathfrak{c} \tilde v''$ is independent of the choice of $\mathfrak{c}>0$.
\end{proposition}

\begin{derivation} \textbf{Proof of \cref{proposition:g-posterior}} \\ 

We adopt the notation used in \cref{lemma:scaling} and \cref{lemma:stationary-points}.

\begin{proof} 
Let $A = a \diag{M}$ for a model with linearisation point $\tilde v' + \tilde v''$.  By \cref{lemma:scaling}, for a model with linearisation point $\tilde v_{\mathfrak{c}} \coloneqq \tilde v' + \mathfrak{c} \tilde v''$, with $\mathfrak{c}>0$, the corresponding regulariser is 
\begin{gather*}
    A_{\mathfrak{c}} = a \begin{bmatrix}
        \diag M' & 0 \\
        0 & \mathfrak{c}^{-2} \diag  M''
        \end{bmatrix} = 
        a \begin{bmatrix}
        \diag M_{\mathfrak{c}}' & 0 \\
        0 & \diag M_{\mathfrak{c}}''
        \end{bmatrix}.
\end{gather*}
With that, \cref{lemma:scaling} and \cref{lemma:stationary-points}, we have
\begin{equation*}
  \|J\|^2_{(M +  A)^{-1}} = \|J_{\mathfrak{c}}\|^2_{(M_{\mathfrak{c}} +  A_{\mathfrak{c}})^{-1}}.
\end{equation*}
Thus the errorbars induced by linearising about $\tilde v'+\tilde v''$ and $\tilde v'+\mathfrak{c} \tilde v''$ are equal for all $\mathfrak{c} > 0$.
\end{proof}
\end{derivation}

This is a diagonal version of what is known in the literature as the g-prior \citep{zellner_1986} or scale-invariant prior \citep{minka2000bayesian}. It has the advantage over the layer-wise prior of only having one free parameter to learn via the evidence. Additionally, unlike the layerwise prior, the posterior corresponding to the g-prior is invariant to the scale of the linearisation point for any value of the free parameter $a\in \R_+$, not just for the one that maximises the evidence $\c{G}_{w_\star}$. A practical implementation must ensure that no entries of $\diag M$ are 0 to preserve positive definiteness in cases where the log-likelihood function is not strictly convex. A further advantage of the diagonal g-prior is that it normalises the scales of the Jacobian entries corresponding to different NN weights, as illustrated in \cref{fig:jac_scales}. For this reason, the diagonal g-prior may, in general, improve the conditioning of the linearised model's loss $\c{L}_{h,A}$.
Indeed, this prior is intimately related to the Jacobi preconditioner.

\begin{figure}[htb]
    \centering
   \includegraphics[width=0.95\textwidth]{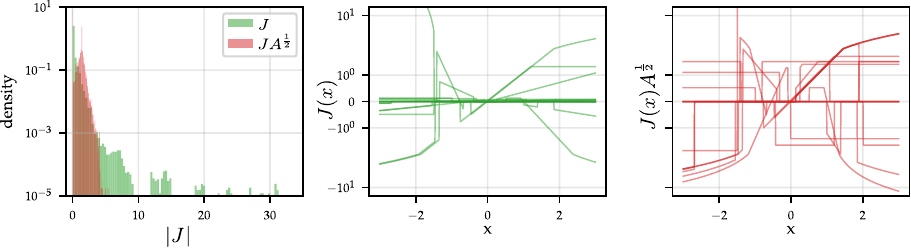}
    \caption{Left: Histogram of the absolute value of Jacobian entries, training data, across model weights and training datapoints. We use the NN depicted in \cref{fig:lin_vs_nn}, with and without g-prior scaling. Middle: 15 randomly chosen Jacobian basis functions. Right:  Same functions with g-prior scaling. }
    \label{fig:jac_scales}
\end{figure}

\begin{remark} \textbf{The history of the g-prior} \\
    The g-prior was originally introduced by \cite{zellner_1986}, It  consists of a centred Gaussian with covariance matching the inverse of the Fisher information matrix. Resultantly, the g-prior ensures inferences are independent of the units of measurement of the covariates \citep{minka2000bayesian}. Since then, it has extensively used in the context of model selection for generalised linear models \citep{Liang2008mixturesg,Sabanes2011hyperg,baragatti2012study}. In the large-scale setting, we have overcome the computational intractability of the Fisher by diagonalising the g-prior while preserving its scale-invariance property.
\end{remark}

\section{Additional observations and discussion}

The above analysis leads to a number of observations and further insights into linearised Laplace. The reader should note that the $\cdot$ notation will be doing some heavy lifting in terms of denoting Jacobian vector products taken such that their dimensions are compatible.

\subsection{Networks with a dense final layer}\label{sec:dense-final-layer}

We look at networks with a dense linear final layer, a (very general) special case. Letting  $\mathfrak{l}$ denote the number of non last layer weights such that $v_{:\mathfrak{l}}$ is the vector of all network parameters but those of the last layer, and $v_{\mathfrak{l}:}$ are the last layer weights, we deal with models of the form
\begin{equation}
    g(v, \cdot) =   \varphi(v_{:\mathfrak{l}}, \cdot) \cdot v_{\mathfrak{l}:},
\end{equation}
where $\varphi(v_{:\mathfrak{l}}, \cdot)$ is the output of the penultimate layer. The derivative of the neural network with respect to the dense final layer weights is
\begin{equation*}
    \partial_{v_{\mathfrak{l}:}} g(v, \cdot) = \varphi(v_{:\mathfrak{l}}, \cdot),
\end{equation*}
and thus the final layer activations $\varphi(v_{:\mathfrak{l}}, \cdot)$ are contained within the Jacobian matrix. 
Consequently, the neural network output $g(v, \cdot)$ is always contained in the linear span of the Jacobian basis.
This motivates \cref{rec:linear-weights}, where we argue for the use of $\tilde v$ for network linearisation, as it allows for an easy linear model error-bars interpretation for the resulting uncertainty. 

Also, the form of the linearised model $h$ simplifies in the dense final layer case when the network is fully normalised. Here $d''$, the dimension of $\c{V}''$, matches $\mathfrak{l}$, and thus we can write $g(v' + v'', \cdot) =   \varphi(v_{:\mathfrak{l}}'', \cdot) \cdot v_{\mathfrak{l}:}'$. The derivative of $\varphi$ in the direction of the linearisation point $\tilde v''_{:\mathfrak{l}}$ is zero
\begin{equation}\label{eq:directional-deriv-zero}
    \partial_{v''_{:\mathfrak{l}}} \varphi(\tilde v''_{:\mathfrak{l}}, \cdot) \cdot \tilde v''_{:\mathfrak{l}} = 0.
\end{equation}
Thus cancellation occurs in \cref{eq:h-def_c5}, as
\begin{align}\label{eq:simple_linear_model}
    h(w, \cdot) &=\varphi(v_{:\mathfrak{l}}, \cdot) \cdot v_{\mathfrak{l}:} + \partial_v g (\tilde v, \cdot) \cdot ( w - \tilde v ) \notag \\
    &=\varphi(v_{:\mathfrak{l}}, \cdot) \cdot v_{\mathfrak{l}:} + \partial_v g (\tilde v, \cdot) \cdot w - \partial_{v_{:\mathfrak{l}}} g(\tilde v, \cdot) \cdot v_{:\mathfrak{l}} - \partial_{v_{\mathfrak{l}:}} g(\tilde v, \cdot) \cdot v_{\mathfrak{l}:}  \notag \\
    &=\varphi(v_{:\mathfrak{l}}, \cdot) \cdot v_{\mathfrak{l}:} + \partial_v g (\tilde v, \cdot) \cdot w - \partial_{v_{:\mathfrak{l}}''} g(\tilde v, \cdot) \cdot v_{:\mathfrak{l}}'' - \varphi(v_{:\mathfrak{l}}, \cdot) \cdot v_{\mathfrak{l}:}  \notag \\
    &=  J(\cdot)  w, \quad w \sim Q.
\end{align}
That is, a linear model based on the features $J(\cdot) {=}\,\partial_v [g(v, \cdot)](\tilde v)$. This removes implementation complications that would stem from considering the zeroth order term in the affine linear model when it comes to finding the MAP of the linearised model $w_\star \in \argmin_{w\in \R^d} \c{L}_{h,A}(w)$.
 
\begin{derivation} \textbf{Jacobian null space of fully normalised networks} \\
    Consider a normalised network $g$, the linearisation point $\tilde v' + \tilde v''$,  and the directional derivative with respect to parameters $ v''$ in the direction of $\tilde v''$, denoted $D_{ \tilde v''} g(\tilde v' + \tilde v'', \cdot)$. On one hand, this is just the partial derivative of $g$ with respect to $v''$ evaluated at $\tilde v''$ and projected onto $\tilde v$, and thus $D_{\tilde v''} g(\tilde v' + c \tilde v'', \cdot) = \partial_{v''} g(\tilde v, \cdot) \cdot \tilde v$. On the other hand, from the limit definition of the directional derivative,
\begin{align*}
    D_{ \tilde v''} g(\tilde v' + c \tilde v'') 
    &= \lim_{\delta \downarrow 0} \frac{1}{\delta}\left[g (\tilde v' + (\delta+c) \: \tilde v'' , \cdot) - g (\tilde v' + c \tilde v'', \cdot) \right] \\
    &=0,
\end{align*}
and thus $\partial_{v''} g(\tilde v, \cdot) \cdot \tilde v = 0$. 

The quantity appearing in equation~\cref{eq:directional-deriv-zero} is $ \partial_{v''_{:\mathfrak{l}}} \varphi(\tilde v''_{:\mathfrak{l}}, \cdot) \cdot \tilde v''_{:\mathfrak{l}} $. We now observe that for a fully normalised network, each of the outputs of the penultimate layer $[\varphi(v_{:\mathfrak{l}}, \cdot)]_{i}$, with the output dimension being indexed by $i$, is a fully normalised network (\cref{def:normalised_networks}) in its own right. Hence, we can apply the same reasoning as above to see that $\partial_{v''_{:\mathfrak{l}}} \varphi(\tilde v''_{:\mathfrak{l}}, \cdot) \cdot \tilde v''_{:\mathfrak{l}} = 0$.
\end{derivation}

\subsection{Optimising linearised networks}
Our adapted linearised Laplace method requires identifying the joint stationary point $(w_{\star},  A_{\star})$. In general, this does not admit a closed-form solution. Instead, we alternate gradient-based optimisation of $\c{L}_{h,  A}$ and $\c{G}_{v}$. 
For normalised networks with dense output layers, implementing the simplified linear model \cref{eq:simple_linear_model} directly yields faster and more stable optimisation. 
Obtaining the gradients of $\c{G}_{v}$ involves computing Hessian log-determinants, which in turn requires approximations in the context of large networks. In this chapter's experiments (\cref{sec:adapting_experiments}), we will rely on the KFAC \citep{Martens15kron} approximation for this. In \cref{chap:sampled_Laplace}, we will introduce a more accurate sample-based approximation.
We go on to provide a derivation for the gradient of $\c{L}_{h,  A}$, \cref{alg:linear_MAP_algorithm} and discuss implementation trade-offs. 

\subsubsection{The linear model loss gradient}
We now discuss the optimisation of the loss for the predictor $h(w, \cdot) = J(\cdot) \cdot w$ where $w \in  \c{V}$ is the linear model's parameter vector. This corresponds to fully normalised networks with a dense final layer. We note that the procedure for the non-simplified Taylor expanded model $g(\tilde v, \cdot) + J(\cdot) \cdot (w - \tilde v)$ is analogous, but the targets are shifted to be  $Y-g(\tilde v, \cdot)+\Phi\tilde v$. Here, we denote NN Jacobians as $J(\cdot) = \partial_v g(\tilde v, \cdot) \in \R^{c\times d}$, we stack then across train points to produce the design matrix $\Phi \in \R^{nc\times d}$, and $c$ is the output dimensionality $| \c{Y}|$.

We wish to optimise $w$ according to the objective $\mathcal{L}_{h,  A}(w) = L(h(w, \cdot)) + \|w\|^{2}_{ A}$. We adopt a first order gradient-based approach. We first consider the gradient of $L(h(w, \cdot)) = \sum_{i}\ell( J(x_{i})\cdot w,  y_{i})$. Using the chain rule and evaluating at an arbitrary $\bar w \in  \c{V}$ we have
\corr{\begin{align*}
 \partial_{w} [L(h(w, \cdot))](\bar w) &= 
    \sum_i \partial_{\hat{y}} [\ell(\hat{y}_i, y_i)](J(x_i)\cdot \bar w) \cdot \partial_w (J(x_{i})\cdot \bar w) \\
    &= \sum_i \partial_{\hat{y}} [\ell(\hat{y}_i, y_i)](J(x_i)\cdot \bar w) \cdot J(x_{i}).
\end{align*}}
Evaluating the affine function $h$ consists of computing the Jacobian vector product $J(x_i) \bar w$. 
This can be done while avoiding computing the Jacobian explicitly by using forward mode automatic differentiation or finite differences.
We find both approaches to work similarly well, with finite differences being slightly faster, and forward mode automatic differentiation more numerically stable.
This chapter's experiments use finite differences, so we present this approach here.  Specifically, we employ the method of \citet{ANDREI2009Accelerated} to select the optimal step size.  \cref{chap:sampled_Laplace} will use automatic differentiation.
We then evaluate the loss gradient at the linear model output, denoting this vector in our algorithm as $ \mathfrak{g} = \partial_{\hat{y}} [\ell(\hat{y}, y)](J(x)\cdot \bar w)$. This gradient can often be evaluated in closed form. Finally, we project $\mathfrak{g}$ onto the weights by multiplying with the Jacobian. This vector Jacobian product is implemented using automatic differentiation. That is, $\mathfrak{g}^{T} J(x_{i})  =\partial_{v} [\mathfrak{g}^{T} \cdot g(v, x_{i})](\tilde v)$. We combine these steps in \cref{alg:linear_MAP_algorithm}.

Evaluating the gradient of $\|\bar w\|^{2}_{ A}$ is trivial. 

\begin{algorithm2e}[H]
	\SetAlgoLined
	\DontPrintSemicolon
	\KwData{Neural network $g$, Observation $x$,  Linearisation point $\tilde v$, Weights to optimise $w$, Likelihood function $\ell(\cdot, y)$, Machine precision $\epsilon$}
	$\delta = \sqrt{\epsilon} (1+\norm{\tilde v}_{\infty}) / \norm{w}_{\infty}$ \tcp*{Set FD stepsize \citep{ANDREI2009Accelerated}}
	$\hat{y} = J(x) \cdot w   \approx \frac{g(x, \tilde v + \delta w ) - g(x, \tilde v - \delta w ) }{2 \delta}$ \tcp*{Two sided FD approximation to Jvp}
	$\mathfrak{g} = \partial_{\hat{y}} [\ell(\hat{y}, y)](J(x)\cdot w)$ \tcp*{Evaluate gradient of loss at $J(x) \cdot w$ }
	$\mathfrak{g}^{T} \cdot J(x) = \partial_{v} [\mathfrak{g}^{T} \cdot g(v, x)](\tilde v) $ \tcp*{Project gradient with backward mode AD}
\KwResult{$\mathfrak{g}^{T} \cdot J(x)$}
\caption{Efficient evaluation of the likelihood gradient for the linearised model}
\label{alg:linear_MAP_algorithm}
\end{algorithm2e}

\subsection{Further implications of our results}

We now discuss details and implications of the presented recommendations and results.

\textbf{Magnitude of linearisation point in normalised networks }
Optimising a normalised neural network returns a solution for the normalised weights (those in $ \c{V}''$) up to some scaling factor $\mathfrak{c}>0$. How is c determined?
Recall, from \cref{eq:directional-deriv-zero}, that for any $v'' \in  \c{V}''$, the directional derivative of the NN output in the direction of $v''$ is zero. This is also illustrated in \cref{fig:contour}. With this in mind, the dynamics of optimisation can be understood by analogy to a Newtonian system in polar coordinates. The weights are a mass upon which the data fit gradient acts as a tangential force. When discretised, this gradient pushes the weights away from zero.
On the other hand, regularisation from the prior term acts like a centripetal force, pushing the weights towards the origin. 
The resulting c is thus proportional to the variance of the gradients of $v''$, and as such dependent on the learning rate and batch size hyperparameters, while being inversely proportional to the regularisation strength, e.g. weight decay. This has been studied extensively in the optimisation literature, including \citet{Laarhoven2017L2,Hoffer2018norm,Cai2019quantitative,li2020intrinsic,Lobacheva2021periodic}. 

\textbf{On network biases in the Jacobian feature expansion }
Most normalisation techniques introduce scale invariance by dividing subsets of network activations by an empirical estimate of their standard deviation. These activations depend on the values of both weights and biases.
On the other hand, practical use of linearised Laplace commonly considers uncertainty due to only network weights \citep{daxberger2021bayesian,Maddox2021fast}, excluding bias entries from Jacobian and Hessian matrices.
This departure from our assumptions can break the scale invariance necessary for \cref{lemma:scaling}. 
Whether invariance is (approximately) preserved for the weights in the bias-exclusion setting depends on the relative effect of weights and biases on each subset of normalised activations. 
Invariance is preserved if the biases have small impact.
Empirically, we find that the inclusion (or exclusion) of biases does not alter the improvements obtained from applying our recommendations (see \cref{fig:hessian_bias_choice}). 

\textbf{Implications for the (non-linearised) Laplace method }
The (non-linearised) Laplace method \citep{ritter2018scalable,Kristiadi2020being} approximates the intractable posterior by means of a quadratic expansion around an optima, but without the linearisation step given in equation \cref{eq:h-def_c5}. As discussed in \cref{sec:choice_of_posterior_mode}, when employing stochastic optimisation, early stopping, or normalisation layers, we will not find a minimiser of $\c{L}_{g, A}$. Without a well-behaved surrogate linear model loss to fall back on, the Laplace method can yield very biased estimates of the model evidence.

\section{Demonstration: hyperparameter selection with the tangent linear model} \label{sec:adapting_experiments}

We proceed to provide empirical evidence for our assumptions and recommendations. 
Specifically, in \cref{subsec:exp_assumptions}, we validate the assumptions made in throughout this chapter.
Then, in \cref{subsec:architectures}, we demonstrate that our recommendations yield improvements across a wide range of architectures.
In these first two subsections, we employ networks containing at most $46$k weights, since this is the largest model for which we can tractably compute the Hessian on an A100 GPU. This choice avoids confounding the effects described throughout the chapter with any further approximations.
In \cref{subsec:large_models}, we show that our recommendations yield performance improvements on the $25$M parameter ResNet-50 network while employing the KFAC approximation to the Hessian \citep{Martens15kron,daxberger2021laplace}. Throughout, we focus on the layerwise prior precision, described in \cref{subsec:layerwise_prior}. We leave extensive evaluation of the g-prior for \cref{chap:sampled_Laplace}.

Unless specified otherwise, we:
1) train a NN to find $\tilde v$ using standard stochastic optimisation algorithms, 2) linearise the network about $\tilde v$ as in \cref{eq:h-def_c5}, 3) optimise the linear model weights using $ \c{L}_{h,  A}$ (\cref{alg:linear_MAP_algorithm}) and layer-wise regularisation parameters with $\mathcal{G}_{w_{\star}}$ \cref{eq:model-evidence-new}, 4) compute the linearised predictive distribution with \cref{eq:linearised_predictive_2}.
We repeat this procedure with 5 random seeds and report mean results and standard error.
For each seed, the methods compared produce the same mean predictions $g(\tilde v, \cdot)$, only differing in their predictive variance. In this setting, the test Negative Log-Likelihood (NLL, lower is better) can be understood as a measure of uncertainty miscalibration. The full set of experimental details for this chapter are provided in \cref{app:adapting_experimental_setup}.

\subsection{Validation of modelling assumptions} \label{subsec:exp_assumptions}
We validate the key conjectures stated throughout the chapter.
If not specified otherwise, we employ a 46k parameter ResNet \citep{he2016deep} with batch-normalisation after every convolutional layer. The output layer is dense, satisfying \cref{eq:simple_linear_model}. 

\textbf{Choice of Hessian } In \cref{sec:choice_of_posterior_mode}, we suggest evaluating the Hessian of $ \c{L}_{h}$ at the linearisation point $\tilde v$ (instead of $w_{\star}$) for model evidence optimisation \cref{eq:model-evidence-new}. This avoids the need to recompute the Hessian throughout optimisation. \Cref{fig:hessian_bias_choice} (left) shows how the improvement from using the recommended model evidence $ \c{G}_{w_{\star}}$, as opposed to $ \c{G}_{\tilde v}$, dominates the effect of the choice of Hessian evaluation~point.

\begin{figure}[htbp]
    \centering
    \includegraphics[width=0.7\linewidth]{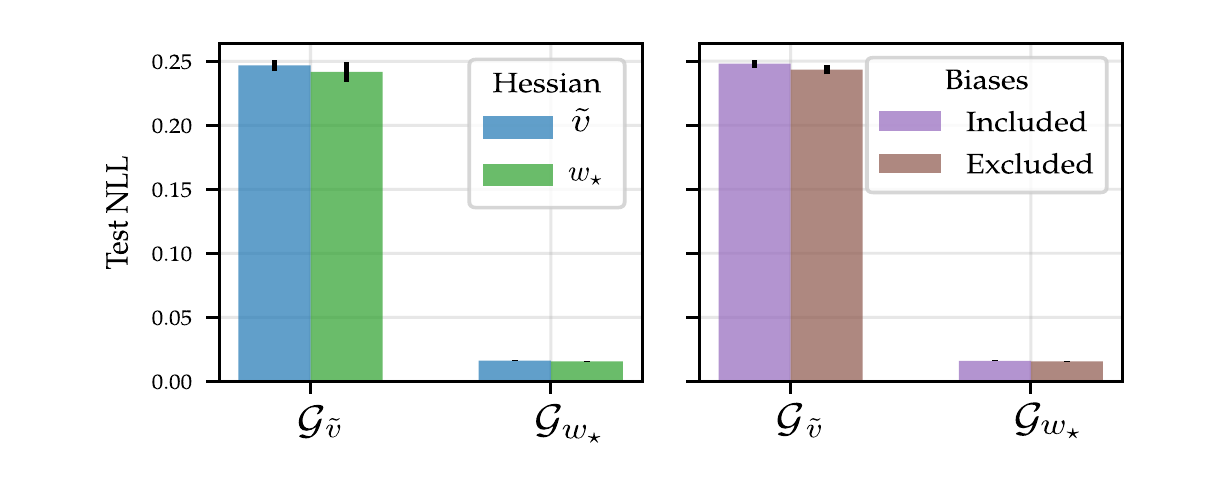}
    \caption{Comparison of the test NLL improvement obtained when switching from $ \c{G}_{\tilde v}$ to $ \c{G}_{w_{\star}}$ to optimise the prior precision $ A$ relative to the impact of (\emph{left}) evaluating the Hessian at $\tilde v$ or $w_{\star}$, and (\emph{right}) excluding network biases from the basis functions. Both plots use a $d=46k$ ResNet with batch norm trained on MNIST.}
    \label{fig:hessian_bias_choice}
\end{figure}

\textbf{Dependence on $\mathfrak{c}$ for isotropic precisions } In \cref{fig:k_scan}, we illustrate the dependence of the predictive posterior on the scale of normalised weights $\mathfrak{c}$ for an isotropic prior precision $ A {=}\, a I$, i.e. \cref{rec:factorised_prior} is ignored. We use a 2.6k parameter 2 hidden layer fully connected NN with layer norm after every layer except the last and a 1d regression task.
Changing $\mathfrak{c}$ changes the optimal $ a$ and, consequently, the predictive uncertainty changes.
With layer-wise $ a$ this effect vanishes (as predicted by \cref{proposition:unique-posterior}).

\textbf{Treatment of NN biases }
 Excluding the Jacobians of network biases from our basis function expansion breaks the scaling properties presented in \cref{lemma:scaling}. In \cref{fig:hessian_bias_choice} (right), we show that the effect of excluding biases is dominated by the choice of the model evidence between $ \c{G}_{w_{\star}}$ and $ \c{G}_{\tilde v}$.

 \begin{figure}[h]
    \centering
    \includegraphics[width=0.65\linewidth]{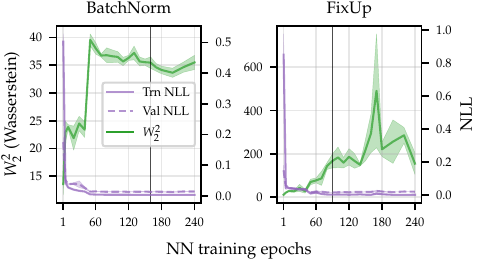}
    \caption{Wasserstein distance between predictive posteriors obtained when using $ \c{G}_{\tilde v}$ and $ \c{G}_{w_{\star}}$ throughout NN training (i.e. the linearisation point $\tilde v$ is changing). The vertical black line indicates optimal (val-based) early stopping.}
    \label{fig:early_stopping}
\end{figure}
 
\textbf{Early stopping }
We evaluate whether more thorough optimisation of the NN weights with $ \c{L}_{g}$ leads to a linearisation point $\tilde v$ closer to $w_{\star}$ in the sense of the implied optimal regularisation and induced posterior predictive distribution. We perform this analysis on normalised and unnormalised networks (for which use the non scale-invariant FixUp regularisation instead \citep{Zhang19FixUp}),
since $\tilde v$ is guaranteed to never match $w_{\star}$ for the former.
Surprisingly, the Wasserstein-2 distance between predictive distributions obtained with $ \c{G}_{\tilde v}$ and $ \c{G}_{w_{\star}}$ increases with more optimisation steps in both cases. Thus, more thorough optimisation does not~help.

\begin{table}[htbp]
\centering
\caption{Validation of recommendations across architectures. All results are reported as negative log-likelihoods (lower is better). In each column, the best performing method is bolded. For each $\mathcal{M}$,  if single or layerwise $ a$ optimisation performs better, it is underlined.}
    \label{tab:model_sweep}
\vspace{0.03in}
\resizebox{\linewidth}{!}{
\begin{tabular}{llllllll}
\toprule
                 &                     &                    \textsc{Transformer} &                            \textsc{CNN} &                         \textsc{ResNet} &                     \textsc{Pre-ResNet} &                          \textsc{FixUp} & \textsc{U-Net} \\
\midrule
\multirow{2}{*}{$ \c{G}_{w_{\star}}$} & single $ a$ &  \textbf{0.162} {\tiny \color{gray} $\pm$ 0.042} &  \textbf{0.025} {\tiny \color{gray} $\pm$ 0.000} &  \textbf{0.017} {\tiny \color{gray} $\pm$ 0.000} &  0.017 {\tiny \color{gray} $\pm$ 0.000} &  \textbf{0.055} {\tiny \color{gray} $\pm$ 0.006} & -1.793 {\tiny \color{gray} $\pm$ 0.050}  \\
                 & layerwise $ a$ &  \textbf{0.162} {\tiny \color{gray} $\pm$ 0.042} &  \textbf{0.025} {\tiny \color{gray} $\pm$ 0.000} &  \textbf{0.016} {\tiny \color{gray} $\pm$ 0.001} &  \underline{\textbf{0.016}} {\tiny \color{gray} $\pm$ 0.000} &  0.061 {\tiny \color{gray} $\pm$ 0.005} & \underline{\textbf{-2.240}} {\tiny \color{gray} $\pm$ 0.027}\\
\cmidrule{1-8}
\multirow{2}{*}{$ \c{G}_{\tilde v}$} & single $ a$ &  0.310 {\tiny \color{gray} $\pm$ 0.060} &  0.253 {\tiny \color{gray} $\pm$ 0.001} &  0.252 {\tiny \color{gray} $\pm$ 0.006} &  \underline{0.220} {\tiny \color{gray} $\pm$ 0.004} &  \underline{0.153} {\tiny \color{gray} $\pm$ 0.021} & -1.164 {\tiny \color{gray} $\pm$ 0.052}\\
                 & layerwise $ a$ &  \underline{\textbf{0.162}} {\tiny \color{gray} $\pm$ 0.042} &  \underline{0.205} {\tiny \color{gray} $\pm$ 0.002} &  \underline{0.236} {\tiny \color{gray} $\pm$ 0.005} &  0.239 {\tiny \color{gray} $\pm$ 0.004} &  0.200 {\tiny \color{gray} $\pm$ 0.018} & \underline{-1.703} {\tiny \color{gray} $\pm$ 0.023} \\
\bottomrule
\end{tabular}
}
\end{table}

\subsection{Validating recommendations across architectures}\label{subsec:architectures}

We evaluate the utility of \cref{rec:linear-weights}, and \cref{rec:factorised_prior} on a range of architectures and tasks:  1) a transformer architecture on the pointcloud-MNIST variable length sequence classification task. This model uses layer norm in alternating layers, 2) a LeNet-style CNN with batch norm placed after every convolutional layer and a dense output layer tasked with MNIST classification, 3) a ResNet with batch norm after every layer except the dense output layer (MNIST classification), 4) the same ResNet but with batch norm substituted by (non scale-invariant) FixUp regularisation (MNIST classification). 5. a pre-ResNet \citep{preresnet}. This architecture differs from ResNet in that batch norm is placed before each weight layer instead of after them; the implication is that there is only 1 normalised group of weights encompassing all weights but those of the dense output layer (MNIST classification). 6. a fully convolutional U-net autoencoder tasked with tomographic reconstruction (regression) of a KMNIST character from a noisy low-dimensional observation. We reproduce the experimental setting of \citet{barbano2021deep} for this task. Group norm is placed after every layer except the last, which is convolutional. 

As shown in \cref{tab:model_sweep}, the application of \cref{rec:linear-weights} yields notably improved performance across all settings. Applying \cref{rec:factorised_prior} yields modest improvements for classification networks with normalisation layers but large improvements for the U-net. Interestingly, layer-wise regularisation degrades performance in the (non-normalised) FixUp ResNet.

\subsection{Large scale models}\label{subsec:large_models}

\begin{table}[bhtp]
\centering
\caption{Test negative log-likelihoods for ResNet-50 on CIFAR10.}
\label{tab:resnet50_cifar}
\begin{tabular}{llll}
\toprule
 &  & \textsc{batch norm} & \textsc{FixUp} \\ \midrule
\multirow{2}{*}{$ \c{G}_{w_{\star, \text{simple}}}$} & single $ a$ & \textbf{0.112} {\tiny \color{gray} $\pm$ 0.004} & 0.128 {\tiny \color{gray} $\pm$ 0.000} \\
 & layerwise $ a$ & \textbf{0.109} {\tiny \color{gray} $\pm$ 0.003} & \underline{\textbf{0.096}} {\tiny \color{gray} $\pm$ 0.000} \\
\midrule       \multirow{2}{*}{$ \c{G}_{w_{\star}}$} & single $ a$ & 0.190 {\tiny \color{gray} $\pm$ 0.005} & 0.249 {\tiny \color{gray} $\pm$ 0.002} \\
 & layerwise $ a$ & 0.194 {\tiny \color{gray} $\pm$ 0.009} & \underline{0.193} {\tiny \color{gray} $\pm$ 0.001} \\
\midrule
\multirow{2}{*}{$ \c{G}_{\tilde v}$} & single $ a$ & 0.570 {\tiny \color{gray} $\pm$ 0.004} & 0.412 {\tiny \color{gray} $\pm$ 0.000} \\
 & layerwise $ a$ & 0.567 {\tiny \color{gray} $\pm$ 0.004} & \underline{0.360} {\tiny \color{gray} $\pm$ 0.000} \\ \bottomrule
\end{tabular}%
\end{table}

We validate our recommendations on the 25M parameter ResNet-50 network trained on the CIFAR10 dataset. This model places batch norm after every layer except the dense output layer. We also consider a normalisation-free FixUp ResNet-50. 
\Cref{tab:resnet50_cifar} shows that both of our recommendations yield better test NLL, with larger gains obtained by the batch norm network.
The normalisation-free FixUp setting does not simplify as in eq.~\cref{eq:simple_linear_model}. Nonetheless, assuming the simplified model evidence, denoted $ \c{G}_{v_{\star, \text{simple}}}$, when obtaining the linear model optima yields improved performance for all models.

\section{Discussion}

This chapter has identified and addressed two pitfalls of a na\"ive application of linearised Laplace to modern NNs in the \emph{post-hoc} setting. First, the optima of the loss function is not found in practice. This invalidates the assumption that the point at which we linearise our model is stationary. However, every linearisation point implies an associated basis function linear model. As we use this model to provide errobars, we propose to choose hyperparameters using the evidence of this model. This requires only the solving of a convex optimisation problem, one much simpler than NN optimisation. Second, normalisation layers introduce an invariance to the scale of NN weights and thus the linearisation point can only be identified up to a scaling factor. We show that to obtain a predictive posterior that is invariant to this scaling factor, the regulariser must be independently parametrised for each normalised group of weights, e.g. different layers. We also show that a classical feature normalisation method, the g-prior, solves this issue.
Our experiments confirm the effectiveness of these recommendations across a wide range of model architectures and sizes. 

With these advancements, and the scalable SGD-based sampling from \cref{chap:SGD_GPs}, we are almost ready to perform Bayesian inference and hyperparameter optimisation with large scale linearised neural networks. The only remaining impediment is computing the log-determinant term in the expression for the model evidence. \cref{chap:sampled_Laplace} will provide the final piece of the puzzle by introducing an accurate method to learn the linearised Laplace prior precision using only posterior samples.

\chapter[Sample-based linearised Laplace]{Scalable uncertainty estimation and hyperparameter learning for neural networks with sample-based linearised Laplace inference }\label{chap:sampled_Laplace}

\ifpdf
    \graphicspath{{Chapter6/Figs/Raster/}{Chapter6/Figs/PDF/}{Chapter6/Figs/}{Chapter6/figures/}}
\else
    \graphicspath{{Chapter6/Figs/Vector/}{Chapter6/Figs/}{Chapter6/figures/}}
\fi

\textit{``One thing that should be learned (...) is the great power of general purpose methods, of methods that continue to scale with increased computation even as the available computation becomes very great.''} --- Richard Sutton\\

The linearised Laplace method, originally introduced by \cite{Mackay1992Thesis}, and reviewed in \cref{sec:Laplace}, has received renewed interest in the context of uncertainty quantification for modern neural networks (NN) \citep{Khan2019approximate,Immer2021Improving,daxberger2021laplace}. The method constructs a surrogate Gaussian linear model for the NN predictions, and uses the error bars of that linear model as estimates of the NN's uncertainty. However, the resulting linear model is very large; the design matrix is sized number of parameters by number of datapoints times number of output classes. Thus, both the primal (weight space) and dual (observation space) formulations of the linear model are intractable. This restricts the method to small network or small data settings. Moreover, the method is sensitive to the choice of regularisation strength for the  linear model \citep{Immer21Selection,antoran2022adapting}. This chapter develops methods to scale inference and hyperparameter selection to very large linear models with a particular focus on linearised neural networks.

To scale inference and hyperparameter selection in Gaussian linear regression, we introduce a sample-based Expectation Maximisation (EM) algorithm. It interleaves E-steps, where we infer the model's posterior distribution over parameters, given some choice of hyperparameters, and M-steps, where the hyperparameters are improved given the current posterior. Our contributions here are two-fold:
\begin{enumerate}[topsep=0pt]
    \item We perform posterior sampling for large-scale linearised neural networks using stochastic gradient descent with the low-variance sample-then-optimise objective introduced in \cref{subsec:variance_reduction_SGD}, which we use to approximate the E-step.
    \item We introduce a method for hyperparameter selection that only requires access to posterior samples, and not the full posterior distribution. This forms our M-step.
\end{enumerate}
Combined, these allow us to perform inference and hyperparameter selection by solving a series of quadratic optimisation problems using stochastic gradient descent, and thus avoiding an explicit cubic cost in any of the problem's properties. Our method readily extends to non-conjugate settings, such as classification problems, through the use of the Laplace approximation. In the context of linearised NNs, our approach also differs from previous work in that it avoids instantiating the full NN Jacobian matrix, an operation requiring as many backward passes as output dimensions in the network.

We demonstrate the strength of our inference technique in the context of the linearised Laplace procedure for image classification on CIFAR100 (100 classes × 50k datapoints)
and Imagenet (1000 classes × 1.2M datapoints) 
using an 11M parameter ResNet-18
and a 25M parameter ResNet-50, respectively.
The methods introduced in this chapter will allow us to perform uncertainty estimation in high-resolution volumetric tomographic image reconstruction in \cref{chap:DIP}.

The rest of this chapter is organised as follows. \cref{sec:EM_for_linearised_NNs} introduces a variational EM algorithm for linearised neural networks. \cref{sec:sample_based_inference} discusses a series of methods to scale up the aforementioned algorithm to the large model and large dataset setting. \cref{sec:sampled_laplace_experiments} demonstrates these methods on large-scale image classification. Finally, \cref{sec:sample_based_conclusion} concludes the chapter.

\section{Variational EM for linearised neural networks}\label{sec:EM_for_linearised_NNs}

We consider the multioutput conjugate Gaussian linear model class, introduced in \cref{chap:linear_models}, and which we review here.
Our choice of basis functions are induced by a first order Taylor expansion of a NN $g: \c{V} \times \c{X} \to \R^c$ around its pre-trained parameters $\tilde v \in \c{V} \subseteq \R^d$.
We will work with fully normalised networks with a dense final layer\footnote{In \cref{chap:DIP}, we will work with models without a dense final layer. Fortunately, the procedures discussed in this chapter may be applied out-of-the-box to this setting by shifting the linear model targets by the constant-in-$w$ terms in the NN's Taylor expansion. 
}. In \cref{sec:dense-final-layer} we showed that, when linearised, these take the simplified form
\begin{gather*}
    h(w, \cdot) = \phi(\cdot) w,
\end{gather*}
where $\phi(x) = \partial_{v} g(\tilde v, x) \in \R^{c\times d}$ is the NN's Jacobian evaluated at $x \in \c{X}$, which acts as a feature expansion of the input.   

With that, the generative process we assume relates our inputs $x_1, \dotsc, x_n \in \c{X}$ and corresponding outputs $y_1, \dotsc, y_n \in \c{Y} \subseteq \R^c$ is
\begin{gather*}
    Y = \Phi w + \c{E} \spaced{with} w \sim \N(0, A^{-1} I) \spaced{and} \c{E} \sim \N(0, B^{-1}),
\end{gather*}
where $Y \in \R^{nc}$ is the concatenation of $y_1, \dotsc, y_n$, $B$ is a block diagonal matrix, built from $(B_i)_{i=1}^n$, a set of $c\times c$ blocks representing the noise precision for each iid observation, and $\smash{\Phi = [\phi(x_1)^T; \dotsc; \phi(x_n)^T]^T \in \R^{nc \times d}}$ is the embedded design matrix. We define \corr{$M = \Phi^T B \Phi \in \R^{d\times d}$}, which matches the curvature of the Gaussian likelihood.
Finally, $A \in \R^{d \times d}$ is a positive definite prior precision matrix, which we treat as a hyperparameter.

\subsection{Conjugate Gaussian regression and the EM algorithm}\label{subsec:conjugate_EM}

Our goal is to infer the posterior distribution for the parameters $ w$ given our observations, under the setting of $A$ most likely to have generated the observed data. We use an iterative procedure inspired by \cite{Mackay1992Thesis}, which alternates computing the posterior for $ w$, denoted $\Pi_{w|Y}$, for a given choice of $A$, and updating $A$, until the pair $(A, \Pi_{w|Y})$ converge to a locally optimal setting. This corresponds to an EM algorithm \corr{\citep{Dempster1977EM,bishop2006pattern}.}

With that, we start with some initial $A \in \R^{d \times d}$, and iterate:
\begin{itemize}[topsep=0pt]
  \item (E step) Given $A$, the posterior for $ w$, denoted $\Pi_{w|Y}$, is computed exactly as
  \begin{equation}
    \textstyle{\Pi_{w|Y} = \N( w_\star, H^{-1}) \spaced{where} H = M + A \spaced{and}  w_\star = H^{-1} \Phi^T  B Y.}
  \end{equation} 
  \item (M step) We lower bound the log-probability density of the observed data, i.e. the evidence, for the model with posterior $\Pi_{w|Y}$ and precision $A'$ as
  \begin{equation} \label{eq:model_evidence_fixed_mean_bound}
    \log p(Y; A') \geq \textstyle{-\frac{1}{2} \| w_\star\|_{A'}^2  -\frac{1}{2}\log\det(I + A'^{-1}M)} + C \eqqcolon \c{M}(w_\star, A'),
  \end{equation}
  for C independent of $A'$. We choose a new setting for $A$ that improves this lower bound.
\end{itemize}

\begin{derivation} \textbf{Derivation of \cref{eq:model_evidence_fixed_mean_bound} as a lower bound on the evidence} \\

To show this we part from the Gaussian ELBO for the Gaussian-linear model given in \cref{eq:expanded_linear_ELBO}
\begin{align*}
    \log p(Y; A) \geq  \c{M}(w_q, \Sigma_q, A) =  \frac{1}{2}\bigr( &- n\log(2 \pi)  -\logdet{B^{-1}} - \logdet{A^{-1}}    \notag \\
    &  - \|w_q\|_A^2 - \|Y - \Phi w_q \|^2_{B} + \logdet \Sigma_q \\
    & - \tr(\Phi \Sigma_q \Phi^T B)  + d - \tr(\Sigma_q A) \bigl).
\end{align*}
and choose $\Sigma_q = (M + A)^{-1} = H^{-1}$, which is the optimal setting, for any value of $w_q$ and $A$. With this we note that 
\begin{gather*}
    \tr{ (\Phi H^{-1} \Phi^T B)}  =  d - \tr (H^{-1} A) = \gamma,
\end{gather*}
are both expressions for the effective dimension (see \cref{eq:forms_effective_dim}), which cancel out. This leaves us with 
\begin{align*}
    \c{M}(w_q, A) =  \frac{1}{2}\bigr(  &  - \|w_q\|_A^2 - \logdet{A^{-1}}  + \logdet \Sigma_q  \notag \\
    &  - \|Y - \Phi w_q \|^2_{B}    - n\log(2 \pi)  -\logdet{B^{-1}} \bigl),
\end{align*}
which matches \cref{eq:model_evidence_fixed_mean_bound} when we set $w_q$ to $w_\star$ and identify the constant in $A$ terms as $C = \frac{1}{2} (- \|Y - \Phi w_q \|^2_{B}    - n\log(2 \pi)  -\logdet{B^{-1}}) = \log p(Y|w_q;B)$.

\end{derivation}

\begin{remark} \textbf{On the ELBO in \cref{eq:model_evidence_fixed_mean_bound}}\\
    $\c{M}(w_\star, A)$ has a variational parameter, the posterior mean and a hyperparameter the prior precision. For each prior precision, there is an optima posterior mean which makes the bound tight. At each M step, we update our prior precision, and the variational posterior's covariance updates automatically with the new regulariser, leaving the posterior mean as the only variational parameter to be found anew in successive E steps. Because the log-likelihood is quadratic, its curvature $M$ is fixed throughout the EM iteration.
\end{remark}

\subsection{Laplace-approximating non-conjugate likelihoods}\label{subsec:non_conj_laplace_EM}

We now consider the setting where the linearised model's loss, defined as
\begin{gather}\label{eq:linearised_loss}
    \c{L}_{h, A}(v) =  L(h(w,\cdot)) + \| w \|^2_A,
    \end{gather}
for $L \colon  \c{Y}^{ \c{V}\times\c{X}} \mapsto \R_+$ of the form  $L(h(w, \cdot)) = \sum_i^n \ell(y_i, h(w, x_i))$, and $\ell$ is a negative log-likelihood function, is non quadratic. That is, $L$ corresponds to a \emph{non-Gaussian} density.

We employ the Laplace approximation (see \cref{subsec:modern_view_laplace} for a review) for the E-step. That is, we construct a Gaussian approximate posterior as
\begin{gather*}
\N(w_\star, H^{-1}) \spaced{with}  w_\star = \argmin_{w \in \R^d} \c{L}_{h, A}(v) \\
   \text{and} \quad H = \Phi^T B \Phi + A\quad.
\end{gather*}
Here, $B\in\R^{nc\times nc}$ is a again a block diagonal matrix built from blocks $B_i = \partial^2_{\hat y_i} \ell(y_i, \hat y_i)$ which we evaluate at predictions $\hat y_i = h(\tilde v, x_i) = g(\tilde v, x_i)$ in place of $h(w_\star, x_i)$, since the latter would change each time the regulariser $A$ is updated, requiring expensive re-evaluation. This decision was recommended in \cref{sec:choice_of_posterior_mode} and ablated in \cref{subsec:exp_assumptions}. 

We plug in the above expressions into the Laplace evidence, given in \cref{eq:model_evidence_fixed_mean_bound}, for the M step. However, this may no longer represent a lower bound on the true evidence.
The EM procedure from \cref{subsec:conjugate_EM} is for the conjugate Gaussian-linear model, where it carries guarantees on non-decreasing model evidence, and thus convergence to a local optimum. These guarantees do not hold for non-conjugate likelihood functions, e.g., the softmax-categorical, where the Laplace approximation is necessary. \corr{Instead, we are guaranteed convergence to a local optima of the evidence of a surrogate model with Laplace approximated likelihood. }

\subsection{The issue of limited scalability} The above inference and hyperparameter selection procedure for $\Pi_{w|Y}$ and $A$ is futile when both $d$ and $nc$ are large. The E-step requires the inversion of a $d \times d$ matrix and the M-step evaluating its log-determinant, both cubic operations in $d$. These may be rewritten to instead yield a cubic dependence on $nc$ (as in \cref{subsec:duality}), but under our assumptions, that too is not computationally tractable. Instead, we now pursue a stochastic approximation to this EM-procedure.

\section{Sample-based inference for the tangent linear model}\label{sec:sample_based_inference}

We now present the chapter's main contribution, a stochastic approximation \citep{Nielsen2000stochastic} to the iterative algorithm presented in the previous section. Our M-step, presented in \cref{subsec:Mackay_update}, requires only access to samples from $\Pi_{w|Y}$. We then touch on a number of practical and implementation matters.  We provide an efficient implementation of the g-prior in \cref{subsec:gprior_implementation}. \Cref{subsec:warm_start_SGD} discusses an efficient implementation of the SGD posterior sampling methods introduced in \cref{chap:SGD_GPs}. These constitute our E-step. We discuss efficient sample-based predictions for linearised neural networks in \cref{subsec:sample_based_predictions}. We conclude with a full description of our inference algorithm in \cref{subsec:sample_laplace_algorithm}, with special attention to its application to image classification.

\subsection{Hyperparameter learning using posterior samples}\label{subsec:Mackay_update}

For now, assume that we have an efficient method of obtaining samples from a zero-mean version of the posterior $\zeta_1, \dotsc, \zeta_k \sim \N(0,\, H^{-1}) \coloneqq \Pi_{w|Y}^0$, and access to $w_\star$, the mean of $\Pi_{w|Y}$.
Evaluating the first order optimality condition for $\c{M}(w_\star, A)$  yields that the optimal choice of $A$ satisfies
\begin{equation}\label{eq:first-order-optimality}
  \| w_\star\|_{A}^2 = \tr\{H^{-1}M\} \eqqcolon \gamma, 
\end{equation}
where the quantity $\gamma$ is the effective dimension of the regression problem (see \cref{subsec:effective_dimension}). It can be interpreted as the number of directions in which the weights $w$ are strongly determined by the data.
Setting $A = a I\,$\footnote{We absorb additional prior structure into the basis functions in \cref{subsec:gprior_implementation}} for $a = \gamma/\|  w_\star \|^2$  yields a contraction step converging towards the optimum of $\c{M}$ \citep{Mackay1992Thesis}. We thus hereon refer to such a contraction step as a \emph{MacKay update}.

\begin{derivation} \textbf{First order optimality condition} \\ 
Consider the derivative of $\c{M}$. We have,
\begin{equation}
  \partial_A \log p(Y; A) = -\frac{1}{2} \left[ \partial_A\| w_\star\|^2_A + \partial_A \log\det(A + M) - \partial_A \log\det A\right],
\end{equation}
where we expanded $\log\det(I + A^{-1}M) = \log\det(A + M) - \log\det A$. Taking the respective derivatives and setting equal to zero at $A$, this leads to the condition 
\begin{equation}
   w_\star w_\star^T = (I - (I + A^{-1}M)^{-1})A^{-1}.
\end{equation}
Post-multiplying by $A$ and applying the push-through identity, we obtain
\begin{equation}
   w_\star w_\star^T A = M(A + M)^{-1}.
\end{equation}
For the above to hold, it is necessary that the traces of both sides are equal. Thus,
\begin{equation}
  \| \tilde w\|^2_{A} = \tr\{\tilde w\bar  w^T A\} = \tr\{M(A + M)^{-1}\} = \gamma ,
\end{equation}
which is the stated first order optimality condition, up to a cyclic permutation. 
\end{derivation}

Computing $\gamma$ directly requires the inversion of $H$, a cubic operation. We instead rewrite $\gamma$ as an expectation with respect to $\Pi_{w|Y}^0$ using \cite{Hutchinson90trace}'s trick, and approximate it using samples as 
\begin{align}\label{eq:eff_dim_estimator}
\gamma = \tr \{ H^{-1}M\} = \tr \{ H^{-\frac{1}{2}}MH^{-\frac{1}{2}}\} &= \E_{\zeta_1 \sim\Pi_{w|Y}^0} [\zeta^{T}_1 M\zeta_1 ] \\
&\approx \textstyle{\frac{1}{k} \sum_{j=1}^k \zeta_j^T \Phi^T  B \Phi \zeta_j}  \coloneqq \hat \gamma. \notag
\end{align}
We then select $a = \hat\gamma/\| w_\star\|^2$. We have thus avoided the explicit cubic cost of computing the log-determinant in the expression for $\c{M}$ (given in \cref{eq:model_evidence_fixed_mean_bound}) or inverting $H$. Due to the block diagonal structure of $B$, $\hat\gamma$ may be computed in $\c{O}(n)$ Jacobian vector products as $\hat \gamma = \textstyle{\frac{1}{k} \sum_{j=1}^k \sum_{i=1}^n \zeta_j^T \phi(x_i)^T  B_i \phi(x_i) \zeta_j} $.%

\subsubsection{Demonstration: MacKay's effective-dimension-based M-step}

\begin{figure}[t]
    \centering
    \includegraphics[width=\linewidth]{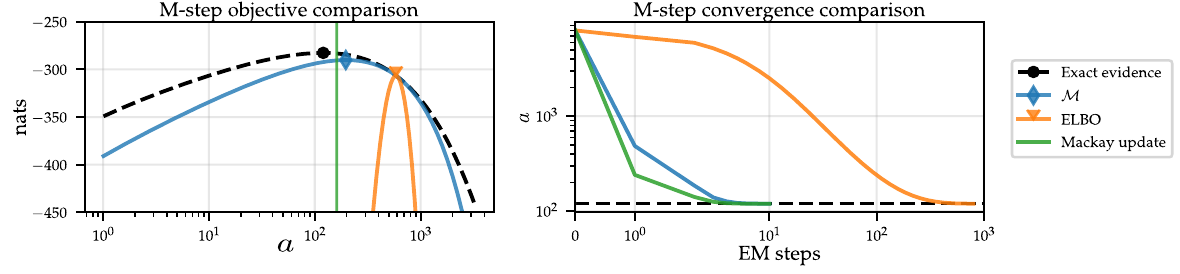}
    \caption{Left: exact model evidence for a linearised 2 hidden layer MLP with layer normalisation together with the lower bound presented in \cref{eq:model_evidence_fixed_mean_bound}, $\c{M}$, and an ELBO where the Gaussian posterior covariance is decoupled from the regulariser. All curves use an initial regulariser of $a=500$ and have a marker placed at their optima. The value proposed by the MacKay update is marked with a vertical green line. Right: values of the regularisation strength $a$ obtained at successive EM iterations while using the different update strategies under consideration for the M step. Note that when we assume access to the exact evidence function, the regulariser converges in a single step and no EM iteration is necessary.}
    \label{fig:Toy_evidence_convergence}
    \vspace{-0.1cm}
\end{figure}

\begin{figure}[t]
    \centering
    \includegraphics[width=\linewidth]{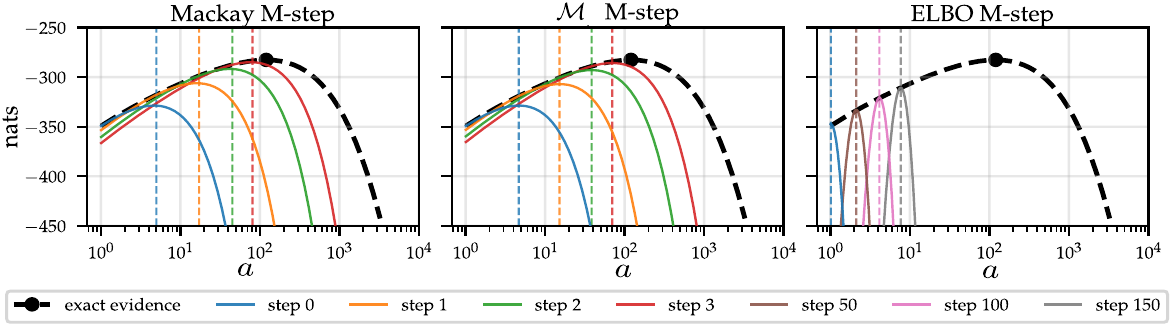}
    \caption{Exact linear model evidence for a linearised 2 hidden layer MLP with layer normalisation together with the lower bound presented in \cref{eq:model_evidence_fixed_mean_bound}, $\c{M}$ (left and middle plots), and an ELBO, where the Gaussian posterior covariance is decoupled from the regulariser (right side plot), at different EM steps. We update the regularisation strength with MacKay's fixed point iteration for the left side plot. Note that $\c{M}$ curves are shown in this plot. We maximise $\c{M}$ in the middle plot and we maximise the ELBO in the right hand side plot. 
    All curves use an initial regulariser of $a=5$ and we place a vertical dashed line at each step's update. Starting below the optimal regularisation strength makes convergence behaviour differ from that of \cref{fig:Toy_evidence_convergence}, which starts from above the optima. }
\label{fig:m_step_toy_convergence}
\end{figure}

We empirically motivate the fixed-point iteration M-step introduced by \cite{Mackay1992Thesis}, by comparing it with alternative approaches to updating hyperparameters. In particular, we compare MacKay's update with the standard Laplace M-step evidence, denoted $\mathcal{M}$ and given in \cref{eq:model_evidence_fixed_mean_bound}, and a Gaussian ELBO, of the form given in \cref{eq:expanded_linear_ELBO}, with both the mean and covariance acting as variational parameters. The latter two approaches differ in that the ELBO's posterior covariance is not clamped to the optimum value for the current regulariser $A$, and thus the bound is less tight. That is, ELBO's covariance does not change with the regulariser while performing the M-step. Both of these objectives differ from the MacKay update in that they provide an objective which requires gradient-based optimisation in the M-step. Instead, the MacKay update has a closed-form.

The plot on the left of \Cref{fig:Toy_evidence_convergence} compares the exact linearised Laplace evidence for a 2 hidden layer MLP with layernorm trained on the toy dataset of \cite{antoran2020depth} with the bound $\c{M}$ \cref{eq:model_evidence_fixed_mean_bound} and with the decoupled ELBO \cref{eq:model_evidence_fixed_mean_bound}. \corr{We evaluate all of these exactly, without resorting to Monte Carlo sampling.} The initial regulariser is set to $a=500$. The ELBO is only tight for regulariser values very close to initialisation, resulting in very small M steps. $\c{M}$ is tangent to the evidence at the same point as the ELBO but presents a much better approximation as we move away from  $a=500$. The optimum of $\c{M}$ is much closer to the optimum of the evidence. The MacKay update does not use a lower bound but instead provides an updated value for $a$ which is even closer to the optimum of the evidence. The right hand side plot shows the change in the regularisation parameter across successive M-steps using the update methods under consideration.  The MacKay M-step converges to the optima of the evidence in 2 steps. Using $\c{M}$ as an objective results in convergence after 5 steps. On the other hand, the ELBO update requires around 100 steps. \Cref{fig:m_step_toy_convergence} further illustrates hyperparameter learning in the 1d toy setting by showing the successive lower bounds obtained by each of the approaches under consideration at each M-step. Interestingly, the MacKay update produces regulariser updates that almost exactly maximise $\c{M}$.

\subsubsection{Demonstration: Comparing estimators of the effective dimension}  

\begin{figure}[t]
    \centering
    \includegraphics[width=0.5\linewidth]{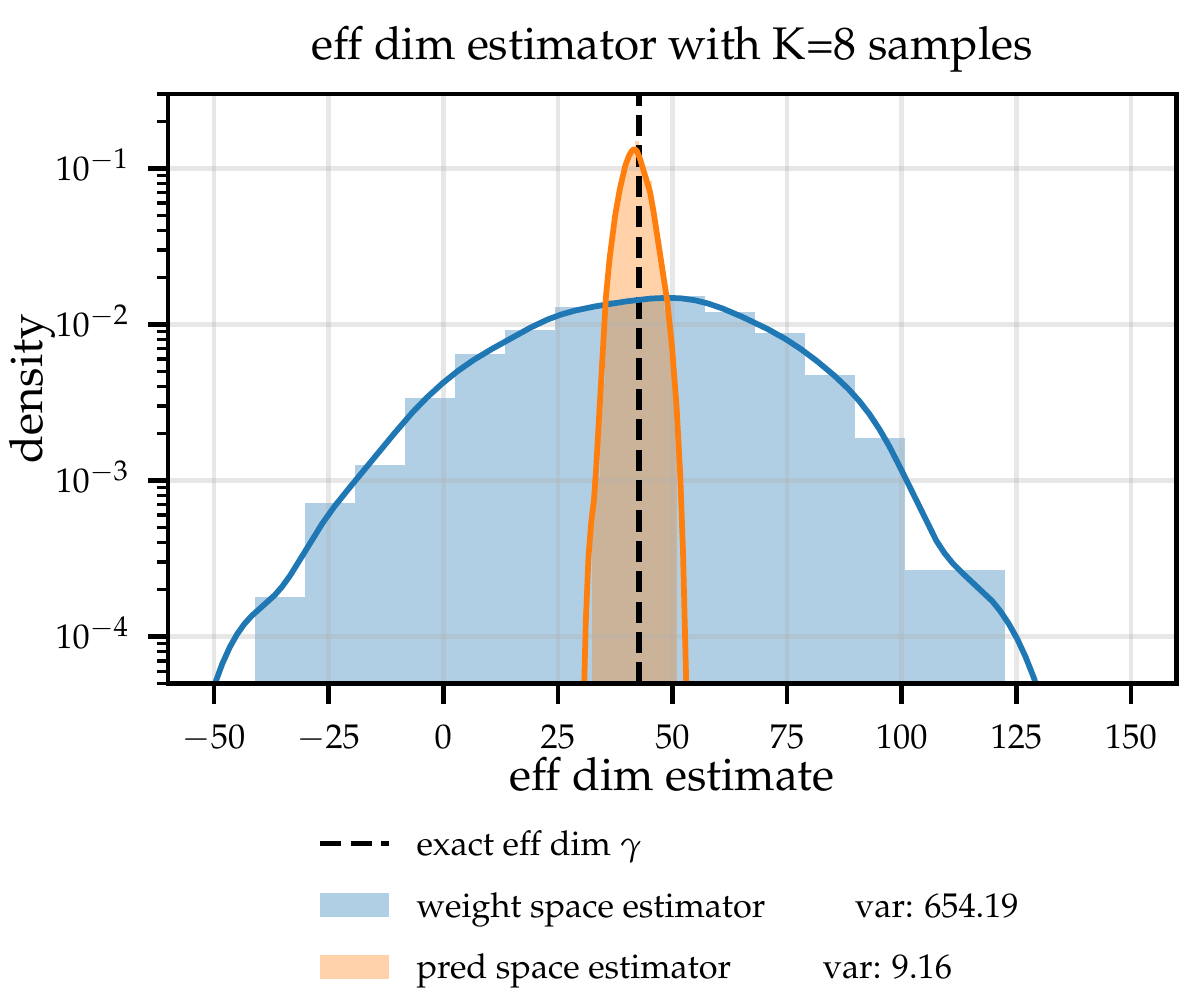}
    \vspace{-0.1cm}
    \caption{Histogram, with bin heights normalised to represent density estimates, of the effective dimension estimates produced by the primal form (weight space) estimator \cref{eq:weight_space_eff_dim_estimator} and the kernelised (prediction space) estimator \cref{eq:eff_dim_estimator}. Both distributions are roughly centred at the true effective dimension but the kernelised estimator presents much lower variance.}
    \label{fig:app_eff_dim_variance}
\end{figure}

The effective dimension estimator introduced in \cref{eq:eff_dim_estimator} is in kernelised form. 
A different unbiased estimator may be obtained in weight-space form following the derivation provided in \cref{eq:forms_effective_dim}.  That is
\begin{gather}\label{eq:weight_space_eff_dim_estimator}
    d - \tr{(AH^{-1})} \approx d - \frac{1}{k}  \sum_{j=1}^k   \zeta_j^T A \zeta_j \spaced{with} \zeta_j \sim \Pi_{w|Y}^0
\end{gather}
\Cref{fig:app_eff_dim_variance} compares both estimators when applied to the 1d toy problem used to generate \cref{fig:toy_EM} from the main text. In particular, we use a linearised 2 hidden layer MLP with 50 hidden units and layernorm after every hidden layer ($d=2700$). We use the ``Matérn'' dataset of \cite{antoran2020depth}. We use 8 samples from the exact linearised Laplace posterior to compute effective dimension estimates and repeat this procedure 1000 times to characterise the behaviour of each estimator. As a reference, we also compute the exact effective dimension using eigendecomposition. 

Both estimators present distributions centred at the true effective dimension value. However, the prediction space (kernelised) estimator presents a much lower variance of 9.16 as opposed to 654.19 for the weight space estimator. Additionally, the weight space estimator distribution places a substantial amount of probability mass on negative effective dimension values. From the form of \cref{eq:weight_space_eff_dim_estimator}, we see that this is due to our 8-sample estimator overestimating posterior variance. On the other hand, the kernelised estimator in \cref{eq:eff_dim_estimator} can only produce positive values.

\begin{remark} \textbf{Extension of the MacKay update to layerwise prior precision parameters} \\

We can leverage the primal form expression for the effective dimension given in \cref{eq:weight_space_eff_dim_estimator} to extend the MacKay update to the layer-wise regulariser setting (see \cref{subsec:layerwise_prior}).

Consider a sub-vector of our weight vector contiguous between the $i$th and $j$th weights written as $w_{\star i:j}$. Note that we only choose contiguous weights for notational convenience but it is not necessary to do so in general. 

The first order optimality condition is satisfied if for any $i, j$ with $i<j$, we have 
\begin{gather} \label{eq:layerwise_Mackay}
   \sum_{k=i}^j [A]_{kk}   w_{\star_k}^2 = j-i -  \sum_{k=i}^j [A]_{kk} [(A + M)^{-1}]_{kk} \coloneqq \gamma_{i:j}.
\end{gather}
We assume $[A]_{kk} = a$ for all $i \leq k < j$. Thus, we may update the regulariser for each separate weight sub-vector as $a = \gamma_{i:j} /\|  w_{\star i:j}\|^2$. However, we find this leads to slower convergence than when estimating a single prior precision for the whole model. Combined with the weight-space estimator of the effective dimension \cref{eq:weight_space_eff_dim_estimator} presenting higher variance, this make layerwise prior precision estimation with the MacKay update less attractive.
\end{remark}

\subsection{Constructing an efficient estimator of the g-prior} \label{subsec:gprior_implementation}

We use the diagonal g-prior, introduced in \cref{subsec:g-prior}, with with a prior precision of the form $a \diag{M}$,
where the $\diag$ operator takes as input a matrix and returns a diagonal version of the matrix. 
This leaves a single free parameter $a \in \R_+$, which will be estimated using MacKay updates, as described in \cref{subsec:Mackay_update}. However, this requires the prior precision to be isotropic. We achieve this by absorbing the scaling structure $\diag{M}$ into our feature expansion as
\begin{equation}\label{eq:g-prior-scaling}
  \smash{\textstyle{\phi'(x) =  \phi(x)\, \diag(s) \spaced{for} s \in \R^{d} \spaced{with entries} s_i = [M]^{-1/2}_{ii}}},
\end{equation}
where $i \leq d$ and $\diag{s}$ denotes a diagonal matrix with entries given by the vector $s$.
 And thus we work with the scaled Jacobian features $\phi'(\cdot)$ throughout,  while assuming a prior precision of the form $A = a I$.
Notice that the  covariance kernels implied by these expansions match $a \phi'(\cdot) \phi'(\cdot')^T = a \phi(\cdot) \diag(M) \phi(\cdot')^T$; our generative model is unchanged. 

We now turn to computing the scaling vector $s$. Na\"ively, each entry would be computed as $s_j = \left(\sum_{i=1}^n e_j^T \phi(x_i)^T  B_i \phi(x_i) e_j\right)^\frac{1}{2}$ for $e_j\: j\leq d$ the unit vectors corresponding to the canonical basis for the euclidean space $\R^d$. This would require $\c{O}(nd)$ Jacobian vector products, which is intractable for large models and datasets.

Instead we stochastically estimate $s$ using $k$ samples as
\begin{gather}\label{eq:stochastic_g_scaler}
    s = \E_{\c{E} \sim \N(0, B^{-1})} (\Phi^T B \c{E})^{\odot -0.5} \approx \left(\frac{1}{k} \sum_{j=1}^k (\Phi^T B \c{E}_j)^{\odot 2} \right)^{\odot -0.5}\spaced{with} \c{E}_j  \sim \N(0, B^{-1}),
\end{gather}
where $^{\odot}$ refers to the elementwise power. The number of Jacobian vector products needed is now \corr{$\c{O}(nk)$}. 
\subsection{Efficient SGD posterior sampling with warm starts}\label{subsec:warm_start_SGD}

All that is left is obtaining the linear model's posterior mean $w_\star$ and sampling from the 0 mean posterior $\Pi^0_{w|Y}$ for the E-step. 
For the former, we target the liner model's loss function $\c{L}_{h,A}$, given in \cref{eq:linearised_loss}, with stochastic gradient descent. 
For the latter, we use the low variance weight-space sampling objective introduced in \cref{subsec:variance_reduction_SGD} and which we re-state here for the reader's convenience
\begin{gather}\label{eq:new-loss2}
  \frac{1}{2}\|\Phi w\|_{{B}}^2 + \frac{1}{2}\|w - w'_0\|_A^2 \spaced{with} w'_0 = w_0 + A^{-1}\Phi^T {B} \c{E} \\
  \text{where} \quad  \c{E} \sim\N(0, B^{-1}) \spaced{and}
    w_0 \sim \N(0, A^{-1}) \notag.
\end{gather}
Notice how we can re-use samples used to estimate the g-prior scaling vectors $\Phi^T B \c{E}_j$ in \cref{eq:stochastic_g_scaler} to compute the regulariser target $w'_0$.

In order to limit computational cost, we sample the stochastic regularisation terms $w'_0$, only once, and keep them fixed throughout EM iteration. This results in the optima of the sampling objective being close for successive iterations with different regularisation strength values. This comes at the cost of a small bias in our estimator which we find to be negligible in practise. We separate $w'_0$ into a sum consisting of a prior sample from $w_0$ and a data dependent term, denoted $A^{-1}\Phi^T B \c{E}$. The former scales with $a^{\nicefrac{-1}{2}}$ while the latter with $a^{-1}$. This allows us to update each term in closed form each time $a$ changes in the M step. We initialise our posterior samples at $w_0$ at the first EM iteration and warm start them with the previously optimised values in successive iterations. Similarly, we warm-start the posterior mode $w_\star$ at the previous solution between iterations, initialising it to zero for the first iteration.
We optimise both our samples and posterior mean using stochastic gradient descent with Nesterov momentum. In particular, we follow the recommendations given in \cref{subsec:the_right_optimiser}. We only depart from this for non-quadratic likelihoods, like softmax cross entropy, where we substitute geometric iterate averaging with a linearly decreasing step-size schedule \citep{Bach2014sgd}.
As a preview of this procedure, we display the SGD optimisation traces for the posterior mean $w_\star$ and samples $\zeta$ throughout all steps of our EM procedure for a linearised ResNet-18 trained on the CIFAR100 dataset in \cref{fig:cifar100_em}.

\begin{figure}[htbp]
    \centering
    \includegraphics[width=0.95\linewidth]{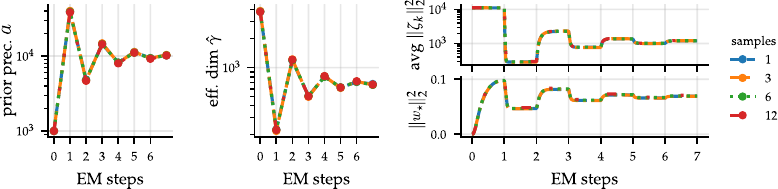}
    \vspace{-0.3cm}
    \caption{Left: prior precision optimisation traces for ResNet-18 on CIFAR100 varying n. samples. Middle: same for the eff. dim. Right: average sample norm and posterior mean norm throughout successive EM steps' SGD runs while varying n. samples. Note that traces almost perfectly overlap.  1 posterior sample is enough to obtain a very accurate estimate of the effective dimension. As a result, the optimisation traces corresponding to different numbers of samples almost perfectly overlap.}
    \label{fig:cifar100_em}
\end{figure}

\begin{remark} \textbf{The impact of SGD's bias on our hyperparameter updates}\\
    Letting $B=bI$ for simplicity, our estimate of the effective dimension amounts to estimating the $K^2$ norm, or output space norm of our samples $\E_\zeta b\|\Phi \zeta\|^2 = b\|h(\zeta, X)\|^2$ \cref{eq:eff_dim_estimator}. The other term appearing in the MacKay update is the $A$ norm of our posterior mean $\|w_\star\|^2_{A}$. 
    In \cref{chap:SGD_GPs} we saw how SGD converges quickly in the output space, but slowly in the weight space, both in an $L_2$ sense (see \cref{fig:exact_vs_low_noise}). As a result, we expect to obtain an accurate estimate of the effective dimension, but not of $\|w_\star\|^2_{A}$. Given that we initialise the MAP setting of the weights at 0 for SGD optimisation, we expect that SGD will result is us underestimating $\|w_\star\|^2_{A}$. In turn, this will lead to overestimation of our regulariser when setting it with $a = \gamma / \|w_\star\|^2_{A}$. Indeed, this issue will appear in very large scale problems in \cref{sec:sampled_laplace_experiments} and in \cref{subsec:volumetric_demo}.
\end{remark}

\subsection{Sample-based linearised Laplace predictions}\label{subsec:sample_based_predictions}

The linearised Laplace distribution over function outputs at an input $x$ is the Gaussian $\N(g(\tilde v, x), \phi(x) H^{-1} \phi(x)^T)$. Here, we are following  \cref{sec:Laplace} and \cref{chap:adapting_laplace} in using the neural network output $g(\tilde v, \cdot)$ as the predictive mean, rather than the surrogate model mean $h(w_\star, \cdot)$.
However, even given $H^{-1}$, evaluating this na\"ively requires instantiating $\phi(x)$, at a cost of $c$ vector-Jacobian products (i.e. backward passes). This is prohibitive for large $c$.
However, expectations of any function $\sigma: \R^c \to \R$ under the predictive posterior can be approximated using only samples from $\Pi_{w|Y}^{0}$ as
\begin{equation}
   \textstyle{\E_{\Pi_{w|Y}}[\sigma] \approx \frac{1}{k} \sum_{j=1}^k \sigma \left( g(\tilde v, x) + \phi(x)\zeta_j\right) \spaced{with} \zeta_1, \dotsc, \zeta_k \sim \Pi^{0}_{w|Y},}
\end{equation}
requiring only $k$ Jacobian-vector products. In practice, for classification, we find $k$ much smaller than the number of classes $c$ suffices.
 
\subsection{Putting the pieces into a single algorithm for image classification} \label{subsec:sample_laplace_algorithm}

 \begin{figure}[t]
    \begin{algorithm2e}[H]\LinesNotNumbered
    	\KwData{initial $a > 0$; $k, k' \in \N$, number of samples for stochastic EM and prediction, respectively.}
    	Compute g-prior scaling vector $s$ as in \cref{eq:g-prior-scaling} %
    	\\\vspace{1mm}
    	Sample random regularisers $w'_{0,1},\dotsc, w'_{0,k}$ per \cref{eq:new-loss2} \\
    	\While{$a$ has not converged}{
    	 Find posterior mode $\bar w$ by optimising linear model loss $\c{L}{(h(w, \cdot))}$, given in \cref{eq:linearised_loss}\\
    	 Draw posterior samples $\zeta_{1}\dotsc \zeta_{k}$ by optimising objective $L'$ with $w'_{0,1},\dotsc, w'_{0,k}$ \\
    	 Estimate effective dimension $\hat \gamma$, per \cref{eq:eff_dim_estimator}, using samples $\zeta_{1}\dotsc \zeta_{k}$\\
    	 Update prior precision $a \gets \hat \gamma/\|\bar w\|_{2}^{2}$
    	}\vspace{1mm}
    Sample $k'$ random regularisers $w'_{0,1},\dotsc, w'_{0,k'}$ using optimised $a$\\
    Draw corresponding posterior samples $\zeta_{1}',\dotsc, \zeta_{k'}'$ using \cref{eq:new-loss2} \\
    \KwResult{posterior samples $\zeta'_{1},\dotsc, \zeta'_{k'}$ }
    \caption{Sampling-based linearised hyperparameter learning and inference}
    \label{alg:algorithm_summary}
    \end{algorithm2e}
\end{figure}

We now combine the methods described so far into a single algorithm that avoids storing Hessian $H$ or covariance matrices $H^{-1}$, computing their log-determinants, or even instantiating Jacobian matrices $\phi(x)$, all of which have prevented the scalability of previous linearised Laplace implementations. We interact with NN Jacobians only through Jacobian-vector and vector-Jacobian products, which have the same asymptotic computational and memory costs as a NN forward-pass \citep{novak2022fast}. Unless otherwise specified, we use the diagonal g-prior and a scalar regularisation parameter. \Cref{alg:algorithm_summary} summarises our method and \cref{fig:toy_EM} shows an illustrative example. 

\begin{figure}[t]
        \centering
        \includegraphics[width=0.9\linewidth]{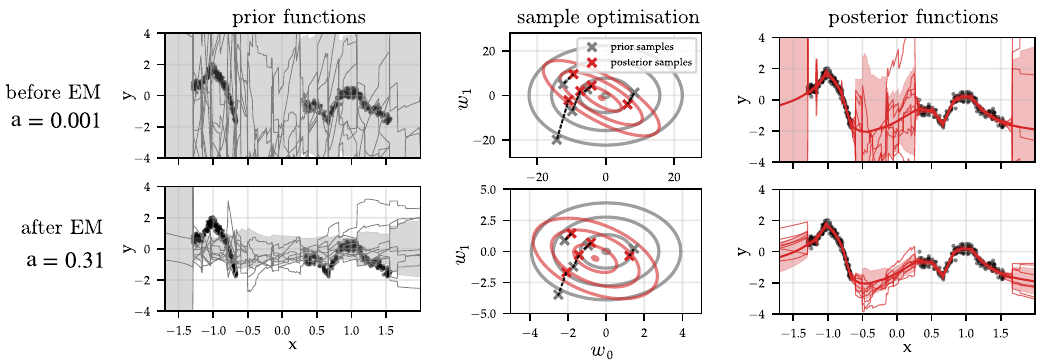}
        \caption{Illustration of our procedure for a fully connected NN on the toy dataset of \cite{antoran2020depth}. Top: prior function samples present large std-dev. (left). When these samples are optimised (middle shows a 2D slice of weight space), the resulting predictive errorbars are larger than the marginal target variance (right).
        Bottom: after EM, the std-dev. of prior functions roughly matches that of the targets (left), the overlap between prior and posterior is maximised, leading to shorter sample trajectories (center), and the predictive errorbars are qualitatively more appealing~(right).\label{fig:toy_EM} 
        }
    \end{figure}

 \subsubsection{An algorithm for image classification} 
 
 \Cref{alg:sampling_classification} provides a detailed procedure for applying our stochastic EM iteration to image classification while using g-prior feature scaling, described in \cref{eq:g-prior-scaling}. Therein, $\sigma$ denotes the softmax function. The curvature of the softmax cross entropy loss at $x_i$, denoted $B_i$, is given by $B_{i} = \text{diag}[(p_i) - p_ip_i^T]$ for $p_i = \sigma(g(\tilde v, x_i))$ denoting our neural network's predictive probabilities. The notation $\odot$ refers to the elementwise product and to the elementwise power when used in an exponent. 
 
 \begin{algorithm2e}[t]\LinesNotNumbered
	\KwData{Linearised network $h$, unscaled feature expansion $\phi$ linearisation point $w_\star$, observations $x_1,\dotsc,x_n$, negative log-likelihood function $\ell$, initial precision $a > 0$, number of samples $k$}
\SetKwProg{Func}{Define}{:}{}
    \Func{$B_i$}{
    $ p_i \gets \sigma(h(w_\star, x_i)) $ \\
    \textbf{return} $\textrm{diag}(p_i) - p_i p_i^T $
    } 
    
\For{$j=1, \dotsc, k$ }{ %
$w^0_j \sim \N(0, a^{-1} I)$  \\
$w'_j \gets a^{-1} \sum_{i=1}^n \phi(x_i)^T  \epsilon_j$ where $\epsilon_{j} \sim \N(0, B_i)$ \\
$\zeta_j \gets w^0_j$
}
$w_\star \gets 0$ \\
$s \gets a^{-1} \left[ \frac{1}{k} \sum_{j=1}^k {w'}_j^{\odot 2}\right]^{\odot -1/2}$ \\
\While{$a$ has not converged}{
\For{$j=1, \dotsc, k$}{
$\zeta_j \gets \text{SGD}_{z}\left(\|\Phi (s\odot z)\|_{\mathrm{B}}^2 + a\|z - w^{0}_j - (s \odot w'_j) \|^2_2, \,\, \text{init}{=}\zeta_j\right) $
}
$w_\star \gets \text{SGD}_{w}\left(\sum_{i=1}^n \ell(y_i, h((s \odot w), x_i)) + a\|w\|^{2}_{2}, \,\, \text{init}{=}w_\star\right)$ \\
$\hat \gamma \gets \frac{1}{k} \sum_{j=1}^k \sum_{i=1}^n \|(\zeta_j \odot s)^T \phi(x_i)^T \|^2_{B_i}$ \\
$a' \gets \hat \gamma / \|w_\star\|^{2}_{2}$ \\
\For{$j=1,\dotsc,k$}{
$w^{0}_j \gets \sqrt{\frac{a}{a'}}w^{0}_j $ \\
$w'_j \gets \frac{a}{a'}w'_j $
}
$a \gets a'$
}
\KwResult{Optimised precision $a$ and weight samples $\zeta_1, \dotsc, \zeta_k$}
\caption{Sampling-based linearised Laplace inference for image classification}
\label{alg:sampling_classification}
\end{algorithm2e}

The key hyperparameters of our algorithm are the number of samples to draw for the EM iteration, the number of EM steps to run, and SGD hyperparameters, namely learning rate, number of steps and batch-size. Empirically, we find that at most 5 EM steps are necessary for hyperparameter convergence and that as little as 1 sample can be used for the algorithm without degrading performance. Choosing SGD hyperparameters is more complicated. However, we are aided by the fact that lower loss values correspond to more precise posterior mean and sample estimates. As a result, we can tune these parameters on the train data, no validation set is required.

\section{Demonstration: Image classification}\label{sec:sampled_laplace_experiments}

We demonstrate our linear model inference and hyperparameter selection approach on the problem of estimating the uncertainty in NN predictions with the linearised Laplace method. 
First, in \cref{sec:lenet_MNIST}, we perform an ablation analysis on the different components of our algorithm using small LeNet-style CNNs trained on MNIST. In this setting, full-covariance Laplace inference (that is, exact linear model inference) is tractable, allowing us to evaluate the quality of our approximations. We then demonstrate our method at scale on CIFAR100 classification with a ResNet-18  (\cref{sec:Resnet18})
and Imagenet with ResNet-50 (\cref{sec:Res50IMAGENET}). 
We look at both marginal and joint uncertainty calibration and at computational cost.

\subsection{Comparison with existing approximations on MNIST}\label{sec:lenet_MNIST}

We first evaluate our approach on MNIST $c{=}10$ class image classification, where exact linearised Laplace inference is tractable. The training set consists of $n{=}60k$ observations and we employ 3 LeNet-style CNNs of increasing size: ``LeNetSmall'' ($d{=}14634$), ``LeNet'' ($d{=}29226$) and ``LeNetBig'' ($d{=}46024$). The latter is the largest model for which we can store the covariance matrix on an A100 GPU. We draw samples and estimate posterior modes using SGD with Nesterov momentum. We use 5 seeds for each experiment, and report the mean and std. error.

\paragraph{Fidelity of sampling-based inference} We compare our methods uncertainty using 64 SGD-based samples against approximate methods based on the NN weight point-estimate (MAP), a diagonal covariance, and against a KFAC estimate of the covariance \citep{Martens15kron,ritter2018scalable} implemented with the \href{https://github.com/AlexImmer/Laplace}{\texttt{Laplace}} library, in terms of similarity to the full-covariance lin. Laplace predictive posterior. As standard, we compute categorical predictive distributions with the probit approximation \citep{daxberger2021laplace}. All methods use the same layerwise prior precision obtained with 5 steps of full-covariance EM iteration. 
The results are on the left hand side of \cref{fig:mnist_preds_and_EM}.
For all three LeNet sizes, the sampled approximation presents the lowest categorical sym. KL and logit W2 distance to the exact lin. Laplace pred. posterior. The fidelity of competing approximations degrades with model size but that of sampling increases.

\begin{figure}[t]
    \centering
    \includegraphics[width=0.48\linewidth]{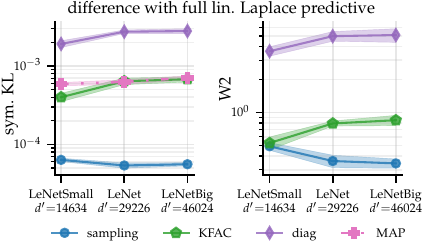}
    \hspace{0.01\linewidth}
    \includegraphics[width=0.24\linewidth]{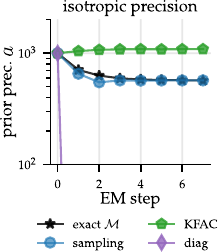}
    \hspace{0.01\linewidth}
    \includegraphics[width=0.23\linewidth]{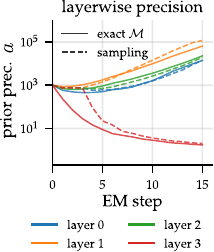}
    \vspace{-0.25cm}
    \caption{Left: similarity to exact lin. Laplace predictions on the MNIST test-set, in terms of symmetric KL and Wasserstein-2 distance, for different approximate methods applied to NNs of increasing size. Centre right: comparison of EM convergence for a single hyperparameter across approximations. Right: layerwise convergence for exact and sampling~methods.     \label{fig:mnist_preds_and_EM}
    }
\end{figure}

\paragraph{Accuracy of sampling hyperparameter selection} We first compare our SGD sampling EM iteration with 16 samples to full-covariance EM on LeNet, both without the g-prior. \Cref{fig:mnist_preds_and_EM}, middle-right, shows that for a single precision hyperparameter, both approaches converge in about 3 steps to the same value. In this setting, the diagonal covariance approximation diverges, and KFAC converges to a biased solution. We also consider learning layer-wise prior precisions by using the layerwise version of MacKay's M-step update. Here, neither the full covariance nor sampling methods converge within 15 EM steps. The precisions for all but the final layer grow in all steps. This reveals a pathology of this prior parametrisation: only the final layer's Jacobian, i.e. the final layer activations, are needed to accurately predict the targets; other features are pruned away.

We further compare the evidence approximations implied by a number of popular Laplace posterior approximations to the one from our SGD-based posterior samples. \Cref{fig:MNIST_bound_tightness_comparison} displays the evidence approximations obtained by plugging different posterior covariance approximations into the Laplace evidence given in \cref{eq:model_evidence_fixed_mean_bound}. In particular, we consider the full-covariance Laplace evidence (denoted $\c{M}$ in the plot), which we note does not match the exact model evidence due to the non-quadratic classification loss, the KFAC approximation to the covariance (labelled KFAC GGN), a single-sample KFAC Fisher estimate of the covariance (KFAC EF), the KFAC empirical Fisher matrix, and a diagonal Laplace covariance. We refer the reader to \cite{Immer21Selection,daxberger2021laplace} for a review of these approximations. We also include  the linear-Gaussian ELBO given in \cref{eq:expanded_linear_ELBO} and discussed in \cref{subsec:conjugate_EM}, where the approximate posterior is given by 16 SGD-based samples. In all cases, we initialise the regulariser at an optima found by applying the EM algorithm while using the full covariance Laplace evidence $\c{M}$ in the M-step. In this way, we may use the deviation of different objectives' optima from the optima of  $\c{M}$ as estimates of the bias in their corresponding approximations. 
The KFAC and KFAC-Fisher approximations result in a systematic overestimation of the evidence optima which grows with model size. This issue is even more pronounced for the diagonal covariance approximation. Surprisingly, we find the empirical Fisher to provide an accurate approximation. A similar finding is reported by \citep{Immer21Selection}. This is surprising, given that the empirical Fisher is known to provide a heavily biased estimate of loss curvature and thus perform poorly for optimisation tasks \citep{Kunstner2019limitations}. The sample-based ELBO shows close to no bias in its optima. This matches our experiments from \cref{sec:Resnet18}, where the sample-based EM algorithm behaves well even when using very few samples.

\begin{figure}[t]
    \centering
    \includegraphics[width=\linewidth]{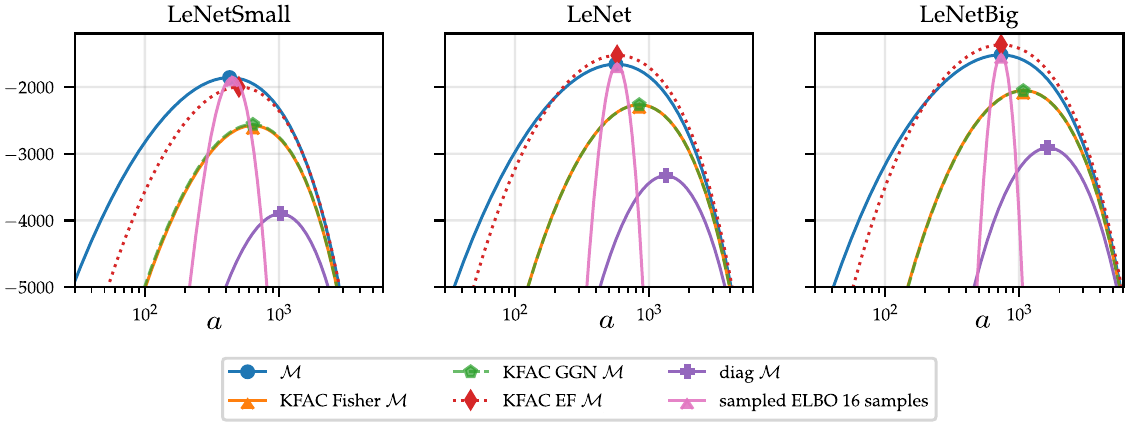}
    \caption{Full covariance linearised Laplace evidence $\c{M}$ together with approximations to this curve that rely on different covariance matrix approximations. A marker is placed at each curve's optima. We consider convolutional networks of increasing size (left to right) trained on the MNIST dataset.
    } \label{fig:MNIST_bound_tightness_comparison}
\end{figure}

\subsection{Predictive performance and robustness on CIFAR-100} \label{sec:Resnet18}

We showcase the stability and performance of our approach by applying it to CIFAR100 $c{=}100$-way image classification. The training set consists of $n{=}50k$ observations, and we employ a ResNet-18 model with $d \approx 11M$ parameters. 
To the best our knowledge, this is the first lin. Laplace approach that is capable of scaling to the CIFAR100 dataset, as the high-parameter and high-output dimensions prove intractable even on modern hardware.
Unless specified otherwise, we run 8 steps of EM with 6 samples to select $a$. We then optimise $64$ samples to be used for prediction.
We run each experiment with 5 different seeds reporting mean and std. error. %

\paragraph{Stability and cost of sampling algorithm} \ \Cref{fig:cifar100_em} shows that our sample-based EM converges in 6 steps, even when using a single sample. At convergence, $a \approx 10^4$ and $\hat{\gamma} \approx 700$, so $2a\gamma = 2\times700\times10^4=1.4\times10^7>1.1\times 10^7 = \tr{M}$. Thus, \cref{eq:effective-dim-condition} is satisfied and our low variance sample-then-optimise objective \cref{eq:new-loss} presents better properties even at convergence. We use 50 epochs of optimisation for the posterior mode and 20 for sampling. When using 2 samples, the cost of one EM step with our method is 45 minutes on an A100 GPU; for the KFAC approximation, this takes 20 minutes.

\begin{figure}[thb]
    \centering
    \includegraphics[width=0.65\textwidth]{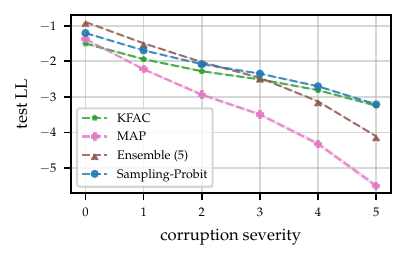}
\caption{Performance under distribution shift for ResNet-18 on CIFAR100.}%
\label{fig:cifar_pred}

\end{figure}

\paragraph{Evaluating performance in the face of distribution shift} \ We employ the standard benchmark for evaluating methods' test Log-Likelihood (LL)  \nomenclature[z-ll]{$LL$}{Log Likelihood} \nomenclature[z-nll]{$NLL$}{Negative Log Likelihood} on the increasingly corrupted data sets of \cite{hendrycks2016baseline,snoek2019can}.
We compare the predictions made with our approach to those from deep ensembles, arguably the strongest baseline for uncertainty quantification in deep learning \citep{lakshminarayanan2017,ashukha2020pitfalls}. \corr{We use a 5 element ensemble, as this is standard in the literature \citep{antoran2020depth,daxberger2020expressive} and the number of vector Jacobian products needed to train it roughly matches the amount used for 7 steps of sample-based EM optimisation with 6 posterior samples.}
We also consider a point-estimated predictions (MAP), and with a KFAC approximation of the lin. Laplace covariance  \citep{ritter2018scalable}. For the latter, constructing full Jacobian matrices for every test point is computationally intractable, so we use 64 samples for prediction, as we do for SGD sampling. The KFAC covariance structure leads to fast log-determinant computation, allowing us to learn layer-wise prior precisions \citep[following][]{Immer21Selection} for this baseline using 10 steps of non-sampled EM. For both lin. Laplace methods, we use the standard probit approximation to the categorical predictive distribution \citep{daxberger2021bayesian}.
\Cref{fig:cifar_pred} shows that for in-distribution inputs, ensembles performs best and KFAC overestimates uncertainty, degrading LL even relative to point-estimated MAP predictions. Conversely, our method improves LL. For sufficiently corrupted data, our approach outperforms ensembles, also edging out KFAC, which fares well here due to its consistent overestimation of uncertainty.

\begin{table}[h]
\begin{tabular}{@{}cccccc@{}}
\multicolumn{1}{c}{} & \multicolumn{1}{c}{$\kappa$} & MAP & Ensemble (5) & KFAC & Sampling \\
\cline{2-6}
\multicolumn{1}{c}{{marginal LL}} & $1$ & \multicolumn{1}{c}{-$1.40\pm0.00$} & \multicolumn{1}{c}{-$\mathbf{0.90\pm0.00}$} & \multicolumn{1}{c}{-$1.12\pm0.01$} & -$1.07\pm0.01$ \\ \midrule
\multicolumn{1}{c}{\multirow{4}{*}{{joint LL}}} & $2$ & \multicolumn{1}{c}{-$13.97\pm0.01$} & \multicolumn{1}{c}{-$6.86\pm0.01$} & \multicolumn{1}{c}{-$\mathbf{4.92\pm0.04}$} & -$5.14\pm0.04$ \\
\multicolumn{1}{c}{} & $3$ & \multicolumn{1}{c}{-$27.89\pm0.03$} & \multicolumn{1}{c}{-$14.17\pm0.03$} & \multicolumn{1}{c}{-$10.83\pm0.12$} & -$\mathbf{10.77\pm0.09}$ \\
\multicolumn{1}{c}{} & $ 4$ & \multicolumn{1}{c}{-$41.83\pm0.03$} & \multicolumn{1}{c}{-$22.29\pm0.04$} & \multicolumn{1}{c}{-$19.02\pm0.22$} & -$\mathbf{18.04\pm0.18}$ \\
\multicolumn{1}{c}{} & $5$ & \multicolumn{1}{c}{-$55.89\pm0.02$} & \multicolumn{1}{c}{-$31.07\pm0.09$} & \multicolumn{1}{c}{-$29.40\pm0.40$} & -$\mathbf{26.75\pm0.26}$ \\ \bottomrule
\end{tabular}

\caption{Comparison of methods' marginal and joint prediction performance for ResNet-18 on CIFAR100.}%
\label{tab:cifar_joint}
\end{table}

\paragraph{Joint predictions} Joint predictions are essential for sequential decision making, but are often ignored in the context of NN uncertainty quantification \citep{Janz19successor}. To address this, we replicate the ``\textit{dyadic sampling}'' experiment proposed by \cite{osband2022evaluating}. We group our test-set into sets of $\kappa$ data points and then uniformly re-sample the points in each set until sets contain $\tau$ points. That is, multiple coppies of the $\kappa$ original points. We then evaluate the LL of each set jointly. Since each set only contains $\kappa$ distinct points, a predictor that models self-covariances perfectly should obtain an LL value at least as large as its marginal LL for all values of $\kappa$.  We use $\tau{=}10(\kappa-1)$
and repeat the experiment for 10 test-set shuffles. 
Our setup remains the same as above but we use Monte Carlo marginalisation to push our Gaussian predictive distribution through the softmax instead of the probit approximation, since the latter discards covariance information. \Cref{tab:cifar_joint} shows that ensembles make calibrated predictions marginally but their joint predictions are poor, an observation also made by \cite{Osband2021epistemic}. Our approach is competitive for all $\kappa$, performing best in the challenging large~$\kappa$~cases.

\paragraph{Calibration of predictive uncertainty}
For the standard CIFAR100 test set, we separate our predicted probabilities into 10 equal width bins between 0 and 1. For each bin, we plot the proportion of targets that coincide with the class for which the predicted probability falls into the bin. This is shown in \cref{fig:calibration_plots}. KFAC overestimates uncertainty at all confidence levels whereas MAP underestimates it. Both sample-based linearised Laplace and ensembling show significantly improved calibration. While ensembles show a small amount of uncertainty overestimation consistently, our method underestimates uncertainty for low predicted probabilities and overestimates it for large predicted probabilities.

\begin{figure}[t]
    \centering
    \includegraphics[width=0.50\linewidth]{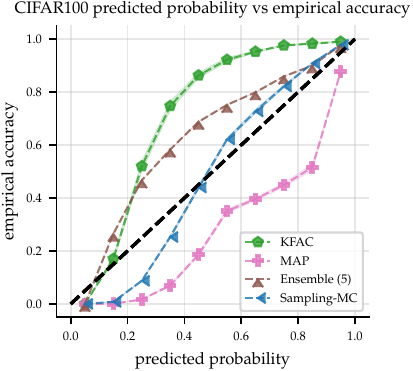}
    \vspace{0.1cm}
    \caption{Confidence vs accuracy plot (also known as a reliability diagram) for our CIFAR100 classification experiment.}
    \label{fig:calibration_plots}
    \vspace{-0.2cm}
\end{figure}

\subsection{Predictive performance on Imagenet}\label{sec:Res50IMAGENET}

We demonstrate the scalability of our approach by applying it to Imagenet $c{=}1000$-way image classification \citep{imagenet}. The training set consists of $n{\approx}1.2M$ observations, and we employ a ResNet-50 model with $d \approx 25M$ parameters. To the best our knowledge, this is the first lin. Laplace approach that is capable of scaling to the Imagenet dataset, as the high-parameter and high-output dimensions prove intractable even on modern hardware.
Our setup largely matches that described in \cref{sec:Resnet18} for CIFAR100 but we run 6 steps of EM with 6 samples to select a single regulariser parameter $a$. We then optimise $90$ samples to be used for prediction. Our computational budget only allows for a single run. Thus, we do not provide errorbars.

\begin{figure}[t]
\centering
\includegraphics[width=0.6\textwidth, trim=0 0 150 0, clip]{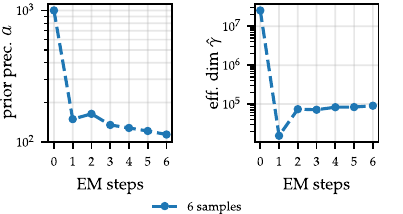}
\caption{Prior precision optimisation trajectories for ResNet-50 on Imagenet.}%
\label{fig:imagenet_pred}
\end{figure}

\paragraph{Stability and cost of sampling algorithm}

At each EM step, we run 10 epochs of optimisation to find the linear model's posterior mean and a single epoch of optimisation to draw samples. Each EM step takes roughly $26$ hours on a TPU-v3 accelerator. \Cref{fig:imagenet_pred} shows the regularisation strength reaches values ${\sim}150$ in 1 step and then drifts slowly towards lower values without fully converging within 6 steps.
This suggests the optimal $a$ value may lie bellow $100$. Unfortunately, we are unable to verify this, as the required computational cost exceeds our budget. 
Taking $a \approx 100$ and $\hat{\gamma} \approx 10^5$, we obtain $2 a \gamma \approx \mathrm{Tr} \;M = 2.5\times 10^7$. According to \cref{eq:effective-dim-condition} and thus our proposed sampling objective is expected to reduce variance throughout optimisation.

\begin{table}
\begin{tabular}{cccccccc}
 & \multirow{2}{*}{$\kappa$} & \multirow{2}{*}{MAP} & \multicolumn{1}{c}{Ensemble} & \multicolumn{2}{c}{KFAC} & \multicolumn{2}{c}{Sampling} \\
                          &   &         &          5 NNs              & init  & 5EM   & 6EM   & $\alpha{=}11.4$  \\ \hline
marginal LL               & 1 & -0.936  & -$\mathbf{0.815}$ & -1.449  & -1.493  & -0.924  & -0.917             \\ \hline
\multirow{4}{*}{joint LL} & 2 & -9.347  & -6.700                 & -6.289  & -6.286  & -7.814  & -$\mathbf{5.611}$  \\
                          & 3 & -18.733 & -13.268                & -12.112 & -12.246 & -15.065 & -$\mathbf{10.675}$ \\
                          & 4 & -28.093 & -20.029                & -19.872 & -20.493 & -22.416 & -$\mathbf{16.154}$ \\
                          & 5 & -37.416 & -26.938                & -29.839 & -31.221 & -29.787 & -$\mathbf{21.981}$ \\ \hline
\end{tabular}
  \caption{Comparison of methods' marginal and joint predictive performance for ResNet-50 on Imagenet.\label{tab:imagenet_joint}}%
\end{table}

\paragraph{Marginal and joint predictions}
Using the same setup from the CIFAR100 joint prediction experiment above, but drawing 90 samples with our method and KFAC, we make marginal and joint predictions on the Imagenet test set. The results are shown in \cref{tab:imagenet_joint}. Our methods regulariser optimisation trajectory in \cref{fig:imagenet_pred} suggests a value lower than the one obtained after 6 EM steps ($a=114$) may be preferred. Thus, we also report results with a value 10 times lower: $a=11.4$. 
For KFAC, we optimise layerwise prior precisions using 5 EM steps. This leads to small precisions which produce underconfident predictions and poor results. We attribute this to bias in the KFAC estimate of the covariance log-determinant. 
For comparison, we include KFAC results with a single regularisation parameter set to our initialisation value $a{=}10000$ (labelled ``init''). This choice maintains or improves performance across $\kappa$ values. 
Similarly to CIFAR100, ensembles obtains the strongest marginal test log-likelihood followed by our sampling approach for both regularisation strength values. KFAC overestimates uncertainty providing worse marginal performance than a single point-estimated network for both regularisation strength values. With $a{=}11.4$, our sampling approach performs best in terms of joint LL. Again, we find ensembles model joint-dependencies poorly. For $\kappa$ values between 2 and 5, their performance is comparable to that of the KFAC approximation. 

Concurrently with the present work, \cite{Deng2022Accelerated} introduce ``ELLA'', a Nyström-based approximation to the Laplace covariance. With ResNet-50 on Imagenet, the authors report a marginal ($\kappa{=}1$) test LL of -$0.948$, which is worse than our MAP model. However, differences in the MAP solution upon which the Laplace approximation is built (theirs obtains -$0.962$ LL) make  \cite{Deng2022Accelerated}'s results not directly comparable with ours. 
ELLA does not provide a model evidence objective and thus \cite{Deng2022Accelerated}'s result relies on validation-based tuning of the regularisation strength.

\section{Discussion}\label{sec:sample_based_conclusion}

This chapter introduced a sample-based approximation to inference and hyperparameter selection in Gaussian linear multi-output models. 
The approach is asymptotically unbiased, allowing us to scale the linearised Laplace method to ResNet models on CIFAR100 and Imagenet without discarding covariance information, which was computationally intractable with previous methods. 
The uncertainty estimates obtained through our method are well-calibrated not just marginally, but also jointly across predictions. Thus, our work may be of interest in the fields of active and reinforcement learning, where joint predictions are of importance, and computation of posterior samples is often~needed.

With this, we have largely delivered on the goals of the thesis; scaling calibrated uncertainty estimation to real-world-sized models and datasets. \cref{chap:DIP}, applies the methods developed so far to uncertainty estimation and experimental design for tomographic image reconstruction. Our scalable methods will allow us to perform uncertainty estimation for high-resolution volumetric reconstructions from neural networks, a problem not tackled before because of its large computational cost. 

\corr{
Since the publication of \cite{antoran2023samplingbased}, which formed the basis for this chapter, there have been further efforts, using both Bayesian and non-Bayesian methods, to obtain calibrated uncertainty estimates from large NNs trained on large datasets. Most notable is the work of \cite{Osband2021epistemic}, who use the sample-then-optimise objective \cref{eq:sample_then_optimise} to draw samples of the weights of a small ad-hoc neural network placed on top of a pre-trained model's final layer activations. This procedure does not even approximately draw samples from the true posterior of the ad-hoc network's weights, but provides calibrated uncertainty estimates in practise, both in terms of marginal and joint predictions. Also worth mentioning is the work of \cite{shen2024variational}, which represents the latest effort to adapt a standard optimiser used in deep learning to learn the mean and variance vector of a mean field variational posterior.
}

\chapter[The linearised deep image prior for computed tomography]{Uncertainty estimation and experimental design for computed tomography with the linearised deep image prior}\label{chap:DIP}  %

\ifpdf
    \graphicspath{{Chapter7/Figs/Raster/}{Chapter7/Figs/PDF/}{Chapter7/Figs/}}
\else
    \graphicspath{{Chapter7/Figs/Vector/}{Chapter7/Figs/}}
\fi

Linear inverse problems in imaging aim to recover an unknown image ${x} \in \R^{d_{x}}$ from measurements $y \in \R^{c}$, which are often modelled by the application of a forward operator $\op \in \R^{c \times d_{x}}$ to the image, and the addition of Gaussian noise $\eps \sim \c{N}(0, \, b^{-1}  I_{c})$. That is
\begin{equation}\label{eq:inverse_problem}
    y = \op {x} + \eps.
\end{equation}
This acquisition model is ubiquitous in machine vision, computed tomography (CT), and magnetic resonance imaging, among other applications. \nomenclature[z-ct]{$CT$}{Computed Tomography}
Due to the inherent ill-posedness of the task (e.g. $c \ll d_{x}$), suitable regularisation, or prior assumptions, are crucial for the stable and accurate recovery of $x$ \citep{tikhonov1977solutions,ito2014inverse}.

In this chapter, we focus on CT. Here, an emitter sends X-ray quanta through the object being scanned. The quanta are captured by $d_{p}$ detector elements placed opposite the emitter. Each row of $\op$ tells us about which regions (pixels) the X-ray quanta will pass through before reaching a detector element. This is illustrated in \cref{fig:CT_diagram}. The number of X-ray quanta measured by a detector pixel conveys information about the attenuation coefficient of the material present along the quanta's path.
This procedure is repeated at $d_{\c{B}}$ angles, yielding a measurement of dimension $c=d_{p}\cdot d_{\c{B}}$, corresponding to the $c \times d_{x}$ sized linear operator $\op$, which is given by the discrete Radon transform.

\begin{figure}[t]
\centering
\includegraphics[width=0.4\linewidth]{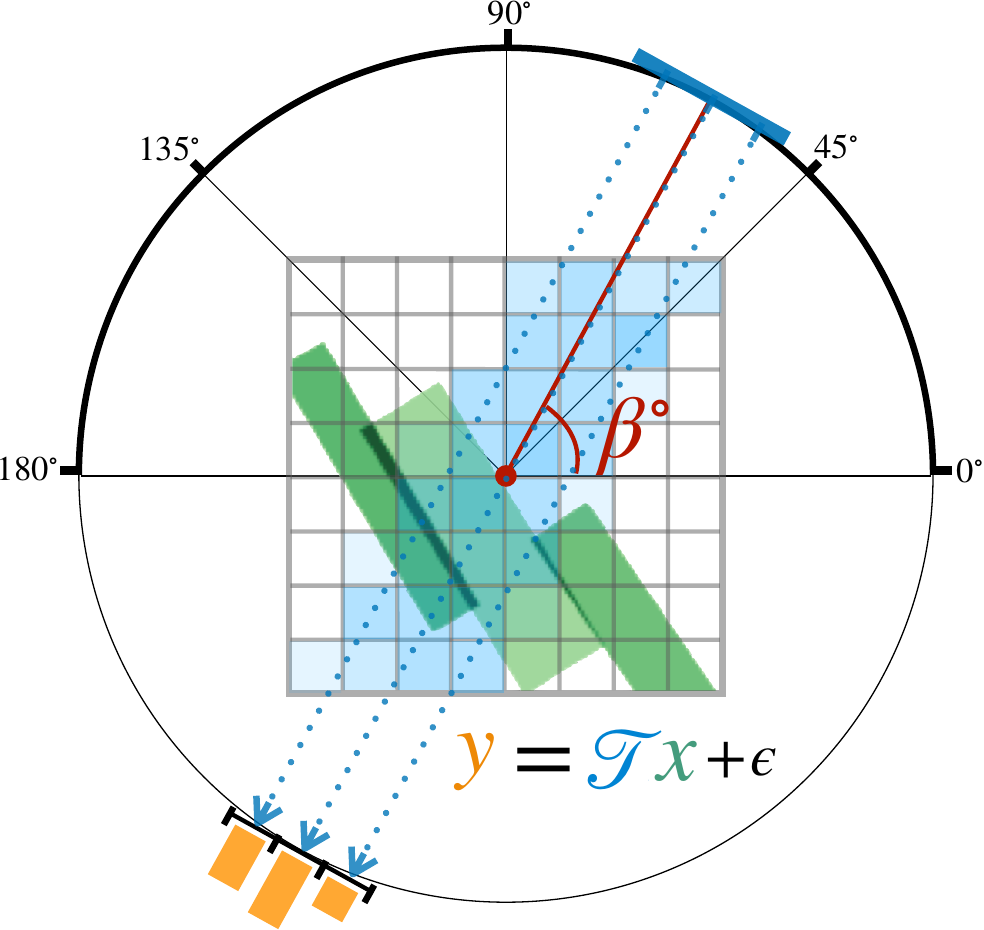}
    \caption{A schematic diagram of 2D parallel beam CT geometry, used in our image reconstruction experiments. In the diagram, the detector is set to angle $\beta$. At this angle, a $d_{p}$ dimensional observation is generated by the application of a $d_{p} \times d_x$ sized block of the $\op$ operator to the input $x$. In this plot, $d_p=3$, $d_x=64$ and the non-zero entries of the $d_{p} \times d_x$ sized block of $\op$ correspond to the pixels with blue colouring that the X-ray quanta pass through. We scan at $d_{\c{B}}$ angles, generating a full $c = d_{p} d_{\c{B}}$ dimensional observation.}
    \label{fig:CT_diagram}
\end{figure}

In recent years, deep-learning based approaches have achieved outstanding performance on a wide variety of tomographic problems \citep{arridge2019solving, ongie2020deep, Wangye:2020}. 
Most deep learning methods are supervised; they rely on large volumes of paired training data. Alas, these often fail to generalise out-of-distribution \citep{antun2020instabilities}; small deviations from the distribution of the training data can lead to severe reconstruction artefacts. Pathologies of this sort call for both unsupervised deep learning methods---free from training data and thus mitigating hallucinatory artefacts \citep{BoraJalal:2017,HeckelHand:2019,Tolle2021meanfield}---and uncertainty quantification 
\citep{kompa2021second,Vasconcelos2022UncertaINR}---informing the user about (un)reliability in reconstructions.

We focus on the deep image prior (DIP), perhaps the most widely adopted unsupervised deep learning approach \citep{Ulyanov:2018}.
DIP regularises the reconstructed image  $\tilde x$ by reparametrising it as the output of a deep convolutional neural network (CNN). It does not require paired training data, relying solely on the structural biases induced by the CNN architecture. 
The DIP has proven effective on tasks ranging from denoising and deblurring to challenging tomographic reconstructions \citep{LiuSunXuKamilov:2019,baguer2020diptv,knopp2021warmstart,DarestaniHeckel:2021,GongLi:2019,cui2021populational,BarutcuGursoyKatsaggelos:2022}.
Nonetheless, the DIP only provides point reconstructions without uncertainty~estimates.

In this chapter, we apply the methods developed through this thesis to equip DIP reconstructions with reliable uncertainty estimates.
In literature, there are two notable probabilistic reformulations of the DIP \citep{Cheng2019Bayesian,Tolle2021meanfield}, but their focus is on preventing overfitting rather than accurately estimating uncertainty.
Distinctly from these, we only estimate the uncertainty associated with a specific reconstruction. We do this by computing Gaussian-linear model type error-bars for a local linearisation of the DIP around its mode \citep{Mackay1992Thesis,Khan19approximate,Immer21Selection}, and refer to the method as \textit{linearised DIP}. Linearised approaches have recently provided state-of-the-art uncertainty estimates for supervised deep learning models \citep{daxberger2021bayesian}. We also explore the incorporation of the total variation (TV) regulariser, ubiquitous in CT reconstruction, as a Bayesian prior for the weights of the linearised model. This regulariser is unnormalised and does not lend itself to standard Laplace (i.e. local Gaussian) approximations \citep{helin2022edgepromoting}. We tackle this issue using predictive complexity prior (PredCP) framework of \cite{nalisnick2020predictive}.

\nomenclature[z-tv]{$TV$}{Total Variation}
\nomenclature[z-dip]{$DIP$}{Deep Image Prior}
\nomenclature[z-predcp]{$\c{P}redCP$}{Predictive Complexity Prior}
\nomenclature[z-cnn]{$CNN$}{Convolutional Neural Network}

We demonstrate our approach on high-resolution CT reconstructions of real-measured 2D and 3d Micro CT ($\mu$CT) projection data. An example of the former is in \cref{fig:walnut_mini}.
Empirically, the method's pixel-wise uncertainty estimates predict reconstruction errors more accurately than existing approaches to uncertainty estimation with the DIP.
This is not at the expense of accuracy in reconstruction: the reconstruction obtained using the standard regularised DIP method \citep{baguer2020diptv} is preserved as the predictive mean, ensuring compatibility with advancements in DIP research.

\begin{figure}[t]
\centering
\includegraphics[width=0.6\textwidth]{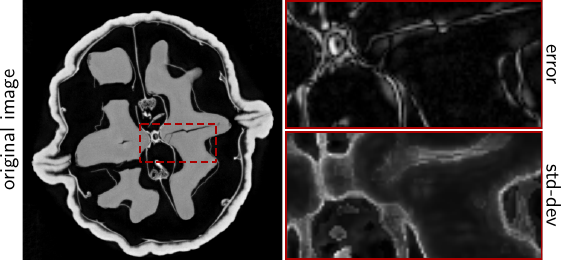}
\caption{X-ray reconstruction ($501{\times}501\, \text{px}^2$) of a walnut (left), the absolute error of its CT reconstruction (top) and pixel-wise uncertainty from the linearised DIP (bottom).}
\label{fig:walnut_mini}
\end{figure}

We then go on to leverage the aforementioned uncertainty estimates to perform adaptive experimental design for CT scan angle selection. We consider a setting where the CT scan is performed in two phases. 
First, a sparse pilot scan is performed to provide data with which to fit adaptive methods.
These are then used to adaptively select angles for a full scan using the linearised deep image prior as a data-dependent prior. 
We demonstrate this procedure with a synthetic dataset where a different ``preferential'' angle is most informative for each image.
Unlike simple linear models, the linearised DIP's designs depend on previously observed targets. This adaptivity allows linearised DIP designs to outperform the equidistant angle baseline, which is almost always used in deployment.

The contributions of this chapter can be summarised as follows.
\begin{itemize}%
    \item We propose a novel approach to bestow reconstructions from the TV-regularised DIP with uncertainty estimates, by linearising the DIP around its optimised reconstruction and providing the linear model's error-bars as a surrogate for those of the DIP. We perform sample-based EM inference in this model, scaling to high resolution real-measured 2d reconstructions and 3d volumetric reconstructions. To be best of our knowledge, this is the first instance of uncertainty estimation for NN-based 3d volumetric CT reconstruction. Our approach yields far more accurate uncertainty estimation than existing probabilistic formulations of the DIP.
    \item We leverage the linearised DIP as a data-dependent prior for CT experimental design. This allows us to perform adaptive design, where successive acquisition locations dependent on previous observed targets (as opposed to only the input locations), while preserving the tractability of a linear model. This method outperforms equidistant angle selection on a synthetic task.
\end{itemize}

The rest of this chapter is organised as follows. \Cref{sec:dip_preliminaries} covers the DIP, the TV regulariser and other preliminaries not introduced earlier in the thesis. \Cref{sec:linerised_DIP} presents the linearised DIP and a novel TV-based prior for the linear model's parameters. \cref{sec:dip_inference} discusses efficient approaches to inference.
\Cref{sec:dip_uncertainty_experiments} presents a demonstration of the linearised DIP on real-measured high-resolution $\mu$CT data. \nomenclature[z-mct]{$\mu CT$}{Micro Computed Tomography} \cref{sec:experimental_design} introduces experimental design with linear models and discusses how the linearised DIP can be used within this framework. Finally, \cref{sec:dip_design_experiments}, demonstrates our approach to experimental design on synthetic data, and \cref{sec:dip_conclusion} concludes the chapter.

\section{Preliminaries}\label{sec:dip_preliminaries}

This section reviews some CT-specific concepts that were not covered in earlier chapters of the thesis. 

\subsection{Total variation regularisation}
The imaging problem, given in \cref{eq:inverse_problem}, admits a linear subspace of solutions consistent with the observation $y$\footnote{This statement disregards the effects of the observation noise, which introduces a strictly convex constraint, but in practise it is very weak.}. Thus, regularisation is needed for stable reconstruction.
Total variation (TV) is perhaps the most well established regulariser \citep{rudin1992nonlinear,chambolle2010introduction}. The anisotropic TV semi-norm of an image vector $x \in \R^{d_{x}}$ imposes an $L_1$ constraint on image gradients:
\begin{align}\label{eq:TV_equation}
    \text{TV}(x) \!=\! \sum_{i,j} | \mathfrak{x}_{i,j} - \mathfrak{x}_{i+1,j}| + \sum_{i,j} |\mathfrak{x}_{i,j} - \mathfrak{x}_{i,j+1}|,
\end{align}
where  $\mathfrak{x} \in \R^{\mathfrak{h} \times \mathfrak{w}}$ denotes the vector $x$ reshaped into an image of height $\mathfrak{h}$ by width $\mathfrak{w}$, and $d_{x} = \mathfrak{h}\cdot \mathfrak{w}$.
This leads to the regularised reconstruction formulation
\begin{equation}\label{eq:simple_optimisation_objective}
\tilde x \in \argmin_{x\in\R^{d_x}} \mathcal{L}(x) \quad \text{with} \quad \mathcal{L}(x) \coloneqq b \|\op x - y\|^2_2 +  \mathrm{TV}(x),
\end{equation}
where the hyperparameter $b>0$ determines the strength of the regularisation relative to the fit term.

\subsection{Bayesian inference for inverse problems}

The Bayesian framework provides a consistent approach to uncertainty estimation in imaging problems \citep{KaipioSomersalo:2005,Stuart,Seeger11scale}. The image to be recovered is treated as a random variable.
Instead of finding a single best reconstruction $\tilde{x}$, we aim to find a posterior $\c{P}_{x|y}$, with density $\rho(x | y)$ that scores every candidate $x\in\R^{d_{x}}$ according to its agreement with the observation $y$ and prior density $\rho(x)$. The loss in \cref{eq:simple_optimisation_objective}
can be viewed as the negative log of an unnormalised posterior density, i.e.\ $\rho(x | y) \!\propto\! \exp(- \mathcal{L}(x))$, and $\tilde x$ as its mode, i.e.\ the \textit{maximum a posteriori} (MAP) estimate. The least squares loss corresponds to a Gaussian likelihood  $p(y | x) = \N (y; \, \op x, I)$ and the TV regulariser to a prior over images $\c{P}$ with density $\rho(x) \propto \exp(-\lambda \TV (x))$. With this, we are ready to crank the lever of Bayesian reasoning, as introduced in \cref{subsec:linear_posterior_inference}.

Our work partially departs from this framework in that it \emph{solely concerns itself with characterising plausible reconstructions around the mode $\tilde{x}$} \citep{Mackay1992Thesis}. 
This has two key advantages, 1) \emph{tractability}: the likelihood induced by NN reconstructions is strongly multi-modal, and both analytically and computationally intractable. In contrast, the posterior for the local model is Gaussian; 2) \emph{interpretablity}: even if we could obtain the full posterior, downstream stakeholders not versed in probability are likely to have little use for it.
A single reconstruction and its pixel-wise uncertainty may be more interpretable to end-users \citep{antoran2021clue,Bhatt2021transparency}.

\subsection{The Deep Image Prior (DIP)} \label{subsec:preliminaries_DIP}

The DIP \citep{Ulyanov:2018,Ulyanov2020deep} reparametrises the reconstructed image as the output of a CNN $g : \R^d \to \R^{d_x}$ with learnable parameters $v \in\R^{d}$ and a fixed input, which we have omitted from our notation for clarity.
The DIP can be seen as a reparametrisation of the reconstructed image that provides a favourable structural bias.
We introduce the optimisation problem  
\begin{equation}\label{eq:DIP_MAP-obj2}
    \tilde{v} \in\argmin_{v\in\R^{d}} b \|\op g(v) - y\|_2^2 +  \TV(g(v)),
\end{equation}
and the recovered image is given by $\tilde x =g(\tilde{v})$. Penalising the TV of the DIP's output avoids the need for early stopping and improves reconstruction fidelity \citep{LiuSunXuKamilov:2019,baguer2020diptv}. 
The standard choice of CNN architecture is the fully convolutional U-net \citep{ronneberger2015u}. We also adopt this architecture in this chapter.  Although the parameters $v$ must be optimised separately for each new measurement $y$, we follow \citep{barbano2021deep,knopp2021warmstart} to reduce the cost with task-agnostic pretraining.

Since its introduction by \cite{Ulyanov:2018,Ulyanov2020deep}, the DIP has been improved with early stopping \citep{Wang2021early}, TV regularisation \citep{LiuSunXuKamilov:2019,baguer2020diptv}, and pretraining  \citep{barbano2021deep,knopp2021warmstart,barbano2023image}.
We build upon these recent advancements by providing a scalable method to estimate the error-bars of DIP's reconstructions.
This is a relatively unexplored topic.
Building upon \cite{aga2018cnngp} and \cite{Novak19CNN}, \cite{Cheng2019Bayesian} show that in the infinite-channel limit, the DIP converges to a Gaussian process (GP). In the finite-channel regime, the authors
approximate the posterior distribution over the DIP's parameters with stochastic gradient Langevin dynamics (SGLD) \citep{Welling11Langevin}.
\cite{laves2020uncertainty} and \cite{Tolle2021meanfield} use factorised Gaussian variational inference \citep{blundell2015weight} and MC dropout \citep{hron18dropout,Vasconcelos2022UncertaINR}, respectively. 
These probabilistic treatments of DIP primarily aim to prevent overfitting, as opposed to accurately estimating uncertainty.
While they can deliver uncertainty estimates, their quality tends to be poor.
In fact, obtaining reliable uncertainty estimates from deep-learning based approaches, like the DIP, largely remains a challenging open problem \citep{antoran2019understanding,Snoek19can, ashukha2020pitfalls,Foong20Approximate, barbano2022probabilistic, antoran2020depth}.

\section{Linearised DIP uncertainty estimation for CT}\label{sec:linerised_DIP}

In this section, we build a probabilistic model to characterise the uncertainty associated with reconstructions around $\tilde{v}$, a mode of the regularised DIP objective obtained using \cref{eq:DIP_MAP-obj2}. \Cref{subsec:linearised_prior} describes the construction of a linearised surrogate for the DIP reconstruction. \Cref{subsec:posterior_predictive} describes how to compute the surrogate model's error-bars and use them to augment the DIP reconstruction. \Cref{subsec:tv_for_NN} discusses how we include the effects of TV regularisation into the surrogate model. 

\subsection{From a prior over parameters to a prior over images}\label{subsec:linearised_prior}

After training the DIP to an optimal TV-regularised setting $\tilde{x} = g(\tilde{v})$ using \cref{eq:DIP_MAP-obj2}, we linearise the network around $\tilde{v}$ by applying \cref{eq:linearised_model}, and obtain the {affine in $w\in\R^d$} function $h(w)$. The error-bars obtained from Bayesian inference with $h(w)$ will tell us about the uncertainty in $\tilde{x}$.
To this end, consider the Bayesian model,
\begin{gather}  \notag
    y|w \sim \N ( \op h(w), b^{-1} I), \quad w  \sim \N ( 0, A^{-1}) \\
    \text{with} \quad  h(w) \coloneqq g(\tilde{v}) + \phi(w - \tilde{v}),\label{eq:Model_weight_space} 
\end{gather}
where $\phi  = \partial_v g(\tilde v) \in \R^{c \times d}$ is the Jacobian of our NN\footnote{In this chapter, $\phi$ is a matrix, as opposed to a function that returns a matrix, because our NN's input is clamped to a constant.}. We will select the precision $A$ to incorporate TV constraints into the computed error-bars in \cref{subsec:tv_for_NN}. We have introduced the noise variance $b^{-1}$ as an additional hyperparameter which we will learn using the marginal likelihood.

\subsubsection{Demonstration: sampling from the linearised DIP prior}

To provide intuition about the linearised model, we push samples from $w\sim \N ({0}, A^{-1})$, through $h$. 
The resulting reconstruction samples are drawn from a Gaussian distribution with covariance $K \in \R^{d_{x} \times d_{x}}$ given by $\phi A^{-1}\phi^{\top}$. We show an example in \cref{fig:samples_from_priors}.
Here, the Jacobian $\phi$ introduces structure from the NN function around the linearisation point $\tilde v$. For this example, we train our NN on CT data simulated by using the KMNIST dataset as the original images. Thus, out prior samples contain features from the KMNIST character that the DIP was trained on.

\begin{figure}[t]
    \centering
    \includegraphics[width=0.6\textwidth]{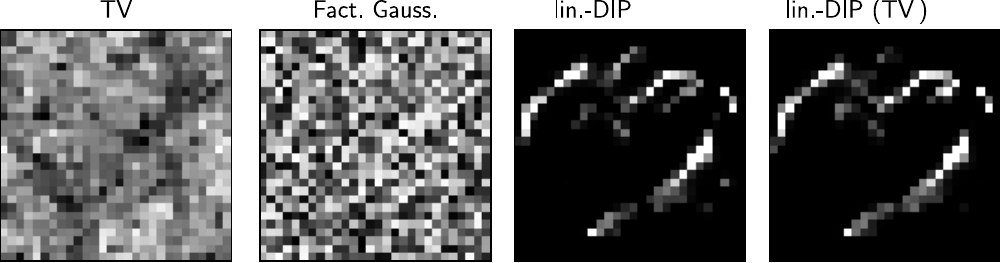}
    \caption{Samples from different priors over the reconstructed image $x$.
    From left to right, the plots show samples from the TV prior with density $\propto \exp(-\TV(x))$, drawn with HMC, from an isotropic Gaussian prior, from a linearised DIP trained on a MNIST character, and from the same model but paired with the TV-PredCP prior over the weights introduced in \cref{subsec:tv_for_NN}.
    The latter leads to smoother samples with less artefacts than the standalone linearised DIP prior.
    }
    \label{fig:samples_from_priors}
\end{figure}

\subsection{Computing the predictive uncertainty}\label{subsec:posterior_predictive} 
We augment the DIP reconstruction $\tilde{x}$ with Gaussian predictive error-bars computed with the linearised model $h$ described in \cref{eq:Model_weight_space}. This yields the predictive distribution $\N (\tilde x, K_{x|y})$.
Denoting the reconstruction space kernel matrix $K = \phi A^{-1}\phi^{\top} \in \R^{d_x \times d_x}$, the observation space prior covariance $K_{yy} = TKT^{\top} \in \R^{c\times c}$, and the cross terms $K_{xy} = KT^{\top} \in \R^{d_x \times c}$,
the posterior covariance $K_{x|y} \in \R^{d_x \times d_x}$ is given by 
\begin{align}\label{eq:posterior_predictive}
    K_{x|y} &=  \phi(b^2 \phi^\top \op^\top  \op \phi + A)^{-1} \phi^\top = K - K_{xy} (K_{y y} + b^{-1} I)^{-1} K_{xy}^{\top}.
\end{align}
Importantly, \cref{eq:posterior_predictive} depends on the inverse of the observation space covariance $K_{y y}  + b^{-1} I$, which we expect to be much lower dimensional than the covariance over reconstructions, or parameters. 
Thus, the cost of computing \cref{eq:posterior_predictive} scales as $\mathcal{O}(d_{x} c^{2})$.

\subsection{Incorporating TV-smoothness into the prior over the weights} \label{subsec:tv_for_NN}

This section aims to design a prior that constraints $h$'s error-bars, such that the model only considers low TV reconstructions as plausible.
Our architecture if fully convolutional. We follow the guiding intuition that if the CNN's filters are smooth, its output will be so as well. 
With this, we place a block-diagonal Matérn-$\nicefrac{1}{2}$ covariance Gaussian prior on the linearised model's weights, similarly to \cite{fortuin2021bayesian}.
In particular, we introduce dependencies between parameters in the same CNN filter by constructing $A$ as a block diagonal matrix. We denote the block corresponding to each filter as $A_k$, where $k$ indexes the filter. These matrix blocks are given by
\begin{equation}\label{eq:CovFunc}
     [A^{-1}_{k}]_{ij, i'j'} = 
            a^{-1}_k\exp\Big({\frac{-\sqrt{(i-i')^2+(j-j')^2}}{\psi_{k}}}\Big),
\end{equation} %
where the tuple $(i,j)$ indexes the spatial location of a specific filter pixel in terms of height and width.
The per-filter lengthscale $\psi_{k}$ regulates the filter smoothness. 
The hyperparameter $a^{-1}_{k}$ determines the marginal prior variance for each filter. 
Both parameters are shared among all filters in an architectural block in the U-net, indexed by $k \in \{1,2,  \ldots , r\}$.  We write $\psi = [\psi_{1}, \psi_{2}, \ldots , \psi_{r}]$ and $a^{-1} = [a^{-1}_{1}, a^{-1}_{2}, \ldots , a^{-1}_{r}]$. A diagram of our U-net architecture that highlights all architectural blocks is provided in \cref{fig:unet-diagram}.  The chosen U-net architecture is fully convolutional and thus \cref{eq:CovFunc} applies to all parameters, reducing to a diagonal covariance for $1\times 1$ convolutions. 

\begin{figure}[t]
    \includegraphics[width=0.6\textwidth]{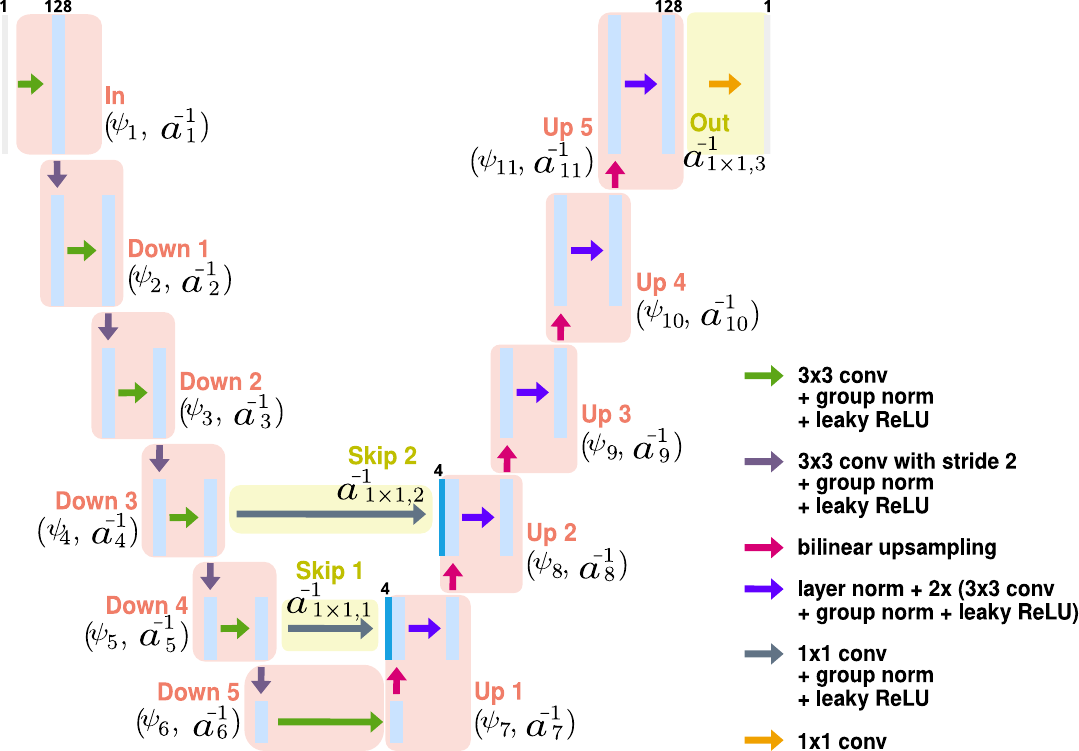}
    \caption{A schematic of the U-net architecture used in our 2d $\mu$CT experiments experiments. For KMNIST, we use a reduced, 3-scale U-net without group norm layers. Each light-blue rectangle corresponds to a multi-channel feature map. We highlight the architectural components corresponding to each block ${1,\dotsc,D}$ for which a separate prior is defined with red and yellow boxes.}
    \label{fig:unet-diagram}
\end{figure}

In \cref{fig:monotonicity_linear_regime}, we experimentally verify that an image generated from a linearised NN prior with smoother filters will present lower TV. In particular, we find a bijective relationship between the each filter's lengthscale $\psi_k$ and the expected TV $\E_{w_k \sim \N( 0,  A_k^{-1} )}[\TV(\phi w)]$ where $w_k$ is a sub-vector of $w$ and the other filter's parameters (non $w_k$) are held fixed.  This suggests we may use the predictive complexity prior (PredCP) framework of \cite{nalisnick2020predictive} to construct a prior over the parameters which acts as a surrogate for the $\TV$ prior. 
In particular, we construct a prior over the lengthscale parameters as:
\begin{gather}\label{eq:ell_prior_DIP}
    \prod_{k=1}^{r}\text{Exp}(\kappa_{k})\left|\frac{\partial \kappa_{k}}{\partial \psi_{k}}\right|, \\
    \text{with }\kappa_{k} := \E_{
    w_k \sim \N (0, A^{-1}_k)
    \prod^{r}_{i=1,i\neq k} \delta(w_{i})}\left[{\TV}(\phi w) \right]\label{eq:kappa_definition}
\end{gather}
where the subscript $_k$ indicates we select the subvector of weights corresponding to CNN block $k$.
We have related a block's contribution to the expected $\TV$, $\kappa_k$, to the block's filter lengthscale $\psi_{k}$ via the change of variables formula. The independence across blocks assumed in \cref{eq:ell_prior_DIP} ensures dimensionality preservation, formally needed in the change of variables. It follows from the triangle inequality that $\sum \kappa_{k}$ is an upper bound on the expectation under the distribution $\E_{w \sim \N (0, A^{-1})}[\TV(\phi w)]$, further motivating the factorisation.

\begin{figure}[t]
    \centering
    \includegraphics[trim={0 0.85cm 0 0},height=1.3in,clip, width=0.95\linewidth]{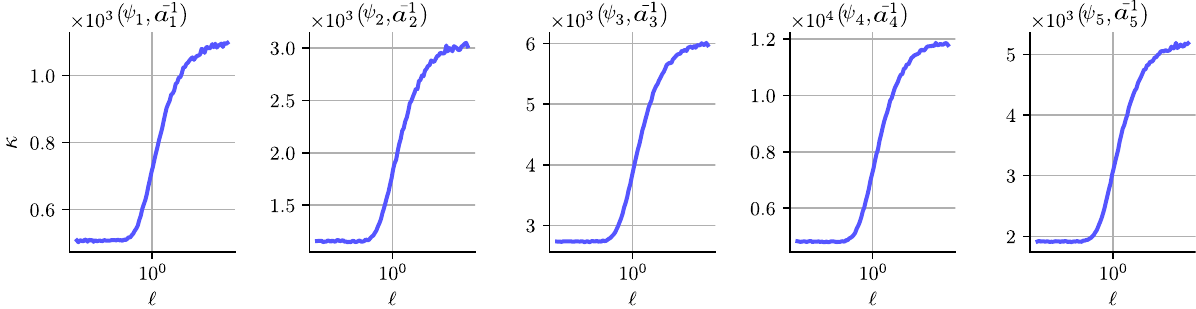}\\[0.5em]
    \includegraphics[width=0.95\linewidth,height=1.4in,trim={0 0 0 0.5cm},height=1.3in,clip]{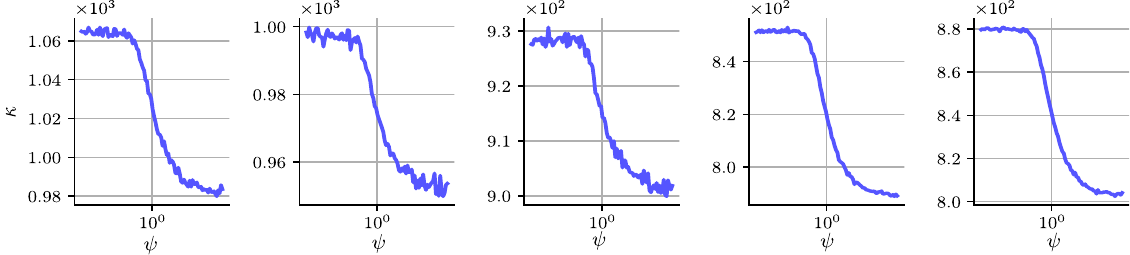}
    \caption{
    Experimental evidence of the monotonicity (and thus invertibility) of the relationship between a CNN block's lengthscale $\psi_k$ and the expected TV $\kappa = \E_{w_k \sim \N( 0,  A_k^{-1} )}[\TV(\phi w)]$,  computed across 50 linearised U-nets trained on different the KMNIST images. The horizontal axis represents lengthscale $\psi \in [0.01, 100]$. $\kappa$ is estimated with $10$k Monte Carlo samples. In the bottom row we scale the marginal variances of $JA^{-1}J^\top$ to be 1 for every value of $\psi$. This decouples $\psi$ from $a^{-1}$, allowing us to observe the smoothing effect from larger lengthscales. }
    \label{fig:monotonicity_linear_regime}
\end{figure}

\begin{derivation} \textbf{Factorising yields an upper bound on the expected TV} \\

Let $\c{S}$ be the set of indices for all adjacent pixel pairs in an image. These images are flattened into $d_x$ length vectors and thus can be indexed by a single number. We denote by $\phi_j \in \R^d$ the row of the Jacobian $\phi$ corresponding to pixel $i$.  We denote by $\phi_{jk}$ the Jacobian row subvector corresponding to  pixel $i$ and weights in NN block $k$. With this
\begin{align*}
 \E_{w\sim\N (0, A^{-1})}\left[{\TV}(\phi w) \right] &= \sum_{(i,j) \in \c{S}} \E_{w\sim \N (0, A^{-1})} | (\phi_{i} w -  \phi_{j} w) |  \\
 &= \sum_{(i,j) \in \c{S}} \E_{w_k\sim \N (0, A^{-1})} | \sum_k^r (\phi_{ik} -  \phi_{jk}) w_k) |] \\
&\leq \sum_{(i,j) \in \c{S}}   \sum_k^r \E_{w_k \sim \N ( 0, A^{-1}_k)}\left[| (\phi_{ik} -  \phi_{jk}) w_k | \right] \\
&=  \sum_k^r \E_{\N (0, A^{-1}_k) \prod^{r}_{i=1,i\neq k}\delta(w_{i})}\left[\sum_{(i,j) \in \c{S}} | (\phi_{i} -  \phi_{j}) w | \right] \\
&= \sum_k^r \kappa_k,
\end{align*}
 
Thus, the separable form of the TV prior as a regulariser ensures that the expected TV under the joint distribution of parameters is also regularised.
\end{derivation}

Note that \cref{eq:ell_prior_DIP} can be computed analytically. However, its direct computation is costly and we instead rely on numerical methods, described in \cref{subsec:CG_for_TVPredCP}. In \cref{fig:samples_from_priors}  we show samples from the linearised NN model where $\psi$ is chosen using the marginal likelihood with TV-PredCP constraints.
Incorporating the TV-PredCP leads to smoother samples with less discontinuities.

\section{Approaches to scalable inference and hyperparameter learning}\label{sec:dip_inference}
In a typical tomography setting, the dimensionality $d_x$ of the image $x$ and $c$ of the observation $y$ can be large, e.g.\ $d_{x} > {\rm1e5}$ and $c >{\rm 5e3}$. Thus holding the input space covariance matrices (e.g.\ $K$ and $K_{x|y}$) in memory is infeasible. This also complicates computing determinants, needed to evaluate Gaussian densities, and to learn hyperparameters. Following \cref{chap:sampled_Laplace}, we develop a series of approaches that avoid instantiating these matrices explicitly. We only access Jacobian and covariance matrices through matrix–vector products.

\cref{subsec:CG_for_TVPredCP} introduces a hyperparameter learning objective that combines the linearised model's evidence with the TV-PredCP prior over filter lengthscales. We approximate the objective's gradients with CG.  \cref{subsec:precond} discusses the computation of a randomised preconditioner for CG. \cref{g-prior-sample-CG} discards the TV-PredCP prior in favour of the g-prior. This allows us to employ the sample-based EM iteration from \cref{chap:sampled_Laplace}, in combination with CG, to accelerate inference. \cref{subsec:sampling_for_3d} discusses the extension of the latter algorithm to very large 3d volumetric reconstructions by substituting CG solves with SGD (as suggested in \cref{chap:SGD_GPs}). Finally, we discuss making sample-based predictions that model covariances between pixels in \cref{sec:post_sampling}

\subsection{Conjugate-gradient hyperparameter learning for the PredCP TV prior}\label{subsec:CG_for_TVPredCP}

In this subsection we consider hyperparameter learning with the TV-PredCP prior introduced in \cref{subsec:tv_for_NN}. Here, the prior precision $A$ is parametrised in terms of the vectors of block-wise marginal variances $a^{-1}\in \R^d$ and block-wise lengthscales $\psi \in \R^d$. To learn $\psi$, we combine the above objective with the TV-PredCP's log-density, which acts as a regulariser. 
The resulting expression used to learn the full set of hyperparameters $(b^{-1}, a^{-1}, \psi)$ resembles a Type-II MAP \citep{williams2006gaussian} objective
\begin{align}
    \log  \,p(y |\psi; \, b^{-1}, a^{-1} ) &+ 
    \log p(\psi; a^{-1})   \notag \\
   \approx  & - \frac{1}{2} b ||y - \op g(\tilde{v})||_{2}^{2} - \frac{1}{2} w_\star^{\top} A w_\star -\frac{1}{2} \log |K_{y y} + b^{-1} I| \notag \\
   &- \sum_{k=1}^{r} \kappa_{k} + \log \left|\frac{\partial \kappa_{k}}{\partial \psi_{k}}\right| +  C, \label{eq:type2MAP}
\end{align}
where $C$ is independent of the hyperparameters and the vector $w_\star \in \R^{d_{w}}$ is the posterior mean of the linear model's parameters (see \cref{sec:choice_of_posterior_mode}). We compute it as
\begin{gather}
    w_\star = A^{-1} \phi^\top \op^\top (K_{yy} + b^{-1} I)^{-1}\left(y + \op \left( \phi\tilde v - g(\tilde v) \right) \right) \label{eq:dip_post_mode}
\end{gather}
and solve the therein contained linear system with CG. The vector we solve against consists of the observations offset by the constant in $w$ terms in the tangent linear model \cref{eq:Model_weight_space}.

The remaining bottleneck in evaluating \cref{eq:type2MAP} is the log-determinant $\log|K_{yy} + b^{-1}I|$, which has a cost $\mathcal{O}(c^{3})$.
Alas, we cannot apply the sample-based MacKay update from \cref{subsec:Mackay_update} to learn hyperparameters other than the entries of a diagonal prior precision matrix. Thus, we resort to gradient descent with CG-based log-determinant gradient trace estimation, as described in \cref{sec:CG_hyperparam_learning}. We use a preconditioner, which we describe in \cref{subsec:precond}. Despite this, the large computational cost associated with this method only allows us to perform a single EM step. We summarise the procedure in \cref{alg:lin-dip}.  We go on to describe efficient estimation of the TV-PreCP term gradients.  

\begin{algorithm2e}
    \SetAlgoLined
    \SetNoFillComment
    \DontPrintSemicolon
    \SetSideCommentLeft
    \KwData{noisy measurements $y$, a CNN $g(\cdot)$, operator $\c{T}$, initial prior precision $A$.
    } 
    $\tilde{v} \gets$ \texttt{fit\_DIP(}$\c{T}, y, g(v)$\texttt{)} \tcp*{by minimising \cref{eq:DIP_MAP-obj2}}
    $w_\star \gets$ \texttt{fit\_linearised\_model(}$\c{T}, y, g(\tilde v)$\texttt{)}\tcp*{using \cref{eq:dip_post_mode}}
    $\c{P} \gets$ \texttt{compute\_preconditioner(} $\c{T}, g(\tilde v), A$ \texttt{)}\tcp*{following \cref{subsec:precond}.}
    ${\sigma}^{2}_{y}, \{a^{-1}_{k}, \psi_{k}\}_{k=1}^{r} \gets$ \texttt{optimise\_hyperparams(}$\c{T}, y, g(\tilde v), w_\star, \c{P}$\texttt{)}\tcp*{with \cref{eq:type2MAP}, and  \cref{eq:tv_grad_est,eq:tv_grad_est_2}. Use preconditioned CG.}
     $\hat{K}_{x|y} \gets$ \texttt{fast\_sampling(}$\c{T}, g(\tilde v), {\sigma}^{2}_{y}, \{a^{-1}_{k}, \psi_{k}\}_{k=1}^{D}, \c{P} $\texttt{)} \tcp*{following \cref{sec:post_sampling} and with preconditioned CG.}
\KwResult{mean reconstruction $g(\tilde{v})$, posterior covariance estimate $\hat{K}_{x|y}$}
\caption{Linearised deep image prior PredCP-TV inference}
\label{alg:lin-dip}
\end{algorithm2e}

\subsubsection{MC sampling for TV-PredCP optimisation}\label{sec:predcptv_computation}

For large images, exact evaluation of the expected TV with \cref{eq:kappa_definition}  is computationally intractable. Instead, we estimate the gradient of $\kappa_k$ with respect to $\theta = (\sigma^2, \psi)$ using a Monte-Carlo approximation of the expectation
\begin{equation}
    \frac{\partial \kappa_{k}}{\partial\theta} = \E_{w_{k}\sim \N (0, A^{-1}_k)}\left[ \frac{\partial \TV(x)}{\partial x} \phi_{k}\frac{\partial w_{k} }{\partial \theta}\right],
    \label{eq:tv_grad_est}
\end{equation}
where $\phi_k=\frac{\partial g(v)}{\partial v_k}|_{v=\tilde v}$. $\frac{\partial \TV(x)}{\partial x}$ is evaluated at the sample $x{=}\phi_{k}w_{k}$ and $\frac{\partial w_{k}}{\partial \theta}$ is the reparametrisation gradient for $w_{k}$, a prior sample of the weights of CNN block $k$.
The gradient for the Jacobian log-determinant term is
\(\frac{\partial}{\partial_{\theta}} \log\left|\frac{\partial \kappa_{k}}{\partial \theta}\right| = \frac{\partial^{2}\kappa_{k}}{\partial \theta^{2}}/\frac{\partial \kappa_{k}}{\partial \theta}\), and since the second derivative of the TV semi-norm is almost everywhere zero, we have
\begin{gather}\label{eq:tv_grad_est_2}
    \frac{\partial^2 \kappa_{k}}{\partial\theta^2} = \E_{w_{k}\sim \N (0, A^{-1}_k)}\left[ \frac{\partial \TV(x)}{\partial \theta} \phi_{k}
    \frac{\partial^2 w_{k}}{\partial \theta^2}
    \right].
\end{gather}

\subsubsection{Demonstration: blockwise lengthscale optimisation}

In \cref{fig:walnut_mrglik_opt_hyper-main}, demonstrate the use of \cref{eq:type2MAP} to learn the full set of prior hyperparameters depicted in \cref{fig:unet-diagram}. We also ablate the TV-PredCP regulariser to better understand its effects. We refer to the ablated setting as ``MLL'' and the setting where the TV-PredCP is kept as ``Type-II MAP''. We use the high-resolution real-measured dataset of \cite{der_sarkissian2019walnuts_data} and provide full details on the experimental setup in \cref{subsec:2d_dip_demo}.

\begin{figure}[t]
    \centering
    \includegraphics[width=\textwidth]{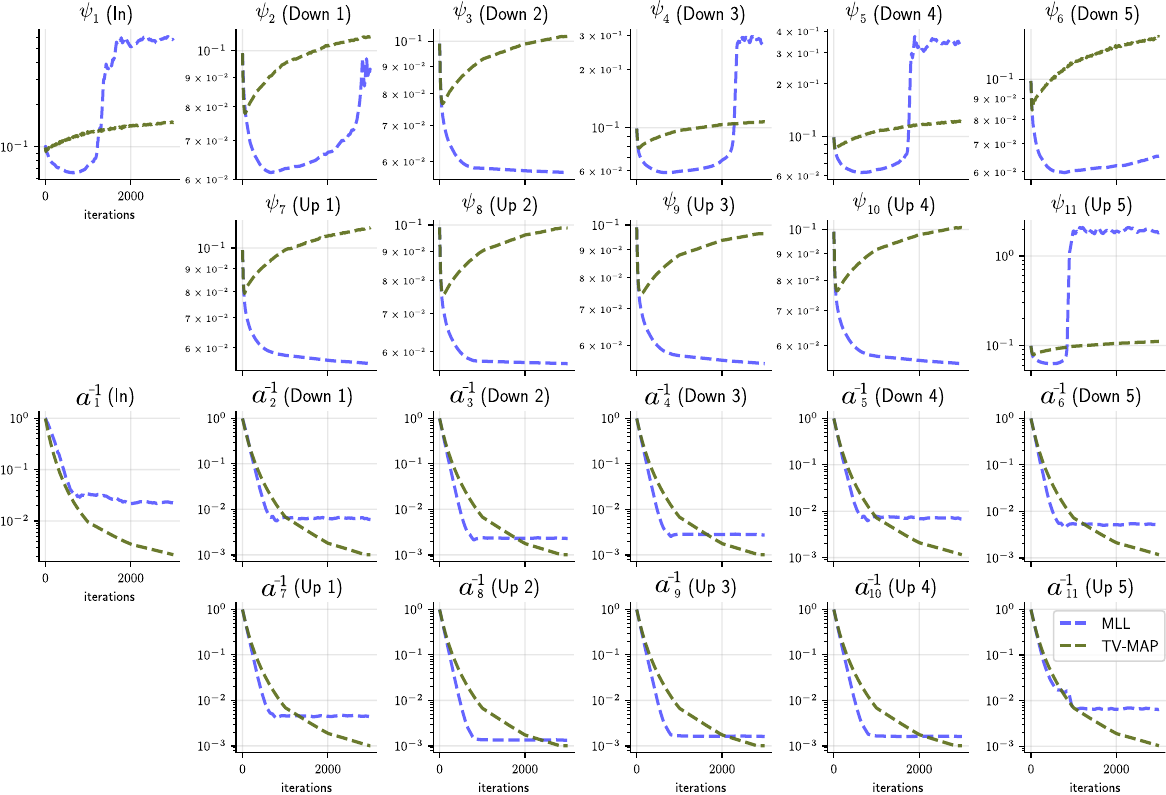}
        \caption{
        Optimisation traces for the lengthscales and marginal variances corresponding to our U-net's $3 \times 3$ convolution layers. We consider both MLL and Type-II MAP and we use the Walnut data described in \cref{sec:dip_uncertainty_experiments}. The TV-PredCP leads to larger prior lengthscales $\psi$ and lower variances $a^{-1}$.
        }
    \label{fig:walnut_mrglik_opt_hyper-main}
\end{figure}

During MLL and Type-II MAP optimisation, many layers' prior variance goes to $a^{-1} \approx 0$.  This phenomenon is known as ``automatic relevance determination'' \citep{Mackay1996BAyESIANNM,Tipping:2001}, and simplifies our linearised network, preventing uncertainty overestimation. 
Type-II MAP hyperparameters optimisation drives  $\psi$ to larger values, compared to MLL. 
This restricts the linearised DIP prior, and thus the induced posterior, to functions that are smooth in a TV sense, leading to smaller error-bars.

\subsection{Randomised SVD preconditioning for CG}\label{subsec:precond}

CG's convergence can greatly be accelerated by using a preconditioner $\c{P}^{-1}$ which approximates $(K_{yy} + b^{-1})^{-1}$. We use randomised SVD methods \citep{Halko2011structure,martinsson2020randomized}, as we find them to be more numerically stable and provide better performance than  pivoted Cholesky methods, despite the latter being more common in the literature \citep{WangPGT2019exactgp}.
Our preconditioner is based on a randomised approximation of
$K_{yy} = \op \phi A^{-1}\phi^{\top}\op^{\top}$
as $ \tilde U \tilde \Lambda \tilde U^{\top}$, where $\tilde U \in \R^{c \times r}$ and $r\ll c$. The steps to construct this approximation are
\begin{itemize}
    \item Constructing a standard normal test matrix $ R \in \R^{c \times r}$ with entries sampled from $\N(0, I)$. 
\item Computing the (thin) QR decomposition $\tilde Q \tilde R = K_{yy}  R$ where $\tilde Q\in \R^{c\times r}$ is an orthonormal matrix.
\item Constructing symmetric matrix $B\in\R^{r\times r}$ (much smaller than $K_{yy}$) as $B = Q^T  K_{yy} Q$.
\item Computing the eigendecomposition  $B = V\Lambda V^\top$, and recovering $\tilde U = QV$.
\end{itemize}
This method requires $\mathcal{O}(r)$ matrix vector products with $K_{yy}$ to construct not only an approximate basis but also its complete factorisation. 
Our approximation is $\c{P} = \tilde U \tilde \Lambda \tilde U^{\top} + b^{-1} I$
The final step is to invert $\c{P}$ in $\c{O}(r^3)$ time using the Woodbury identity 
\begin{gather*}
    \c{P}^{-1} = bI - b^2 \tilde U   (b\tilde U^{\top} \tilde U + \tilde \Lambda^{-1})^{-1} \tilde U^{\top}.
\end{gather*}
We choose a value of $r=400$.

\begin{algorithm2e}[p]\LinesNotNumbered
    \KwData{Linearised network $h$, linearisation point $\tilde v$, measurements $y$, discrete Radon transform $\op$, U-Net Jacobian $\phi$, initial precision $ a > 0$, number of samples $m$, noise precision $B = bI$}\vspace{2mm}
    \DontPrintSemicolon 
    \SetKwFunction{FMain}{Kvp}
    \SetKwProg{Fn}{Function}{:}{}
    \Fn{\FMain{$v$, $ a$, $\op\phi$, s, $\mathrm{B}^{-1}$}}{
        \KwRet $ \op\phi (  a^{-1} \text{diag}(s^{\odot 2}) ) \phi^T \op^T v + \mathrm{B}^{-1}v$ \;
      }\vspace{2mm}
    $s \gets ( \sum_{i<c} (\op_{i}\phi)^{ \odot 2} )^{\nicefrac{-1}{2}} $ \tcp*{$i$ indexes output dimensions and $\odot$ refers to an operation applied to vector entries elementwise.}
	\While{ $ a$ has not converged}{
	$\c{P} \gets \text{Compute-preconditioner}(\texttt{Kvp})$\\
	\For{$j=1, \dotsc, m$}{
        $\zeta_j^0 \gets \op\phi(s \odot  w_{0,i}) + \c{E}_j \,\,\,
    \text{where}\,\,\, \c{E}_j \sim \N(0, \mathrm{B}^{-1}) \,\,\, \text{and} \,\,\,  w_{0,i} \sim \N(0, A^{-1})$\\
	$\texttt{c}_j \gets \text{CG}\left( \texttt{Kvp}, \zeta_j^0,\,\,   \text{precond.}{=}\c{P}\right)$ \\
        $\zeta_{j} \gets w_{0,j} - \op\phi (  a^{-1} \text{diag}(s^{\odot 2}) ) \phi^T \op^T \texttt{c}_{j}$
	}
        $\delta \gets \op(\phi \tilde v - h(\tilde v))$\\
        $\texttt{c} \gets \text{CG}\left(\texttt{Kvp}, \,
        y{+}\delta,\,\,\text{precond.}{=}\c{P}\right)$\\
	$w_\star \gets s \odot  a^{-1}\phi^T\op^T \texttt{c}$\\
	$\hat \gamma \gets \frac{1}{m} \sum_{j=1}^m   
        \|\op\phi( s \odot\zeta_j)\|_{2}^2$ \\
	$ a' \gets \hat \gamma / \|w_\star\|^{2}_2$ \\
	$ a \gets  a'$\\
 }
\KwResult{Optimised precision $a$}
\caption{Kernelised sampling-based linearised inference for CT }
\label{alg:sampling_kernelised}
\end{algorithm2e}

\subsection{Scalable sample-based hyperparameter learning with the g-prior}\label{g-prior-sample-CG}

As shown in \cref{fig:walnut_mrglik_opt_hyper-main}, stochastic-gradient optimisation of hyperparameters requires thousands of steps.  When the number of observations $c$ is large, solving many linear systems to estimate the log determinant gradient at each gradient step becomes too expensive to be practical. 

We address this issue by adapting the sample-based EM iteration of \cref{chap:sampled_Laplace} to the 2d CT setting. We use the diagonal g-prior $A = a \tr(\phi^\top \op^\top \op \phi)$, implemented by scaling the feature vectors, as in \cref{eq:g-prior-scaling}. However, unlike in   \cref{chap:sampled_Laplace}, $n=1$ and thus computing the scaling vectors in closed form is tractable.
For the E-step, we draw samples using the weight-space form of pathwise conditioning, given in \cref{eq:pathwise_weights} and re-stated here for the CT setting
\begin{gather*}
        \zeta_i = w_{0,i} - A^{-1} \phi^\top \op^\top (K_{yy} + b^{-1} I)^{-1}\left(\c{E}_i + \op  \phi w_{0,i} \right)  \\
    \text{with} \quad 
  \c{E}_i \sim\N(0, b^{-1} I) \spaced{and}
  w_{0,i} \sim \N(0, A^{-1}) \notag,
\end{gather*}
and we use preconditioned CG for to solve the linear system. We do not warm-start our CG iteration, drawing new prior and noise samples $(\c{E}_i, w_{0,i})$ at succesive E steps. 
We find the linear model's posterior mode $w_\star$ by using preconditioned CG to solve \cref{eq:dip_post_mode}. The preconditioner is described in \cref{subsec:precond}.

There exists unidentifiability between our isotropic noise precision parameter $b$ and the prior precision $a$. We resolve this by fixing $b=1$ and setting
\begin{gather*}
    a  = \hat \gamma / \|w_\star\|^{2} \spaced{with} \hat \gamma = \frac{1}{m} \sum_{i=1}^m   
        \|\op\phi \zeta_i\|_{2}^2 
\end{gather*}
in the M step. We run this EM algorithm, described in detail in \cref{alg:sampling_kernelised}, to convergence.

\subsubsection{Demonstration: CG sample-based EM iteration}

\begin{figure}[t]
    \centering
    \includegraphics[width=\linewidth]{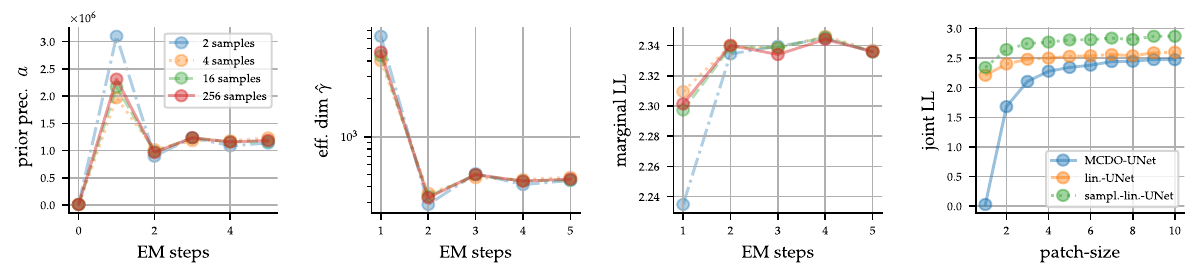}
    \caption{Left 3 plots: traces of prior precision, eff. dim., and marginal test LL vs EM steps for our tomographic reconstruction task with $c=7680$ described in \cref{sec:dip_uncertainty_experiments}. Right: joint test LL for varying image patch sizes for sample-based EM inference with the g-prior, inference in the TV-PredCP DIP model (\cref{subsec:tv_for_NN}, labelled ``lin-Unet'') and MC Dropout (labelled ``MCDO''). }
    \label{fig:DIP_EM_convergence}
\end{figure}

Similarly to image classification (recall \cref{sec:sampled_laplace_experiments}), the key hyperparameter is the number of samples to draw for the EM iteration. Again, as shown in \cref{fig:DIP_EM_convergence}, the number of samples can be kept low (e.g. 2), and we find around 5 steps to suffice for convergence of the prior precision. Our large preconditioner results in CG always hitting the desired low error tolerance within 10 steps. When the problem is small enough for CG to be tractable, preconditioning makes our kernelised EM algorithm notably faster than its primal form SGD-based counterpart from \cref{sec:sampled_laplace_experiments}.

\subsection{SGD sampling EM iteration for very large reconstructions} \label{subsec:sampling_for_3d}

In settings where the dimensionality of the observation vector $y$ is very large, i.e. $c \geq 50k$, CG may fail to converge quickly. 3d volumetric reconstruction is an example of such a setting. Here, the the dimensionality of the observation can be larger than the number of parameters of the 3d CNN used for reconstruction, i.e. $c \gtrsim d$. We deal with this, by substituting CG-based sampling for SGD-based sampling in our sample-based EM algorithm, described in \cref{g-prior-sample-CG}. That is, we apply \cref{alg:sampling_kernelised}, but with every instance of \texttt{CG} substituted with \texttt{SGD}. We use the Nesterov plus geometric averaging SGD variant for quadratic problems described in \cref{subsec:the_right_optimiser}.

\subsection{Posterior covariance matrix estimation by sampling}\label{sec:post_sampling}

The covariance matrix $K_{x | y}$ is too large to fit into memory for high-resolution tomographic reconstructions. Instead, we draw samples from $\N (x; 0, K_{x|y})$ via pathwise conditioning as
\begin{gather*}
        x_i =  \phi w_{0,i} - \phi A^{-1} \phi^\top \op^\top (K_{yy} + b^{-1} I)^{-1}\left(\c{E}_i + \op  \phi w_{0,i} \right)  \\
    \text{with} \quad 
  \c{E}_i \sim\N(0, b^{-1} I) \spaced{and}
  w_{0,i} \sim \N(0, A^{-1}). \notag
\end{gather*}
We compute the solution to the linear systems via Preconditioned CG for 2d reconstruction problems and via SGD for 3d problems.

Since only nearby pixels are expected to be correlated, we estimate cross covariances for patches of only up to $10\times 10$ adjacent pixels. Using larger patches yields no improvements. We use the biased, but lower variance, estimator  $\hat K_{x|y} = \frac{1}{2m} [\sum_{j=1}^m \textrm{diag}(x_{j})^{\odot 2} + x_{j} x_{j}^{\top}]$ for $({x}_{i})_{i=1}^m$ samples from the 0-mean posterior predictive distribution over a given patch.

\section{Demonstration: uncertainty estimation in CT with the linearised DIP}\label{sec:dip_uncertainty_experiments}

We now demonstrate the approach on real-measured cone-beam $\mu$CT data of a walnut \citep{der_sarkissian2019walnuts_data}. We first consider reconstructing the middle (2d) slice of the target volume from sparse measurements in \cref{subsec:2d_dip_demo}. We then demonstrate the scalability of our methods by estimating uncertainty for the full 3d volumetric reconstruction in \cref{subsec:sampling_for_3d}. In all cases, we have access to a ``ground truth'' reconstruction obtained from an exhaustive dense scan.  We use the pre-trained U-net models from \cite{barbano2021deep}.In all cases, we have access to a “ground truth” reconstruction obtained from an

\subsection{Uncertainty estimation for image reconstruction}\label{subsec:2d_dip_demo}

We begin by reconstructing a 2d image.
We target the $501\times 501\,\text{px}^2$ ($d_x=251\,001$) central slice of the volumetric data of \cite{der_sarkissian2019walnuts_data}. We consider two levels of sparsity. The first uses a subset of measurements taken from $d_{\c{B}}=60$ angles and $d_p = 128$ detector rows ($c=7680$). The second setting uses $d_p = 256$ detector rows and thus $c{=}15360$.  Here, $K$ is too large to store in memory and $K_{y y}$ too expensive to assemble repeatedly.
Our U-net has $d{=}2.97M$ parameters.

\begin{figure}[thb]
    \centering
    \includegraphics[width=\linewidth]{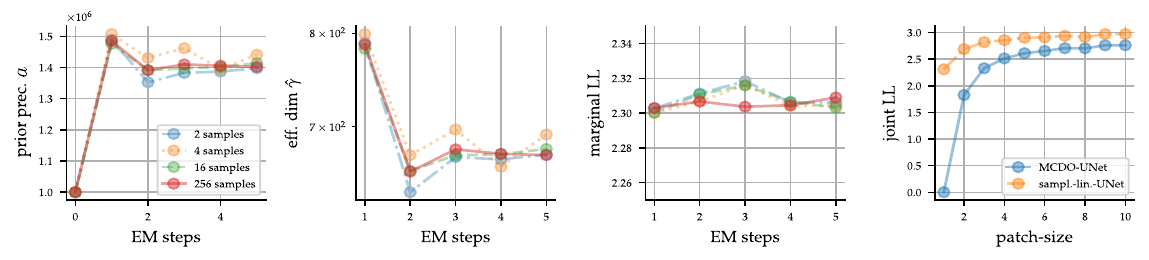}
    \caption{Left 3 plots: traces of prior precision, eff. dim., and marginal test LL vs EM steps for our tomographic reconstruction task with $c=15360$ described in \cref{sec:dip_uncertainty_experiments}. Right: joint test LL for varying image patch sizes for sample-based EM inference with the g-prior, inference in the TV-PredCP DIP model (\cref{subsec:tv_for_NN}, labelled ``lin-Unet'') and MC Dropout (labelled ``MCDO''). In this case, our initialisation for $a$ is close to the optima; its value only changes by around $50\%$ throughout EM iteration, and mostly in the first step.}
    \label{fig:DIP_EM_convergence_120}
\end{figure}

\begin{table}[htb]
\vspace{-0.2cm}
\caption{Tomographic reconstruction: test LL and wall-clock times (A100 GPU) for both 2s reconstruction data sizes.}
\vspace{0.0cm}
\resizebox{\textwidth}{!}{
\begin{tabular}{cccccccccc}
\multicolumn{1}{c}{} & \multicolumn{4}{c}{$c=7680$}                                &  & \multicolumn{4}{c}{$c=15360$}                               \\ \cline{2-5} \cline{7-10} 
\multicolumn{1}{c}{} & \multicolumn{2}{c}{LL} & \multicolumn{2}{c}{wall-clock time (min.)} &  & \multicolumn{2}{c}{LL} & \multicolumn{2}{c}{wall-clock time (min.)} \\
\multicolumn{1}{c}{Method} &
  \multicolumn{1}{c}{\small{marginal}} &
  \multicolumn{1}{c}{\small{$(10\times 10)$}} &
  \multicolumn{1}{c}{params optim.} &
  \multicolumn{1}{c}{prediction} &
   &
  \multicolumn{1}{c}{\small{marginal}} &
  \multicolumn{1}{c}{\small{$(10\times 10)$}} &
  \multicolumn{1}{c}{params. optim.} &
  \multicolumn{1}{c}{prediction} \\ \cline{1-5} \cline{7-10} 
MCDO-Ug(v) & 0.028 & 2.474 & $0$ & $3^\prime$ &  & 0.002 & 2.762 & $0$ & $3^\prime$\\
lin.-Ug(v) & 2.214 & 2.601 & $1260^\prime$ & $14^\prime$ & & $-$ & $-$ & $-$ & $-$\\
sampl.-lin.-Ug(v) & \textbf{2.341} &  \textbf{2.869}& $12^\prime$ &  $14^\prime$ &  & \textbf{2.310} & \textbf{2.972} &    $15^\prime$ &  $14^\prime$               \\ \cline{1-5} \cline{7-10} 
\end{tabular}
}
\label{Tab:walnuttab}
\vspace{-0.25cm}
\end{table}

\paragraph{Comparing hyperparameter learning schemes}

We consider two different families for the prior precision over the weights $A$, each is matched with a different inference scheme. The first is the CNN-block-wise Matérn-$\nicefrac{1}{2}$ TV-PredCP prior introduced in \cref{subsec:tv_for_NN}. We pair it with CG-based marginal likelihood estimation for hyperparameter learning, as described in \cref{subsec:CG_for_TVPredCP}. The large cost of this approach only allows us to perform a single EM step and we are restricted to the smaller $c=7680$ setting. The second model is the g-prior, which we combine with CG-sampling-based EM iteration, described in \cref{g-prior-sample-CG}. We label this method ``sampled'' in our plots. Unless otherwise specified, we use 16 samples for stochastic EM, and 1024 for~prediction.
While, the TV-PredCP model's layerwise prior variance and lengthscales take $21$ hours to converge (the corresponding optimisation traces are in \cref{fig:walnut_mrglik_opt_hyper-main}, the g-prior model takes only 12 minutes---both on an A100 GPU. These times are provided in \cref{Tab:walnuttab}. 
\cref{fig:DIP_EM_convergence} and \cref{fig:DIP_EM_convergence_120} show that sample-based EM iteration converges within 4 steps and using as few as 2 samples for both the $c{=}7680$ setting and the $c{=}15360$ setting (although the reported times use 5 steps and 16 samples). Avoiding explicit estimation of the covariance log-determinant gradient provides us with a two order of magnitude speedup.

\begin{figure}[thb]
    \centering
    \includegraphics[width=0.85\textwidth]{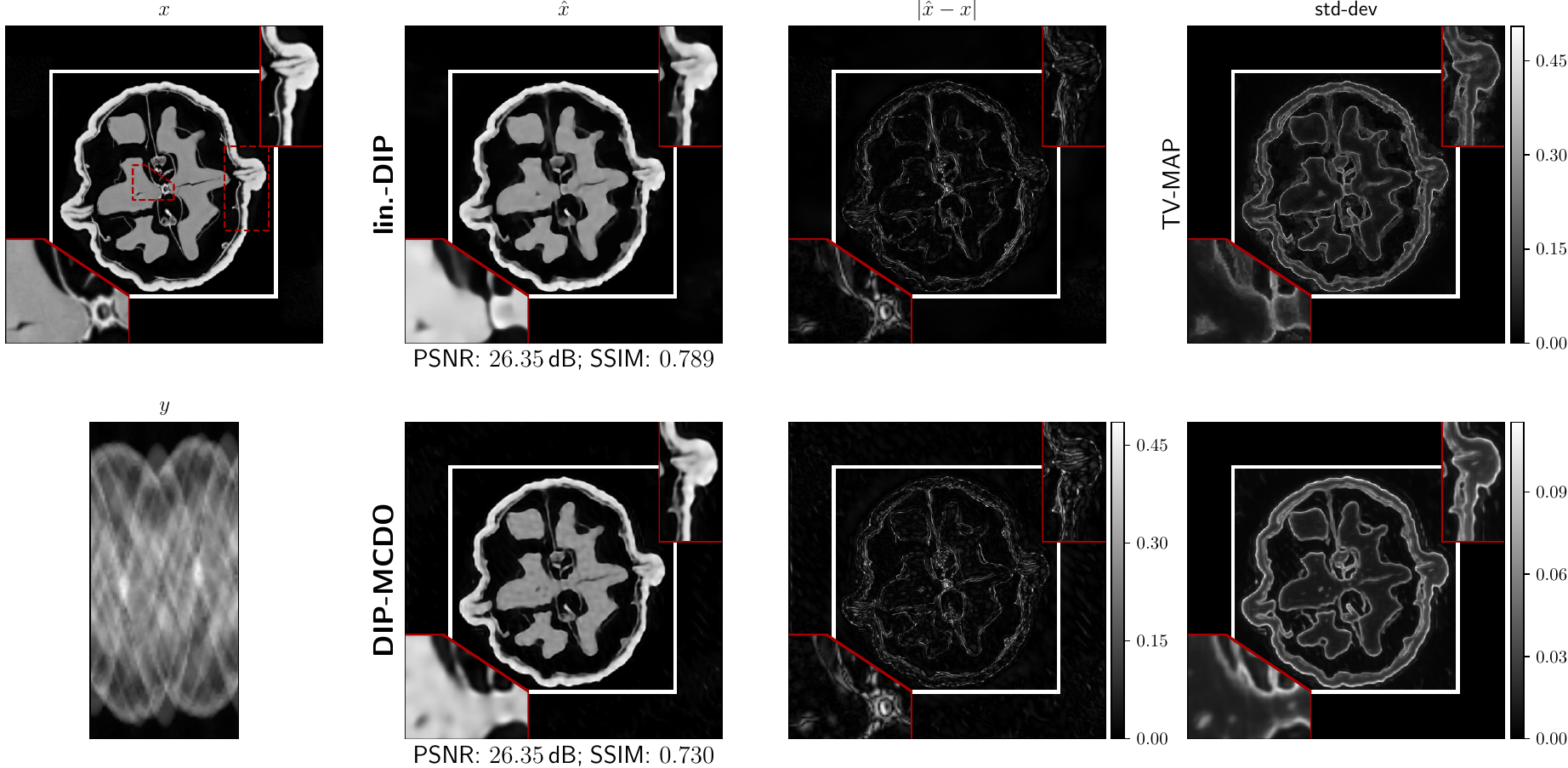}
    \caption{
    Reconstruction of a $501\times 501\,\text{px}^2$ slice of a scanned Walnut from $c=7680$ dimensional measurements using lin.-DIP (using the TV-PredCP prior) and DIP-MCDO along with their respective uncertainty estimates. The zoomed regions (outlined in red) are given in top-left.
    }
    \label{fig:main_walnut_qualitative}
\end{figure}

\begin{figure}[thb]
    \centering
    \includegraphics[width=\linewidth]{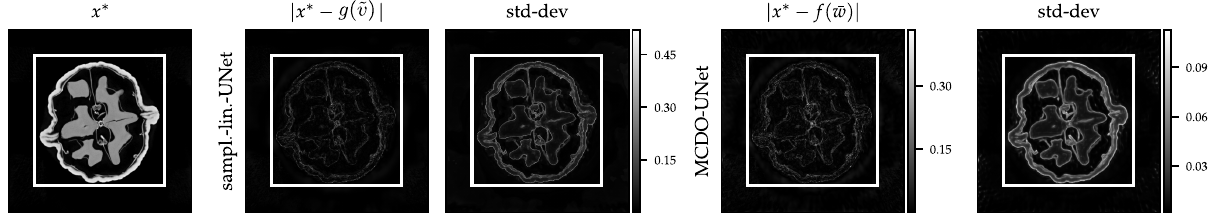}
    \caption{Original 501$\times$501 pixel  walnut image and reconstruction error for a $c{=}15360$ dimensional observation, along with pixel-wise std-dev obtained with sampling lin. Laplace and MCDO.}
    \label{fig:120_qualitative_DIP}
\end{figure}

\paragraph{Evaluating predictive performance}
We compare both of our linearised DIP models with MC dropout (MCDO), the most common baseline for NN uncertainty estimation in tomographic reconstruction \citep{laves2020uncertainty, Tolle2021meanfield}.
\Cref{fig:main_walnut_qualitative} shows, qualitatively, that the marginal standard deviation assigned to each pixel by the linearised DIP (TV-PredCP) aligns with the pixelwise error in the U-net reconstruction in a fine-grained manner in the $c=7680$ setting.
By contrast, MCDO, spreads uncertainty more uniformly across large sections of the image. \cref{fig:120_qualitative_DIP} shows a similar result but for the $c{=}15360$ setting and the g-prior DIP model. 
\Cref{Tab:walnuttab}, \cref{fig:DIP_EM_convergence} and  
\cref{fig:DIP_EM_convergence_120} show that the Log-Likelihood obtained with the g-prior sampling EM DIP  exceeds that obtained with the TV-PredCP model, potentially due to the former optimising the prior precision to convergence, while we can only afford a single EM step for the latter. Both methods outperform MCDO, in terms of both marginal and joint LL. Interestingly, MCDO's predictions are very poor marginally but improve significantly when considering covariances.

\begin{figure}[t]
    \centering
    \includegraphics[width=0.45\linewidth]{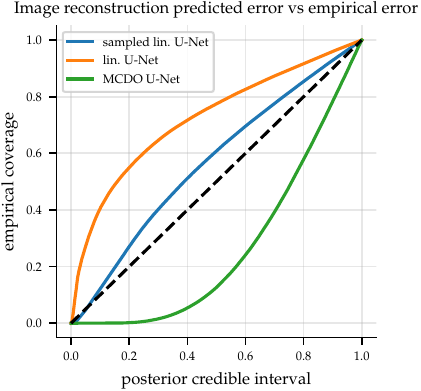}
    \caption{Empirical coverage of test targets for posterior credible intervals of increasing width for our U-net 2d tomographic reconstruction experiment with $c=7680$. Both linearised DIP variances are under-confident, although the g-prior sampling EM variant is much better calibrated. MCDO is overconfident.}
    \label{fig:Unet-calibration_plots}
\end{figure}

\paragraph{Uncertainty calibration}
For the $c=7680$ setting, we compute normalised residuals by subtracting our predictions from the targets and dividing by the predictive standard deviation. Our predictive distribution for these normalised residuals is the centred unit variance Gaussian. We consider posterior credible intervals centred at 0 and of increasing width and plot the proportion of test points that fall within them in the left side plot of \cref{fig:Unet-calibration_plots}. We find dropout inference to underestimate the magnitude of the residuals across all credible interval widths. Linearised inference with TV-PredCP consistently overestimates uncertainty, potentially due to non-converged EM underfitting. The g-prior combined with 5 steps of EM barely overestimates uncertainty, presenting the best overall calibration. 

\begin{figure}[thb]
    \centering
    \includegraphics[width=0.75\linewidth]{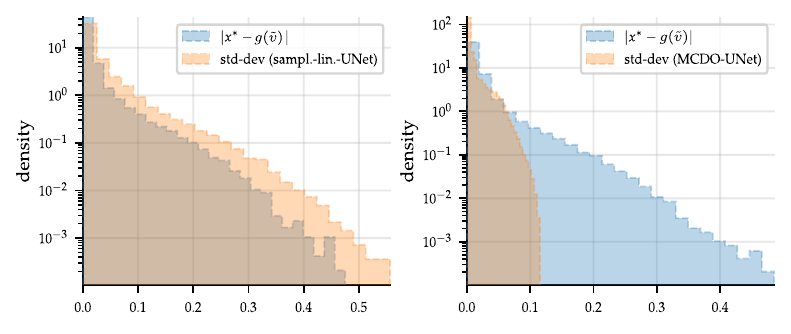}
    \caption{Histogram of the absolute pixelwise error computed between the reconstructed walnut image, given $c=7680$ observations, and the ground-truth for both lin.-Unet with g-prior (left) and MCDO-Unet (right). We overlay histograms of both methods' predictive standard deviations across pixels.}
    \label{fig:hist_DIP}
\end{figure}

We further asses calibration by comparing the histogram of the reconstruction errors made by our U-Net to the histogram of marginal, i.e. pixelwise, predictive standard deviation in \cref{fig:hist_DIP} for the $c=7680$ setting and in \cref{fig:120_hist_DIP} for the $c=15360$ setting.
In both plots, sample-based linearised Laplace inference slightly overestimates uncertainty and MCDO systematically underestimates uncertainty in the pixels where the reconstruction error is largest. Interestingly, our method shows to be slightly worsely calibrated in the more data-rich setting; the reconstruction error decreases faster than the predictive standard deviation with the addition of new data.

\begin{figure}[thb]
\vspace{-0.2cm}
    \centering
    \includegraphics[width=0.75\linewidth]{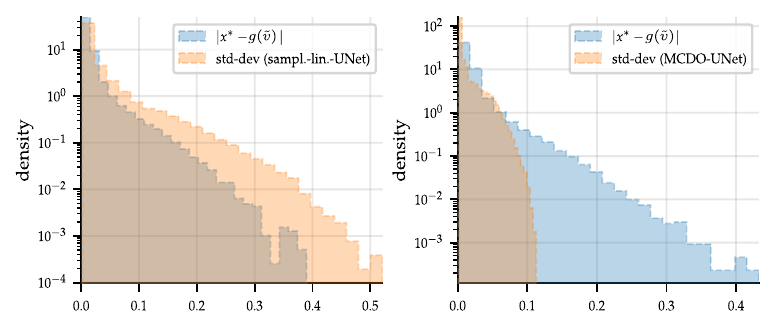}
    \caption{Histogram of the absolute pixelwise error computed between the reconstructed walnut image, given $c=15360$ observations, and the ground-truth for both lin.-Unet with g-prior (left) and MCDO-Unet (right). We overlay histograms of both methods' predictive standard deviations across pixels.}
    \label{fig:120_hist_DIP}
    \vspace{-0.4cm}
\end{figure}

\begin{figure}[thb]
    \centering
    \includegraphics[width = 0.8\textwidth]{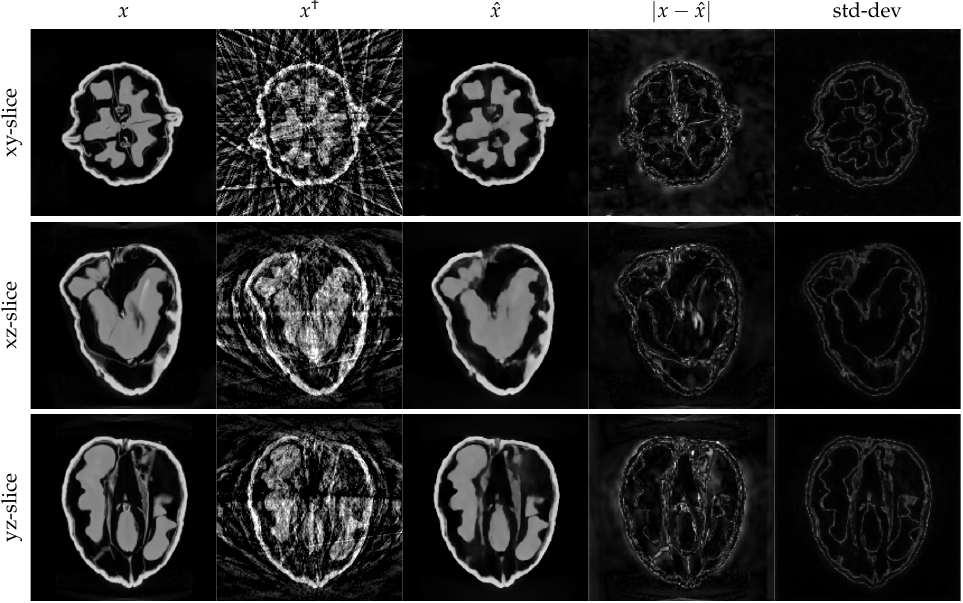}
    
    \caption{From left to right: 1) Ground truth reconstructions of three $167\times 167\,\text{px}^2$ slices from the $167\times 167\times 167\,\text{px}^3$  Walnut data from \cite{der_sarkissian2019walnuts_data}. 2) Filtered backprojections (i.e. reconstructions obtained by pseudoinverting the operator $\c{T}$) from $c=1.6M$ observations. 3) Unet reconstructions. 4) Absolute error in Unet reconstructions. 5) Pixelwise standard deviations obtained with the linerised Unet and the g-prior. }
    \label{fig:3d_qualitative}
\end{figure}

\subsection{Volumetric uncertainty estimation} \label{subsec:volumetric_demo}

\begin{figure}[thb]
    \centering
    \includegraphics[width=0.45\linewidth]{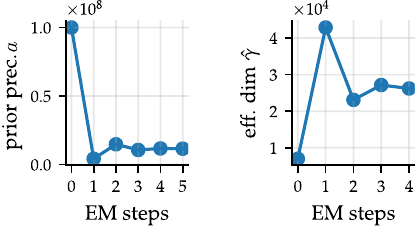}
    \caption{Traces of prior precision $\alpha$ and eff. dim. $\hat{\gamma}$ vs EM steps for the $c=1.6M$ 3d volumetric reconstruction task.}
    \label{fig:3d_EM_convergence}
\end{figure}

\begin{figure}[thb]
    \centering
    \includegraphics[width = 0.7\textwidth]{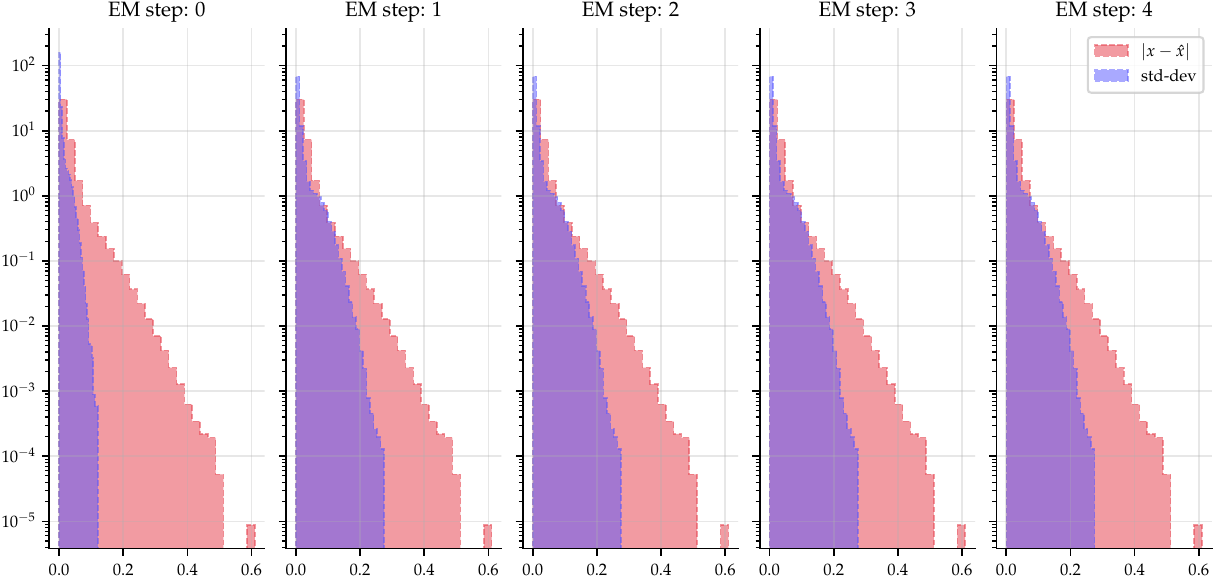}
    \caption{Histograms (y-axes are normalised to represent empirical densities) of the voxel-wise error computed between the reconstructed 3d volumetric walnut and the ground-truth, along with the histograms of pixelwise predictive standard deviations across voxels.}
    \label{fig:3d_histograms}
\end{figure}

 We consider 3d reconstruction of the Walnut data with a downscaled resolution of $d_x = (167px)^3 \approx 4,65M$ voxels, from $d_{\c{B}}=20$ equally distributed angles, and we sub-sample projection rows and columns by a factor of 3. This corresponds to a $c=1.6M$ dimensional observation space. We use a $d \approx 5M$ parameter 3d CNN. We perform sample-based EM inference, drawing samples with SGD. This procedure is illustrated in \cref{fig:3d_EM_convergence}. It is not clear from the plot that EM has converged, but we can not afford the computation for more than 4 steps. To the best of my knowledge, this is the first instance of uncertainty estimation for deep-learning based volumetric image reconstruction. Three slices of the reconstructed volume, along with their respective error and uncertainty maps are provided in \cref{fig:3d_qualitative}.
We provide error and pixelwise uncertainty histograms in \cref{fig:3d_histograms}. Our method underestimates uncertainty in the tails, but this is somewhat alleviated with successive EM steps.

\FloatBarrier

\section{Linearised DIP Bayesian experimental design for CT}\label{sec:experimental_design}

In CT, Bayesian experimental design leverages an a-priori model to select the scanning angles which are expected to yield the highest fidelity reconstruction. 
Adaptive design further incorporates information gained at previous angles to inform subsequent angle selections \citep{Chaloner1995review}. 
These methods are of great practical interest since they promise to reduce radiation dosages and scanning times.
Alas, existing CT design methods often struggle to improve over equidistant angle choice \citep{Shen2022learningtoscan}. 
Furthermore, the requisite of additional computations before subsequent scans makes adaptive methods impractical for many~applications.

Critically important to experimental design is the choice of prior \citep{chi2015modelerror,foster2021Design}.
Linear models allow for tractable computation of quantities of interest for experimental design, but their predictive uncertainty is independent of previously measured values, disallowing adaptive design \citep{Burger2021lineardesign}. 
More complex model choices make inference difficult, necessitating approximations which can degrade performance \citep{Helin2021GaussianTV,Shen2022learningtoscan}.

This section aims to make adaptive design practical by considering a setting where the CT scan is performed in two phases. 
First, a sparse pilot scan is performed to provide data with which to fit a adaptive methods.
These are then used to select angles for a full scan. 
We demonstrate this procedure with a synthetic dataset where a different ``preferential'' angle is most informative for each image.
Preferential directions appear commonly in industrial CT for material science and in medical CT for medical implant assessment.
We use the linearised Deep Image Prior (DIP) \citep{barbano2022probabilistic} as a data-dependent prior for adaptive design which preserves the tractability of conjugate Gaussian-linear models.
Unlike simple linear models, the linearised DIP outperforms the equidistant angle baseline.
Finally, we show that designs obtained with the linearised DIP perform well under traditional (non DIP-based) regularised-reconstruction.

\cref{sec:dip_design_methods} covers sequential inference in the conjugate Gaussian-linear setting.
\cref{subsec:linear_design} introduces experimental design with linear models and linearised neural networks. Finally, \cref{sec:dip_design_experiments} demonstrates our approach on a synthetic CT scanning angle selection task.

\begin{figure}[t]
        \centering
        \includegraphics[width=0.7\linewidth]{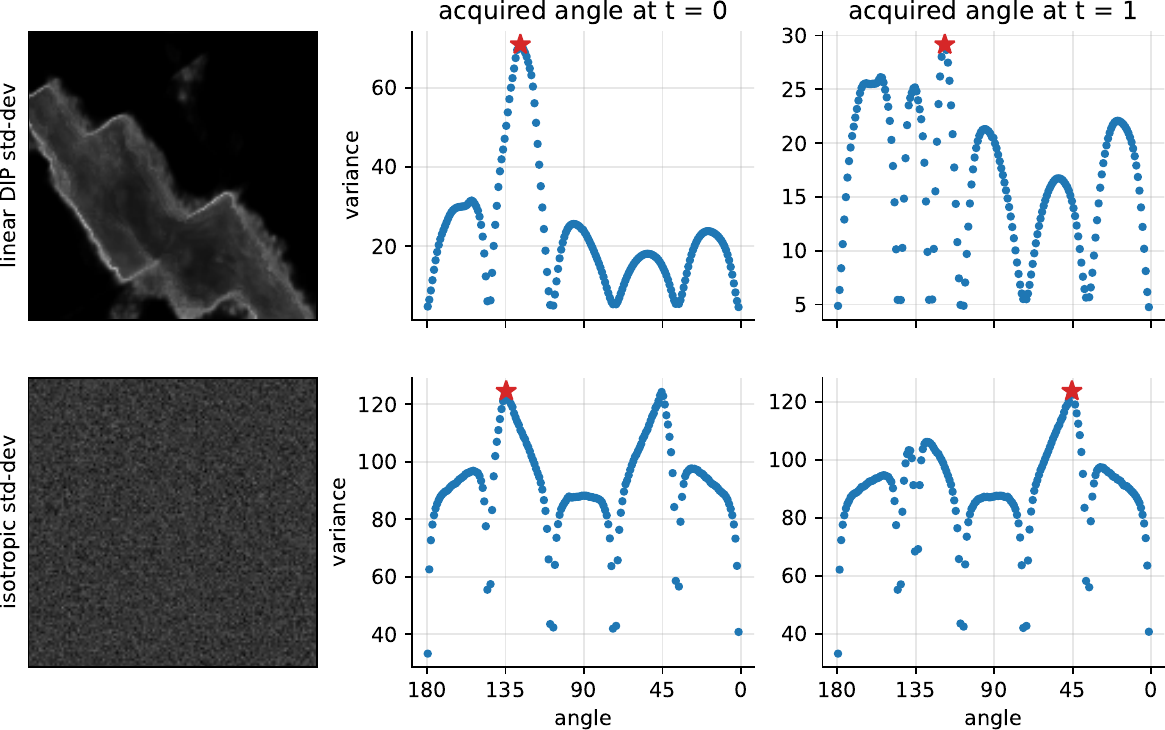}
    \caption{Top row: the linearised DIP assigns prior variance to pixels where edges are present, guiding angle selection so that X-ray quanta cover these pixels. Bottom row: the isotropic linear model's variance does not depend on the measurements. Angles $45$ and $135$ are chosen since they are oblique and maximise quanta path-length in the image.}
    \label{fig:variance_angles}
\end{figure}

\subsection{Sequential inference with linear(ised) models}\label{sec:dip_design_methods}

Let $\c{B}_{a}$ be the set of \emph{all} possible angles at which we can scan. The task is to choose the subset of angles $\c{B} \subset \c{B}_{a}$ which produces the highest-fidelity reconstruction. We shall add angles sequentially over $T$ steps. The set $\c{B}^{(t)}$ denotes the chosen angles up to step $t < T$, and $\bar{\c{B}}^{(t)}=\c{B}_{a} \setminus \c{B}^{(t)}$ the angles left to choose from. $\c{B}^{(0)}$ denotes the set of angles used in the initial pilot scan, and $\c{B} = \c{B}^{(T)}$ the full design.
We incorporate a decision to scan at angle $\beta \in \bar{\c{B}}^{(t)}$ by concatenating the matrix $\op^{\beta} \in \R^{d_{p} \times d_{x}}$, which contains a row for each detector pixel at angle $\beta$, to the operator.
After step $t$, the operator $\op^{(t)} \in \R^{d_{p} \cdot d_{\c{B}^{(t)}} \times d_{x}}$ stacks $d_{\c{B}^{(t)}}$ of these matrices, with $d_{\c{B}^{(t)}}=|\c{B}^{(t)}|$.
$\bar{\op}^{(t)} \in \R^{d_{p}\cdot d_{\bar{\c{B}}^{(t)}} \times d_{x}}$ denotes the forward operator for the angles left to choose~from.

For design, we place a multivariate Gaussian prior on $x$ with zero mean and covariance matrix $K \in \R^{d_{x} \times d_{x}}$. Together with the Gaussian noise model in \cref{eq:inverse_problem}, this gives a conjugate Gaussian-linear model. The vector $y^{(t)}\in \R^{d_{p}\cdot d_{\c{B}^{(t)} }}$ of all measurements at step $t$ is distributed~as 
\begin{gather*}
    y^{(t)} | x \sim \c{N}(\op^{(t)} x,\, b^{-1} I_{c}) \quad \text{with} \quad  x \sim \c{N}(0, K).
\end{gather*}
Thus, $K^{(t)}_{yy} {+}\,b^{-1}I$, with $K^{(t)}_{yy} = \op^{(t)} K (\op^{(t)})^{\top}$, is the measurement covariance and the posterior over $x$ is %
\begin{gather}
    x | y^{(t)} \sim \c{N}(\mu_{x|y^{(t)}}, K_{x|y^{(t)}}), \notag \\ 
    \text{with} \quad  \mu_{x|y^{(t)}}=K(\op^{(t)})^{\top}(K_{yy}^{(t)} {+}\,b^{-1}I)^{-1}y^{(t)}, \notag \\
    \text{and} \quad K_{x|y^{(t)}} = K - K(\op^{(t)})^{\top} (K_{yy}^{(t)}{+}\,b^{-1}I)^{-1}\op^{(t)}K.\label{eq:linear_posterior}
\end{gather}
The predictive covariance $K_{x|y^{(t)}}$ completely characterises the uncertainty of the reconstruction at step $t$ and is the building block for the angle selection criteria in \cref{subsec:linear_design}.

With this, a concern may be that natural images often exhibit heavy-tailed non-Gaussian statistics \citep{Seeger11scale}. 
Furthermore, by \cref{eq:linear_posterior}, $K_{x|y^{(t)}}$ depends on the choice of angles through $\op^{(t)}$, but not on the measurements made at said angles $y^{(t)}$, precluding adaptive design. 
In \cref{subsec:prior_cov}, we will address both of these concerns by constructing a very flexible data dependent covariance kernel from the Jacobian of a NN, recovering adaptive design capabilities.

\subsection{Experimental design with linear(ised) models} \label{subsec:linear_design}

\textbf{Acquisition objectives.} Since the linear design task is submodular \citep{Seeger2009submodular}, we greedily add one single angle per acquisition step \footnote{Submodularity guarantees this procedure obtains a score within a $(1-\nicefrac{1}{e})$ factor of the optimal strategy.}. We consider two popular acquisition objectives.

The first objective, \emph{expected information gain} (EIG) \citep{Mackay1992InformationBasedOF}, is the expected reduction in the posterior entropy $\mathbb{H}(\c{P}_{x|y})$ from scanning at angle $\beta$. At step $t$, it is given by
\begin{gather}\label{eq:eig}
    \text{EIG} :=  \mathbb{H}(\c{P}_{x|y^{(t)}} )  - \E_{{y^{\beta} | y^{(t)}}} [\mathbb{H}(\c{P}_{ x | y^{(t)}, y^{\beta}})] = \logdet(b^{-1} I_{d_{\c{B}^{(t)}}} + \op^{\beta}K_{x | y^{(t)}}(\op^{\beta})^{\top}) + C 
\end{gather}
where the constant $C = -\logdet(b^{-1} I)$ is independent of the angle choice. 
{Intuitively, the determinant of the matrix $\op^{\beta}K_{x | y^{(t)}}(\op^{\beta})^{\top}\in\R^{d_p\times d_p}$ penalises angles for which different detector elements make correlated measurements and the log term encourages the measurements from all detector pixels to be similarly informative.} EIG is known as a (D)eterminant-optimal objective.

\begin{derivation} \textbf{Expected information gain} \\

The entropy of a multivariate Gaussian  is $\mathbb{H}(\N (\mu, K)) = \frac{1}{2} \logdet(K) + \frac{d}{2}(\log(2 \pi) + 1)$.

We compute the posterior covariance log-determinant at time $t$ from the covariance at time $t-1$ using the matrix determinant lemma
\begin{align*}
\logdet(K_{x | y^{(t)}}) = &-  \logdet(K_{x | y^{(t-1)}}^{-1}) 
- \logdet(b I) \\
&- \logdet(b^{-1} I 
+ \op^{(t)} K_{x | y^{(t-1)}}\op^{\top,(t)}).
\end{align*}
Note that both sides of the equality are independent of the targets $y$. Thus we drop the expectation in \cref{eq:eig}.
With that, we have
\begin{align*}
    \text{EIG} =& \logdet(K_{x | y^{(t-1)}}) -  \logdet(K_{x | y^{(t)}})  \\ =&\logdet(K_{x | y^{(t-1)}}) - [-  \logdet(K_{x | y^{(t-1)}}^{-1}) - \logdet(b I) \\ &- \logdet(b^{-1} I + \op^{\beta}K_{x | y^{(t-1)}}(\op^{\beta})^{\top})] \\
    =&-\logdet(b^{-1} I) + \logdet(b^{-1} I + \op^{\beta}K_{x | y^{(t-1)}}(\op^{\beta})^{\top})\\
    =& \logdet(b^{-1} I + \op^{\beta}K_{x | y^{(t-1)}}(\op^{\beta})^{\top}) + C
\end{align*}
where the constant $C = -\logdet(b^{-1} I)$ is independent of angle choice, yielding the angle selection objective.
\end{derivation}

\begin{remark} \textbf{What information are we gaining?}

EIG quantifies the information (in nats) we expect to gain by observing the detector elements' measurements for an angle or set of angles \citep{Mackay1992InformationBasedOF}.
    EIG is also equal to the mutual information between the reconstruction $x$ and the new measurement $y^\beta$ conditional on the previous measurements $y^{(t-1)}$, i.e. $MI({x}, {y}^{\beta} | y^{(t-1)})$, giving an interpretation as aiming to select the angle $\beta$ most informative towards the reconstruction. 
For fixed model hyperparameters, EIG is always greater or equal than $0$ since making additional measurements cannot increase the uncertainty in the reconstruction. 
\end{remark}

The second objective, which we find to perform better empirically, is to choose the angles for which our prediction has the largest \emph{expected squared error} (ESE) in measurement space
\begin{gather}\label{eq:ESE}
    \text{ESE} := \E_{y^{\beta}, \,x|y^{(t)}}[(y^{\beta} - \op^{\beta} x)^{\top}(y^{\beta} - \op^{\beta} x)] = \text{Tr}(\op^{\beta}K_{x | y^{(t)}}(\op^{\beta})^{\top}) + C.
\end{gather}
This objective is equivalent to EIG in the setting where our detector has a single pixel. 

\begin{remark} \textbf{Motivating ESE}\\
The ESE objective in \cref{eq:ESE} aims to minimise the squared prediction error in measurement space. Objectives of this kind are commonly known as (A)verage-optimal. However, ESE is A-optimal over measurement space $y$, not over image space $x$. ESE is crucially different from minimising the arguably more relevant expected squared reconstruction error, a more computationally expensive criterion. ESE can be understood as a na\"ive simplification of EIG, by discarding correlations between detector pixels, making $\logdet(\op^{\beta}K_{x | y^{(t-1)}}(\op^{\beta})^{\top})$ match $\sum_{i\leq d_{p}} \log [\op^{\beta}K_{x | y^{(t-1)}}(\op^{\beta})^{\top}]_{ii}$. Then, the order of log and sum are switched, something that will only preserve the output (up to a constant independent of $\beta$) if every element under the sum is the same.
Having reached this point, since the log function is monotonic, it does not affect angle selection and the criterion matches the trace of $\op^{\beta}K_{x | y^{(t-1)}}(\op^{\beta})^{\top}$.
\end{remark}

\textbf{Efficient acquisition.} 
Constructing the matrix $\op^{\beta}K_{x | y^{(t)}}(\op^{\beta})^{\top}$ repeatedly for each candidate angle $\beta \in \bar{\c{B}}^{(t)}$ requires $\c{O}(d_{p}\cdot d_{\bar{\c{B}}^{(t)}})$ matrix vector products, which is very costly even for moderate size scanners.
Instead, we estimate the matrix for every angle simultaneously by drawing $m$ samples from $\c{N}(0, \bar{\op}^{(t)}K_{x | y^{(t)}} (\bar{\op}^{(t)})^{\top} )$. That is, we sample $\R^{d_{p}\cdot d_{\bar{\c{B}}^{(t)}}}$ sized vectors containing the concatenated ``pseudo measurements'' for each unused angle $\beta \in \bar{\c{B}}^{(t)}$. We again use pathwise conditioning 
\begin{gather}
    \bigoplus_{\beta \in \bar{\c{B}}^{(t)}} \!y^{\beta}_{i} = \bar{\op}^{(t)} \!\left(x_{i} - K(\op^{(t)})^{\top} (K_{yy}^{(t)}{+}\,b^{-1} I)^{-1}(\c{E}{i} + \op^{(t)}x_{i})\right) \quad \text{with} \notag \\
     x_{i}\sim \c{N}(0, K  ) \quad \text{and} \quad \c{E}_{i}\sim \c{N}(0, b^{-1} I )\label{eq:matheron_design}.
\end{gather}
Here, $i \in \{1,...,m\}$ indexes different samples and $\bigoplus$ denotes the concatenation of vectors generated for each $\beta \in \bar{\c{B}}^{(t)}$. Now, for each angle $\beta \in \bar{\c{B}}^{(t)}$, we compute
\begin{gather*}
    \op^{\beta}K_{x | y^{(t)}}(\op^{\beta})^{\top} \approx \frac{1}{m} \textstyle{\sum_{i=1}^{m}} y^{\beta}_{i} (y^{\beta}_{i})^{\top},
\end{gather*}
which is then used to estimate the acquisition objective \cref{eq:eig} or \cref{eq:ESE}. The log term makes EIG estimates only asymptotically unbiased (i.e.\ as $m\to \infty$) but we find the bias to be insignificant. Once the angle $\beta$ that maximises \cref{eq:eig} or \cref{eq:ESE} is chosen, we update $K_{yy}^{(t+1)}$ as
\begin{equation}
    K_{yy}^{(t+1)} = \begin{bmatrix}
        K_{yy}^{(t)} & \op^{(t)}K(\op^{(t+1)})^{\top} \\
        \op^{(t+1)}K(\op^{(t)})^{\top} & \op^{(t + 1)}K(\op^{(t+1)})^{\top}
    \end{bmatrix},\label{eq:covariance_update}
\end{equation}
and repeat the procedure, i.e.\ return to \cref{eq:matheron_design}.

\begin{figure}[thb]
\includegraphics[width=\textwidth]{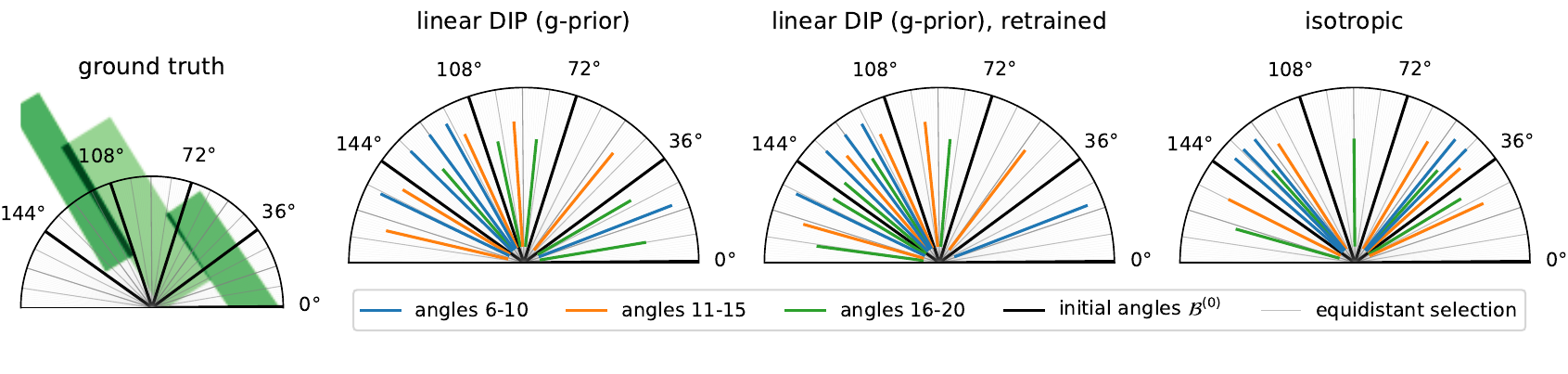}
\caption{First 20 angles selected by each method under consideration for an example image.} \label{fig:angle_selection_image1}
\end{figure}

\subsection{Construction of the prior covariance $K$} \label{subsec:prior_cov}

Now we describe the construction of the Gaussian prior covariance $K\in \R^{d_{x} \times d_{x}}$ over reconstructions. We consider a range of models, building from very simple models to flexible data-driven ones that allows for adaptive design. 

\textbf{Isotropic model.} The simple choice $K =  I_{d_{x}}$ assumes uncorrelated pixels, and it implies a ridge regulariser for the reconstruction, which is known to perform poorly in imaging.

\textbf{Matérn-$\nicefrac{1}{2}$ Process.} \cite{Antoran2022Tomography}, and also \cref{subsec:tv_for_NN}, employ the Matérn-$\nicefrac{1}{2}$  covariance $[K]_{ij,i'j'} = \exp (-\psi^{-1}\sqrt{(i-i')^{2} + (j-j')^{2}})$, where $(i, j)$ index the pixel locations in the image $x$ in terms of height and width respectively, as a surrogate for the TV regulariser. 

\textbf{Linearised deep image prior} This data-driven prior is constructed by first fitting a DIP model on the measurements taken during the pilot scan with \cref{eq:DIP_MAP-obj2}. We then adopt a linear model on the basis expansion given by the Jacobian of the trained U-net, denoted $\phi \in \R^{d_{x} \times d}$. %
The resulting  covariance matrix $K = \phi A^{-1} \phi^{\top}$  incorporates information about the pilot measurements on which the NN was trained through its Jacobians $\phi$. It assigns higher prior variance being near the edges in the reconstruction (this is shown in \cref{fig:variance_angles}), which are most sensitive to a change in U-net parameters. The covariance $A^{-1} \in \R^{d \times d}$ weights different Jacobian entries. We consider two different structures for $A^{-1}$.
\begin{itemize}[leftmargin=*]
\vspace{-0.1cm}
  \setlength\itemsep{0pt}
     \item The filter-wise block-diagonal matrix of \cref{subsec:tv_for_NN}. This choice uses a large number of hyperparameters and thus risks overfitting to the pilot scan measurements.
    \item The neural g-prior, introduced in \cref{subsec:g-prior}. We implement it through feature scaling, as described in \cref{subsec:gprior_implementation}. We update the feature scaling vectors every $5$ acquired angles.
\end{itemize}

The Matérn model has its lengthscale as a free hyperparameter. Learning this hyperparameter from the data makes the model adaptive. The filter-wise DIP prior has filter-wise marginal variances and lengthscales. We set these such that the model evidence is maximised given the pilot scan measurements using gradient-based optimisation. Since the number of pilot observations is small, the exact evidence \cref{eq:GP_evidence} is tractable.
We omit the global prior variance scaling hyperparameter from all models since the choice of this value only alters the width of the posterior errorbars, but not their shape. As a result, experimental design is invariant to the choice of global prior variance scaling\footnote{I was made aware of this by David Janz via personal communication and then I verified it experimentally.}. The same is true for the isotropic observation noise precision $b$.

\begin{figure}[thb]
\centering
\includegraphics[width=0.22\linewidth]{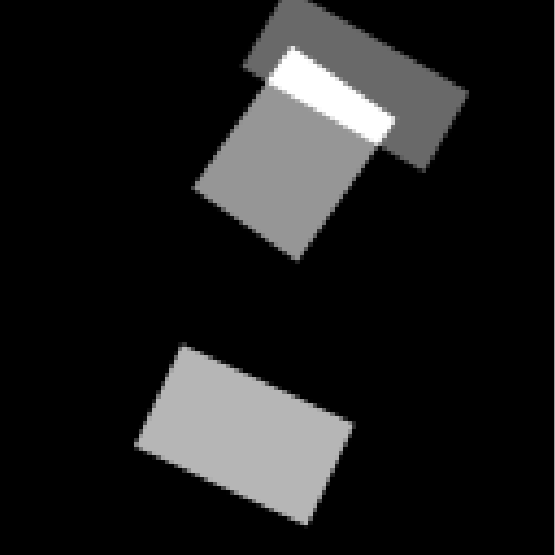}
\includegraphics[width=0.22\linewidth]{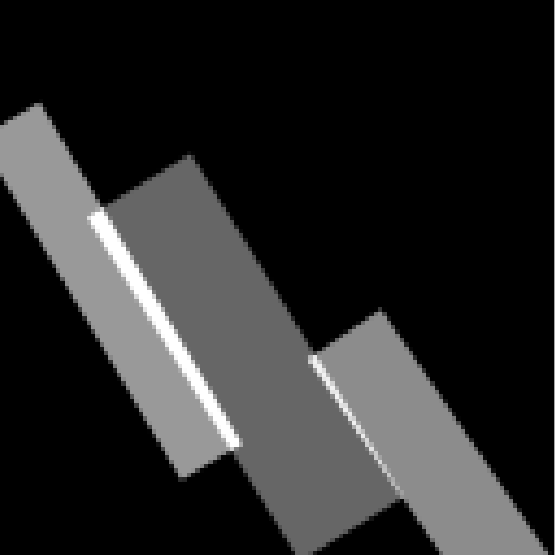}
\includegraphics[width=0.22\linewidth]{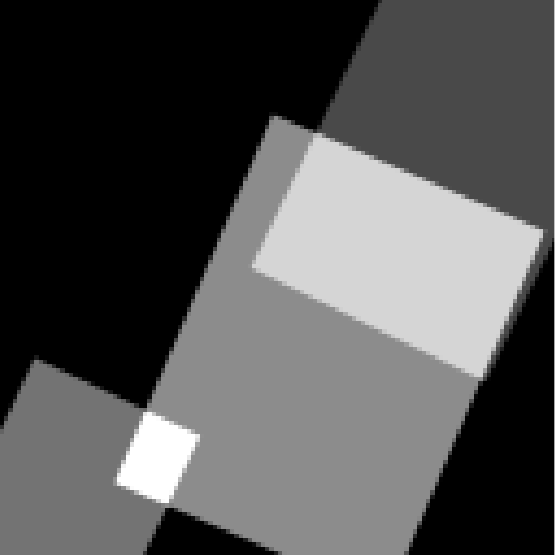}
\includegraphics[width=0.22\linewidth]{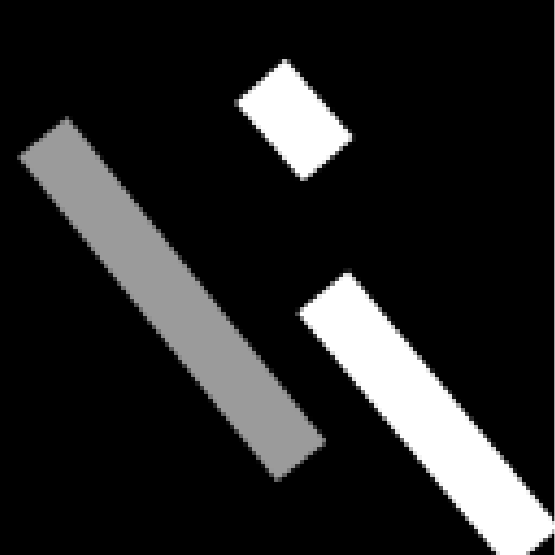}
\caption{Examples of synthetic images used for our experiments.}
    \label{fig:example_images}
\end{figure}

\section{Demonstration: designing CT angle selection strategies}\label{sec:dip_design_experiments}

We now demonstrate the experimental design objectives from \cref{subsec:linear_design} coupled with the models from \cref{subsec:prior_cov}. In almost all real-world CT deployments, the scanning angles are chosen to be equidistant. This strategy is known to be very hard to beat, and we will use it as our strong baseline. We will also test weather the DIP-based designs work well exclusively for DIP-based reconstructions or if they generalise to non-NN-based reconstruction methods. 

\paragraph{Experimental setup}
We simulate CT measurements $y$ from $128\!\times\!128$ ($d_x=16384$) pixel images of 3 superimposed rectangles. Their orientation is sampled from a single normal distribution with zero mean and standard deviation $2.86^\circ$. Thus, images in this class contain edges in roughly two perpendicular ``preferential'' directions (see \cref{fig:angle_selection_image1} and \cref{fig:example_images}). 
We simulate CT measurements by applying the discrete Radon transform operator $\op \in \R^{c \times d_{x}}$ and adding Gaussian noise with standard deviation matching 5\,\% of the average absolute value of the noiseless measurements $\op x$ to generate $y$. 
We divide the scanning range $[0^\circ, 180^{\circ})$ into 200 selectable angles (i.e.\ $|\c{B}_{a}| = 200)$. The pilot scan measures at 5 equidistant angles, on which  we fit all models' hyperparameters and the linearised DIP's U-net.
Then, we apply the methods in \cref{subsec:linear_design} to produce designs consisting of 35 additional angles.
For every 5 acquired angles, we evaluate reconstruction quality using both the DIP \cref{eq:DIP_MAP-obj2}, and the traditional NN-free TV regularised approach \cref{eq:simple_optimisation_objective}. We include equidistant and random angle selection as strong and weak baselines, respectively. On an A100 GPU, a full linearised DIP acquisition step with $K=3000$ samples takes 9 seconds and the full design takes 5 minutes.

For the \textbf{linearised DIP}, we consider both training our U-net and prior hyperparameters only on the pilot scan, and also retraining every 5 angles. \Cref{fig:angle_selection_image1} shows both approaches can identify and prioritise the preferential direction, leading to reconstructions that \emph{outperform the equidistant angle baseline by over 1.5 dB} in the range of $[10, 15]$ angles. This is shown in \cref{fig:main-figure-psnr-comparison}. 
During this initial stage, the linearised DIP requires roughly \emph{30\% less scanned angles} to match the equidistant baseline's performance.
The performance gap decreases as we select more angles, although linearised DIP remains more efficient even after 40 angles. 
Retraining the U-net provides most benefits in the large angle regime. It increases focus on preferential directions and consistently provides gains ${>}0.5$dB after 20 angles. All gains over the equidistant baseline are obtained with both DIP reconstruction \cref{eq:DIP_MAP-obj2} and traditional TV regularised reconstruction \cref{eq:simple_optimisation_objective}. 

\begin{figure}[t]
\begin{minipage}{0.329\textwidth}
  \includegraphics[width=\linewidth]{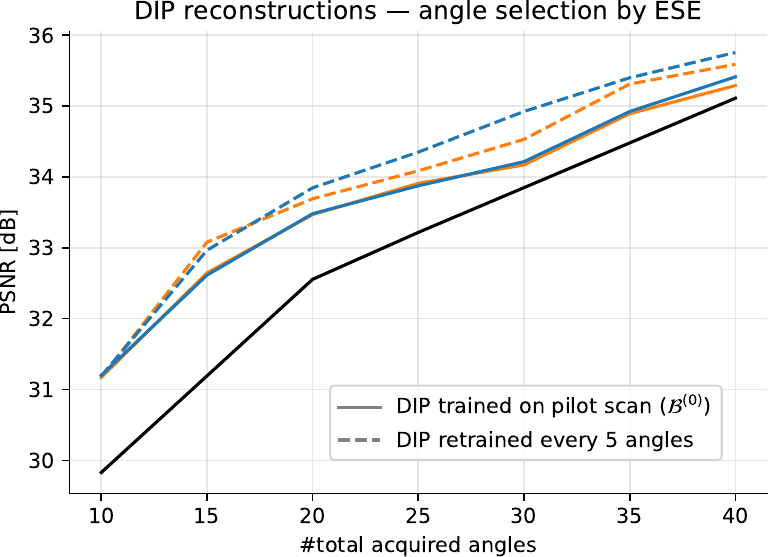}
\end{minipage}\hfill
\begin{minipage}{0.329\textwidth}
  \includegraphics[width=\linewidth]{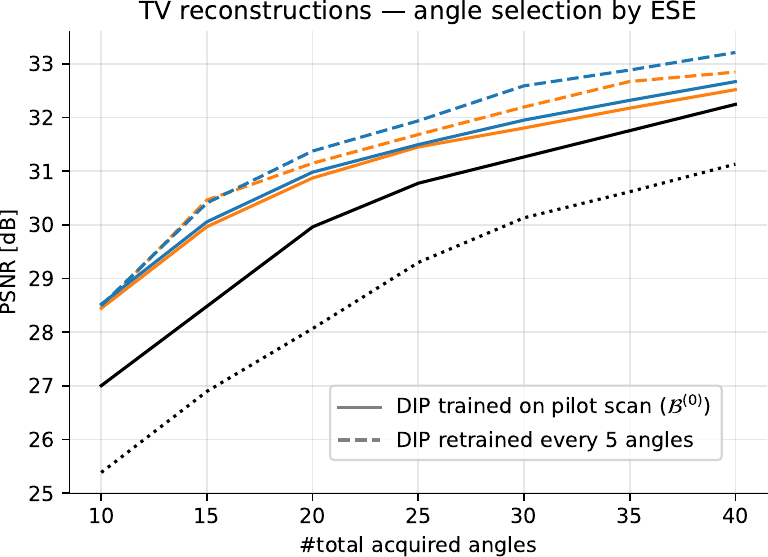}
\end{minipage}\hfill
\begin{minipage}{0.329\textwidth}%
  \includegraphics[width=\linewidth]{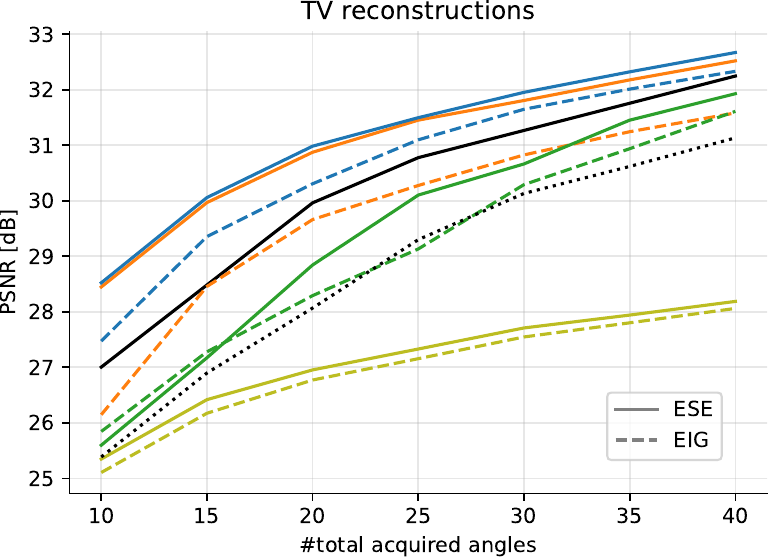}
\end{minipage}
\\[0.5em]
\centering
\includegraphics[width=0.85\linewidth]{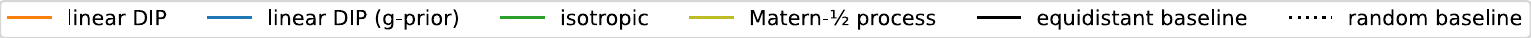}
\caption{Reconstruction PSNR vs n. angles scanned, averaged across 30 images (5\% noise).}\label{fig:main-figure-psnr-comparison}
\end{figure}

The \textbf{isotropic and Matérn-$\nicefrac{1}{2}$} models' uncertainty estimates are independent of the pilot measurements. These models prioritise clustered sets of oblique angles which maximise the length of quanta trajectories in the image. They perform similar to or worse than random. This negative result is due to the lengthscale hyperparameter overfitting to the small amount of data from the pilot scan and taking very large values. This makes the predictive variance insensitive to previous acquisitions, as shown in \cref{fig:gp_first_8}. For contrast, we display the g-prior DIP's acquisitions in \cref{fig:g-prior_first_8}.

\begin{figure}[htb]
\centering
{\sffamily\scriptsize Matérn-$\nicefrac{1}{2}$ model, first 8 acquisitions}\\
\includegraphics[width=0.9\linewidth]{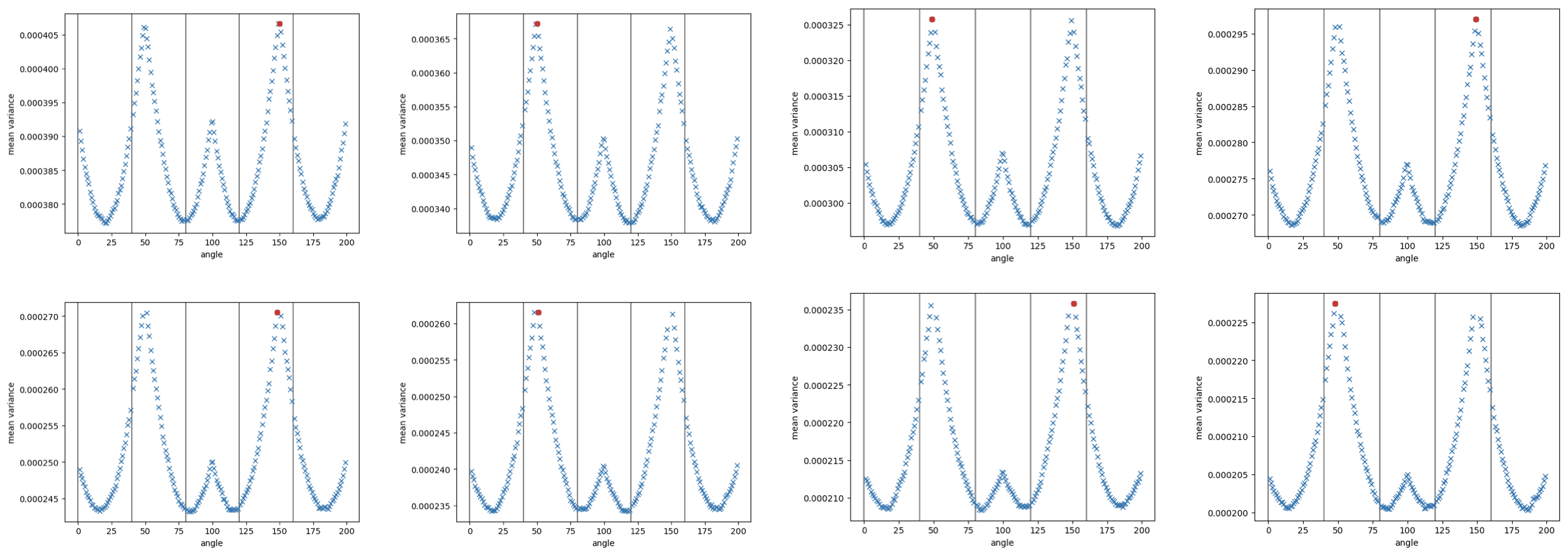}
\caption{Variance assigned to each candidate angle during the first 8 design steps by our Matérn-$\nicefrac{1}{2}$  model. }
\label{fig:gp_first_8}
\end{figure}

\begin{figure}[htb]
\centering
{\sffamily\scriptsize Linearised DIP with g-prior, first 8 acquisitions}\\
\includegraphics[width=\linewidth]{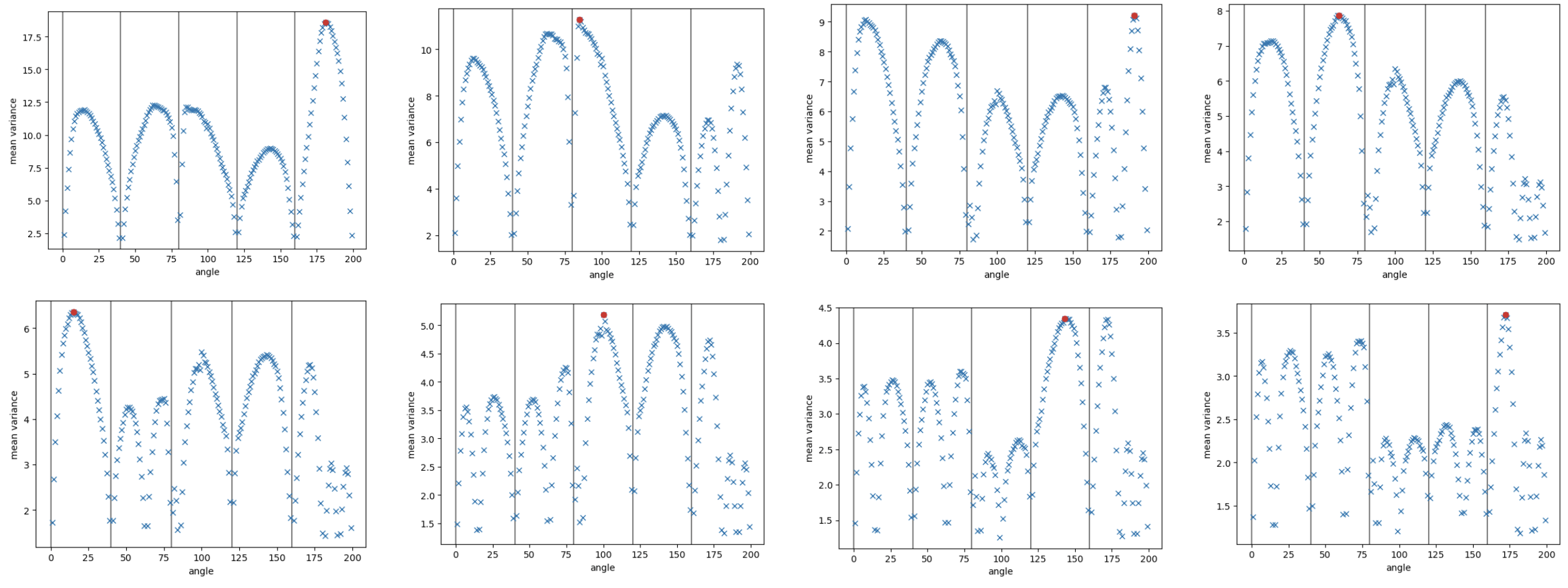}
\caption{Variance assigned to each candidate angle during the first 8 design steps by our linearised DIP model with the g-prior.}
\label{fig:g-prior_first_8}
\end{figure}

\textbf{ESE outperforms EIG} across models. For the linearised DIP, this gap is smaller when using the g-prior. This is surprising, given that EIG takes covariances into account, but ESE doesn't. 
We hypothesise that model misspecification and hyperparameter overfitting may result in poor measurement covariance estimates, in turn degrading the EIG estimates.

\begin{remark} \textbf{On the dangers of combining model evidence maximisation with data acquisition} \\
    \Cref{fig:gp_first_8} shows the Matérn-$\nicefrac{1}{2}$  model concentrates its selection on oblique angles, and this does not change as more data is acquired. This results in a very non-diverse angle set which achieves very poor performance. To understand why this happens we first remark that the Matérn-$\nicefrac{1}{2}$  model generalises the isotropic model and the two are equal when the lengthscale is set to $\psi=0$. We investigate the hyperparameters chosen by the model evidence for the Matérn-$\nicefrac{1}{2}$  model and find that for all images the lengthscale is in the range [40-70]. This value is very large relative to the size of the image ($128 \times 128$) and represents an assumption that the reconstructed image has only 2 or 3 regions with different pixel intensity values. Under this assumption, only taking measurements at 3 different angles is justified.  
    Each new angle introduced into the operator reduces the predictive variance of every unseen angle almost equally. 
    As a result, the relative assignment of predictive variance in angle space remains roughly constant throughout design steps.

Although it is well known that experimental design is very sensitive to the choice of prior \citep{chi2015modelerror,foster2021Design}, the ease with which the relatively simple Matérn-$\nicefrac{1}{2}$  model can overfit was unexpected to us.
\end{remark}

\section{Discussion}\label{sec:dip_conclusion}

Having laid the groundwork for scalable uncertainty estimation and hyperparameter selection for linearised neural networks in \cref{chap:SGD_GPs}, \cref{chap:adapting_laplace} and \cref{chap:sampled_Laplace}, this chapter has applied these advances to tomographic reconstruction.  
In particular, we have introduced a probabilistic formulation of the deep image prior (DIP) that utilises a linearisation of the network around the parameters that output the candidate reconstruction. The approach yields far better uncertainty estimates on 2d image reconstructions from real-measured $\mu$CT data than MC dropout-based approaches standard in the field of CT. Furthermore, our method is the first to have been applied to  3d volume reconstructions.

Motivated by standard practise in the field of CT, we developed a bespoke TV-based prior for our linearised NN. However, we found it to be dominated, in terms of both computational efficiency and calibration of uncertainty estimates, by the more general diagonal g-prior, introduced in \cref{subsec:g-prior}. 

Finally, we applied linearised DIP inference to adaptively select scanning angles for CT. 
Our results suggest that dependence on the measurement data, i.e.\ adaptivity, is key to outperforming equidistant angle selection, a notoriously strong baseline in CT reconstruction \citep{Shen2022learningtoscan,Helin2021GaussianTV}. Distinctly from previous work, our methods only necessitate a pilot scan instead of being fully online, increasing applicability. We observe the largest gains in the 10 to 20 angle regime, where our designs reduce the angle requirement by roughly 30\% without loss of reconstruction quality. This is true for both traditional TV-regularised and DIP-based reconstructions.

With this, we conclude the technical content of this thesis. The following chapter reviews the main contributions of the work and discusses exciting avenues for future work.

\chapter{Conclusions and future work} \label{chap:conclusion} %

\ifpdf
    \graphicspath{{Chapter8/Figs/Raster/}{Chapter8/Figs/PDF/}{Chapter8/Figs/}}
\else
    \graphicspath{{Chapter8/Figs/Vector/}{Chapter8/Figs/}}
\fi

This thesis has studied the problem of large scale Bayesian inference in linear models and neural networks. We made contributions of both fundamental nature, furthering our understanding of linearised neural networks, and also of practical nature by introducing a number of learning algorithms that scale well in both the number of observations and model parameters. 
We have strived for these methods to be fully compatible with existing (and hopefully future) progress in the field of deep learning. 
I hope that this work will contribute towards the development and real-world deployment of uncertainty-aware data-driven decision making systems. 

We go on to provide a recap of our contributions in \cref{sec:recap}, while giving a critical overview of their strengths and weaknesses.
In \cref{sec:future_work} we discuss avenues for future work.

\section{Recap of contributions}\label{sec:recap}

\textbf{\cref{chap:SGD_GPs}} proposed using SGD, the workhorse algorithm of deep learning, to perform posterior inference in Gaussian processes. Traditionally, lack of scalability has been a major impediment to the use of these models; the cost of exact inference is cubic in the number of observations, i.e. $\c{O}(n^3)$. 
SGD has not been considered for this task in the past, in part because it provides worse formal convergence guarantees than alternatives like CG \citep{Boyd2014convex}.
At a high level, our key insight is that full convergence of SGD is not necessary to obtain good performance. 
SGD converges very fast in the directions of parameter space that matter for prediction, and very slowly in others. 
Additionally, this convergence is monotonic in the number of steps, making SGD an anytime method amenable to early stopping.
These two features, combined with SGD's linear cost, i.e. $\c{O}(n)$, per step allows it to handily outperform other inference schemes when dealing with datasets of more than $n\approx 100$k observations. 
The weakness of SGD is its lack of convergence in most low-eigenvalue directions of parameter space. These make little, but some, difference for prediction, resulting in SGD always providing an approximate solution.
Thus, below $n\approx 100$k, conjugate gradients is likely to converge faster while providing an effectively exact solution.

\textbf{\cref{chap:adapting_laplace}} studied the applicability of the linearised Laplace model evidence to modern neural networks. Motivated by the heavy dependence on this parameter of the calibration of the linearised Laplace errorbars, we focused on learning the prior precision with the evidence. We first interrogated the validity of the Laplace approximation's assumption that we Taylor expand about a mode of the posterior. In fact, satisfying this assumption is not practical in modern deep learning. Nevertheless, we showed that every expansion point implies an associated basis function linear model. As we use
this model to provide errobars, we propose to choose hyper-
parameters using the evidence of this model. This requires
only the solving of a convex optimisation problem, one
much simpler than NN optimisation.
 We then showed that, for neural networks with normalisation layers---that is, practically all modern architectures---the predictive posterior covariance can only be identified up to a scalar constant, or a constant per normalised group of weights. We introduce two prior classes which produce a predictive posterior invariant to this scaling constant. The first is a diagonal Gaussian prior with layerwise precision parameters fit to maximise the model evidence. The second is the diagonal g-prior, which only has an isotropic scaling parameter that can be set to any value.

\textbf{\cref{chap:sampled_Laplace}} combines the efficient SGD posterior sampling from \cref{chap:SGD_GPs} with the developments in linearised model hyperparameter learning of \cref{chap:adapting_laplace} into a scalable sample-based EM algorithm.
The key component is our M step, which builds upon \cite{mackay1992practical}'s update for the prior precision. Our update makes much more progress per step than traditional gradient based optimisation with the model evidence and can be estimated with only posterior samples. 
We combine these methods with a number of matrix-free linear algebra techniques and SGD warm starting to scale linearised inference to ResNet-50 ($25$M parameters) and Imagenet ($1.2$M observations and $1000$ output dimensions). To the best of our knowledge, this is the first time Bayesian inference has been performed in this setting without assuming some degree of independence across weights in the model. 
Linearised inference performs particularly well in terms of joint predictions, which are key to sequential decision making. 
However, despite our methods being more accurate and scalable than other Bayesian approximations, they still introduce very significant overhead compared to training a single neural network. Furthermore, we used the diagonal g-prior in our experiments. Since this prior only has a single free parameter, it may be cheaper to set its value with cross validation than to use our EM iteration. 

\textbf{\cref{chap:DIP}} applies the methods developed in \cref{chap:SGD_GPs,chap:adapting_laplace,chap:sampled_Laplace} to uncertainty estimation in CT reconstructions from the deep image prior.
On 2d images, we obtain more calibrated uncertainty estimates than previous probabilistic approaches to DIP reconstruction. 
The scalability of our inference methods allows us to apply them to uncertainty estimation in 3d volumetric reconstructions from the deep image prior. To the best of our knowledge, this is the first time neural network uncertainty has been estimated on this large scale task.
We concluded by leveraging the errorbars from the linearised deep image prior to guide scanning angle selection in CT. This allows us to reduce the number of scans needed to obtain a constant reconstruction quality.  
We also constructed a bespoke total-variation based prior for the linearised DIP, but we found its performance dominated by the more-scalable diagonal g-prior. This was true for both uncertainty estimation and experimental design. Perhaps we should have payed more attention to the Richard Sutton quote that preceded \cref{chap:sampled_Laplace}.

\section{Future Work} \label{sec:future_work}

\paragraph{Scalable hyperparameter learning for GPs and linearised neural networks} A clear avenue for future work is leveraging SGD posterior sampling to learn GP hyperparameters. One way to do this is to use these posterior samples in the existing Hutchinson estimator of the evidence log-determinant gradient. However, I am not optimistic about this direction because gradient-based optimisation of the evidence requires many steps. This makes updating our samples for each new step is very expensive. It would be more interesting to generalise the MacKay update, used in \cref{chap:sampled_Laplace}, beyond marginal prior precisions.

\paragraph{Online Laplace for normalised networks} The developments of \cref{chap:adapting_laplace} are focused on the post-hoc setting, where we have access to a pre-trained neural network. An exciting line of research is online Laplace, where the hyperparameters are learnt simultaneously with the network weights. However, these methods are incompatible with normalisation layers, ostensibly for the same reasons described in \cref{chap:adapting_laplace}. In \cite{lin2023online}, a paper not included in this thesis, we did some work relating online Laplace methods to the tangent linear model. It would be good to further leverage this connection, and the results of \cref{chap:adapting_laplace}, to make online Laplace amenable to normalisation layers.

\paragraph{Sequential decision making with neural networks} It seems plausible that given large enough datasets, modern large-scale neural models will rarely encounter out of distribution scenarios. Thus, the utility of model uncertainty as a tool for rejecting spurious model behaviour may decrease. However, I do not think that the more general problem of sequential decision making can be solved in the same way. Thus, I am particularly optimistic about this application of of Bayesian inference with neural networks. In particular, I am excited about the use of the linearised DIP to design CT scanning strategies for the real-world. 
Furthermore, the experimental design methods of \cref{chap:DIP} may be applied to magnetic resonance imaging, where the forward operator is a Fourier transform, almost out of the box.

\begin{spacing}{0.9}

\bibliographystyle{apalike}
\cleardoublepage
\bibliography{References/references} %

\end{spacing}

\begin{appendices} %

\chapter{Experimental setup details for \cref{chap:adapting_laplace}}\label{app:adapting_experimental_setup}

Here, we provide the details of our experimental setup which were omitted from the main text.

\section{Experiments with full Hessian computation}\label{app:experimental_setup_small}

This subsection concerns the experiments which use small architectures for which exact Hessian computation is tractable. These experiments are described in \cref{subsec:exp_assumptions} and \cref{subsec:architectures} of the main text. We first describe the setup components shared among architectures and then provide architecture-specific details. We exclude details for the U-net used in \cref{subsec:architectures}. Instead we provide these together with a brief description of the tomographic reconstruction task it performs in \cref{appendix:unet-setup}.

Unless specified otherwise, NN weights $\tilde v$ are learnt using SGD, with an initial learning rate of 0.1, momentum of 0.9, and weight decay of $1\times 10^{-4}$. We trained for 90 epochs, using a multi-step LR scheduler with a decay rate of 0.1 applied at epochs 40 and 70. This is a standard  choice for CNNs and is default in the examples provided by \href{https://pytorch.org/vision/stable/models.html}{Pytorch}.

The linear weights $w_\star$ are optimised using Adam and with their gradients calculated using \cref{alg:linear_MAP_algorithm}. We use a learning rate of $1\times 10^{-4}$ and train for 100 epochs. We set the initial regularisation parameter to be isotropic $A = a I$ with $a = 1\times 10^{-4}$.

\subsection{CNN}

Our CNN is based on the LeNet architecture with a few variations found in more modern neural networks. The architecture contains 3 convolutional blocks, followed by global average pooling in the spatial dimensions, a flatten operation, and finally a fully-connected layer. The convolutional blocks consist of \textsc{Conv} $\rightarrow$ \textsc{ReLU} $\rightarrow$ \textsc{BatchNorm}. Instead of using max pooling layers, as in the original LeNet variants, we use convolutions with a stride of 2. The first convolution is $5 \times 5$, while the next two are $3 \times 3$. As described in the main text, we consider architectures of 3 different sizes. \cref{tab:lenet_models} shows the number of the filters and number of parameters for each size of this model. The \emph{Big} model's values where chosen to create a model as large as possible while keeping full-covariance Laplace inference tractable on one A100 GPU.

\begin{table}[htb]
    \caption{Architecture parameters for the CNNs used in experiments.}
    \label{tab:lenet_models}
    \centering
    \begin{tabular}{cccccc}
    \toprule
         &  \textsc{Conv1 Filters} &  \textsc{Conv2 Filters} &  \textsc{Conv3 Filters} & \textsc{Params.} & \textsc{Hessian Size} \\ \midrule
     Big & 42 & 48 & 60 & 46 024 & 15.68 GB \\
     Med & 32 & 32 & 64 & 29 226 & 6.36 GB \\
     Small & 16 & 32 & 32 & 14 634 & 1.60 GB \\
     \bottomrule
    \end{tabular}
\end{table}

\subsection{ResNet, Pre-ResNet, and Biased-ResNet}\label{appendix:ResNet_arq}

Our ResNet is based on our CNN architecture. We replace the second and third convolutional blocks with residual blocks. The main branch of the residual blocks consist of \textsc{Conv} $\rightarrow$  \textsc{BatchNorm} $\rightarrow$ \textsc{ReLU} $\rightarrow$ \textsc{Conv} $\rightarrow$  \textsc{BatchNorm}. We apply a final \textsc{ReLU} layer after the residual is added. All of the convolutions in the residual blocks use the same number of filters. In order to downsample our features between blocks we use $1 \times 1$ convolutions with a stride of 2. \cref{tab:resnet_models} shows the number of the filters and number of parameters used for each size of this model.

Our Pre-ResNet architecture is identical to the ResNet except the main branch cosists of \textsc{BatchNorm} $\rightarrow$ \textsc{ReLU} $\rightarrow$ \textsc{Conv} $\rightarrow$  \textsc{BatchNorm} $\rightarrow$ \textsc{ReLU} $\rightarrow$ \textsc{Conv}, and we do not apply a \textsc{ReLU} after adding the residual.

\begin{table}[htb]
    \caption{Architecture parameters for the ResNets used in experiments.}
    \label{tab:resnet_models}
    \centering
      \resizebox{\textwidth}{!}{%
    \begin{tabular}{cccccc}
    \toprule
         &  \textsc{Conv1 Filters} &  \textsc{ResBlock 1 Filters} &  \textsc{ResBlock 2 Filters} & \textsc{Params.} & \textsc{Hessian Size} \\ \midrule
     Big & 22 & 42 & 64 & 45 576 & 15.48 GB \\
     Med & 16 & 32 & 64 & 26 874 & 5.38 GB \\
     Small & 12 & 24 & 32 & 14 814 & 1.64 GB \\
     \bottomrule
    \end{tabular}
    }
\end{table}

Note that the standard ResNet architecture does not apply biases in the convolution layers. The only biases in the entire network are placed in the dense output layer. For our experiment where biases are included in the Jacobian feature expansion in \cref{subsec:exp_assumptions}, we modify the ResNet architecture to include biases in all convolutional layers in addition to the already present final dense layer bias. These biases account for a small increase in parameters, to 14 898, 26 986, and 45 726, in the small, medium, and big cases, respectively.

\subsection{FixUp ResNet}

Our FixUp-ResNet architecture follows the standard ResNet structure described above, with the additional FixUP offsets and multipliers described in \citep{Zhang19FixUp}. We also follow \citet{Zhang19FixUp} in zero initialising the dense layer, and scaling the convolution weight initialisation as a function of the depth of the network.

When training FixUp-ResNets, we use the Adam optimiser with a fixed learning rate of 0.01.

\subsection{Transformer}

Our Transformer architecture contains two encoder layers with two attention heads each, and no dropout. Its input is a sequence of tokens, to which we apply a linear embedding. We add a learnable \texttt{class} embedding for each input. This \texttt{class} token is used to classify the input. We do not use positional encoding, preserving permutation invariance in the input. The sizes of the embeddings and the MLP hidden dimensions are provided in \cref{tab:transformer_models}.

When training Transformers, we use the Adam optimiser with a learning rate of $3 \times 10^{-3}$. We use an exponential learning rate decay with a gamma of 0.99 applied after every epoch of training.

\begin{table}[htb]
    \caption{Architecture parameters for the Transformers used in experiments.}
    \label{tab:transformer_models}
    \centering
    \begin{tabular}{cccccc}
    \toprule
         &  \textsc{MLP Dim} &  \textsc{Embedding Dim} &  \textsc{Params.} & \textsc{Hessian Size} \\ \midrule
     Big & 120 & 50  & 45 900 & 15.70 GB \\
     Med & 80 & 40  & 27 090 & 5.47 GB \\
     Small & 60 & 30  & 15 520 & 1.79 GB \\
     \bottomrule
    \end{tabular}
\end{table}

\section{U-Net tomographic reconstruction of KMNIST digits}\label{appendix:unet-setup}

In this section, we provide experimental details for the tomographic reconstruction results in \cref{subsec:architectures}.

 Our setup almost exactly replicates that of \citet{barbano2022probabilistic} and \citet{Antoran2022Tomography}, which form the basis of \cref{chap:DIP}. We refer to this chapter for an introduction to  tomographic reconstruction with the deep image prior.

We 
use 10 test images from the KMNIST dataset, which consists of $28\times 28$ grey-scale images of Hiragana characters \citep{clanuwat2018deep}, we simulate $y$ with $20$ angles taken uniformly from the range $0^\circ$ to $180^\circ$, and add $5\%$ white noise to the projected inputs $\c{T}x$.
We reconstruct $x$ using the Deep Image Prior (DIP) \citep{Ulyanov18prior}, which parametrises the reconstruction $x$ as the output of a U-net $g(v)$ \citep{ronneberger2015u}.

We use the U-net like architecture deployed by \citet{barbano2021deep}. Group norm is placed before every $3\times 3$ convolution operation. The U-net architecture is a encoder-decoder, fully convolutional deep model constructing multi-level feature maps.
We identify 3 distinct blocks for both the encoder branch and and 2 blocks for the the decoder branch: \texttt{In}, \texttt{Down}$_0$, \texttt{Down}$_{1}$, and \texttt{Up}$_{0}$ and \texttt{Up}$_{1}$, respectively. 
The \texttt{In} block consists of a $3\times 3$ convolution. 
\texttt{Down} blocks consist of a $3\times 3$ convolution with stride of 2 followed by a $3\times 3$ convolution operation and a bi-linear up-sampling. 
The \texttt{Up} blocks instead consist of two successive $3\times 3$ convolutional operations. Given the use of the leaky ReLU non-linearity, the normalised parameter groups of this network coincide with the described blocks.
The number of channels is set to 32 at every scale.
Multi-channel feature maps from the \texttt{In} block and from \texttt{Down}$_{0}$ are first transformed via a $1\times 1$ convolutional operation to 4 channel feature maps and then fed to \texttt{Up}$_{1}$, \texttt{Up}$_{0}$.
The reconstructed image is obtained as the output of \texttt{Up}$_{1}$ further processed via a $1\times 1$ convolutional layer.
The total number of parameters is 78k.
This is too many for full Hessian construction on GPU but we get around this issue by performing inference in the lower dimensional space of observations, as described in 
\cref{subsec:posterior_predictive} and \cref{subsec:CG_for_TVPredCP}.
We refer to \citet{Antoran2022Tomography} for a full list of hyperparameters involved in training the U-net.

The prior covariance $A^{-1}$ is a filter-wise block-diagonal matrix which applies separate regularisation to the parameters of each block in the U-net. This matches the prior described in \cref{subsec:tv_for_NN}, but without the Matérn covariance structure.
For the single regulariser experiment, we keep the same prior structure but tie the marginal prior variance of all parameters. That is, we ensure all entries of the diagonal of $A^{-1}$ are the same. The parameters of these regularisers are learnt via model evidence optimisation, as described in the main text. 

\section{Large scale experiments}\label{app:experimental_setup_big}

For scaling linearised Laplace to ResNet-50 with 25M parameters, we employ a Kronecker-factorisation of the Hessian/GGN. This is a common way to scale the Laplace approximation to large models \citep{daxberger2021laplace} and was originally proposed in \citet{ritter2018scalable}. We use the recently-released \texttt{laplace} library\footnote{\url{https://github.com/AlexImmer/Laplace}} \citep{daxberger2021laplace} for fitting the KFAC Laplace models. For ResNets with batch norm, we use the reference implementation from the \texttt{torchvision} package\footnote{\url{https://pytorch.org/vision/stable/models.html}}. For ResNets with fixUp, we use a popular open-source implementation\footnote{\url{https://github.com/hongyi-zhang/FixUp}}. To train the ResNet parameters, which will be used as the linearisation points $\tilde v$, we use the same hyperparameters as described at the top of \cref{app:experimental_setup_small} for both batch norm and FixUp ResNets.

\end{appendices}

\printthesisindex %

\end{document}